\documentclass[11pt]{article}
\usepackage{jmlr2e}
\pdfoutput=1

\usepackage[english]{babel}
\usepackage[autostyle, english = american]{csquotes}
\usepackage{manfnt}
\MakeOuterQuote{"}

\usepackage{bm}
\usepackage{bbm}
\usepackage{mathtools}

\newtheorem{assume}{Assumption}
\newtheorem{alg}{Algorithm}

\newtheorem{assumeapp}{Assumption}
\newtheorem{assumep}{Assumption}

\usepackage{multirow}
\usepackage{booktabs}
\usepackage{color}
\usepackage[svgnames]{xcolor}
\definecolor{dkgreen}{rgb}{0,0.6,0}
\definecolor{gray}{rgb}{0.5,0.5,0.5}
\definecolor{mauve}{rgb}{0.58,0,0.82}
\usepackage[shortlabels]{enumitem}
\usepackage{subcaption}

\newcommand{\iid}{\stackrel{\text{i.i.d}}{\sim}}

\newcommand{\cvp}{\stackrel{p}{\to}}
\newcommand{\disteq}{\stackrel{d}{=}}

\newcommand{\cl}{\mathrm{cl}}

\DeclareMathOperator*{\argmin}{arg\,min}

\newcommand{\mcA}{\mathcal{A}}

\newcommand{\mcC}{\mathcal{C}}

\newcommand{\mcE}{\mathcal{E}}
\newcommand{\mcF}{\mathcal{F}}
\newcommand{\mcG}{\mathcal{G}}

\newcommand{\mcI}{\mathcal{I}}

\newcommand{\mcK}{\mathcal{K}}
\newcommand{\mcL}{\mathcal{L}}

\newcommand{\mcN}{\mathcal{N}}

\newcommand{\mcP}{\mathcal{P}}
\newcommand{\mcQ}{\mathcal{Q}}
\newcommand{\mcR}{\mathcal{R}}
\newcommand{\mcS}{\mathcal{S}}

\newcommand{\mcV}{\mathcal{V}}
\newcommand{\mcW}{\mathcal{W}}
\newcommand{\mcX}{\mathcal{X}}

\newcommand{\mcZ}{\mathcal{Z}}

\newcommand{\whomega}{\widehat{\omega}}
\newcommand{\bmomega}{\bm{\omega}_n}
\newcommand{\whbmomega}{\widehat{\bm{\omega}}_n}

\newcommand{\whomegavec}{(\widehat{\omega}_1, \ldots, \widehat{\omega}_n)}

\newcommand{\fnone}{\tilde{f}_n(l, l', 1)}
\newcommand{\fnzero}{\tilde{f}_n(l, l', 0)}
\newcommand{\fnsum}{\tilde{f}_n(l, l')}
\newcommand{\llp}{(l, l')}
\newcommand{\dldl}{\; dl\,dl'}
\newcommand{\intsq}{\int_{[0, 1]^2}}
\newcommand{\compactset}{S_d}
\newcommand{\intsqbb}{ \int_{([0, 1]^q)^2}}
\newcommand{\dldlbb}{\; d\bm{l}\,d\bm{l}'}
\newcommand{\fnonebb}{\tilde{f}_n(\bm{l}, \bm{l}', 1)}
\newcommand{\fnzerobb}{\tilde{f}_n(\bm{l}, \bm{l}', 0)}
\newcommand{\simindep}{\stackrel{\text{indep}}{\sim}}

\newcommand{\emprisk}{\mathcal{R}_n}

\newcommand{\empriskhat}{\widehat{\mcR}_n}

\newcommand{\losslipconst}{L_{\ell}}
\newcommand{\lossbounda}{a_{\ell}}
\newcommand{\lossboundc}{C_{\ell}}

\newcommand{\fnholderconst}{L_f}
\newcommand{\fnboundabove}{M_f}

\newcommand{\optimalK}{K_{n, \text{uc}}^*}

\newcommand{\degree}{\mathrm{deg}}

\usepackage{lastpage}
\jmlrheading{24}{2023}{1-\pageref{LastPage}}{7/21; Revised
4/23}{5/23}{21-0841}{Andrew Davison and Morgane Austern}

\ShortHeadings{Asymptotics of Network Embeddings Learned via Subsampling}{Davison and Austern}
\firstpageno{1}
    
\begin{document}

\title{Asymptotics of Network Embeddings Learned via Subsampling}

\author{\name Andrew Davison \email ad3395@columbia.edu \\
       \addr Department of Statistics\\
       Columbia University\\
       New York, NY 10027-5927, USA
       \AND
       \name Morgane Austern \email morgane.austern@gmail.com \\
       \addr Department of Statistics\\
       Harvard University\\
       Cambridge, MA 02138-2901, USA}

\editor{Tina Eliassi-Rad}

%\editor{Kevin Murphy and Bernhard Sch{\"o}lkopf}

\maketitle

\begin{abstract}%
    Network data are ubiquitous in modern machine learning, with tasks of interest including node classification, node clustering and link prediction. A frequent approach begins by learning an Euclidean embedding of the network, to which algorithms developed for vector-valued data are applied. For large networks, embeddings are learned using stochastic gradient methods where the sub-sampling scheme can be freely chosen. Despite the strong empirical performance of such methods, they are not well understood theoretically. Our work encapsulates representation methods using a subsampling approach, such as node2vec, into a single unifying framework. We prove, under the assumption that the graph is exchangeable, that the distribution of the learned embedding vectors asymptotically decouples. Moreover, we characterize the asymptotic distribution and provided rates of convergence, in terms of the latent parameters, which includes the choice of loss function and the embedding dimension. This provides a theoretical foundation to understand what the embedding vectors represent and how well these methods perform on downstream tasks. Notably, we observe that typically used loss functions may lead to shortcomings, such as a lack of Fisher consistency. 
\end{abstract}


%!TEX root = ms.tex

\section{Introduction} \label{sec:intro}

% Paragraph on tasks of interest for networks, relational data
Network data are commonplace in modern-day data analysis tasks. Some examples of network data include social networks detailing interactions between users, citation and knowledge networks between academic papers, and protein-protein interaction networks, where the presence of an edge indicates that two proteins in a common cell interact with each other. With such data, there are several types of tasks we may be interested in. Within a citation network, we can classify different papers as belonging to particular subfields \citep[a community detection task; e.g][]{fortunato_community_2010,fortunato_community_2016}. In protein-protein interaction networks, it is too costly to examine whether every protein pair will interact together \citep{qi_evaluation_2006}, and so given a partially observed network we are interested in predicting the values of the unobserved edges. As users join a social network, they are recommended individuals who they could interact with \citep{hasan_survey_2011}.

% Talk about embeddings as a useful approach to this, highlight advantages
% of doing so
A highly successful approach to solve network prediction tasks is to first learn an embedding or latent representation of the network into some manifold, usually a Euclidean space. A classical way of doing so is to perform principal component analysis or dimension reduction on the Laplacian of the adjacency matrix of the network \citep{belkin_laplacian_2003}. This originates from spectral clustering methods \citep{pothen_partitioning_1990, shi_normalized_2000, ng_spectral_2001}, where a clustering algorithm is applied to the matrix formed with the eigenvectors corresponding to the top $k$-eigenvalues of a Laplacian matrix. One shortcoming is that for large data sets, computing the SVD of a large matrix to obtain the eigenvectors becomes increasingly computationally restrictive. Approaches which scale better for larger data sets originate from natural language processing (NLP). DeepWalk \citep{perozzi_deepwalk_2014} and node2vec \citep{grover_node2vec_2016} are both network embedding methods which apply embedding methods designed for NLP, by treating various types of random walks on a graph as "sentences", with nodes as "words" within a vocabulary. We refer to \citet{hamilton_representation_2017} and \citet{cai_comprehensive_2018} for comprehensive overviews of algorithms for creating network embeddings. See \citet{agrawal_minimum-distortion_2021} for a discussion on how such embedding methods are related to other classical methods such as multidimensional scaling, and embedding methods for other data types.

To obtain an embedding of the network, each node or vertex of the network (say $u$) is represented by a single $d$-dimensional vector $\omega_u \in \mathbb{R}^d$, which are learned by minimizing a loss function between features of the network and the collection of embedding vectors. There are several benefits to this approach. As the learned embeddings capture latent information of each node through a Euclidean vector, we can use traditional machine learning methods (such as logistic regression) to perform a downstream task. The fact that the embeddings lie within a Euclidean space also means that they are amenable to (stochastic) gradient based optimization. One important point is that, unlike in an i.i.d setting where subsamples are essentially always obtained via sampling uniformly at random, here there is substantial freedom in the way in which subsampling is performed. \citet{veitch_empirical_2018} shows that this choice has a significant influence in downstream task performance.

% Despite the success of these methods, not particularly well understood
% with respects to (blah)
Despite their applied success, our current theoretical understanding of methods such as node2vec are lacking. We currently lack quantitative descriptions of what the embedding vectors represent and the information they contain, which has implications for whether the learned embeddings can be useful for downstream tasks. We also do not have quantitative descriptions for how the choice of subsampling scheme affects learned representations. The contributions of our paper in addressing this are threefold:
\begin{enumerate}[label=\alph*)]
    \item Under the assumption that the observed network arises from an exchangeable graph, we describe the limiting distribution of the embeddings learned via procedures which depend on minimizing losses formed over random subsamples of a network, such as node2vec \citep{grover_node2vec_2016}. The limiting distribution depends both on the underlying model of the graph and the choice of subsampling scheme, and we describe it explicitly for common choices of subsampling schemes, such as uniform edge sampling
    \citep{tang_line_2015} or random-walk samplers \citep{perozzi_deepwalk_2014,grover_node2vec_2016}. 
    \item Embedding methods are frequently learned via minimizing losses which depend on the embedding vectors only through their pairwise inner products. We show that this restricts the class of networks for which an informative embedding can be learned, and that networks generated from distinct probabilistic models can have embeddings which are asymptotically indistinguishable. We also show that this can be fixed by changing the loss to use an indefinite or Krein inner product between the embedding vectors. We illustrate on real data that doing so can lead to improved performance in downstream tasks.
    \item We show that for sampling schemes based upon performing random walks on the graph, the learned embeddings are scale-invariant in the following sense. Suppose that we have two identical copies of a network generated from a sparsified exchangeable graph, and on one we delete each edge with probability $p \in (0, 1)$. Then in the limit as the number of vertices increases to infinity, the asymptotic distributions of the embedding vectors trained on the two networks will be asymptotically distinguishable. We highlight that this may provide some explanation as to the desirability of using random walk based methods for learning embeddings of sparse networks. 
\end{enumerate}

\subsection{Motivation} \label{sec:intro:motivation}

We note that several approaches to learn network embeddings \citep{perozzi_deepwalk_2014, tang_line_2015, grover_node2vec_2016} do so by performing stochastic gradient updates of the embedding vectors $\omega_i \in \mathbb{R}^d$ by updates
\begin{equation} \label{eq:intro:pos_neg_loss}
    \omega_i \longleftarrow \omega_i - \eta \frac{\partial \mcL}{\partial \omega_i} \quad \text{ where } \quad \mcL = 
    - \sum_{(i, j) \in \mcP} \log \sigma\big( \langle \omega_i, \omega_j \rangle \big)  - \sum_{(i, j) \in \mcN } \log \big\{ 1 - \sigma\big( \langle \omega_i, \omega_j \rangle \big) \big\}.
\end{equation}
Here $\sigma(x) = (1+e^{-x})^{-1}$ is the sigmoid function, the sets $\mcP$ and $\mcN$ are pairs of nodes which are chosen randomly at each iteration (referred to as positive and negative samples respectively) and $\eta > 0$ is a step size. The goal of the objective is to force pairs of vertices within $\mcP$ to be close in the embedding space, and those within $\mcN$ to be far apart. At the most basic level, we could just have that $\mcP$ consists of edges within the graph and $\mcN$ non-edges, so that vertices which are disconnected from each other are further apart in the embedding space than those which are connected. In a scheme such as node2vec, $\mcP$ arises
through a random walk on the network, and $\mcN$ arises by choosing vertices
according to a unigram negative sampling distribution for each vertex in
the random walk $\mcP$. 

For simplicity, assume that the only information available for training is a fully observed adjacency matrix $(a_{ij})_{i, j}$ of a network $\mcG$ of size $n$. Moreover, we let $\mcP$ and $\mcN$ be random sets which consist only of pairs of vertices which are connected ($a_{ij} = 1$) and not connected ($a_{ij} = 0$) respectively. In this case, if we write $S(\mcG) = \mcP \cup \mcN$, then the algorithm scheme described in \eqref{eq:intro:pos_neg_loss} arises from trying to minimize the empirical risk function (which depends on the underlying graph $\mcG$)
\begin{equation} \label{eq:intro:sgd_loss}
    \mcR_n(\omega_1, \ldots, \omega_n) := \sum_{i \neq j} \mathbb{P}\big( (i, j) \in S(\mcG) \,|\, \mcG \big) \ell\big( \langle \omega_i, \omega_j \rangle, a_{ij} \big)
\end{equation}
with a stochastic optimization scheme \citep{robbins_stochastic_1951}, where we write $\ell(y, x) = -x \log \sigma(y) - (1-x) \log(1 - \sigma(y))$ for the cross entropy loss. 

This means that the optimization scheme in \eqref{eq:intro:pos_neg_loss} attempts to
find a minimizer $(\widehat{\omega}_1, \ldots, \widehat{\omega}_n)$ of the function $\mcR_n(\omega_1, \ldots, \omega_n)$ defined in \eqref{eq:intro:sgd_loss}. We ask several questions about these minimizers where there is currently little understanding:

\begin{enumerate}[label=\alph*)]
    \item[Q1:] To what extent is there a unique minimizer to the empirical risk \eqref{eq:intro:sgd_loss}?
    \item[Q2:] Does the distribution of the learnt embedding vectors $(\widehat{\omega}_1, \ldots, \widehat{\omega}_n)$ change as a result of changing the underlying sampling scheme? If so, can we describe quantitatively how?
    \item[Q3:] During learning of the embedding vectors, are we using a loss which limits the information we can capture in a learned representation? If so, can we fix this in some way?
\end{enumerate}

Answering these questions allow us to evaluate the impact of
various heuristic choices made in the design of algorithms such as node2vec, where our results will allow us to describe the impact with respect to downstream tasks such as edge prediction. We go into more depth into these questions below, and discuss in Section~\ref{sec:intro:outline} how our main results help address these questions.

\subsubsection{Uniqueness of minimizers of the empirical risk} 
We highlight that the loss and risk functions in \eqref{eq:intro:pos_neg_loss} and \eqref{eq:intro:sgd_loss} are invariant under any joint transformation of the embedding vectors $\omega_i \to Q \omega_i$ by an orthogonal matrix $Q$. As a result, we can at most ask whether the gram matrix $\Omega_{ij} = \langle \omega_i, \omega_j \rangle$ induced by the embedding vectors is uniquely characterized. This is challenging as the embedding dimension $d$ is significantly less than the number of vertices $n$ - even for networks involving millions of nodes, the embedding dimensions used by practitioners are of the order of magnitude of hundreds. As a result the gram matrix is rank constrained. Consequently, when reformulating \eqref{eq:intro:sgd_loss} to optimize over the matrix $\Omega$, the optimization domain is non-convex, meaning answering this question is non-trivial. Answering this allows us to understand whether the embedding dimension fundamentally influences the representation we are learning, or instead only influences how accurately we can learn such a representation. 

\subsubsection{Dependence of embeddings on the sampling scheme choice in learning}

While we know that random-walk schemes such as node2vec are empirically successful, there has been little discussion as to how the representation learnt by such schemes compares to (for example) schemes where we sample vertices randomly and look at the induced subgraph. This is useful
for understanding their performance on downstream tasks such as 
community detection or link prediction. Another useful example is for
when embeddings are used
for causal inference \citep{veitch_using_2019}, where there is the needed
to validate assumptions that the embeddings containing 
information relevant
to the prediction of propensity scores and expected outcomes. A final example arises
in methods which try and attempt to "de-bias" embeddings through
the use of adaptive sampling schemes \citep{rahman_fairwalk_2019},
to understand what extent they satisfy different fairness criteria.

We are also interested in understanding how the hyperparameters of a 
sampling scheme affect the expected value and variance of gradient estimates
when performing stochastic gradient descent. The distinction is important, as
the expected value influences the empirical risk being minimized - therefore the underlying representation - and the variance
the speed at which an optimization algorithm converges \citep{dekel_optimal_2012}. When using stochastic gradient 
descent in an i.i.d data setting, the mini-batch size does not effect
the expected value of the gradient estimate given the observed data, but only 
its variance, which decreases as the mini-batch size increases. However, for a scheme 
like node2vec, it is not clear whether hyperparameters such as the random walk length, or 
the unigram parameter affect the expectation or variance of the gradient estimates (conditional on the graph $\mcG$).

\subsubsection{Information-limiting loss functions}
One important property of representations which make them useful for downstream tasks 
are their ability to differentiate between different graph structures. One way to
examine this is to consider different probabilistic models for a network, and to
then examine whether the resulting embeddings are distinguishable from each other.
If they are not, then this suggests some information about the network
has been lost in learning the representation. By examining the range of distributions
which have the same learned representation, we can understand this information loss
and the effect on downstream task performance.

\subsection{Overview of results}

\subsubsection{Embedding methods implicitly fit graphon models}

We highlight that the loss in \eqref{eq:intro:sgd_loss} is the same as the loss obtained by maximizing the log-likelihood formed by a probabilistic model for the network of the form 
\begin{equation} \label{eq:intro:prob_model_form}
    \begin{split} 
        a_{ij} \,|\, \omega_i, \omega_j & \sim \mathrm{Bernoulli}\big( \sigma( \langle \omega_i, \omega_j \rangle ) \big) \text{ independently for $i \neq j$} \\
        \omega_i & \sim \mathrm{Unif}(C) \text{ independently for $i \in [n]$,}  
    \end{split}
\end{equation}
using stochastic gradient ascent. Here $C \subseteq \mathbb{R}^d$ is a closed set corresponding to a constrained set for the embedding vectors. In the limit as the number of vertices increases to infinity, such a model corresponds to an exchangeable graph \citep{lovasz_large_2012}, as the infinite adjacency matrices are invariant to a permutation of the labels of the vertices. 

In an exchangeable graph, each vertex $u$ has a latent feature $\lambda_u \sim \mathrm{Unif}[0, 1]$, with edges arising independently with $a_{uv} \,|\, \lambda_u, \lambda_v \sim \mathrm{Bernoulli}(W(\lambda_u, \lambda_v))$ for a function $W: [0, 1]^2 \to [0, 1]$ called a graphon; see \cite{lovasz_large_2012} for an overview. Such models can be thought of as generalizations of a stochastic block model \citep{holland_stochastic_1983}, which have a correspondence to when the function $W$ is a piecewise constant function on sets $A_i \times A_j$ for some partition $(A_i)_{i \in [k]}$ of $[0, 1]$, with the partitions acting as the different communities within the SBM. If $\pi_i$ is the size of $A_i$, and
we write $W_{ij}$ for the value of $W(l, l')$ on $A_i \times A_j$, this
is equivalent to the usual presentation of a stochastic block model
\begin{equation}
    c(u) \iid \mathrm{Categorical}(\pi), \qquad a_{uv} \,|\, c(u), c(v) \simindep \mathrm{Bernoulli}( W_{c(u),c(v)}).
\end{equation}
where $c(i)$ is the community label of vertex $u$. One can also
consider sparsified exchangeable graphs, where for a graph on $n$
vertices, edges are generated with probability $W_n(\lambda_u, \lambda_v) = \rho_n W(\lambda_u, \lambda_v)$ for a graphon $W$ and 
a \emph{sparsity factor} $\rho_n \to 0$ as $n \to \infty$. This
accounts for the fact that most real world graphs are not "dense"
and do not have the number of edges scaling as $O(n^2)$; in a sparsified
graphon, the number of edges now scales as $O(\rho_n n^2)$. 

For the purposes of theoretical analysis, we look at the minimizers of \eqref{eq:intro:sgd_loss} when the network $\mcG$ arises as a finite sample observation from a sparsified exchangeable graph whose graphon is
sufficiently regular. We then examine statistically the behavior of the minimizers as the number of vertices grows towards infinity. As embedding methods are frequently used on very large networks, a large sample statistical analysis is well suited for this task. One important observation is that \emph{even when the observed data is from a sparse graph, embedding methods which fall under \eqref{eq:intro:prob_model_form} are implicitly fitting a
dense model to the data}. As we know empirically that embedding methods
such as node2vec produce useful representations in sparse settings,
we introduce the sparsity to allow some insight as to how this can occur.

\subsubsection{Types of results obtained} 
We now discuss our main results, with a general overview followed by explicit examples. In Theorems~\ref{thm:embed_learn:converge_1}~and~\ref{thm:embed_learn:converge_3}, we show that under regularity assumptions on the graphon, in the limit as the number of vertices increases to infinity, we have for any sequence of minimizers $(\widehat{\omega}_1, \ldots, \widehat{\omega}_n)$ to $\mcR_n(\omega_1, \ldots, \omega_n)$ that 
\begin{equation} \label{eq:intro:consistency}
    \frac{1}{n^2} \sum_{i, j} \big| \langle \whomega_i, \whomega_j \rangle - K(\lambda_i, \lambda_j) \big| = O_p(r_n)
\end{equation}
for a function $K: [0, 1]^2 \to \mathbb{R}$ we determine, and rate $r_n \to 0$. Both $K$ and $r_n$ depend on the graphon $W$ and the choice of sampling scheme. The rate also depends on the embedding dimension $d$; we note that our results
may sometimes require $d \to \infty$ as $n \to \infty$ in order for $r_n \to 0$, but will always do so sub-linearly with $n$. As a result \eqref{eq:intro:consistency} allows us to guarantee that on average, the inner products between embedding vectors contain some information about the underlying structure of the graph, as parameterized through the graphon function $W$. One notable application of this type of result is that it allows us to give guarantees for the asymptotic risk on edge prediction tasks, when using the values $S_{ij} = \langle \whomega_i, \whomega_j \rangle$ as scores to threshold for whether there is the presence of an edge $(i, j)$ in the graph. Our results apply to sparsified exchangeable graphs whose graphons are either piecewise constant (corresponding to a stochastic block model), or piecewise H\"{o}lder continuous. 

To show how our results address the questions introduced in Section~\ref{sec:intro:motivation}, and to highlight the connection with using the embedding vectors for edge prediction tasks, we give explicit examples (with minimal additiional notation) of results which can be obtained from the main theorems of the paper. For the remainder of the section, suppose that
\begin{equation*}
    \ell(y, x) := -x \log(\sigma(y)) - (1-x) \log(1 - \sigma(y) ) \qquad \qquad \Big( \text{with } \sigma(y) = \frac{e^y}{1 + e^y} \Big)
\end{equation*}
denotes the cross-entropy loss function (where $y \in \mathbb{R}$ and $x \in \{0, 1\}$). We consider graphs which arise from a sub-family of stochastic block models - frequently called SBM$(p, q, \kappa)$ models - where a graph of size $n$ is generated via the probabilistic model
\begin{equation}
    c(u) \iid \mathrm{Unif}( \{1, \ldots, \kappa \} ), \qquad a_{uv} \,|\, c(u), c(v) \simindep \begin{cases} 
    \mathrm{Bernoulli}(\rho_n p) & \text{ if } c(u) = c(v), \\ \mathrm{Bernoulli}( \rho_n q) & \text{ if } c(u) \neq c(v).
    \end{cases}
\end{equation}
Here $\rho_n$ is a sparsifying sequence. For our results below, we will consider the cases when $\rho_n = 1$ or $\rho_n = (\log n)^2 / n$ (so $\rho_n \to 0$ in the second case). With regards to the choice of sampling schemes, we consider two choices:
\begin{enumerate}[label=\roman*)]
    \item Uniform vertex sampling: A sampling scheme where we select $100$ vertices uniformly at random, and then form a loss over the induced sub-graph formed by these vertices.
    \item node2vec: The sampling scheme in node2vec where we use a walk length of $50$, select $1$ negative samples per vertex using a unigram distribution with $\alpha = 0.75$. (See either \citet{grover_node2vec_2016}, or Algorithm~\ref{alg:random_walk} in Section~\ref{sec:sampling_formula}, for more details.)
\end{enumerate}
Recall that defining a sampling scheme and a loss function induces a empirical risk as given in \eqref{eq:intro:sgd_loss}, with the sampling scheme defining sampling probabilities $\mathbb{P}((u, v) \in S(\mcG) \,|\, \mcG)$. Below we will give theorem statements for two types of empirical risks, depending on how we combine two embedding vectors $\omega_u$ and $\omega_v$ to give a scalar. The first uses a regular positive definite inner product $\langle \omega_u, \omega_v \rangle$, and the second uses a \emph{Krein inner product}, which takes the form $\langle \omega_u, S \omega_v \rangle$ where $S$ is a diagonal matrix with entries $\{ +1, -1\}$.

Supposing we have embedding vectors $\omega_u \in \mathbb{R}^{2d}$, we consider the risks
\begin{align}
    \mcR_n(\omega_1, \ldots, \omega_n) & := \sum_{i \neq j} \mathbb{P}\big( (i, j) \in S(\mcG) \,|\, \mcG \big) \ell\big( \langle \omega_i, \omega_j \rangle, a_{ij} \big), \\
    \mcR^B_n(\omega_1, \ldots, \omega_n) & := \sum_{i \neq j} \mathbb{P}\big( (i, j) \in S(\mcG) \,|\, \mcG \big) \ell\big( \langle \omega_i, S_d \, \omega_j \rangle, a_{ij} \big),
\end{align}
where $S_d = \mathrm{diag}(I_d, -I_d) \in \mathbb{R}^{2d \times 2d}$ and $I_d \in \mathbb{R}^{d \times d}$ is the $d$-dimensional identity matrix. With this, we are now in a position to state results of the form given in \eqref{eq:intro:consistency}. As it is easier to state results when using the second risk $\mcR^B_n(\omega_1, \ldots, \omega_n)$, we will begin with this, and state two results corresponding to either the uniform vertex sampling scheme, or the node2vec sampling scheme. We then discuss implications of the results afterwards.

% Begin with the uniform sampling result
\begin{theorem}
    \label{thm:intro:unif_samp}
    Suppose that we use the uniform vertex sampling scheme described above, we choose the embedding dimension $d = 2 \kappa$, and $\rho_n = 1$ for all $n$. Then for any sequence of minimizers $(\whomega_1, \ldots, \whomega_n)$ to $\mcR^B_n(\omega_1, \ldots, \omega_n)$, we have that 
    \begin{equation*}
        \frac{1}{n^2} \sum_{i,j} \big| \langle \whomega_i, S_d \,\whomega_j \rangle - K_{c(i), c(j)} \big| \to 0
    \end{equation*}
    in probability as $n \to \infty$, where $K \in \mathbb{R}^{\kappa \times \kappa}$ is the matrix
    \begin{equation*}
        K_{lm} = \begin{cases} \log(p/(1- p)) & \text{ if } l = m, \\ \log(q/(1- q)) & \text{ if } l \neq m \end{cases}
    \end{equation*}
\end{theorem}

% Make the statement about node2vec
\begin{theorem}
    \label{thm:intro:node2vec}
    Suppose in Theorem~\ref{thm:intro:unif_samp} we instead use the node2vec sampling scheme described earlier, and now either $\rho_n = 1$ or $\rho_n = (\log n)^2 / n$. Then the same convergence guarantee holds, except now the matrix $K \in \mathbb{R}^{\kappa \times \kappa}$ takes the form
    \begin{align*}
        K_{lm} & = 
            \log\Big(\frac{ p \kappa}{1.02 (1 - \rho_n p) (p + (\kappa - 1) q}  \Big) \qquad \text{ if } l = m,\\
             & = \log\Big(\frac{ q \kappa}{1.02 (1 - \rho_n q) (p + (\kappa - 1) q}  \Big)
        \qquad \text{ if } l \neq m.
    \end{align*}
\end{theorem}
With these two results, we make a few observations:
\begin{enumerate}[label=\roman*)]
    \item In our convergence theorems, we say that \textbf{for any sequence of minimizers}, the matrix $( \langle \whomega_i, S_d \,\whomega_j \rangle)_{i,j}$ will have the same limiting distribution. Although here we explicitly choose $d = 2\kappa$, $d$ can be any sequence which which diverges to infinity (provided it does so sufficiently slowly) and have the same result hold. Consequently, this suggests that up to symmetry and statistical error, the minimizers of the empirical risk will be essentially unique, giving an answer to Q1.
    \item For different sampling schemes, we are able to give a closed form description of the limiting distribution of the matrices $( \langle \whomega_i, S_d \,\whomega_j \rangle)_{i,j}$, and we can see that they are different for different sampling schemes. This affirms Q2 as posed above in the positive. One interesting observation from the Theorems~\ref{thm:intro:unif_samp}~and~\ref{thm:intro:node2vec} is the dependence on the sparsity factor. While a uniform vertex sampling scheme does not work well in the sparsified setting (and so we give convergence results only when $\rho_n = 1$) in node2vec \textbf{the representation remains stable in the limit when $\rho_n \to 0$}. 
    \item Theorem~\ref{thm:intro:unif_samp} tells us that if we use a uniform sampling scheme, then using the Krein inner product during learning and the $S_{ij} = \langle \whomega_i, S_d \whomega_j \rangle$ as scores, we are able to perform edge prediction up to the information theoretic threshold. 
    \item If in Theorem~\ref{thm:intro:node2vec} we instead let the walk length in node2vec to be of length $k$, the $1.02$ term in the limiting distribution for node2vec would be replaced by $1 + k^{-1}$. This means that in the limit $k \to \infty$, the limiting distribution is independent of the walk length. We discuss later in Section~\ref{sec:sampling:variance} the roles of the hyperparameters in node2vec, and argue that the walk length places a role in only reducing the variance of gradient estimates.
\end{enumerate}

So far we have only given results for minimizers of the loss $\mcR^B_n(\omega_1, \ldots, \omega_n)$. We now give an example of a convergence result for $\mcR_n(\omega_1, \ldots, \omega_n)$, and afterwards discuss how this result addresses Q3 as posed above.

\begin{theorem}
    \label{thm:intro:unif_samp_pos}
    Suppose the graph arises from a SBM($p, q, 2$) model. Let $\sigma^{-1}(y) = \log(y/(1-y)$ denote the inverse sigmoid function.
    Suppose that we use the uniform vertex sampling scheme described above, the embedding dimension satisfies $d \geq 2$ and $\rho_n = 1$. Then for any sequence of minimizers $(\whomega_1, \ldots, \whomega_n)$ to $\mcR_n(\omega_1, \ldots, \omega_n)$, we have that 
    \begin{equation*}
        \frac{1}{n^2} \sum_{i,j} \big| \langle \whomega_i, \whomega_j \rangle - K_{c(i), c(j)} \big| = o_p(1) \qquad  \text{ where } K = \begin{pmatrix} K_{11} & K_{12} \\ K_{12} & K_{11} \end{pmatrix}
    \end{equation*}
    and the values of $K_{11}$ and $K_{12}$ depend on $p$ and $q$ as follows:
    \begin{enumerate}[label=\alph*)]
        \item If $p \geq q$ and $p + q \geq 1$, then $K_{11} = \sigma^{-1}(p)$ and $K_{12} = \sigma^{-1}(q)$;
        \item If $p \geq q$ and $p + q < 1$, then $K_{11} = - K_{12} = \sigma^{-1}((1 + p -q)/2)$;
        \item If $p < q$ and $p + q \geq 1$, then $K_{11} = K_{12} = \sigma^{-1}( (p+q)/2)$;
        \item Otherwise, $K_{11} = K_{12} = 0$.
    \end{enumerate}
\end{theorem}

From the above theorem we can see that the representation produced is not an invertible function of the model from which the data arose. For example in the regime where $p \geq q$ and $p + q < 1$, the representation depends only on the size of the gap $p - q$, and so one can choose different values of $(p, q)$ for which the limiting distribution is the same. This answers the first part of Q3. (We discuss this further in Section~\ref{sec:embed_learn:inner_prod}; see the discussion after Proposition~\ref{thm:embed_learn:sbm_example_1}.) In contrast, this does not occur in Theorem~\ref{thm:intro:unif_samp} - the representation learned is an invertible function of the underlying model. Theorem~\ref{thm:intro:unif_samp_pos} also highlights that, when using only the regular inner product during training and scores $S_{ij} = \langle \whomega_i, \whomega_j \rangle$, there are regimes (such as when $p < q$) where the scores produced will be unsuitable for purposes of edge prediction.

The fundamental difference between Theorems~\ref{thm:intro:unif_samp}~and~\ref{thm:intro:unif_samp_pos} is that the risk $\mcR_n^B(\omega_1, \ldots, \omega_n)$ we consider in Theorem~\ref{thm:intro:unif_samp} arises by making the implicit assumption that the network arises from a probabilistic model $a_{ij} \,|\, \omega_i, \omega_j \sim \mathrm{Bernoulli}\big( \sigma( \langle \omega_i, S_d \, \omega_j \rangle ) \big)$. This means the inverse-logit matrix of edge probabilities are not constrained to be positive-definite, whereas using $\langle \omega_i, \omega_j \rangle$ as in \eqref{eq:intro:prob_model_form} to give $\mcR_n(\omega_1, \ldots, \omega_n)$ places a positive-definite constraint on this matrix. This can be interpreted as a form of model misspecification of the data generating process. To address the information loss which occurs when parameterizing the loss through inner products $\langle \omega_i, \omega_j \rangle$, we can fix this by replacing it with a Krein inner product. This gives an answer to the second part of Q3. We later demonstrate that making this change can lead to improved performance when using the learned embeddings for downstream tasks on real data (Section~\ref{sec:exper:real_data}), suggesting these findings are not just an artefact of just the type of models we consider.

\subsection{Related works} \label{sec:intro:related_works}

There is a large literature looking at embeddings formed via
spectral clustering methods under various 
network models from a statistical perspective; see e.g \citet{ma_determining_2021, deng_strong_2021} for some recent examples. 
For models supporting a natural community structure,
these frequently take the form of giving guarantees on the behavior of the embeddings, and then argue that using a clustering method
with the embedding vectors allows for weak/strong consistency of community detection. See
\citet{abbe_community_2017} for an overview of the
information theoretic thresholds for the different type of recovery
guarantees.

\citet{lei_consistency_2015} consider spectral clustering using the eigenvectors of the adjacency matrix for a stochastic block model. \citet{rubin-delanchy_statistical_2017} consider spectral embeddings using both the adjacency matrix and Laplacian matrices from models arising from generative models of the form  $A_{ij} | X_i, X_j \sim \mathrm{Bernoulli}( \langle X_i, I_{p, q} X_j \rangle)$ where $I_{p, q} = \mathrm{diag}(I_p, -I_q)$) and the $X_i \in \mathbb{R}^d$ are i.i.d random variables with $p, q, d$ known and fixed - such graphs are referred to frequently as dot product graphs. These allow for a broader class of models than stochastic block models, such as mixed-membership models. The $q = 0$ case was considered by \citet{tang_limit_2018}, with central limit theorem results given in \citet{levin_limit_2021}; see \citet{athreya_statistical_2018} for a 
broader review of statistical analyses of various methods on these graphs. In \citet{lei_network_2021}, they consider similar models where $A_{ij} | Z_i, Z_j \sim \mathrm{Bernoulli}( \langle Z_i, Z_j \rangle_{\mcK})$ where $\mcK$ is a Krein space (formally, this is a direct sum of Hilbert spaces equipped with an indefinite inner product, formed by taking the difference of the inner products on the summand Hilbert spaces), with their results applying to non-negative definite graphons and graphons which are H\"{o}lder continuous for exponents $\beta > 1/2$. They then discuss the estimation of the $Z_i$ using the eigendecomposition of the adjacency matrix (which we have noted can be viewed as a type of embedding) from a functional data analysis perspective. We note that in our work we do not directly assume a model of such a form, but some of our proofs use some similar ideas.

With regards to embeddings learned via random walk approaches such as node2vec \citep{grover_node2vec_2016}, there are a few works which study modified loss functions. To be precise, these suppose that each vertex $u$ has two embedding vectors $\omega_u \in \mathbb{R}^d$ and $\eta_u \in \mathbb{R}^d$, with terms of the form $\langle \omega_i, \omega_j \rangle$ replaced in the loss with $\langle \omega_i, \eta_j \rangle$, and $\omega_u$, $\eta_u$ are allowed to vary independently with each other. \citet{qiu_network_2018} study several different embedding methods within this context (including those involving random walks) where they explicitly write down the closed form of the minimizing matrix $(\langle \omega_i, \eta_j \rangle)_{ij}$ for the loss having averaged over the random walk process when $d \geq n$ and $n$ is fixed. In order to be always able to write down explicitly the minimizing matrix, they rely on the assumption that $d \geq n$ and that $\eta_j$ and $\omega_j$ are unconstrained of each other, so that the matrix $(\langle \omega_i, \eta_j \rangle)_{ij}$ is unconstrained. This avoids
the issues of non-convexity in the problem. We note that in our work we are able to handle the case where we enforce the constraints $\eta_j = \omega_j$ (as in the original node2vec paper) and $d \ll n$, so we address the non-convexity.

\citet{zhang_consistency_2021} then considers the same minimizing matrix as in \citet{qiu_network_2018} for stochastic block models, and examines the best rank $d$ approximation (with respect to the Frobenius norm) to this matrix, in the regime where $n \to \infty$ and $d$ is less than or equal to the number of communities. We comment that our work gives convergence guarantees under broad families of sampling schemes, including - but not limited to - those involving random walks, and for general smooth graphons rather than only stochastic block models. \citet{veitch_empirical_2018} discusses the role of subsampling as a model choice, within the context of specifying stochastic gradient schemes for empirical risk minimization for learning network representations, and highlights the role they play in empirical performance.

\subsection{Notation and nomenclature} \label{sec:intro:notation}

For this section, we write $\mu$ for the Lebesgue measure, $\mathrm{int}(A)$ the interior of a set $A$ and $\mathrm{cl}(A)$ as the closure of $A$. We say that a \emph{partition} $\mathcal{Q}$ of $X \subseteq \mathbb{R}^d$, written $\mathcal{Q} = (Q_1, \ldots, Q_{\kappa})$, is a finite collection of pairwise disjoint, connected sets whose union is $X$, and $\mu(\mathrm{int}(Q)) > 0$ and $\mu( \cl(Q) \setminus \mathrm{int}(Q) ) = 0$ for all $Q \in \mcQ$. For a partition $\mcQ$ of $X$, we define 
\begin{equation*}
    \mathcal{Q}^{\otimes 2} := \{ Q_i \times Q_j \,:\, Q_i, Q_j \in \mathcal{Q} \},
\end{equation*}
which gives a partition of $X^2$. A \emph{refinement} $\mathcal{Q}'$ of $\mathcal{Q}$ is a partition $\mathcal{Q}'$ where for every $Q' \in \mathcal{Q}'$, there exists a (necessarily unique) $Q \in \mcQ$ such that $Q' \subseteq Q$. 

We say a function $f: X \to \mathbb{R}$ is H\"{o}lder$(X, \beta, M)$, where $X \subseteq [0, 1]^d$ is closed and $\beta \in (0, 1]$, $M > 0$ are constants, if \begin{equation*}
    | f(x) - f(y) | \leq M \| x - y \|_2^{\beta} \qquad \text{ for all } x, y \in X.
\end{equation*}
We say a function $f : X \to \mathbb{R}$ is piecewise H\"{o}lder$(X, \beta, M, \mathcal{Q})$ if the following holds: for any $Q \in \mathcal{Q}$, the restriction $f|_Q$ admits a continuous extension to $\cl(Q)$, with this extension being H\"{o}lder$(\mathrm{cl}(Q), \beta, M)$. Similarly, we say that a function $f : X \to \mathbb{R}$ is piecewise continuous on $\mcQ$ if for every $Q \in \mcQ$, $f|_Q$ admits a continuous extension to $\cl(Q)$.

For a graph $\mcG = (\mcV, \mcE)$ with vertex set $\mcV \subseteq \mathbb{N}$ and edge set $\mcE$, we let $A = (a_{uv})_{u, v \in \mcV}$ denote the adjacency matrix of $\mcG$, so $a_{uv} = 1$ iff $(u, v) \in \mcE$. Here we consider undirected graphs with no self-loops, so $(u, v) \in \mcE \iff (v, u) \in \mcE$; we count $(u, v)$ and $(v, u)$ together as one edge. For such a graph, we let
\begin{itemize}
    \item $E[\mcG] = \sum_{u < v} a_{uv} = \frac{1}{2} \sum_{u \neq v} a_{uv}$ denote the number of edges of $\mcG$;
    \item $\degree(u) = \sum_{v} a_{uv}$ denotes the degree of the vertex $u$, so $\sum_u \degree(u) = 2 E[\mcG]$.
\end{itemize}
A subsample $S(\mcG)$ of a graph $\mcG$ is a collection of vertices $\mcV(S(\mcG))$, along with a symmetric subset of the adjacency matrix of $\mcG$ restricted to $\mcV(S(\mcG))$; that is, a subset of $(a_{uv})_{u, v \in \mcV(S(\mcG))}$. The notation $(i, j) \in S(\mcG)$ therefore refers to whether $a_{ij}$ is an element of the aforementioned subset of $(a_{uv})_{u, v \in \mcV(S(\mcG))}$.

In the paper, we consider sequences of random graphs $(\mcG_n)_{n \geq 1}$ generated by a sequence of graphons $(W_n)_{n \geq 1}$. A graphon is a symmetric measurable function $W: [0, 1]^2 \to [0, 1]$. To generate these graphs, we draw latent variables $\lambda_i \sim U[0, 1]$ independently for $i \in \mathbb{N}$, and then for $i < j$ set 
\begin{equation*}
    a^{(n)}_{ij} | \lambda_i, \lambda_j \sim \mathrm{Bernoulli}(W_n(\lambda_i, \lambda_j) )
\end{equation*}
independently, and $a^{(n)}_{ji} = a^{(n)}_{ij}$ for $j < i$. We then let $\mcG_n$ be the graph formed with adjacency matrix $A^{(n)}$ restricted to the first $n$ vertices. Unless mentioned otherwise, we understand that references to $\lambda_i$ and $a_{ij}$ - now dropping the superscript $(n)$ - refer to the above generative process. For a graphon $W$, we will denote
\begin{itemize}
    \item $\mcE_W = \int_0^1 \int_0^1 W\llp \dldl$ for the edge density of $W$;
    \item $W(\lambda, \cdot) = \int_0^1 W(\lambda, y) \, dy$ for the degree function of $W$;
    \item $ \mcE_W(\alpha) = \int_0^1 W(\lambda, \cdot)^{\alpha} \, d\lambda$, so $\mcE_W(1) = \mcE_W$. 
\end{itemize}
Given a sequence of random graphs $(\mcG_n)_{n \geq 1}$ generated in the above fashion, we define the random variables $E_n := E[\mcG_n]$ and $\degree_n(u)$ for the number of edges, and degrees of a vertex $u$ in $\mcG_n$, respectively.

For triangular arrays of random variables $(X_{n, k})$ and $(Y_{n, k})$, we say that $X_{n, k} = o_{p; k}(Y_{n, k} )$ if for all $\epsilon > 0$, $\delta > 0$, there exists $N_{\epsilon, \delta}(k)$ such that for all $n \geq N_{\epsilon, \delta}(k)$ we have that $\mathbb{P}\big( |X_{n, k} | > \delta | Y_{n, k} | \big) < \epsilon$. If $N_{\delta, \epsilon}(k)$ can be chosen uniformly in $k$, then we simply write $X_{n, k} = o_p(Y_{n, k} )$. We use similar notation for $O_p(\cdot)$, $\omega_p(\cdot)$ (where $X_n = \omega_p(Y_n)$ iff $Y_n = o_p(X_n)$), $\Omega_p(\cdot)$ (where $X_n = \Omega_p(Y_n)$ iff $Y_n = O_p(X_n)$) and $\Theta_p(\cdot)$ (where $X_n = \Theta_p(Y_n)$ iff $X_n = O_p(Y_n)$ and $Y_n = O_p(X_n)$). For non-stochastic quantities, we use similar notation, except that we drop the subscript $p$. Throughout, we use the notation $| \cdot |$ to denote the measure of sets; specifically, if $A \subseteq \mathbb{N}$ then $|A|$ is the number of elements of the set $A$, and if $A \subseteq \mathbb{R}$ then $|A|$ or $\mu(A)$ is the Lebesgue measure of the set $A$. Similarly, for sequences and functions, we use $\| \cdot \|_p$ to denote the $\ell_p$ or $L^p$ norms respectively. The notation $[n]$ indicates the set of integers $\{1, \ldots, n \}$.

\subsection{Outline of paper} \label{sec:intro:outline}

In Section~\ref{sec:framework}, we discuss the main object of study in the paper, and the assumptions we require throughout. The assumptions concern the data generating process of the observed network, the behavior of the subsampling scheme used, and the properties of the loss function used to learn embedding vectors. Section~\ref{sec:embed_learn} consist of the main theoretical results of the paper, giving a consistency result for the learned embedding vectors under different subsampling schemes. Section~\ref{sec:sampling_formula} gives examples of subsampling schemes which our approach allows us to analyze, and highlights a scale invariance property of subsampling schemes which perform random walks on a graph. In Section~\ref{sec:exper}, we demonstrate on real data the benefit in using an indefinite or Krein inner product between embedding vectors, and demonstrate the validity of our theoretical results on simulated data. Proofs are deferred to the appendix, with a brief outline of the ideas used for the main results given in Appendix~\ref{sec:proof_sketch}.
%!TEX root = ms.tex

\section{Framework of analysis} \label{sec:framework}

We consider the problem of minimizing the empirical risk function 
\begin{equation} \label{framework:eq:empirical_loss}
    \mathcal{R}_n(\omega_1, \ldots, \omega_n) = \sum_{i, j \in [n], i \neq j} \mathbb{P}\left( (i, j) \in S(\mathcal{G}_n) \big| \mathcal{G}_n   \right) \ell(B(\omega_i, \omega_j), a_{ij})
\end{equation}
where we have that
\begin{enumerate}[label=\roman*)]
    \item the embedding vectors $\omega_i \in \mathbb{R}^d$ are $d$-dimensional (where $d$ is allowed to grow with $n$), with $\omega_i$ corresponding to the embedding of vertex $i$ of the graph;
    \item $\ell : \mathbb{R} \times\{0, 1\} \to [0, \infty)$ is a non-negative loss function;
    \item $B : \mathbb{R}^d \times \mathbb{R}^d \to \mathbb{R}$ is a (bilinear) similarity measure between embedding vectors; and 
    \item $S(\mathcal{G}_n)$ refers to a stochastic subsampling scheme of the graph $\mathcal{G}_n$, with $\mcG_n$ representing a graph on $n$ vertices.
\end{enumerate}
We now discuss our assumptions for the analysis of this object, which relate to a generative model of the graph $\mcG_n$, the loss function used, and the properties of the subsampling scheme. For purposes of readability, we first provide a simplified set of assumptions, and give a general set of assumptions for which our theoretical results hold in Appendix~\ref{sec:app:assumptions}.

\subsection{Data generating process of the network}

We begin by imposing some regularity conditions on the data
generating process of the network. Recall that we assume the graphs $(\mathcal{G}_n)_{n \geq 1}$ are generated from a graphon process with latent variables $\lambda_i \iid \mathrm{Unif}[0, 1]$ and generating graphon $W_n(l, l') = \rho_n W(l, l')$, where $W$ is a graphon and $\rho_n$ is a sparsity factor which may shrink to zero as $n \to \infty$.

\begin{remark}
    \label{rmk:framework:higher_dim}
    The above assumption corresponds to the graph $\mcG_n$ being an exchangeable graph. Parameterizing such graphs through a graphon $W: [0, 1]^2 \to \mathbb{R}$ and one dimensional latent variables $\lambda_i \sim U[0, 1]$ is a canonical choice as a result of the Aldous-Hoover theorem \citep[e.g][]{aldous_representations_1981}, and is extensive in the network analysis literature. However, this is not the only possible choice for the latent space. More generally we could consider some probability measure $Q$ on $\mathbb{R}^q$, and a symmetric measurable function $\widetilde{W} : (\mathbb{R}^q)^2 \to [0, 1]$, where the graph is generated by assigning a latent variable $\tilde{\lambda}_i \sim Q$ independently for each vertex, and then joining vertices $i < j$ with an edge independently of each other with probability $\widetilde{W}(\tilde{\lambda}_i, \tilde{\lambda}_j)$.
    
    From a modelling perspective a higher dimensional latent space is desirable; an interesting fact is that any such graph is equivalent in law to one drawn from a graphon with latent variables $\lambda_i \sim U[0, 1]$ \citep[Theorem 7.1]{janson_standard_2009}. As a simple illustration of this fact, suppose that users in a social network graph have characteristics $x_i \in \{0, 1\}^q$ for some $q \in \mathbb{N}$, and that two individuals $i$ and $j$ are connected in the network (independently of any other pair of users) with probability $\widetilde{W}(x_i, x_j)$, which depends only on their characteristics. Assuming that the $x_i$ are drawn i.i.d from a distribution $p(x)$ on $[0, 1]^q$, we can always simulate such a distribution by partitioning $[0, 1]$ according to the probability mass function $p(x)$, drawing a latent variable $\lambda_i \sim U[0, 1]$, and then assigning $x_i$ to the value corresponding to the part of the partition of $[0, 1]$ for which $\lambda_i$ landed in. Letting $\phi: [0, 1] \to \{0, 1\}^q$ denote this mapping, the model is then equivalent to a one with a graphon $W(\phi(\lambda_i), \phi(\lambda_j))$. Consequently, our results will be presented mostly in terms of graphons $W: [0, 1]^2 \to [0, 1]$. However, they can be extended with relative ease to graphons with higher dimensional latent spaces, which we discuss further in Section~\ref{sec:embed_learn:higher_dim}.
\end{remark}

\begin{assume}[Regularity + smoothness of the graphon] \label{assume:simple:graphon_ass}
    We suppose that the sequence of graphons $(W_n = \rho_n W)_{n \geq 1}$ generating $(\mcG_n)_{n \geq 1}$ are, up to weak equivalence of graphons \citep{lovasz_large_2012}, such that i) the graphon $W$ is piecewise H\"{o}lder$([0, 1]^2$, $\beta_W$, $L_W$, $\mcQ^{\otimes 2})$ for some partition $\mcQ$ of $[0, 1]$ and constants $\beta_W \in (0, 1]$, $L_W \in (0, \infty)$; ii) there exist constants $c_1, c_2 > 0$ such that $W \geq c_1$ and $1 - \rho_n W \geq c_2$ a.e; and iii) the sparsifying sequence $(\rho_n)_{n \geq 1}$ is such that $\rho_n = \omega(log(n)/n )$.
\end{assume}

\begin{remark}
    We will briefly discuss the implications of the above assumptions. The smoothness assumptions in a) are standard when assuming networks are generated from graphon models \citep[e.g][]{wolfe_nonparametric_2013,gao_rate-optimal_2015, klopp_oracle_2017, xu_rates_2018}. The assumption in b) that $W$ is bounded from below is strong, and is weakened in the most general assumptions listed in Appendix~\ref{sec:app:assumptions}. This, along with the assumption that $\rho_n = \omega( \log(n)/n)$, implies that the degree structure of $\mcG_n$ is regular, in that the degrees of every vertex are roughly of the same order, and will grow to infinity as $n$ does; this is a limitation in that real world networks do not always exhibit this type of behavior, and have either scale-free or heavy-tailed degree distributions \citep[e.g][]{albert_diameter_1999,broido_scale-free_2019,zhou_power-law_2020}. Regardless of the sparsity factor, graphon models will tend to have structural deficits; for example, they tend to not give rise
    to partially isolated substructures \citep{orbanz_subsampling_2017}. We note that assumptions on the sparsity factor where $n \rho_n$ grows like $(\log n)^c$ for some $c \geq 1$, remain standard when using graphons as a tool for theoretical analyses \citep[e.g][]{wolfe_nonparametric_2013, borgs_private_2015, klopp_oracle_2017, xu_rates_2018,oono_graph_2021}. Future work could
    extend our results to generalizations of graphon models, such as graphex models \citep{veitch_class_2015, borgs_sampling_2019}, which better account for issues of sparsity and regularity of graphs.
\end{remark}

\subsection{Assumptions on the loss function and \texorpdfstring{$B(\omega, \omega')$}{B(w, w')}}

We now discuss our assumptions on the loss function $\ell(y, x)$, which we follow with a discussion as to the form of the functions $B(\omega, \omega')$. 

\begin{assume}[Form of the loss function] \label{assume:simple:loss}
    We assume that the loss function is equal to the cross-entropy loss
    \begin{equation}
        \ell(y, x) := -x \log\big( \sigma(y) \big) - (1 -x) \log\big( 1- \sigma(y) \big) \text{ for } y \in \mathbb{R}, x \in \{0, 1\},
    \end{equation}
    where $\sigma(y) := (1 + e^{-y})^{-1}$ is the sigmoid function. 
\end{assume}

We note that our analysis can be extended to loss functions of the form
\begin{equation*}
    \ell(y, x) := - x\log\big( F(y) \big) - (1-x) \log\big( 1 - F(y) \big),
\end{equation*}
where $F$ corresponds to a distribution which is continuous, symmetric 
about $0$ and strictly log-concave. This includes the probit loss (Assumption~\ref{assume:loss_prob}), or more general classes of strictly convex functions $\ell(y, x)$ which include the squared loss $\ell(y, x) = (y -x)^2$ (Assumption~\ref{assume:loss}). We now discuss the form of $B(\omega, \omega')$. 

\begin{assume}[Properties of the similarity measure $B(\omega, \omega')$] \label{assume:simple:bilinear} 
    Supposing we have \linebreak embedding vectors $\omega, \omega' \in \mathbb{R}^d$, we assume that the similarity measure $B$ is equal to one of the following bilinear forms: 
    \begin{enumerate}[label=\roman*)]
        \item $B(\omega, \omega') = \langle \omega, \omega' \rangle$ (i.e a regular or definite inner product) or
        \item $B(\omega, \omega') = \langle \omega, I_{d_1, d - d_1} \omega' \rangle = \langle \omega_{[1:d_1]}, \omega'_{[1:d_1]} \rangle - \langle \omega_{[(d_1+1):d]}, \omega'_{[(d_1 + 1):d]} \rangle$ for some $d_1 \leq d$ (i.e an indefinite or Krein inner product);
    \end{enumerate}    
    where $I_{p, q} = \mathrm{diag}( I_p, - I_q )$, $\omega_A = (\omega_i)_{i \in A}$ for $A \subseteq [d]$, and $[a:b] = \{a, a + 1, \ldots, b\}$.
\end{assume}

\subsection{Assumptions on the sampling scheme}

We now introduce our assumptions on the sampling scheme. For most subsampling schemes, the probability that the pair $(i, j)$ is part of the subsample $S(\mcG_n)$ depends only on \emph{local} features of the underlying graph $\mathcal{G}_n$. We formalize this notion as follows:

\begin{assume}[Strong local convergence] \label{assume:simple:slc}
    There exists a sequence $(f_n(\lambda_i, \lambda_j, a_{ij}))_{n \geq 1}$ of $\sigma(W)$-measurable functions, with $\mathbb{E}[ f_n(\lambda_1, \lambda_2, a_{12} )^2 ] < \infty$ for each $n$, such that 
    \begin{equation*}
        \max_{i, j \in [n], i \neq j} \Big| \frac{n^2 \mathbb{P}((i,j)\in S(\mathcal{G}_n)|\mathcal{G}_n) }{ f_n(\lambda_i, \lambda_j, a_{ij}) } - 1 \Big| = O_p(s_n)
    \end{equation*}
    for some non-negative sequence $s_n = o(1)$.
\end{assume}

We refer to the $f_n$ as sampling weights. This condition implies that the probability that $(i, j)$ is sampled depends approximately on only local information, namely the latent variables $\lambda_i$, $\lambda_j$ and the value of $a_{ij}$, i.e that 
\begin{equation}
    \mathbb{P}\big( (i, j) \in S(\mcG_n) \,|\, \mcG_n \big) \approx \frac{f_n(\lambda_i, \lambda_j, a_{ij})}{n^2}  \text{ for all } i, j \in [n].
\end{equation}
As a result of the concentration of measure phenomenon, many sampling frameworks satisfy this condition (see Section~\ref{sec:sampling_formula}). This includes those used in practice, such as uniform vertex sampling, uniform edge sampling \citep{tang_line_2015}, along with "random walk with unigram negative sampling" schemes like those of Deepwalk \citep{perozzi_deepwalk_2014} and node2vec \citep{grover_node2vec_2016}. In particular, we are able to give explicit formulae for the sampling weights in these scenarios. We also impose some regularity conditions on the conditional averages of the sampling weights.

\begin{assume}[Regularity of the sampling weighs] 
    \label{assume:simple:samp_weight_reg}
    We assume that, for each $n$, the functions
    \begin{equation*}
        \fnone := f_n(l, l', 1) W_n(l, l') \text{ and } \fnzero := f_n(l, l', 0) (1 - W_n(l, l') )
    \end{equation*}
    are piecewise H\"{o}lder$([0, 1]^2, \beta, \fnholderconst, \mathcal{Q}^{\otimes 2})$. $\mcQ$ is the same partition as in Assumption~\ref{assume:simple:graphon_ass}, but the exponents $\beta$ and $\fnholderconst$ may differ from that of $\beta_W$ and $L_W$ in Assumption~\ref{assume:simple:graphon_ass}. We moreover suppose that $\tilde{f}_n(l, l', 1)$ and $\tilde{f}_n(l, l', 0)$ are uniformly bounded in $L^{\infty}([0, 1]^2)$, are are also uniformly bounded below and away from zero. 
\end{assume}

\begin{remark} \label{rmk:framework:ex_formula}
    % are introduced for sake of convenience, as the rates of convergence we obtain later depend on the smoothness of the $\fnone$ and $\fnzero$, which may be different than that of $W$. However, 
    For all the sampling schemes we consider, the conditions on $\fnone$ and $\fnzero$ will follow from Assumption~\ref{assume:simple:graphon_ass} and the formulae for the sampling weights we derive in Section~\ref{sec:sampling_formula}; in particular, the exponent $\beta$ will be a function of $\beta_W$ and the particular choice of sampling scheme. To illustrate this, if we suppose that we use a random walk scheme with unigram negative sampling \citep{perozzi_deepwalk_2014} as later described in Algorithm~\ref{alg:random_walk}, we show later (Proposition~\ref{sec:sampling:rw_uni_stat_formula}) that 
    \begin{align}
        \tilde{f}_n(\lambda, \lambda', 1) & = \frac{ 2 k W(\lambda, \lambda') }{ \mcE_W } \\
        \tilde{f}_n(\lambda, \lambda', 0) & = \frac{l (k+1) (1 - \rho_n W(\lambda, \lambda'))} { \mcE_W \mcE_W(\alpha) } \big\{ W(\lambda, \cdot) W(\lambda', \cdot)^{\alpha} + W(\lambda, \cdot)^{\alpha} W(\lambda', \cdot) \big\} 
    \end{align}
    where $k$, $l$ and $\alpha \in (0, 1]$ are hyperparameters of the sampling scheme. In particular, if $W$ is piecewise H\"{o}lder with exponent $\beta$, then we show (Lemma~\ref{app:sampling:prod_deg_holder}) that $\tilde{f}_n(\lambda, \lambda', 1)$ and $\tilde{f}_n(\lambda, \lambda', 0)$ will be piecewise H\"{o}lder with exponent $\alpha \beta$.
\end{remark}
%!TEX root = ms.tex

\section{Asymptotics of the learned embedding vectors} \label{sec:embed_learn}

In this section, we discuss the population risk corresponding to the empirical risk \eqref{framework:eq:empirical_loss}, show that any minimizer of \eqref{framework:eq:empirical_loss} converges to a minimizer of this population risk, and then discuss some implications and uses of this result. 

\subsection{Convergence of empirical risk to population risk} 

Given the empirical risk \eqref{framework:eq:empirical_loss}, and assuming that the embedding vectors are constrained to lie within a compact set $S_d = [-A, A]^d$ for some $A$, our first result shows that the population limit analogue of \eqref{framework:eq:empirical_loss} has the form
\begin{equation} \label{eq:loss_converge:risk}
    \begin{split} 
        \mathcal{I}_n[K] := \int_{[0, 1]^2} \Big\{ \tilde{f}_n(l, l', 1) \ell( K(l, l'), 1) + \tilde{f}_n(l, l', 0) \ell( K(l, l'), 0) \Big\} \, dl dl',
    \end{split}
\end{equation}
where the domain consists of functions $K(l, l') = B(\eta(l), \eta(l'))$ for functions $\eta: [0, 1] \to \compactset$. We can interpret $\eta$ as giving embedding vectors $\eta(\lambda)$ for vertices with latent feature $\lambda$, with $K(\lambda, \lambda')$ then measuring the similarity between two vertices with latent features $\lambda$ and $\lambda'$. We write 
\begin{equation} \label{eq:loss_converge:K_minima_set}
    Z(\compactset) := \Big\{ K\,:\, K(l, l') = B(\eta(l), \eta(l') ) \text{ for }\eta : [0, 1] \to \compactset \Big\}
\end{equation}
for all such functions $K$ which are represented in this fashion. We then have that the minimized empirical risk $\mcR_n(\bmomega)$ converges to the minimized population risk $\mcI_n[K]$:

\begin{theorem} \label{thm:loss_converge}
    Suppose that Assumptions~\ref{assume:simple:graphon_ass},~\ref{assume:simple:loss},~\ref{assume:simple:bilinear},~\ref{assume:simple:slc}~and~\ref{assume:simple:samp_weight_reg} hold. Let $\compactset = [-A, A]^d$ be the $d$-dimensional hypercube of radius $A$. Then we have that, writing $\bmomega = (\omega_1, \ldots, \omega_n)$, 
    \begin{equation*}
        \Big| \min_{\bm{\omega}_n \in (S_d)^n } \mathcal{R}_n(\bm{\omega}_n)  - \min_{K \in Z(S_d) } \mathcal{I}_n[K] \Big| = O_p\Big(  s_n + \frac{ d^{3/2} \mathbb{E}[f_n^2]^{1/2} }{ n^{1/2} } + \frac{ (\log n)^{1/2}  }{ n^{\beta/(1 + 2\beta)} } \Big),
    \end{equation*}
    where we write 
    \begin{equation*}
        \mathbb{E}[f_n^2] = \mathbb{E}[f_n(\lambda_1, \lambda_2, a_{12})^2] = \intsq \{ f_n(l, l', 1)^2 W_n(l, l') + f_n(l, l', 0)^2 (1 - W_n(l, l') ) \} \, dl dl'.
    \end{equation*}
    In the case where $\fnone$ and $\fnzero$ are piecewise constant on a
    partition $\mcQ^{\otimes 2}$ where $\mcQ$ is of size $\kappa$, we have
    \begin{equation*}
        \Big| \min_{\bm{\omega}_n \in (S_d)^n } \mathcal{R}_n(\bm{\omega}_n)  - \min_{K \in Z(S_d) } \mathcal{I}_n[K] \Big| = O_p\Big(  s_n + \frac{ d^{3/2} \mathbb{E}[f_n^2]^{1/2} }{ n^{1/2} } +  \frac{ (\log \kappa)^{1/2}  }{ n^{1/2} } \Big),
    \end{equation*}
\end{theorem}

The proof can be found in Appendix~\ref{sec:app:loss_converge_proof} (with Theorem~\ref{app:loss_converge_proof:main_theorem} stating a more general result under the assumptions listed in Appendix~\ref{sec:app:assumptions}), with a proof sketch in Appendix~\ref{sec:proof_sketch}.

\begin{remark} \label{rmk:embed_learn:remark_1}
    The error term above consists of three parts. The $s_n$ term relates to the fluctuations of the empirical sampling probabilities to the sampling weights $\fnone$ and $\fnzero$. The second term arises as the penalty for getting uniform convergence of the loss functions when averaged over the adjacency assignments. The final term arises from using a stochastic block approximation for the functions $\fnone$ and $\fnzero$, and optimizing the tradeoff between the number of blocks for approximating these functions, and the relative error in the proportion of the $\lambda_i$ in a block versus the size of the block.
\end{remark} 

\begin{remark} \label{rmk:embed_learn:remark_1a}
    Typically for random walk schemes we have that $s_n = O((\log(n)/n\rho_n)^{1/2})$ and $\mathbb{E}[f_n^2] = O(\rho_n^{-1})$ under Assumption~\ref{assume:simple:graphon_ass}, and so the error term is of the form
    \begin{equation*}
        O_p\Big( \Big( \frac{ \max\{ \log n, d^{3} \} }{ n \rho_n } \Big)^{1/2} + \Big( \frac{  \log n }{ n^{2\beta/(1 +2 \beta) } } \Big)^{1/2} \Big).
    \end{equation*}
    One affect of this is that as $\rho_n$ decreases in magnitude, the permissable embedding dimensions decrease also; we also always require that $d \ll n$ in order for the rate $r_n \to 0$.
\end{remark}

\subsection{Convergence of the learned embedding vectors}
\label{sec:embed_learn:embed_converge}

We now argue that the minimizers of \eqref{framework:eq:empirical_loss}
converge in an appropriate sense to a minimizer of $\mcI_n[K]$ over a 
constraint set which depends on the choice of similarity measure $B(\omega,
\omega')$. Before considering any constrained estimation of $\mcI_n[K]$, 
we highlight that depending on the form of $\ell(y, x)$, we can write down
a closed form to the unconstrained minimizer of $\mcI_n[K]$ over all
(symmetric) functions $K$. When $\ell(y, x)$ is the cross-entropy loss, 
by minimizing the integrand of $\mcI_n[K]$ point-wise, the unconstrained
minimizer of $\mcI_n[K]$ will equal 
\begin{equation} \label{eq:embed_learn:optimalK}
    \optimalK := \sigma^{-1}\Big( \frac{ \fnone}{ \fnone + \fnzero } \Big) \text{ where } \sigma^{-1}(x) = \log\Big( \frac{x}{1-x} \Big).
\end{equation}
As $\fnone$ and $\fnzero$ are proportional to $W_n(l, l')$ and $1 - W_n(l, l')$ respectively, we are learning a re-weighting of the original graphon. As a special case, if the sampling formulae are such that $f_n(l, l', 1) = f_n(l, l', 0)$ (so the probability that a pair of vertices is sampled is asymptotically independent of whether they are connected in the underlying graph) then \eqref{eq:embed_learn:optimalK} simplifies to the equation $\optimalK = \sigma^{-1}(W_n)$. This is the case for a sampling scheme which samples vertices uniformly at random and then returns the induced subgraph (Algorithm~\ref{alg:psamp}). Otherwise, $\optimalK$ will still depend on $W_n$, but may not be an invertible transformation of $W_n$; for example, for a
random walk sampler with walk length $k$, one negative sample per positively sampled vertex, and a unigram negative sampler with $\alpha = 1$ (Algorithm~\ref{alg:random_walk}), we get that
\begin{equation}
    \optimalK = \log\Big(  \frac{ W(\lambda_i, \lambda_j) \mcE_W (1 + k^{-1}) }{ (1 - \rho_n W(\lambda_i, \lambda_j)) W(\lambda_i, \cdot) W(\lambda_j, \cdot) } \Big).
\end{equation}

As a result of Theorem~\ref{thm:loss_converge}, we posit that when taking $d 
\to \infty$ as $n \to \infty$, the embedding vectors learned via minimizing
\eqref{framework:eq:empirical_loss} will converge to a minimizer of $\mcI_n[K]$
when $K$ is constrained to the "limit" of the sets $\mcZ(S_d)$ in \eqref{eq:loss_converge:K_minima_set} as $d \to \infty$. As this set depends on
$B(\omega, \omega')$, whether $B(\omega, \omega')$ is a positive-definite
inner product (or not) corresponds to whether $K$ is constrained to being
non-negative definite (or not) in the following sense: suppose $K$ allows an
expansion of the form
\begin{equation}
    K(l, l') = \sum_{i=1}^{\infty} \mu_i(K) \phi_i(l) \phi_i(l') \quad \text{(as a limit in $L^2([0, 1]^2)$)}
\end{equation}
for some numbers $(\mu_i(K))_{i \geq 1}$ and orthonormal functions $(\phi_i)_{i \geq 1}$. Then, are
the $\mu_i$ all non-negative - in which case $K$ is non-negative definite - or
not? We prove in Appendix~\ref{sec:app:holder_props} that as a consequence of our assumptions, we can write 
\begin{equation} \label{eq:embed_learn:optimalK_eigen}
    \optimalK(l, l') = \sum_{i=1}^{\infty} \mu_i(\optimalK) \phi_{n, i}(l) \phi_{n, i}(l') \quad \text{(as a limit in $L^2([0, 1]^2)$)}
\end{equation}
where for each $n$ the collection of functions $(\phi_{n, i} )_{i \geq 1}$ are orthonormal. With this, we begin with giving a convergence guarantee when $\mu_i(\optimalK) \geq 0$ for all $i, n \geq 1$. In this case, $\optimalK$ is the limiting distribution of the inner products of the embedding vectors learned via minimizing \eqref{framework:eq:empirical_loss}.

\begin{theorem} \label{thm:embed_learn:converge_1}
    Suppose that Assumptions~\ref{assume:simple:graphon_ass},~\ref{assume:simple:loss},~\ref{assume:simple:slc}~and~\ref{assume:simple:samp_weight_reg} hold. Also suppose that Assumption~\ref{assume:simple:bilinear} holds with $B(\omega, \omega') = \langle \omega, \omega' \rangle$ with $\omega \in \mathbb{R}^d$. Finally, suppose that in \eqref{eq:embed_learn:optimalK_eigen} the $\mu_i(\optimalK)$ are non-negative for all $n, i \geq 1$. Then there exists $A'$ sufficiently
    large such that whenever $A \geq A'$, for any sequence of minimizers $(\whomega_1, \ldots, \whomega_n) \in \argmin_{\bmomega \in ([-A, A]^d)^n} \emprisk(\bmomega)$, we have that 
    \begin{align*}
        \frac{1}{n^2} \sum_{i, j \in [n] } &\big| \langle \whomega_i, \whomega_j \rangle - \optimalK(\lambda_i, \lambda_j) \big| = O_p( \tilde{r}_n^{1/2} )\\
        & \text{ where } \tilde{r}_n = s_n + \frac{ d^{3/2} \mathbb{E}[f_n^2]^{1/2} }{ n^{1/2} } +  \frac{ (\log n)^{1/2} }{ n^{\beta/(1 + 2\beta)} } + \Big( \frac{ \log n}{ n } \Big)^{\beta/2} + d^{-1/2 - \beta}.
    \end{align*}
    In the case where the $\fnone$ and $\fnzero$ are piecewise constant on a fixed partition $\mcQ^{\otimes 2}$ for all $n$, where $\mcQ$ is a partition of $[0, 1]$ into $\kappa$ parts, then $\optimalK$
    is piecewise constant on $\mcQ^{\otimes 2}$ also, there exists
    $q \leq \kappa$ such that, then provided $d \geq q$, the above convergence result holds with
    \begin{equation*}
        \tilde{r}_n = s_n + \frac{ d^{3/2} \mathbb{E}[f_n^2]^{1/2}}{ n^{1/2} } + \frac{ ( \log \kappa)^{1/2} }{ n^{1/2} }. 
    \end{equation*}
\end{theorem}

See Theorem~\ref{app:embed_converge_proof:embed_convergence} in Appendix~\ref{sec:app:embed_converge_proof} for the proof, with the latter theorem holding under more general regularity conditions. We highlight that in the above theorem, one can also take $B(\omega, \omega') = \langle \omega, I_{d, d'} \omega' \rangle$ with $\omega \in \mathbb{R}^{d + d'}$ and $I_{d, d'} = \mathrm{diag}(  I_d, -I_{d'} )$ and have the convergence theorem also hold, with the
$d^{3/2}$ term being replaced by a $(d+d')^{3/2}$ term.

\begin{remark}
    In the above bound for $\tilde{r}_n$, the first three terms correspond to the terms in the convergence of the loss function as in Theorem~\ref{thm:loss_converge}. The fourth term arises from relating the matrix $(\optimalK(\lambda_i, \lambda_j))_{i,j}$ back to the function $\optimalK$. The fifth term arises from the error in considering the difference between $\optimalK$ and the best rank $d$ approximation to $\optimalK$; in particular, if $\optimalK$ is actually finite rank in that $\mu_i(\optimalK) = 0$ for all $i \geq q$, for some $q$ free of $n$, then provided $d \geq q$ we can discard the $d^{-1/2 - \beta}$ term, and so under the conditions in which the rate in Theorem~\ref{thm:loss_converge}
    converges to zero, the rate in Theorem~\ref{thm:embed_learn:converge_1} also goes to zero as $n \to \infty$. 
    
    In general, from the above result we can argue that there exists a sequence of embedding dimensions $d = d(n)$ such that $\tilde{r}_n \to 0$ as $n \to \infty$, albeit possibly at a slow rate (by choosing e.g $d = (\log n)^c$ for $c$ very small). If the $\fnone$ and $\fnzero$ are piecewise constant on a partition of size $\kappa$, then it is in fact possible to obtain consistency as soon as $d = \kappa$ and $d' = 0$. Here, there is a tradeoff between choosing $d$ large enough so that we get a good rank $d$ approximation to $\optimalK$, and keeping the capacity of the optimization domain sufficiently small that the convergence of the minimal loss values is quick (see Remark~\ref{rmk:sec:embed_converge:rates_tradeoff} for a discussion of choosing $d$ optimally). 
    
    We finally note that in the statement of Theorem~\ref{thm:embed_learn:converge_1} the constant $A$ is held fixed; it is however possible to take $A = O(\log n)$ and have the bound $\tilde{r}_n$ increase only by a multiplicative factor of $O( (\log n)^c )$ for some constant $c$.
\end{remark}

In the case where some of the $\mu_i(\optimalK)$ are negative, we can obtain a similar result which gives convergence to $\optimalK$, although now choosing $B(\omega, \omega') = \langle \omega, I_{d_1, d_2} \omega' \rangle$ is necessary. We show later in Proposition~\ref{thm:embed_learn:sbm_example_1} an example of a two community SBM which highlights the necessity of using a Krein inner product. 

\begin{theorem} \label{thm:embed_learn:converge_2}
    Suppose that Assumptions~\ref{assume:simple:graphon_ass},~\ref{assume:simple:loss},~\ref{assume:simple:bilinear},~\ref{assume:simple:slc}~and~\ref{assume:simple:samp_weight_reg} hold. Given an embedding dimension $d = d(n)$, pick $d_1$ and $d_2 = d -d_1$ in $B(\omega, \omega') = \langle \omega, I_{d_1, d_2} \omega' \rangle$ where $I_{d, d'} = \mathrm{diag}(  I_d, -I_{d'} )$, such that $d_1$ is equal to the number of non-negative values out of the $d$ absolutely largest values of $\mu_i(\optimalK)$ in \eqref{eq:embed_learn:optimalK_eigen}. Then there exists $A'$ sufficiently large such that whenever $A \geq A'$, for any sequence of minimizers $(\whomega_1, \ldots, \whomega_n) \in \argmin_{\bmomega \in ([-A, A]^d)^n} \emprisk(\bmomega)$, we have that 
    \begin{align*}
        \frac{1}{n^2} \sum_{i, j \in [n] } &\big| B(\whomega_i, \whomega_j) - \optimalK(\lambda_i, \lambda_j) \rangle \big| = O_p( \tilde{r}_n^{1/2} )\\
        & \text{ where } \tilde{r}_n = s_n + \frac{ d^{3/2} \mathbb{E}[f_n^2]^{1/2} }{ n^{1/2} } +  \frac{ (\log n)^{1/2} }{ n^{\beta/(1 + 2\beta)} } + \Big( \frac{ \log n}{ n } \Big)^{\beta/2} + d^{- \beta}.
    \end{align*}
    In the case where the $\fnone$ and $\fnzero$ are piecewise constant on a fixed partition $\mcQ^{\otimes 2}$ for all $n$, where $\mcQ$ is a partition of $[0, 1]$ into $\kappa$ parts, then there exists $q \leq \kappa$
    for which, as soon as $d = d_1 + d_2 \geq q$, we have that the above convergence result holds with
    \begin{equation*}
        \tilde{r}_n = s_n + \frac{ d^{3/2} \mathbb{E}[f_n^2]^{1/2}}{ n^{1/2} } + \frac{ ( \log \kappa)^{1/2} }{ n^{1/2} }. 
    \end{equation*}
\end{theorem}

\begin{remark}
    \label{rmk:sec:embed_converge:rates_tradeoff}
    The $d^{-\beta}$ term above is the analogue of the $d^{-1/2 - \beta}$ term in Theorem~\ref{thm:embed_learn:converge_1}, which arises from the fact that the decay of the $\mu_i(\optimalK)$ as a function of $i$ is quicker when we can guarantee that they are all positive. Consequently, we have analogous remarks for that if the $\mu_i(\optimalK)$ are all zero for $i \geq \kappa$, then as soon as $\min\{d_1, d_2 \} \geq \kappa$, this term will disappear. Similarly, the $d^{-\beta}$ term arises from looking
    at the best rank $d$ approximation to $\optimalK$. As the eigenvalues
    can be positive and negative, the choice of $d_1$ and $d_2$ means we choose the top $d$ eigenvalues (by absolute value) for any given
    $d$, and so we can obtain the $d^{-\beta}$ rate. To see how the rates of convergence are affected by the optimal choice of embedding dimension $d$, when $s_n = O((\log(n)/n\rho_n)^{1/2})$ and $\mathbb{E}[f_n^2] = O(\rho_n^{-1})$, optimizing over $d$ gives
    \begin{equation*}
        \tilde{r}_n = \Big( \frac{ \log n}{n \rho_n} \Big)^{1/2} + \Big( \frac{ \log n}{n^{2\beta/(1+2\beta)}}  \Big)^{1/2} + \Big( \frac{ \log n}{n}  \Big)^{\beta/2} + (n\rho_n)^{-\beta/(3+2\beta)},
    \end{equation*}
    and so the last term will tend to dominate in the sparse regime. 
\end{remark}

To summarize, Theorems~\ref{thm:embed_learn:converge_1} and \ref{thm:embed_learn:converge_2} characterize the distribution of pairs of embedding vectors, through the similarity measure $B(\omega, \omega')$ used for training. They show that the distribution of embedding vectors asymptotically decouple in that, in an average sense, the distribution of $B(\widehat{\omega}_i, \widehat{\omega}_j)$ depends only on the latent features $(\lambda_i, \lambda_j)$ for the respective vertices. Moreover, when we have a cross-entropy loss and the similarity measure $B(\omega, \omega')$ is correctly specified, we can explicitly write down the limiting distribution in terms of the sampling formulae corresponding to the choice of sampling scheme, and the original generating graphon.

\subsection{Extension to graphons on higher dimensional latent spaces}
\label{sec:embed_learn:higher_dim}

As discussed earlier in Remark~\ref{rmk:framework:higher_dim}, it is possible to consider graphons more generally as functions $W: (\mathbb{R}^q)^2 \to [0, 1]$ with latent variables $\bm{\lambda}_i$ drawn from some probability distribution on $\mathbb{R}^q$. As these can always be made equivalent to graphons $W: [0, 1]^2 \to [0, 1]$, there is a natural question as to whether our results can be applied to higher dimensional graphons. To illustrate that we can do so, here we illustrate what occurs when we have a graphon with latent variables $\bm{\lambda}_i \sim U([0, 1]^q)$ independently for some $q \in \mathbb{N}$, with a graphon function $W: ([0, 1]^q)^2 \to [0, 1]$:

\begin{assume}[Graphon with high dimensional latent factors]
    \label{assume:graphon_ass_high}
    Suppose that the \linebreak $(\mcG_n)_{n \geq 1}$ are generated by a sequence of graphons $(W_n = \rho_n W)_{n \geq 1}$ where; the latent parameters $\bm{\lambda}_i \sim \mathrm{Unif}([0, 1]^q)$ for some $q \in \mathbb{N}$; the graphon $W: ([0, 1]^q)^2 \to [0, 1]$ is symmetric and piecewise H\"{o}lder$( ([0, 1]^q)^2, \beta_W, L_W,\mcQ^{\otimes 2})$ for some partition $\mcQ$ of $[0, 1]$; there exist constants $0 < c < C < 1$ such that $c \leq W \leq C$ a.e; and $\rho_n = \omega( \log(n)/n)$. Moreover, we suppose that the functions
    \begin{equation*}
        \fnonebb := f_n(\bm{l}, \bm{l}', 1) W_n( \bm{l}, \bm{l}') \quad \text{ and } \quad \fnzerobb := f_n( \bm{l}, \bm{l}', 0) (1 - W_n( \bm{l}, \bm{l}') ),
    \end{equation*} 
    defined for $\bm{l}, \bm{l}' \in [0, 1]^q$, are piecewise H\"{o}lder$([0, 1]^q)^2, \beta, L_f, \mcQ^{\otimes 2})$ for some exponent $\beta$; are uniformly bounded above; and uniformly bounded below and away from zero. 
\end{assume}

To apply our existing results, we will make use of the following theorem.

\begin{theorem}
    \label{thm:graphon_equiv}
    Let $W$ be a graphon on $[0, 1]^q$ which is H\"{o}lder$(([0, 1]^q)^2, \beta, L)$. Then there exists an equivalent graphon $W'$ on $[0, 1]$ which is H\"{o}lder$([0, 1], \beta q^{-1}, L')$ where $L'$ depends only on $L$ and $q$. Moreover, for any $p \in [1, \infty]$ and function $f: [0, 1] \to \mathbb{R}$ we have that $\| f(W) \|_{L^p(([0, 1]^q)^2)} = \| f(W') \|_{L^p([0, 1]^2)}$.
\end{theorem}

\begin{proof}[Proof of Theorem~\ref{thm:graphon_equiv}]
    The first part is simply Theorem~2.1 of \cite{janson_can_2021}, which uses the fact that there exists a measure preserving map $\phi: [0, 1] \to [0, 1]^q$ which is H\"{o}lder($q^{-1}$, $C_{q}$) for some constant $C_{q}$, in which case $W^{\phi}(x, y) := W(\phi(x), \phi(y))$ is equivalent to $W$ and is H\"{o}lder$([0, 1], \beta q^{-1}, L C_{q})$. The second part then follows by the change of variables formulae and the fact that $\phi$ is measure preserving.
\end{proof}

In this setting, the population risk \eqref{eq:loss_converge:risk} is now of the form 
\begin{equation}
    \mcI_n[K] := \intsqbb \big\{ \fnonebb \ell(K(\bm{l}, \bm{l}'), 1) + \fnzerobb \ell( K(\bm{l}, \bm{l}'), 0) \big\} \dldlbb.  
\end{equation}
We can now obtain analogous versions of Theorems~\ref{thm:loss_converge}~and~\ref{thm:embed_learn:converge_2} as follows:

\begin{theorem}
    \label{thm:loss_converge:high_dim}
    Suppose that Assumptions~\ref{assume:simple:loss},~\ref{assume:simple:bilinear},~\ref{assume:simple:slc}~and~\ref{assume:graphon_ass_high} hold. Writing $S_d = ([-A, A]^d)^n$, we get that
    \begin{equation*}
        \Big| \min_{\bm{\omega}_n \in (S_d)^n } \mathcal{R}_n(\bm{\omega}_n)  - \min_{K \in Z(S_d) } \mathcal{I}_n[K] \Big| = O_p\Big(  s_n + \frac{ d^{3/2} \mathbb{E}[f_n^2]^{1/2} }{ n^{1/2} } +  \frac{ (\log n)^{1/2}}{ n^{\beta/(q + 2\beta)} } \Big).
    \end{equation*}
\end{theorem}

The proof of Theorem~\ref{thm:loss_converge:high_dim} follows
immediately by Theorem~\ref{thm:loss_converge} and Theorem~\ref{thm:graphon_equiv}.

\begin{theorem}
    \label{thm:embed_converge:high_dim}
    Suppose that Assumptions~\ref{assume:simple:loss},~\ref{assume:simple:bilinear}~and~\ref{assume:graphon_ass_high} hold, and
    that we use Algorithm~\ref{alg:random_walk} (random walk + unigram negative sampling) for the sampling scheme with $\alpha \in (0, 1]$, so that $\beta = \beta_W \alpha$ in Assumption~\ref{assume:graphon_ass_high}. 
    Under the same assumptions on the choice of the embedding dimension $d = d(n)$ as given in Theorem~\ref{thm:embed_learn:converge_2}, it follows that there exists $A'$ sufficiently large such that whenever $A \geq A'$, for any sequence of minimizers $(\whomega_1, \ldots, \whomega_n) \in \argmin_{\bmomega \in ([-A, A]^d)^n} \mcR_n(\bmomega)$, we have that 
    \begin{align*}
        \frac{1}{n^2} \sum_{i, j \in [n] } &\big| B(\whomega_i, \whomega_j) - \optimalK(\bm{\lambda}_i, \bm{\lambda}_j) \big| = O_p( \tilde{r}_n^{1/2} )
    \end{align*}
    where 
    \begin{gather*}
        \tilde{r}_n = \Big( \frac{ \log(n) }{ n \rho_n } \Big)^{1/2} + \frac{ d^{3/2} }{ (n \rho_n)^{1/2} }  +  \frac{ (\log n)^{1/2} }{ n^{\beta/(q + 2\beta)} } + \Big( \frac{ \log n}{ n } \Big)^{\beta/2q} + d^{- \beta/q}, \\
        \optimalK(\bm{\lambda}_i, \bm{\lambda}_j) = \log\Big( \frac{ 2 W(\bm{\lambda}_i, \bm{\lambda}_j) \mcE_W(\alpha) (1 + k^{-1})^{-1} }{ l (1 - \rho_n W(\bm{\lambda}_i, \bm{\lambda}_j)) \cdot \{ W(\bm{\lambda}_i, \cdot) W(\bm{\lambda}_j, \cdot)^{\alpha} + W(\bm{\lambda}_i, \cdot)^{\alpha} W(\bm{\lambda}_j, \cdot)  \}  } \Big).
    \end{gather*}
\end{theorem}

See page~\pageref{app:embed_converge:high_dim}
for the proof of Theorem~\ref{thm:embed_converge:high_dim}.

\begin{remark}
    We note that the rates of convergence in Theorems~\ref{thm:loss_converge:high_dim}~and~\ref{thm:embed_converge:high_dim}
    depend on the dimension of the latent parameters. This cannot be avoided
    by our proof strategy - if we manually modified the proof, rather than simply applying Theorem~\ref{thm:graphon_equiv}, we would still
    end up with the same rates of convergence. For example, part of our bounds depend on the decay of the eigenvalues of the operator $\optimalK$, which under our smoothness assumptions will have eigenvalues $\mu_d$ decay as $O(d^{-\beta/q} )$ \citep{birman_estimates_1977}. We highlight that such dependence on the latent dimension is common for other tasks involving networks, such as graphon estimation \citep{xu_rates_2018}, and
    such dependence commonly arises in non-parametric estimation tasks
    \citep{tsybakov_introduction_2008}.
\end{remark}

\begin{remark}
    We highlight that there is some debate as to the types of graphs which can arise from latent variable models when the latent dimension is high \citep{seshadhri_impossibility_2020,chanpuriya_node_2020}. We highlight that
    this is distinct from matters of what embedding dimensions should be
    chosen when fitting an embedding model, as methods such as node2vec are not necessarily trying to recover exactly the latent variables used as part of a
    generative model. For example, from Theorem~\ref{thm:embed_converge:high_dim} above, if we suppose that $W(\bm{\lambda}_i, \bm{\lambda}_j) = \rho_n \langle \bm{\lambda}_i, \bm{\lambda}_j \rangle$ and substitute this
    into the given formula for $\optimalK$, we can see that $\optimalK(\bm{\lambda}_i, \bm{\lambda}_j)$ is not a function of $\langle \bm{\lambda}_i, \bm{\lambda}_j \rangle$ due to the $W(\bm{\lambda}_i, \cdot) W(\bm{\lambda}_j, \cdot)^{\alpha}$ terms in the denominator.
\end{remark}

\subsection{Importance of the choice of similarity measure}
\label{sec:embed_learn:inner_prod}

Theorem~\ref{thm:embed_learn:converge_1} only applies when the $\mu_i(\optimalK)$ in \eqref{eq:embed_learn:optimalK_eigen} are all non-negative, and Theorem~\ref{thm:embed_learn:converge_2} only applies to the case where we have some negative $\mu_i(\optimalK)$ and we make the choice of $B(\omega, \omega') = \langle \omega, I_{d_1, d_2} \omega' \rangle$. We now study the case where there are some negative $\mu_i(\optimalK)$ and we choose $B(\omega, \omega') = \langle \omega, \omega' \rangle$.

\begin{theorem} \label{thm:embed_learn:converge_3}
    Suppose that Assumptions~\ref{assume:simple:graphon_ass},~\ref{assume:simple:loss},~\ref{assume:simple:slc}~and~\ref{assume:simple:samp_weight_reg} hold, and suppose that Assumption~\ref{assume:simple:bilinear} holds with $B(\omega, \omega') = \langle \omega, \omega' \rangle$ denoting the inner product on $\mathbb{R}^d$. Define
    \begin{align*}
        \mcZ_d^{\geq 0}(A) := \big\{ K(l, l') = \langle \eta(l), \eta(l) \rangle \,:\, \eta: [0, 1] \to [-A, A]^d \big\}, \quad \mcZ^{\geq 0} := \mathrm{cl}\Big( \bigcup_{d \geq 1} \mcZ^{\geq 0}(A) ),
    \end{align*}
    where the closure is taken in a suitable topology (see Appendix~\ref{sec:app:embed_converge_proof:props of sets}). Note that the set $\mcZ^{\geq 0}$ does not depend on $A$ (see Lemma~\ref{app:embed_converge_proof:mcZ_free_of_A}). Then there exists a unique minimizer $K_n^*$ to $\mcI_n[K]$ over $\mcZ^{\geq 0}$. Under some further regularity conditions (see Theorem~\ref{app:embed_converge_proof:embed_convergence}), there exists $A'$ and a sequence of embedding dimensions $d = d(n)$, such that whenever $A \geq A'$, for any sequence of minimizers $\whomegavec \in \argmin_{\bmomega \in ([-A, A]^d)^n} \mcR_n(\bmomega)$, we have that  
    \begin{align*}
        \frac{1}{n^2} \sum_{i, j \in [n] } \big| \langle \whomega_i, \whomega_j \rangle - K_n^*(\lambda_i, \lambda_j) \big| = o_p(1).
    \end{align*}
    If moreover we know that $\fnone$ and $\fnzero$ are piecewise constant on a fixed partition $\mcQ^{\otimes 2}$ for all $n$, where $\mcQ$ is a partition of $[0, 1]$ into $\kappa$ parts, then $K_n^*$ is also piecewise constant on the partition $\mcQ^{\otimes 2}$, and can be calculated exactly via a finite dimensional convex program.
\end{theorem}

In the case where we select $B(\omega, \omega') = \langle \omega, \omega' \rangle$, we now argue that this leads to a lack of injectivity - it will not be possible to distinguish two different graph distributions from the learned embeddings alone. As a consequence, there is necessarily some information about the network lost, the importance of which depends on the downstream task at hand. For example, suppose the graph is generated by a two-community stochastic block model with even sized communities, with within-community edge probability $p$ and between-community edge probability $q$. We then have the following:

\begin{proposition} \label{thm:embed_learn:sbm_example_1}
    Suppose that the graphon $W_n(\cdot, \cdot)$ corresponds to a SBM with two communities of equal size, such that the within-community edge probability is $p$ and the between-community edge probability is $q$; i.e that 
    \begin{equation*}
        W_n(l, l') = \begin{cases} p & \text{ if } (l, l') \in [0, 1/2)^2 \cup [1/2, 1]^2, \\
            q & \text{ if } (l, l') \in [0, 1/2) \times [1/2, 1] \cup [1/2, 1] \times [0, 1/2); \end{cases}
    \end{equation*}
    and that we learn embeddings using a cross entropy loss and a uniform vertex subsampling scheme (Algorithm~\ref{alg:psamp} in Section~\ref{sec:sampling_formula}). Then the global minima of $\mcI_n[K]$ over $\mcZ^{\geq 0}$ is given by
    \begin{equation*}
        K^*(l, l') = \begin{cases} K_{11}^* & \text{ if } (l, l') \in [0, 1/2)^2 \cup [1/2, 1]^2 \\
            K_{12}^* & \text{ if } (l, l') \in [0, 1/2) \times [1/2, 1] \cup [1/2, 1] \times [0, 1/2) \end{cases}
    \end{equation*}
    where 
    \begin{enumerate}[label=\alph*)]
        \item if $p \geq q$ and $p + q \geq 1$, then $K_{11}^* = \sigma^{-1}(p)$, $K_{12}^* = \sigma^{-1}(q)$;
        \item if $p \geq q$ and $p + q < 1$, then $K_{11}^* = -K_{12}^* = \sigma^{-1}( \tfrac{1+p-q}{2} )$;
        \item if $p < q$ and $p + q \geq 1$, then $K_{11}^* = K_{12}^* = \sigma^{-1}( \tfrac{p+q}{2} )$;
        \item otherwise, $K_{11}^* = 0$, $K_{12}^* = 0$. 
    \end{enumerate}
\end{proposition}

The proof is given in Appendix~\ref{sec:app:additional_results} (page~\pageref{thm:embed_learn:sbm_example_1:proof}). With this, we make a few remarks.

\emph{Lack of injectivity:} As mentioned earlier, we can have multiple graphons $W$ for which the minima of $\mcI_n[K]$ over non-negative definite $K$ are identical; for instance, note that in the above example when $p > q$ and $p + q < 1$, then the minima of $\mcI_n[K]$ over non-negative definite $K$ depends only on the gap $p-q$. 

\emph{Loss of information:} In the case where $p > q$ and $p + q < 1$, Theorem~\ref{thm:embed_learn:converge_3} and Proposition~\ref{thm:embed_learn:sbm_example_1} tell us that the embedding vectors learned via minimizing \eqref{framework:eq:empirical_loss} will satisfy 
\begin{align*}
    \frac{1}{n^2} \sum_{i, j} \Big| \langle \whomega_i, \whomega_j \rangle & - K^*(\lambda_i, \lambda_j) \Big| = o_p(1)  \nonumber \\
    & \text{ where }  K^*(\lambda_i, \lambda_j) = \begin{cases} \sigma^{-1}\big( \frac{1 + p-q}{2} \big) & \text{ if } (\lambda_i, \lambda_j) \in [0, 1/2)^2 \cup [1/2, 1]^2 \\
        -\sigma^{-1}\big( \frac{1 + p-q}{2} \big)  & \text{ otherwise.} \end{cases}
\end{align*}
In particular, the generating graphon cannot be directly recovered from $K^*$ as it only identified up to the value of $p - q$. Despite this, we note that $K^*$ still preserves the community structure of the network, in that $K^*(\lambda_i, \lambda_j) > 0$ if and only if $\lambda_i$ and $\lambda_j$ belong to the same community. It therefore follows that asymptotically, on average the learned embedding vectors corresponding to vertices in the same community are positively correlated, whereas those in opposing communities are negatively correlated. 

When the minima is a constant function (such as when $q > p$ above), the limiting distribution $K^*$ contains no usable information about the underlying graphon, and therefore neither do the inner products of the learned embedding vectors. We discuss when this occurs for general graphon models in Proposition~\ref{thm:extra_results:learn_nothing}. In all, this highlights the advantage in using a Krein inner product between embedding vectors, as these issues are avoided. Later in Section~\ref{sec:exper:real_data} we observe empirically the benefits of making such a choice. 

\subsection{Application of embedding convergence: performance of link prediction}

We discuss the asymptotic performance of embedding methods when used for a link prediction downstream task. Consider the scenario where we make a partial observation $A^{\text{obs}} = (A^\text{obs}_{ij})$ of an underlying network $A = (A_{ij})$, with the property that if $A^{\text{obs}}_{ij} = 1$ then $A_{ij} = 1$, but if $A^{\text{obs}}_{ij} = 0$, we do not know whether $A_{ij} = 1$ or $A_{ij} = 0$. For example, this model is appropriate for when we are wanting to predict the future evolution of a network. The task is then to make predictions about $A$ using the observed data $A^{\text{obs}}$.

In the context above, link prediction algorithms frequently use the network $A^{\text{obs}}$ to produce a score $S_{ij}$ corresponding to the likelihood of whether the pair $(i, j)$ is an edge in the network $A$. The scores are usually interpreted so that the larger $S_{ij}$ is, the more likely it will occur that $A_{ij} = 1$. We consider metrics to evaluate performance of the form
\begin{equation}
    \label{eq:embed_learn:score_loss_1}
    D(S, B) = \frac{1}{n(n-1)} \sum_{i \neq j} d(S_{ij}, B_{ij})
\end{equation}
when using the scores $S$ to predict the presence of edges in a network $B$. We write $d(s, b)$ for a discrepancy measure between the predicted score $s$ and an observed edge or non-edge $b$ in the test set. For example, in the case where 
\begin{equation}
    \label{eq:embed_learn:score_loss_2}
    d_{\tau}(s, b) := b \mathbbm{1}\big[ s \geq \tau] + (1 - b) \mathbbm{1}\big[ s < \tau]
\end{equation}
is a zero-one loss (having thresholded the scores by $\tau$ to obtain a $\{0, 1\}$-valued prediction), \eqref{eq:embed_learn:score_loss_1} becomes the misclassification error. Smoother losses can be obtained by using 
\begin{align}
    \label{eq:embed_learn:score_loss_3}
    d(s, b) & = -b \log( \sigma(s) ) - (1 - b) \log(1 - \sigma(s) ), \text{ or } \\
    \label{eq:embed_learn:score_loss_4}
    d(s, b) & = \max\{ 0, 1 - (2b - 1) s \} \quad \text{ (provided } s \in (0, 1))
\end{align}
i.e the softmax cross-entropy or hinge losses respectively. Given a network embedding with embedding vectors $\omega_v$ for each vertex $v$, one frequent way of producing scores is to let $S_{ij} = B(\omega_i, \omega_j)$ where $B(\cdot, \cdot)$ is a similarity measure as in Assumption~\ref{assume:simple:bilinear}. By applying Theorems~\ref{thm:embed_learn:converge_1},~\ref{thm:embed_learn:converge_2}~or~\ref{thm:embed_learn:converge_3}, we can begin to analyze the performance of a link prediction method using scores produced by embeddings learned via minimizing \eqref{framework:eq:empirical_loss}.

\begin{proposition} \label{thm:embed_learn:link_prediction_consistency}
    Let $\mathbb{A}_n$ be the set of symmetric adjacency matrices on $n$ vertices with no self-loops. Suppose that $(A^{\text{obs},(n)})_{n \geq 1}$ is a sequence of adjacency matrices drawn from a graphon process satisfying the conditions in one of Theorems~\ref{thm:embed_learn:converge_1},~\ref{thm:embed_learn:converge_2}~or~\ref{thm:embed_learn:converge_3}, with $\whomegavec$ denoting the embedding vectors learned via minimizing \eqref{framework:eq:empirical_loss} using $A^{\text{obs},(n)}$. Let $K_n^*$ be the minimal value of $\mcI_n[K]$ which appears in the aforementioned convergence theorems, and $\tilde{r}_n^{1/2}$ the corresponding convergence rate. Recall that $B(\omega, \omega')$ denotes the similarity measure in Assumption~\ref{assume:simple:bilinear}. Write $\Omega_n = ( B(\whomega_i, \whomega_j) )_{i, j}$ and $K_n = ( K_n^*(\lambda_i, \lambda_j) )_{i, j}$ for the scoring matrices formed by using the learned embeddings from minimizing \eqref{framework:eq:empirical_loss} and $K_n^*$ respectively. Then we have that for any loss function $d(s, b)$ which is Lipschitz in $s$ for $a \in \{0, 1\}$ that 
    \begin{equation*}
        \sup_{B \in \mathbb{A}_n } \Big| D(\Omega_n, B) - D( K_n^*, B)  \Big| = o_p(1).
    \end{equation*}
    When $D_{\tau}(S, B)$ denotes \eqref{eq:embed_learn:score_loss_1} using the zero-one loss $d_{\tau}(s, b)$ with threshold $\tau$, further assume that there exists a finite set $E \subseteq \mathbb{R}$ for which 
    \begin{equation}
        \label{eq:embed_learn:stab_condition}
        \lim_{\epsilon \to 0} \sup_{\tau \in \mathbb{R} \setminus E} \sup_{n \in \mathbb{N} } \big| \big\{ \llp \in [0, 1]^2 \,:\,  K_n^*(l, l') \in [\tau - \epsilon, \tau + \epsilon] \big) \big\} \big| = 0.
    \end{equation}
    Then for any sequence $\epsilon_n \to 0$ with $\epsilon_n = \omega( \tilde{r}_n^{1/2})$ as $n \to \infty$, we have that 
    \begin{equation*}
        \sup_{\tau \in \mathbb{R} \setminus E } \sup_{B \in \mathbb{A}_n } \Big|  D_{\tau}(\Omega_n, B) - D_{\tau + \epsilon_n}(K_n^*, B) \Big| \cvp 0 \text{ as } n \to \infty.
    \end{equation*}
\end{proposition}

See Appendix~\ref{sec:app:additional_results} (page~\pageref{app:other_results:link_prediction_consistency}) for a proof. 

\begin{remark}
    We note that examples of loss functions $d(s, b)$ which are Lipschitz include the hinge loss \eqref{eq:embed_learn:score_loss_4}, along with any `clipped' version of the softmax cross entropy loss \eqref{eq:embed_learn:score_loss_3}, where the scores are truncated so that the loss does not become unbounded as $s \to \pm \infty$. A sufficient condition for the regularity condition \eqref{eq:embed_learn:stab_condition} to hold is that the total number of jumps in the distribution functions associated to the $K_n^*$ for all $n$ is finite; for example, this occurs if $K_n^*$ is a piecewise constant function.
\end{remark} 

We now illustrate a use of the theorem above, in the context of the censoring example introduced at the beginning of the section. Suppose that the network $A$ is generated via a graphon $W$. We then calculate that
\begin{equation*}
    \label{eq:loss_embed_converge:censoring_1}
    \mathbb{P}\big( A^{\text{obs} }_{ij} = 1 \,|\, \lambda_i, \lambda_j \big) = \mathbb{P}\big( A^{\text{obs}}_{ij} = 1 \,|\, A_{ij} = 1, \lambda_i, \lambda_j) W(\lambda_i, \lambda_j)
\end{equation*}
independently across all pairs $(i, j)$ (as the probability that $A^{\text{obs}} = 1$ given $A_{ij} = 0$ is zero). If we further have that $\mathbb{P}\big( A^{\text{obs}}_{ij} = 1 \,|\, A_{ij} = 1, \lambda_i, \lambda_j) = g(\lambda_i, \lambda_j)$ for some symmetric, measurable function $g: [0, 1]^2 \to [0, 1]$, then $A^{\text{obs}}$ also has the law of an exchangeable graph. As a simple example, we could consider $g(\lambda_i, \lambda_j) = p$, corresponding to edges being randomly deleted from $A$. 

If we instead assume that $A^{\text{obs}}$ has the law of an exchangeable graph with graphon $\widetilde{W}$, then we can calculate that 
\begin{equation*}
     \mathbb{P}(A_{ij} = 1 \,|\, \lambda_i, \lambda_j) = \widetilde{W}(\lambda_i, \lambda_j) + \mathbb{P}\big( A_{ij} = 1 \,|\, A^{\text{obs}}_{ij} = 0, \lambda_i, \lambda_j \big) (1 - \widetilde{W}(\lambda_i, \lambda_j) )
\end{equation*}
independently across all pairs $(i, j)$. Again, if $\mathbb{P}\big( A_{ij} = 1 \,|\, A^{\text{obs}}_{ij} = 0, \lambda_i, \lambda_j \big) = \tilde{g}(\lambda_i, \lambda_j)$, then $A$ will have the law of an exchangeable graph too. For example, in the context of the social network example, one may suppose that the likelihood of an edge forming between two vertices is linked to the proportion of users who they are both connected with, or that it is linked to their respective degrees. We could then hypothesize that e.g
\begin{equation*}
    \tilde{g}(\lambda_i, \lambda_j) = \int_0^1 \widetilde{W}(\lambda_i, y) \widetilde{W}(y, \lambda_j) \, dy \quad \text{ or } \quad \tilde{g}(\lambda_i, \lambda_j) = \widetilde{W}(\lambda_i, \cdot) \widetilde{W}(\lambda_j, \cdot).
\end{equation*}
If either of the conditions hold, we can switch between using $\tilde{g}$ or $g$ by using $\tilde{g} = (1 - g)W (1 - gW)^{-1}$ and $g = \widetilde{W} (\widetilde{W} + \tilde{g}(1 - \widetilde{W} ) )^{-1}$ respectively.

Now suppose that we learn an embedding using the network $A^{\text{obs}}$ to produce a scoring matrix $S$ (as described above) to make predictions about $A$. Moreover assume that in \eqref{framework:eq:empirical_loss} we use the cross-entropy loss, a Krein inner product for the bilinear from $B(\omega, \omega')$, and that $A^{\text{obs}}$ satisfies the conditions in Theorem~\ref{thm:embed_learn:converge_2}. This implies that the optimal value of $\mcI_n[K]$ (where $\fnone$ and $\fnzero$ are functions of $\widetilde{W}$, and so we make the dependence on $\widetilde{W}$ explicit) is given by $\optimalK$ as in \eqref{eq:embed_learn:optimalK}. Provided the number of vertices in $A^{\text{obs}}$ is large, Proposition~\ref{thm:embed_learn:link_prediction_consistency} tells us that $D(S, A)$ will be approximately equal to $D(\optimalK, A)$. When $d(s, a)$ is the softmax cross-entropy loss, we then get that 
\begin{align}
    \label{eq:link_prediction:cross_entropy}
    D(\optimalK, A) \approx - \intsq \Bigg\{ & W(l, l')  \log\Big( \frac{ \fnone[\widetilde{W}]  }{ \fnone[\widetilde{W}] + \fnzero[\widetilde{W}]} \Big) \\ & + (1 - W(l, l')) \log\Big( \frac{ \fnzero[\widetilde{W}]  }{ \fnone[\widetilde{W}] + \fnzero[\widetilde{W}]} \Big) \Bigg\} \, dl dl'. \nonumber
\end{align}
With the expression on the right hand side, it is then possible to numerically investigate for which network models $W$ (given a fixed entropy) will a particular choice of sampling scheme be effective in combating particular types of censoring. This is because once the entropy of $W$ has been fixed, minimizing the RHS in \eqref{eq:link_prediction:cross_entropy} corresponds to minimizing the KL divergence $D_{KL}( P_{W} \,||\, \widetilde{P}_{\widetilde{W}})$ between the measures with densities
\begin{equation*}
    P_W(l, l', x) := W(l, l') \big[ 1 - W(l, l') \big]^{1-x} \text{ and }  \widetilde{P}_{\widetilde{W}}(l, l', x) = \frac{ \fnone[\widetilde{W}]^x \big[ \fnzero[\widetilde{W}] \big]^{1-x}  }{ \fnone[\widetilde{W}] + \fnzero[\widetilde{W}]}
\end{equation*}
defined for $(l, l') \in [0, 1]^2$ and $x \in \{0, 1\}$. 
%!TEX root = ms.tex

\section{Asymptotic local formulae for various sampling schemes} \label{sec:sampling_formula}

In this section we show that frequently used sampling schemes satisfy the strong local convergence assumption (Assumption~\ref{assume:simple:slc}) and give the corresponding sampling formulae and rates of convergence. We leave the corresponding proofs to Appendix~\ref{sec:app:sampling_formula_proof}. 
We begin with a scheme which simply selects vertices of the graph at random.

\begin{alg}[Uniform vertex sampling] 
    \label{alg:psamp}
    Given a graph $\mcG_n$ and number of samples $k$, we select $k$ vertices from $\mcG_n$ uniformly and without replacement, and then return $S(\mcG_n)$ as the induced subgraph using these sampled vertices.
\end{alg}

\begin{proposition} \label{sec:sampling:psamp_formula}
    Suppose that Assumption~\ref{assume:simple:graphon_ass} holds. Then for Algorithm~\ref{alg:psamp}, Assumptions~\ref{assume:simple:slc} and \ref{assume:simple:samp_weight_reg} hold with
    \begin{equation*}
        f_n(\lambda_i, \lambda_j, a_{ij} ) = k(k-1), 
        %\qquad g_n(\lambda_i) = k
    \end{equation*}
    $s_n = 0$, $\mathbb{E}[f_n^2] = \rho_n k^2(k-1)^2$ and $\beta = \beta_W$. %and $\gamma_s = \gamma_W$.
\end{proposition}

We now consider uniform edge sampling \citep[e.g][]{tang_line_2015}, complemented with a unigram negative sampling regime \citep[e.g][]{mikolov_distributed_2013}. We recall from the discussion in Section~\ref{sec:intro:motivation} that a negative sampling scheme is intended to force pairs of vertices which are negatively sampled to be far apart from each other in an embedding space, in contrast to those which are positively sampled.

\begin{alg}[Uniform edge sampling with unigram negative sampling]
    \label{alg:unifedge+ns}
    Given a graph $\mcG_n$, number of edges to sample $k$ and number of negative samples $l$ per `positively' sampled vertex, we perform the following steps:
    \begin{enumerate}[label=\roman*)]
        \item Form $S_0(\mcG_n)$ by sampling $k$ edges from $\mcG_n$ uniformly and without replacement;
        \item We form a sample set of negative samples $S_{ns}(\mcG_n)$ by drawing, for each $u \in \mcV(S_0(\mcG_n))$, $l$ vertices $v_1, \ldots, v_l$ i.i.d according to the unigram distribution 
        \begin{equation*}
            \mathrm{Ug}_{\alpha}\big( v \,|\, \mcG_n \big) = \frac{ \mathbb{P}\big( v \in \mcV(S_0(\mcG_n)) \,|\, \mcG_n )^{\alpha} }{ \sum_{u \in \mcV_n} \mathbb{P}\big( u \in \mcV(S_0(\mcG_n)) \,|\, \mcG_n )^{\alpha} }
        \end{equation*}
        and then adjoining $(u, v_i) \to S_{ns}(\mcG_n)$ if $a_{u v_i} = 0$.
    \end{enumerate}
    We then return $S(\mcG_n) = S_0(\mcG_n) \cup S_{ns}(\mcG_n)$. 
\end{alg}

\begin{proposition} \label{sec:sampling:unif_edge_uni_formula}
    Suppose that Assumption~\ref{assume:simple:graphon_ass} holds. Then for Algorithm~\ref{alg:unifedge+ns}, Assumptions~\ref{assume:simple:slc} and \ref{assume:simple:samp_weight_reg} hold with 
    \begin{align*}
        f_n(\lambda_i, \lambda_j, a_{ij} ) & = \begin{dcases*}
          \frac{2k}{\mcE_W \rho_n} & if $a_{ij} = 1$, \\
          \frac{ 2k l }{ \mcE_W \mcE_W(\alpha) } \big\{ W(\lambda_i, \cdot) W(\lambda_j, \cdot)^{\alpha} + W(\lambda_j, \cdot) W(\lambda_i, \cdot)^{\alpha} \big\} & if $a_{ij} = 0$;
    \end{dcases*}
    %g_n(\lambda_i) & = \frac{2k}{\mcE_W }\Big( W(\lambda_i, \cdot) +  \frac{ l W(\lambda_i, \cdot)^{\alpha} }{ \mcE_W(\alpha)} \cdot \int_0^1 (1 - \rho_n W(\lambda_i, y) ) W(y, \cdot) \, dy \Big)
    \end{align*}
    with $s_n = (\log(n) / n \rho_n )^{1/2}$, $\mathbb{E}[f_n^2] = O(\rho_n^{-1})$, and $\beta = \beta_W \min\{ \alpha, 1 \}$. %and $\gamma_s = \min\{ \gamma_W, \gamma_d, \gamma_d/\alpha \}$.
\end{proposition}

Alternatively to using a unigram distribution for negative sampling, one other approach is to select edges (such as via uniform sampling as above), and then return the induced subgraph as the entire sample. 

\begin{alg}[Uniform edge sampling and induced subgraph negative sampling]%
    \label{alg:unifedge+inducedsg}%
    Given a graph $\mcG_n$ and number of edges $k$ to sample, we perform the following steps:
    \begin{enumerate}[label=\roman*)]
        \item Form $S_0(\mcG_n)$ by sampling $k$ edges from $\mcG_n$ uniformly and without replacement;
        \item Return $S(\mcG_n)$ as the induced subgraph formed from all of the vertices $u \in \mcV(S_0(\mcG_n))$. 
    \end{enumerate}
\end{alg}

\begin{proposition} \label{sec:sampling:unif_edge_induced_formula}
    Suppose that Assumption~\ref{assume:simple:graphon_ass} holds. Then for Algorithm~\ref{alg:unifedge+inducedsg}, Assumptions~\ref{assume:simple:slc} and \ref{assume:simple:samp_weight_reg} hold with 
    \begin{align*}
        f_n(\lambda_i, \lambda_j, a_{ij} ) & = \begin{dcases*}
         \frac{4k}{\mcE_W \rho_n} + \frac{ 4k(k-1) W(\lambda_i, \cdot) W(\lambda_j, \cdot) }{ \mcE_W^2  } & if $a_{ij} = 1$, \\
         \frac{ 4k(k-1) W(\lambda_i, \cdot) W(\lambda_j, \cdot) }{ \mcE_W^2 }  & if $a_{ij} = 0$;
    \end{dcases*} 
    %g_n(\lambda_i) & = \frac{2k W(\lambda_i, \cdot)}{\mcE_W } 
    \end{align*}
    with $s_n = (\log(n)/n\rho_n)^{1/2}$, $\beta = \beta_W$, and $\mathbb{E}[f_n^2] = O(\rho_n^{-1})$. %and $\gamma_s = \min\{ \gamma_d, \gamma_W \}$.  
\end{proposition}

We can also consider random walk based sampling schemes \citep[see e.g.][]{perozzi_deepwalk_2014}.

\begin{alg}[Random walk sampling with unigram negative sampling]
    \label{alg:random_walk}
    Given a graph $\mcG_n$, a walk length $k$, number of negative samples $l$ per positively sampled vertex, unigram parameter $\alpha$ and an initial distribution $\pi_0(\cdot \,|\, \mcG_n)$, we 
    \begin{enumerate}[label=\roman*)]
        \item Select an initial vertex $\tilde{v}_1$ according to $\pi_0$;
        \item Perform a simple random walk on $\mcG_n$ of length $k$ to form a path $(\tilde{v}_i)_{i \leq k+1}$, and report $(\tilde{v}_i, \tilde{v}_{i+1})$ for $i \leq k$ as part of $S_0(\mcG_n)$;
        \item For each vertex $\tilde{v}_i$, we select $l$ vertices $(\eta_j)_{j \leq l}$ independently and identically according to the unigram distribution 
        \begin{equation*}
            \mathrm{Ug}_{\alpha}( v \,|\, \mcG_n) = \frac{  \mathbb{P}\big( \tilde{v}_i = v \text{ for some } i \leq k \,|\, \mcG_n \big)^{\alpha}   }{ \sum_{u \in \mcV_n}  \mathbb{P}\big( \tilde{v}_i = u \text{ for some } i \leq k \,|\, \mcG_n \big)^{\alpha}    }
        \end{equation*}
        and then form $S_{ns}(\mcG_n)$ as the collection of $(\tilde{v}_i, \eta_j)$ which are non-edges in $\mcG_n$;
    \end{enumerate}
    and then return $S(\mcG_n) = S_0(\mcG_n) \cup S_{ns}(\mcG_n)$.
\end{alg}

In the above scheme, there is freedom in how we can specify the initial vertex of the random walk. Here we will do so using the stationary distribution of a simple random walk on $\mcG_n$, namely $\pi_0( v \,|\, \mcG_n) = \degree_n(v)/2E_n$, as this simplifies the analysis of the scheme.

\begin{proposition} \label{sec:sampling:rw_uni_stat_formula}
    Suppose that Assumption~\ref{assume:simple:graphon_ass} holds. Then for Algorithm~\ref{alg:unifedge+inducedsg} with choice of initial distribution $\pi_0( v \,|\, \mcG_n) = \degree_n(v)/2E_n$, Assumptions~\ref{assume:simple:slc} and \ref{assume:simple:samp_weight_reg} hold with
    \begin{align*}
        f_n(\lambda_i, \lambda_j, a_{ij} ) & = \begin{dcases*}
         \frac{2k}{\mcE_W \rho_n} & if $a_{ij} = 1$, \\
         \frac{ l(k+1) }{\mcE_W \mcE_W(\alpha) } \big\{ W(\lambda_i, \cdot) W(\lambda_j, \cdot)^{\alpha} + W(\lambda_j, \cdot) W(\lambda_i, \cdot)^{\alpha} \big\} & if $a_{ij} = 0$;
    \end{dcases*} 
    %g_n(\lambda_i) & = \frac{ k W(\lambda_i, \cdot) }{\mcE_W} +  \frac{ (k+1)l W(\lambda_i, \cdot)^{\alpha} }{ \mcE_W(\alpha) \mcE_W } \int_0^1 (1 - \rho_n W(\lambda_i, y) ) W(y, \cdot) \, dy 
    \end{align*}
    with $s_n = (\log(n)/n\rho_n)^{1/2}$, $\mathbb{E}[f_n^2] = O(\rho_n^{-1})$, and $\beta = \beta_W \min\{ \alpha, 1 \}$.% and $\gamma_s = \min\{ \gamma_W, \gamma_d, \gamma_d/\alpha \}$.
\end{proposition}

One important property of the samplers discussed in Algorithms~\ref{alg:unifedge+ns},~\ref{alg:unifedge+inducedsg}~and~\ref{alg:random_walk} is that they are essentially invariant to the scale of the underlying graph, in that the dominating parts of the expressions for the $\tilde{f}_n(l, l', x)$ are free of the sparsity factor $\rho_n$. We write this down for the random walk sampler.

\begin{lemma} \label{sec:sampling:scale_free}
    For Algorithm~\ref{alg:random_walk}, under the conditions of Proposition~\ref{sec:sampling:rw_uni_stat_formula} we get that 
    \begin{align*}
        \tilde{f}_n(\lambda_i, \lambda_j, 1) & = \frac{ 2k W(\lambda_i, \lambda_j) }{ \mcE_W} \\
        \tilde{f}_n(\lambda_i, \lambda_j, 0) & = \frac{ l (k+1) }{ \mcE_W \mcE_W(\alpha) } \big\{ W(\lambda_i, \cdot) W(\lambda_j, \cdot)^{\alpha} + W(\lambda_i, \cdot)^{\alpha} W(\lambda_j, \cdot) \big\} \cdot (1 - \rho_n W(\lambda_i, \lambda_j) ).
    \end{align*}
    In particular, we have that $\tilde{f}_n(\lambda_i, \lambda_j, 1)$ is free of $\rho_n$, and 
    \begin{equation*}
        \tilde{f}_n(\lambda_i, \lambda_j, 0) = \frac{ l (k+1) }{ \mcE_W \mcE_W(\alpha) } \big\{ W(\lambda_i, \cdot) W(\lambda_j, \cdot)^{\alpha} + W(\lambda_i, \cdot)^{\alpha} W(\lambda_j, \cdot) \big\} \cdot (1 + O(\rho_n))
    \end{equation*}
\end{lemma}

\begin{remark}
    We note that in algorithmic implementations of negative sampling schemes in practice, there is usually not an explicit check for whether the negatively sampled edges are non-edges in the original graph. This is done for the reason that graphs encountered in the real world are frequently sparse, and so the check would take up computational time while only having a small effect on the learnt embeddings. This would correspond to removing the $(1 - \rho_n W(\lambda_i, \lambda_j))$ factor in the above formula for $\tilde{f}_n(\lambda_i, \lambda_j, 1)$, and so Lemma~\ref{sec:sampling:scale_free} reaffirms the above reasoning.
\end{remark}

\subsection{Expectations and variances of random-walk based gradient estimates}
\label{sec:sampling:variance}

Throughout we have studied the empirical risk $\mcR_n(\omega_1,\ldots,
\omega_n)$ induced through using a stochastic gradient scheme to learn
a network embedding, given a subsampling scheme $S(\mcG)$. Subsampling
schemes used by practitioners (such as in node2vec) depend on some choice
of hyperparameters. These are selected either via
a grid-search, or by using default suggestions - for
example, the unigram sampler in Algorithm~\ref{alg:random_walk} is
commonly used with $\alpha = 0.75$, as recommended in \citet{mikolov_distributed_2013}. A priori, the role of such parameters is
not obvious, and so we give some insights into the role of particular hyperparameters
within the random walk scheme described in Algorithm~\ref{alg:random_walk}. We focus on the expected value and variance of the gradient estimates used
during training. 

To illustrate the importance of these two values, we
discuss first what happens in a traditional empirical risk minimization setting, where given data $x_1, \ldots, x_n
\in \mathbb{R}^p$ where $n$ is large and a loss function $L(x; \theta)$,
we try to optimize over $\theta$ the empirical loss function $L_n(\theta) :=
\sum_{i=1}^n L(x_i; \theta)$ by using a stochastic gradient scheme. More
specifically, we obtain a sequence $(\theta_t)_{t \geq 1}$ via
\begin{equation*}
    \theta_t = \theta_{t-1} - \eta_t G_t \text{ where } \mathbb{E}[G_t] = \nabla L_n(\theta)
\end{equation*}
given an initial point $\theta_0$, step sizes $\eta_t$ and a random gradient estimate $G_t$. We then run this for a sufficiently large number of 
iterations $t$ such that
$\theta_t \approx \argmin_{\theta} L_n(\theta)$; see e.g \citet{robbins_stochastic_1951}. For the empirical risk minimization setting
detailed above, one common approach has $G_t$ take the form
\begin{equation*}
    G_t = \frac{1}{m} \sum_{l = 1}^m \nabla L(\tilde{x}_m ; \theta_{t-1}) 
\end{equation*}
where $\tilde{x}_l$ are sampled i.i.d uniformly from $\{x_1, \ldots, x_n \}$
for each $l \in [m]$. We then get $\mathbb{E}[G_t] = \nabla L_n(\theta_{t-1})$ for any choice of $m$, and $\mathrm{Var}(\| G_t \|_2) = O(m^{-1})$ when 
assuming that the gradient of $L$ is bounded. In general, the 
variance of the gradient estimates determines
the speed of convergence of a stochastic gradient scheme - the smaller the
variance, the quicker the convergence \citep{dekel_optimal_2012} - and
so choosing a larger batch size $k$ should leave to better convergence.
Importantly, when comparing two gradient estimates, we cannot make a 
bona-fide comparison of their variances without ensuring that they
have similar expectations, as otherwise the two schemes are optimizing
different empirical risks. 

In the network embedding setting, to form a gradient estimate we could take independent subsamples
$S_1(\mcG), \ldots, S_m(\mcG)$ and average over these, to get an estimator
which (when averaging over the subsampling process) gives an unbiased
estimator of the gradient of the empirical risk $\mcR_n(\omega_1, \ldots, \omega_n)$. This also has the variance of the gradient estimates decaying
as $O(m^{-1})$. A more interesting question is to study what occurs when we
only use one subsampling scheme $S(\mcG)$ per gradient estimate - as in practice - and vary the hyperparameters.
For example, in the random walk scheme Algorithm~\ref{alg:random_walk},
as a consequence of Proposition~\ref{sec:sampling:rw_uni_stat_formula},
under the assumptions of Theorem~\ref{thm:embed_learn:converge_2}, the
matrix $B(\whomega_i, \whomega_j)$ is approximately equal to
\begin{equation*}
    \optimalK(\lambda_i, \lambda_j) = \log\Big( \frac{ 2 W(\lambda_i, \lambda_j) \mcE_W(\alpha) (1 + k^{-1})^{-1} }{ l (1 - \rho_n W(\lambda_i, \lambda_j)) \cdot \{ W(\lambda_i, \cdot) W(\lambda_j, \cdot)^{\alpha} + W(\lambda_i, \cdot)^{\alpha} W(\lambda_j, \cdot)  \}  } \Big),
\end{equation*}
which is essentially free of the random walk length $k$ once $k$ is sufficiently large. A natural question is to
therefore ask what the role of $k$ is in such a setting. In the result
below, we highlight that the role of $k$ leads to producing gradient
estimates with reduced variance. The proof is
given on page~\pageref{sampling:sgd_variance:proof}.

\begin{proposition} \label{thm:sampling:sgd_variance}
    Let $S(\mcG_n)$ be a single instance of the subsampling scheme described
    in Algorithm~\ref{alg:random_walk} given a graph $\mcG_n$. Define the random vector
    \begin{equation*}
        G_i = \frac{1}{k} \sum_{j \in \mcV_n \setminus \{ i \} } \mathbbm{1}\big[ (i, j) \in S(\mcG_n) \big] \omega_j \ell'( \langle \omega_i, \omega_j \rangle, a_{ij} )
    \end{equation*}
    so $\mathbb{E}[G_i | \mcG_n] = k^{-1} \nabla_{\omega_i} \mcR_n(\omega_1, \ldots, \omega_n)$. Supposing that Assumptions~\ref{assume:simple:graphon_ass},~\ref{assume:simple:loss}~and~\ref{assume:simple:bilinear} hold, then we have that, writing $s_n = (\log(n) / n \rho_n)^{1/2}$, 
    \begin{align*}
        \mathbb{E}[G_i | \mcG_n] = \frac{1}{n^2} \sum_{j \in \mcV_n \setminus \{i\} } \big\{ \frac{ 2 a_{ij} }{ \mcE_W \rho_n} + \frac{ l (1+k^{-1}) H(\lambda_i, \lambda_j) (1 - a_{ij} )}{ \mcE_W \mcE_W(\alpha)} \big\} \omega_j \ell'(\langle \omega_i, \omega_j \rangle, a_{ij} ) \cdot (1 + o_p(s_n))
    \end{align*}
    for some function $H(\lambda_i, \lambda_j)$ free of $k$, and letting $G_{ir}$ be the $r$-th component of $G_i$, we have that 
    \begin{equation*}
        \mathrm{Var}[ G_{ir} \,|\, \mcG_n] = O_p\Big( \frac{1}{nk} \Big)
    \end{equation*}
    uniformly over all $i$ and $r$. In particular, the representation learned by Algorithm~\ref{alg:random_walk} is approximately invariant to the walk length $k$ for large $k$, as guaranteed by Theorem~\ref{thm:embed_learn:converge_2}; the gradients are asymptotically free of the walk length $k$ when $k$ and $n$ are large; and the $\ell_{\infty}$ norm of the variance of the gradients decays as $O_p(1/nk)$. 
\end{proposition}
%!TEX root = ms.tex

\section{Experiments} \label{sec:exper}

We perform experiments\footnote{Code is available at \url{https://github.com/aday651/embed-asym-experiments}.} on both simulated and real data, illustrating the validity of our theoretical results. We also highlight that the use of a Krein inner product $\langle \omega, \mathrm{diag}(I_p, -I_q) \omega' \rangle$ between embedding vectors can lead to improved performance when using the learned embeddings for downstream tasks.

\subsection{Simulated data experiments}
\label{sec:exper:sim}

To illustrate our theoretical results, we perform two different sets of experiments on simulated data. The first demonstrates some potential limitations of using the regular inner product between embedding vectors in the empirical risk being optimized. The second demonstrates the validity of the sampling formulae for different sampling schemes.

For the first experiment, we consider generating networks with $n$ vertices, where each vertex is given a latent vector $Z_i \sim N(0, I_{ (p_{+} + p_{-}) })$ drawn independently (where $p_{+}, p_{-} \in \mathbb{N}$), with edges formed between vertices independently with probability 
\begin{equation*}
    \mathbb{P}(A_{ij} = 1 | Z_i, Z_j) = \sigma\big( B_{p_{+}, p_{-}}(Z_i, Z_j) \big) \text{ for } i < j.
\end{equation*}
Here $\sigma(x) = (1+e^{-x})^{-1}$ is the sigmoid function, and $B_{r, s}(\omega, \omega') = \langle \omega, \mathrm{diag}(I_r, -I_s) \omega' \rangle$ for any $r, s \geq 1$. We simulate twenty networks for each possible combination of: $n = 200$, $400$, $800$, $1200$, $1600$, $2400$, $3200$, or $4800$; and $(p_{+}, p_{-})$ equal to $(4, 0)$, $(4, 4)$, $(8, 0)$, or $(8, 8)$. We then train each network using a constant step-size SGD method with a uniform vertex sampler for 40 epochs\footnote{By epochs, we are referring to the cumulative number of pairs of vertices which are used to form a gradient at each iteration, relative to the total number of edges in the graph.}, using a similarity measure $B_{q_{+}, q_{-}}$ between embedding vectors for various values of $(q_{+}, q_{-})$. Some are equal to $(p_{+}, p_{-})$, so that the similarity measure used for the data generating process and training are identical.  Some are greater than $(p_{+}, p_{-})$, so the data generating process still falls within the constraints of the model. Finally, we also let some be less than $(p_{+}, p_{-})$, in which case the data generating process falls outside the specified model class for learning. With the learned embeddings $(\widehat{\omega}_1, \ldots, \widehat{\omega}_n)$ we then calculate the value of 
\begin{equation}
    \label{eq:experiments:l1_error}
    \frac{1}{n^2} \sum_{i, j \in [n] } \Big| B_{q_{+}, q_{-}}\big( \widehat{\omega}_i, \widehat{\omega}_j \big) - B_{p_{+}, p_{-}}(Z_i, Z_j) \Big|.
\end{equation}
In words, we are computing the average $L^1$ error between the estimated edge logits using the learned embeddings (with a bilinear form $B_{q_{+}, q_{-}}$ between embedding vectors in the loss function), and the actual edge logits used to generate the network. The results are displayed in Figure~\ref{fig:normal_recovery}. By the convergence theorems discussed in Sections~\ref{sec:embed_learn:embed_converge}~and~\ref{sec:embed_learn:inner_prod}, we expect that \eqref{eq:experiments:l1_error} will be $o_p(1)$ if and only if $p_{+} \leq q_{+}$ and $p_{-} \leq q_{-}$, and indeed this is the trend displayed in Figure~\ref{fig:normal_recovery}.

\begin{figure}[t!]
    \centering
    \includegraphics[width=0.75\textwidth]{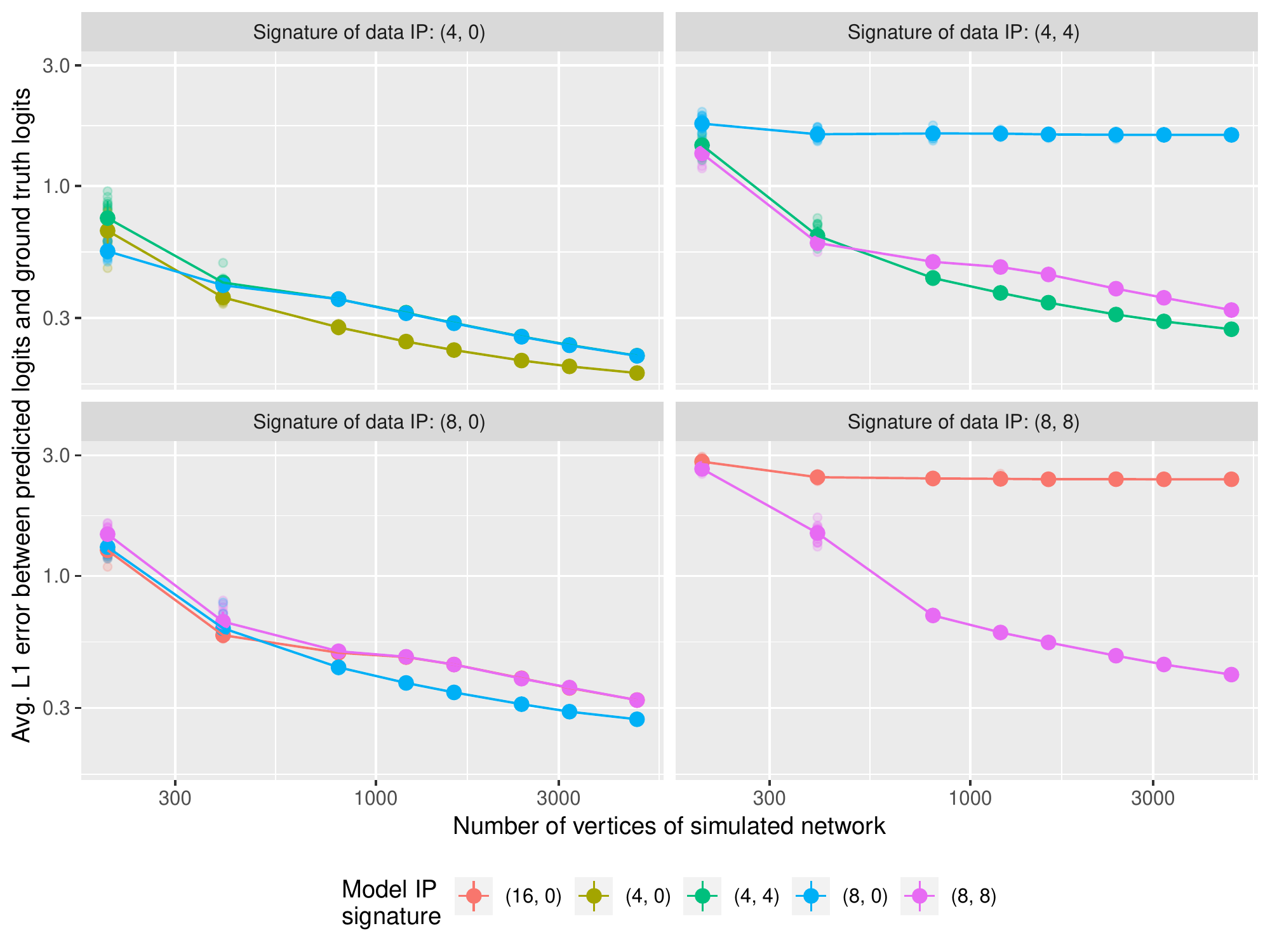}
    \caption{Simulation results for recovery of latent variables for different similarity measures $B(\omega, \omega')$ for generating the network and for learning. The $x$-axis are the number of vertices, and the $y$-axis is the calculated value of \eqref{eq:experiments:l1_error}. The results for each of the 20 runs per experiment are displayed translucently, with the average across these simulation runs given in bold.}
    \label{fig:normal_recovery}
\end{figure}

For the second result, we illustrate the validity of the sampling formulae calculated in Section~\ref{sec:sampling_formula}. To do so, we begin by generating a network of $n$ vertices from one of the following stochastic block models, where $\pi$ denotes the community sizes and $P$ the community linkage matrices:
\begingroup 
\allowdisplaybreaks
\begin{align*}
    \text{SBM1:} & \qquad \pi = (1/3, 1/3, 1/3), \qquad 
    && P = \begin{pmatrix} 
        0.7 & 0.3 & 0.1 \\
        0.3 & 0.5 & 0.6 \\
        0.1 & 0.6 & 0.2
    \end{pmatrix}; \\
    \text{SBM2:} & \qquad \pi = (0.1, 0.2, 0.2, 0.3, 0.2), \qquad 
    && P= \begin{pmatrix}
        0.75 & 0.87 & 0.025 & 0.81 & 0.25 \\
        0.87 & 0.93 & 0.58 & 0.48 & 0.45 \\
        0.025 & 0.58 & 0.68 & 0.15 & 0.48 \\
        0.81 & 0.48 & 0.15 & 0.80 & 0.92 \\
        0.25 & 0.45 & 0.48 & 0.92 & 0.62 
    \end{pmatrix}.
\end{align*}
\endgroup
Here each vertex is assigned a latent variable $\lambda_i \sim \mathrm{Unif}([0, 1])$ which is used to determine the corresponding community (depending on where $\lambda_i$ lies within the partition of $[0, 1]$ induced by $\pi$). As illustrated in Sections~\ref{sec:embed_learn}~and~\ref{sec:sampling_formula}, depending on the sampling scheme ($\textbf{samp}$), and whether we use a regular or Krein inner product ($\textbf{IP}$) as the similarity measure $B(\omega, \omega')$ between embedding vectors (recall Assumption~\ref{assume:bilinear}), there is a function $K^*_{\textbf{samp}, \textbf{IP}}$ for which the minimizers of \eqref{framework:eq:empirical_loss} satisfy
\begin{equation}
    \label{eq:experiments:sbm_error}
    \frac{1}{n^2} \sum_{i, j \in [n] } \Big| B(\widehat{\omega}_i, \widehat{\omega}_j ) - K^*_{\textbf{samp}, \textbf{IP}}(\lambda_i, \lambda_j) \Big| = o_p(1).
\end{equation}
We note that for stochastic block models, when we choose $B(\omega, \omega') = \langle \omega, \omega' \rangle$ - corresponding to minimizing $\mcI_n[K]$ over $\mcZ^{\geq 0}$ - we can numerically compute the formula for $K^*_{\textbf{samp}, \textbf{IP}}$ via a convex program as a result of Proposition~\ref{app:embed_converge_proof:sbm_exist}. In the case where we choose $B(\omega, \omega')$ to be a Krein inner product, the discussion in Section~\ref{sec:embed_learn:embed_converge} tells us that we can write down the minima of $\mcI_n[K]$ over $\mcZ$ exactly. 

For each generated network, we train using either a) a random vertex sampler or a random walk + unigram sampler, and b) either the regular or Krein inner product for $B(\omega, \omega')$. We then calculate the value of \eqref{eq:experiments:sbm_error} for each possible form of $K^*_{\textbf{samp}, \textbf{IP}}$ for the sampling schemes and inner products we consider. The experiments are then repeated for the same values of $n$, and number of networks per choice of $n$, as in the first experiment; the results are displayed in Figure~\ref{fig:sbm1_recovery}. From the figure, we observe that the LHS of \eqref{eq:experiments:sbm_error} decays to zero only when the choice of $K^*_{\textbf{samp}, \textbf{IP}}$ corresponds to the sampling scheme and inner product actually used, as expected.

\begin{figure}[tbp]
    \centering
    \begin{subfigure}[t]{\textwidth}
        \centering
        \includegraphics[width=0.80\textwidth]{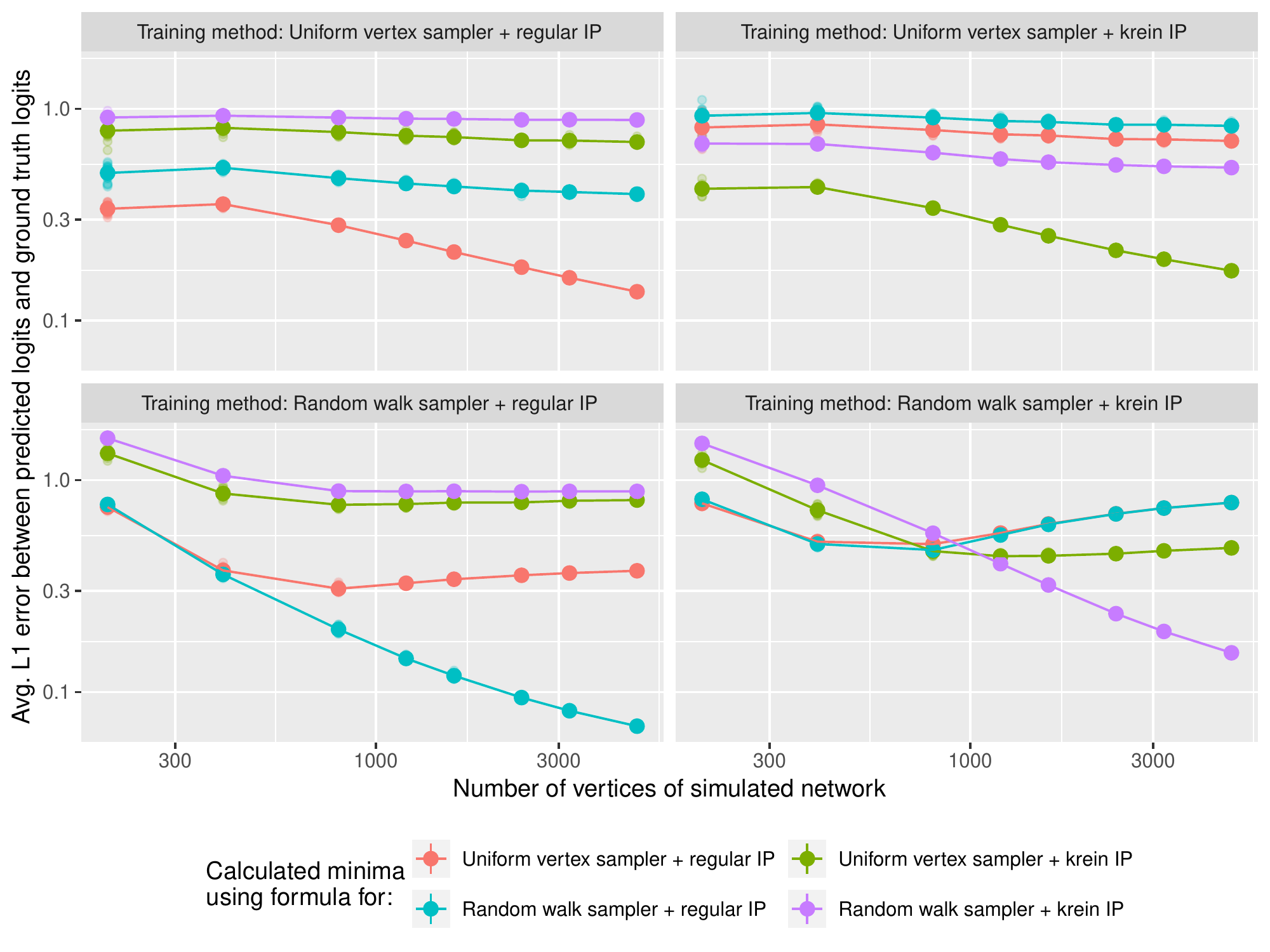}
        \caption{SBM1}
    \end{subfigure}%
    \par\medskip
    \begin{subfigure}[t]{\textwidth}
        \centering
        \includegraphics[width=0.80\textwidth]{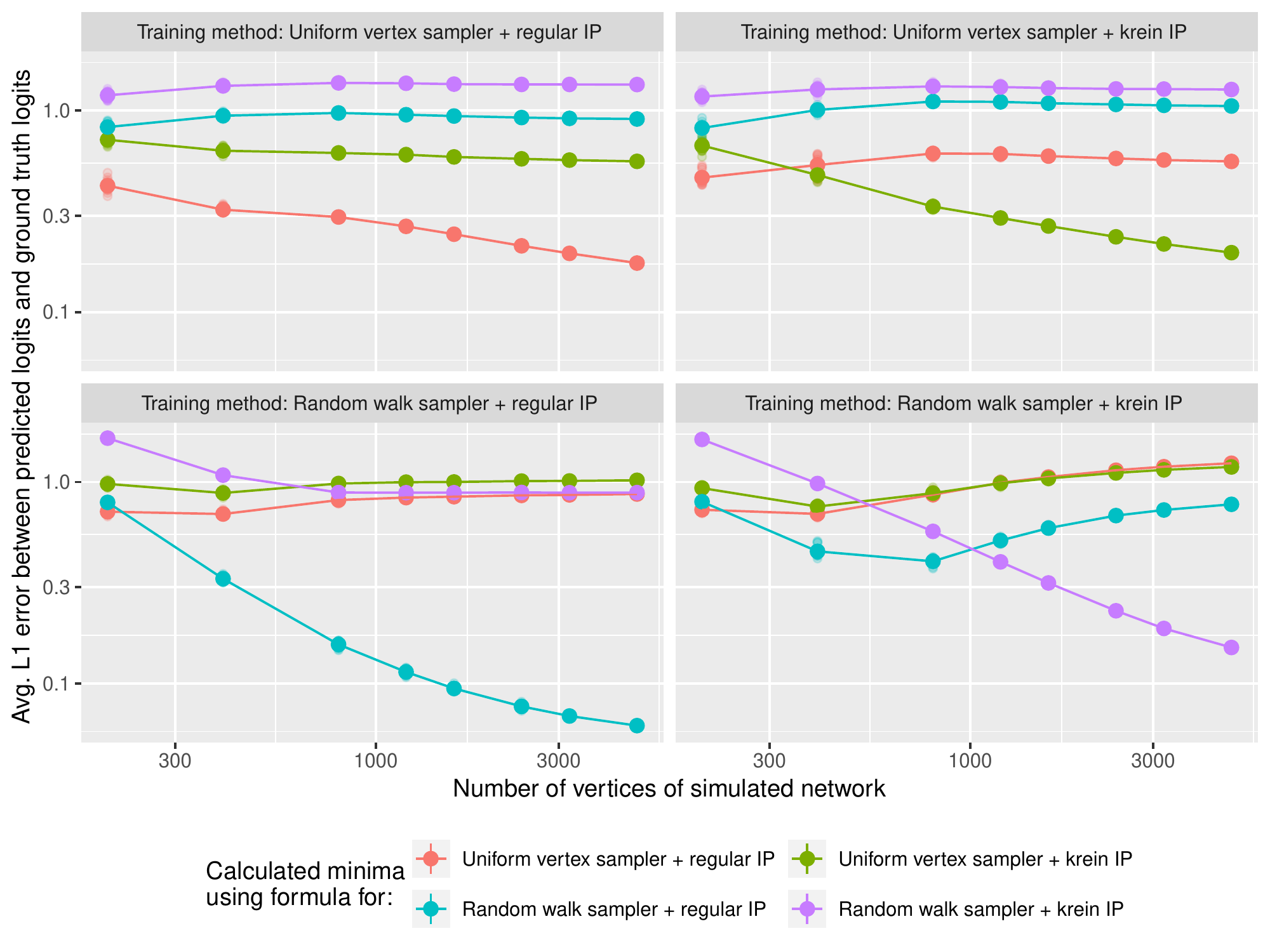}
        \caption{SBM2}
    \end{subfigure}
    \caption{Plots of the values of \eqref{eq:experiments:sbm_error} for different sampling formulae against the number of vertices of the network, when trained under different sampling schemes and different SBM models. }
    \label{fig:sbm1_recovery}
\end{figure}

\subsection{Real data experiments}
\label{sec:exper:real_data}

We now demonstrate on real data sets that the use of the Krein inner product leads to improved prediction of whether vertices are connected in a network, and as a consequence can lead to improvements in downstream tasks performance. To do so, we will consider a semi-supervised multi-label node classification task on two different data sets: a protein-protein interaction network \citep{grover_node2vec_2016,breitkreutz_biogrid_2008} with 3,890 vertices, 76,583 edges and 50 classes; and the Blog Catalog data set \citep{tang_relational_2009} with 10,312 vertices, 333,983 edges and 39 classes.

For each data set, we perform the same type of semi-supervised experiments as in \citet{veitch_empirical_2018}. We learn 128 dimensional embeddings of the networks using two sampling schemes - random walk/skipgram sampling and p-sampling, both augmented with unigram negative samplers - and either a regular inner product (with signature $(128, 0)$) or a Krein inner product (with signature $(64, 64)$). We simultaneously train a multinomial logistic regression classifier from the embedding vectors to the vertex classes, with half of the labels censored during training (to be predicted afterwards), and the normalized label loss kept at a ratio of 0.01 to that of the normalized edge logit loss. 

After training, we draw test sets according to three different methods (uniform vertex sampling, a random walk sampler and a p-sampler), and calculate the associated macro F1 scores\footnote{For a multi-class classification problem, the F1 score for a class is the harmonic average of the precision and recall; the macro F1 score is then the arithmetic average of these quantities over all the classes.}. The results of this are displayed in Table~\ref{tab:real_data}, and the plots of the normalized edge loss during training for each of the data sets can be found in Figure~\ref{fig:real_data}. From these, we observe that for each of the data sets when using p-sampling with a unigram negative sampler, there is a large decrease in the normalized edge loss during training when using the Krein inner product compared to the regular inner product. We also see a sizeable increase in the average macro F1 scores. For the skipgram/random walk sampler, we do not observe an improvement in the edge logit loss, but observe a minor increase in macro F1 scores.

\begin{table}[tbp]
    \centering
    \small
    \begin{tabular}{@{}cccccc@{}}
    \toprule[1pt]
    \multirow{2}{*}{Dataset} & \multirow{2}{*}{Sampling scheme} & \multirow{2}{*}{Inner product} & \multicolumn{3}{c}{Average macro F1 scores} \\ \cmidrule(l){4-6} \addlinespace[0.1em]
     &  &  & Uniform & Random walk & p-sampling \\ \addlinespace[0.1em] \toprule[1pt]
    \multirow{4}{*}{PPI} & Skipgram/RW + NS & Regular & 
    0.203       &        0.250    &     0.246 \\  
     & Skipgram/RW + NS & Krein &  
     0.245       &        0.298     &    0.290    \\ \addlinespace[0.1em]
     & p-sampling + NS & Regular & 
     0.408   &            0.423    &     0.417 \\ 
     & p-sampling + NS& Krein & 
     0.486   &         0.468    &    0.461 \\ \midrule
    \multirow{4}{*}{Blogs} & Skipgram/RW + NS & Regular &
    0.154  &   0.192   & 0.194 \\ 
     & Skipgram/RW + NS & Krein &
     0.250  &  0.279   &   0.285\\ \addlinespace[0.1em]
     & p-sampling + NS & Regular &
     0.132   & 0.155   &  0.166 \\ 
     & p-sampling + NS & Krein &
     0.349  & 0.291   & 0.290 \\ 
    \bottomrule[1pt]
    \end{tabular}
    \caption{Average macro F1 scores for semi-supervised classification for different data sets, sampling schemes, choice of similarity measure $B(\omega, \omega')$ between embedding vectors, and method of sampling test vertices.}
    \label{tab:real_data}
\end{table}

\begin{figure}[t!]
    \centering
    \includegraphics[width=0.75\textwidth]{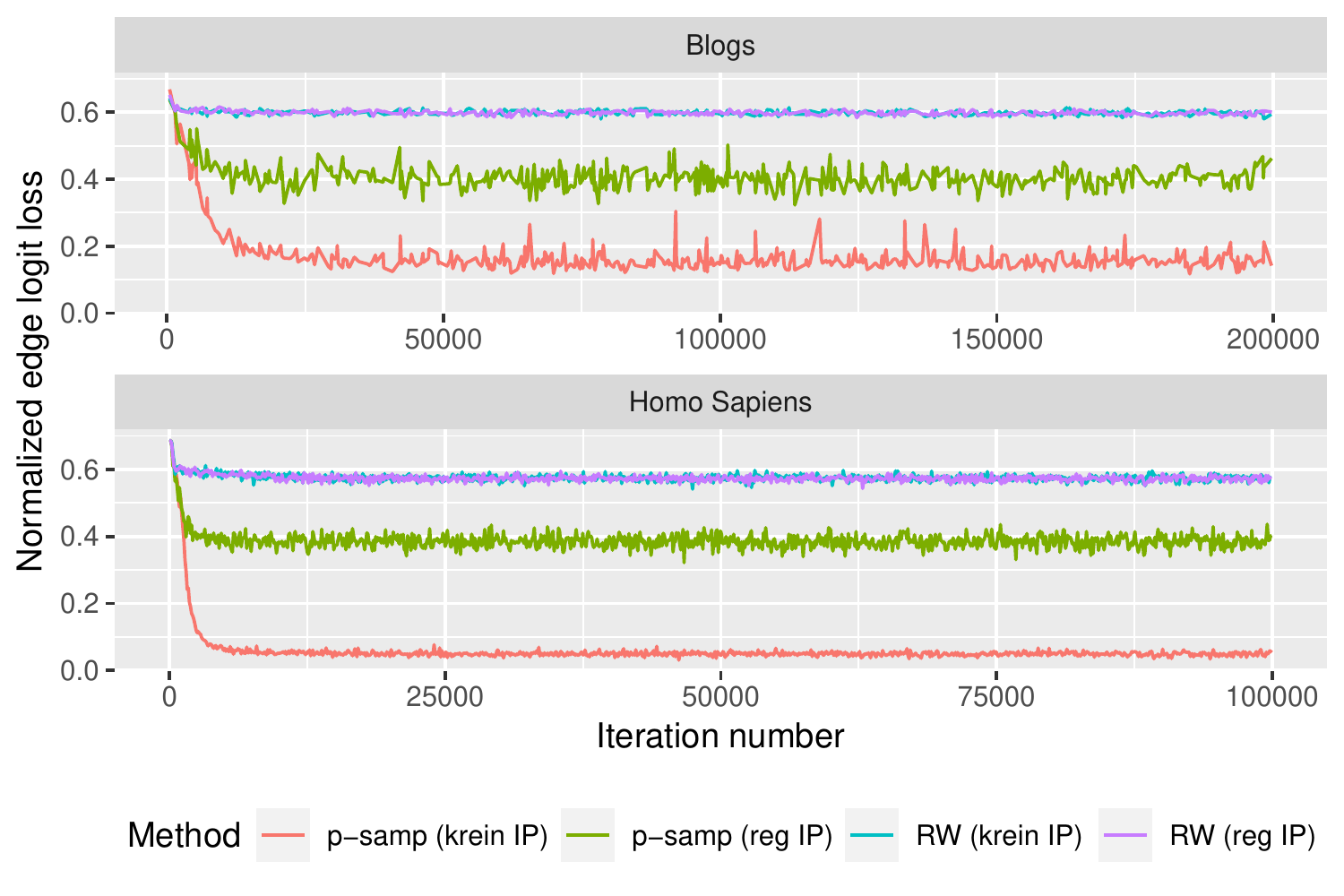}
    \caption{Normalized edge logit loss against iteration number for the homo-sapiens data set and blogs data set, for different sampling schemes and choice of similarity measure $B(\omega, \omega')$ between embedding vectors.}
    \label{fig:real_data}
\end{figure}

\newpage
%!TEX root = ms.tex

\section{Discussion} \label{sec:conc}

In our paper, we have obtained convergence guarantees for embeddings learnt via minimizing empirical risks formed through subsampling schemes on a network, in generality for subsampling schemes which depend only on local properties of the network. As a consequence of our theory, we also have argued that using an inner product between embedding vectors in losses of the form \eqref{framework:eq:empirical_loss} can limit the information contained within the learned embedding vectors. Mitigating this through the use of a Krein inner product instead can lead to improved performance in downstream tasks.

We note that our results apply within the framework of (sparsified) exchangeable graphs. While such graphs are convenient for theoretical purposes, and can reflect how real world networks are sparse, they are generally not capable of capturing the power-law type degree distributions of observed networks. There are alternative families of models for network data which are not vertex exchangeable and alleviate some of these problems, such as graphs generated by a graphex process \citep{veitch_class_2015,borgs_sampling_2017,borgs_sparse_2018}, along with other models such as those proposed by \citet{caron_sparse_2017} and \citet{crane_edge_2018}. As these models all contain enough structure similar to that of exchangeability (such as through an underlying point process to generate the network - see \citet{orbanz_subsampling_2017} for a general discussion on these points), we anticipate that our overall approach can be used to analyze the performance of embedding methods on broader classes of models for networks.

Our theory only considers embeddings learnt in an unsupervised, transductive fashion, whereas inductive methods for learning network embeddings are increasing popular. We highlight that inductive methods such as GraphSAGE \citep{hamilton_inductive_2017} work by parameterizing node embeddings through an encoder (possibly with the inclusion of nodal covariates), with the output embeddings then trained through a DeepWalk procedure. Provided that the encoder used is sufficiently flexible so that the range of embedding vectors is unconstrained (which is likely the case for the neural network architectures frequently employed), our results still apply in that we can give convergence guarantees for the output of the encoder analogously to Theorems~\ref{thm:embed_learn:converge_1}, \ref{thm:embed_learn:converge_2} and \ref{thm:embed_learn:converge_3}.

\section*{Acknowledgements} \label{sec:ack}

We acknowledge computing resources from Columbia University's Shared Research Computing Facility project, which is supported by NIH Research Facility Improvement Grant 1G20RR030893-01, and associated funds from the New York State Empire State Development, Division of Science Technology and Innovation (NYSTAR) Contract C090171, both awarded April 15, 2010. Part of this work was completed while M.~Austern was at Microsoft Research, New England. We thank the two anonymous reviewers and the editor for their feedback, which significantly improved the readability and contributions of the paper.

\newpage
\appendix

%!TEX root = ..\ms.tex

\section{Technical Assumptions} \label{sec:app:assumptions}

Here we introduce a more general set of technical assumptions than
those introduced in Section~\ref{sec:framework} for which our
technical results hold. For convenience, at points we will duplicate our
assumptions to keep the labelling consistent, and so Assumptions~\ref{assume:graphon_ass},\ref{assume:loss}~and~\ref{assume:samp_weight_reg} are
generalizations of Assumptions~\ref{assume:simple:graphon_ass},~\ref{assume:simple:loss}~and~\ref{assume:simple:samp_weight_reg} respectively,
and Assumptions~\ref{assume:bilinear}~and~\ref{assume:slc} are the same
as Assumptions~\ref{assume:simple:bilinear}~and~\ref{assume:simple:slc}
respectively.

\begin{assumeapp}[Regularity and smoothness of the graphon] \label{assume:graphon_ass}
    We suppose that the underlying sequence of graphons $(W_n = \rho_n W)_{n \geq 1}$ generating $(\mcG_n)_{n \geq 1}$ are, up to weak equivalence of graphons \citep{lovasz_large_2012}, such that:
    \begin{enumerate}[label=\alph*)]
        \item The graphon $W$ is piecewise H\"{o}lder$([0, 1]^2$, $\beta_W$, $L_W$, $\mcQ^{\otimes 2})$ for some partition $\mcQ$ of $[0, 1]$ and constants $\beta_W \in (0, 1]$, $L_W \in (0, \infty)$;
        \item The degree function $W(x, \cdot)$ is such that $W(x, \cdot)^{-1} \in L^{\gamma_d}([0, 1])$ for some exponent $\gamma_d \in (1, \infty]$;
        \item The graphon $W$ is such that $W^{-1} \in L^{\gamma_W}([0, 1]^2)$ for some exponent $\gamma_W \in [1, \infty]$;
        \item There exists a constant $C > 0$ such that $1 - \rho_n W \geq C$ a.e;
        \item The sparsifying sequence $(\rho_n)_{n \geq 1}$ is such that  $\rho_n = \omega( n^{-(\gamma_d - 1)/\gamma_d})$ if $\gamma_d \in (1, \infty)$, and $\rho_n = \omega(log(n)/n )$ if $\gamma_d = \infty$.
    \end{enumerate}
\end{assumeapp}

\begin{assumeapp}[Properties of the loss function] \label{assume:loss}
    Assume that the loss function $\ell(y, x)$ is non-negative, twice differentiable and strictly convex in $y \in \mathbb{R}$ for $x \in \{0, 1\}$, and is injective in the sense that if $\ell(y, x) = \ell(\tilde{y}, x)$ for $x =0$ and $x = 1$, then $y = \tilde{y}$. Moreover, we suppose that there exists $p \in [1, \infty)$ (where we call $p$ the growth rate of the loss function $\ell$) such that
    \begin{enumerate}[label=\roman*)]
        \item For $x \in \{0, 1\}$, the loss function $\ell(y, x)$ is locally Lipschitz in that there exists a constant $\losslipconst$ such that 
        \begin{equation*}
                \big| \ell(y, x) - \ell(y', x) \big| \leq \losslipconst \max\{ |y|, |y'| \}^{p-1} | y -y'| 
                \text{ for all } y, y' \in \mathbb{R};
        \end{equation*}
        \item Moreover, there exists constants $\lossboundc > 0$ and $\lossbounda > 0$ such that, for all $y \in \mathbb{R}$ and $x \in \{0, 1\}$, we have 
        \begin{gather*}
            \frac{1}{\lossboundc} ( |y|^p - \lossbounda ) \leq \ell(y, 1) + \ell(y, 0) \leq \lossboundc ( |y|^p + \lossbounda ), \qquad           \Big| \frac{d}{dy} \ell(y, x) \Big| \leq \lossboundc ( |y|^{p-1} + \lossbounda ).
        \end{gather*}
        These conditions ensure that $\ell(y, 1)$ and $\ell(y, 0)$ grows like $|y|^p$ as $y \to +\infty$ and $y \to -\infty$ respectively.
    \end{enumerate}
\end{assumeapp}

Note that the cross-entropy loss satisifies the above conditions with
$p = 1$, and also satisifies the conditions below: 

\begin{assumep}[Loss functions arising from probabilistic models] \label{assume:loss_prob}
    In addition to requiring all of Assumption~\ref{assume:loss} to hold, we additionally suppose that there exists a c.d.f $F$ for which 
    \begin{equation*}
        \ell(y, x) = \ell_F(y, x) := - x\log\big( F(y) \big) - (1-x) \log\big( 1 - F(y) \big),
    \end{equation*}
    where $F$ corresponds to a distribution which is continuous, symmetric about $0$, strictly log-concave, and has an inverse which is Lipschitz on compact sets.
\end{assumep}

In addition to the cross-entropy loss, the above assumptions allows for probit losses (taking $F$ to be the c.d.f of a Gaussian distribution). Note that for such loss functions, the value of $p$ is linked to the tail behavior of the distribution in that it behaves as $\exp(-|y|^p)$ - for instance, the logistic distribution is sub-exponential and the cross entropy loss satisifies Assumption~\ref{assume:loss_prob} with $p = 1$, whereas a Gaussian is sub-Gaussian and thus Assumption~\ref{assume:loss_prob} will hold with $p = 2$.

\begin{assumeapp}[Properties of the similarity measure $B(\omega, \omega')$] \label{assume:bilinear} 
    Supposing we have \linebreak embedding vectors $\omega, \omega' \in \mathbb{R}^d$, we assume that the similarity measure $B$ is equal to one of the following bilinear forms: 
    \begin{enumerate}[label=\roman*)]
        \item $B(\omega, \omega') = \langle \omega, \omega' \rangle$ (i.e a regular or definite inner product) or
        \item $B(\omega, \omega') = \langle \omega, I_{d_1, d - d_1} \omega' \rangle = \langle \omega_{[1:d_1]}, \omega'_{[1:d_1]} \rangle - \langle \omega_{[(d_1+1):d]}, \omega'_{[(d_1 + 1):d]} \rangle$ for some $d_1 \leq d$ (i.e an indefinite or Krein inner product);
    \end{enumerate}    
    where $I_{p, q} = \mathrm{diag}( I_p, - I_q )$, $\omega_A = (\omega_i)_{i \in A}$ for $A \subseteq [d]$, and $[a:b] = \{a, a + 1, \ldots, b\}$.
\end{assumeapp}

\begin{assumeapp}[Strong local convergence] \label{assume:slc}
    There exists a sequence $(f_n(\lambda_i, \lambda_j, a_{ij}))_{n \geq 1}$ of $\sigma(W)$-measurable functions, with $\mathbb{E}[ f_n(\lambda_1, \lambda_2, a_{12} )^2 ] < \infty$ for each $n$, such that 
    \begin{equation*}
        \max_{i, j \in [n], i \neq j} \Big| \frac{n^2 \mathbb{P}((i,j)\in S(\mathcal{G}_n)|\mathcal{G}_n) }{ f_n(\lambda_i, \lambda_j, a_{ij}) } - 1 \Big| = O_p(s_n)
    \end{equation*}
    for some non-negative sequence $s_n = o(1)$.
\end{assumeapp}

\begin{assumeapp}[Regularity of the sampling weighs] 
    \label{assume:samp_weight_reg}
    We assume that, for each $n$, the functions
    \begin{equation*}
        \fnone := f_n(l, l', 1) W_n(l, l') \text{ and } \fnzero := f_n(l, l', 0) (1 - W_n(l, l') )
    \end{equation*}
    are piecewise H\"{o}lder$([0, 1]^2, \beta, \fnholderconst, \mathcal{Q}^{\otimes 2})$, where $\mcQ$ is the same partition as in Assumption~\ref{assume:graphon_ass}a), but the exponents $\beta$ and $\fnholderconst$ may differ from that of $\beta_W$ and $L_W$ in Assumption~\ref{assume:graphon_ass}a). We moreover suppose that $\tilde{f}_n(l, l', 1)$ and $\tilde{f}_n(l, l', 0)$ are uniformly bounded in $L^{\infty}([0, 1]^2)$, are positive a.e, and that $\tilde{f}_n(l, l', 1)^{-1}$ and $\tilde{f}_n(l, l', 0)^{-1}$ are uniformly bounded in $L^{\gamma_s}([0, 1]^2)$ for some constant $\gamma_s \in [1, \infty]$.
\end{assumeapp}

%!TEX root = ../ms.tex

\section{Proof outline for Theorems~\ref{thm:loss_converge},~\ref{thm:embed_learn:converge_1},~\ref{thm:embed_learn:converge_2} and \ref{thm:embed_learn:converge_3}} \label{sec:proof_sketch}

We begin with outlining the approach of the proof of Theorem~\ref{thm:loss_converge}; that is, the convergence of the empirical risk to the population risk. Note that in the expression of the empirical risk $\mathcal{R}_n(\bmomega)$, as a consequence of Assumption~\ref{assume:simple:slc}, we are able to replace the sampling probabilities in $\mathcal{R}_n(\bmomega)$ with the $f_n(\lambda_i, \lambda_j, a_{ij})/n^2$. After also including the terms with $i = j$, $i \in [n]$ as part of the summation (which is possible as we are adding $O(n)$ terms to an average of $O(n^2)$ quantities), we can asymptotically consider minimizing the expression
\begin{equation*}
    \widehat{\mathcal{R}}_n(\omega_1, \ldots, \omega_n) := \frac{1}{n^2} \sum_{i, j \in [n]^2} f_n(\lambda_i, \lambda_j, a_{ij}) \ell(B(\omega_i, \omega_j), a_{ij}).
\end{equation*}
To proceed further, we now suppose that $W$ corresponds to a stochastic block model; more specifically, we suppose there exists a partition $\mathcal{Q}= (A_1, \ldots, A_{\kappa})$ of $[0, 1]$ into intervals for which $W(\cdot, \cdot)$ is constant on the $A_l \times A_{l'}$ for $l, l' \in [\kappa(n)]$. Note that $f_n(\cdot, \cdot, x)$ is implicitly a function of $W(\cdot, \cdot)$ for $x \in \{0, 1\}$, and therefore it also piecewise constant on $\mathcal{Q}$. As an abuse of notation, we write $f_n(l, l', x)$ for the value of $f_n(\lambda_i, \lambda_j, x)$ when $(\lambda_i,\lambda_j) \in A_{l} \times A_{l'}$. If we write
\begin{align*}
    \mathcal{A}_n(l) & := \big\{ i \in [n] \,:\, \lambda_i \in A_l \big\}, \\
    \mathcal{A}_n(l, l') &:= \big\{ i, j \in [n] \,:\, \lambda_i \in A_l, \lambda_j \in A_{l'} \big\} = \mathcal{A}_n(l) \times \mathcal{A}_n(l')
\end{align*}
we can then perform a decomposition of $\widehat{\mathcal{R}}_n$ into a sum
\begingroup 
\allowdisplaybreaks
\begin{align*}
    \widehat{\mathcal{R}}_n(\omega_1, \ldots, \omega_n) & := \frac{1}{n^2} \sum_{l, l' \in [\kappa]} \sum_{(i, j) \in \mathcal{A}_n(l, l') } f_n(l, l', a_{ij} ) \ell( B(\omega_i, \omega_j), a_{ij} )\nonumber  \\ 
    & = \sum_{l, l' \in [\kappa] }  \frac{ |\mathcal{A}_n(l, l')| }{n^2} \cdot \frac{1}{|\mathcal{A}_n(l, l')| }  \sum_{(i, j) \in \mathcal{A}_n(l, l') } f_n(l, l', a_{ij} ) \ell( B(\omega_i, \omega_j), a_{ij} )    
    \label{eq:decouple:rndecomp}.
\end{align*}
\endgroup
For now working conditionally on the $\lambda_i$, we note that for each of the $(l, l')$, the gap between the averages
\begin{equation}
    \label{eq:proof_sketch:avg1}
    \frac{1}{|\mathcal{A}_n(l, l')| }  \sum_{(i, j) \in \mathcal{A}_n(l, l') } f_n(l, l', a_{ij} ) \ell( B(\omega_i, \omega_j), a_{ij} )    
\end{equation}
and 
\begin{equation}
    \label{eq:proof_sketch:avg2}
    \frac{1}{|\mathcal{A}_n(l, l')| } \sum_{(i, j) \in \mathcal{A}_n(l, l') } \big\{ \tilde{f}_n(l, l', 1) \ell( B(\omega_i, \omega_j), 1)  + \tilde{f}_n(l, l', 0) \ell( B(\omega_i, \omega_j), 0) \big\},
\end{equation}
where we recall that $\tilde{f}_n(l, l', x) = f_n(l, l', 1) W(l, l')^x [1 - W(l, l') ]^{1-x}$, will be small asymptotically. In particular, the difference of the two has expectation zero as the expected value of \eqref{eq:proof_sketch:avg1} conditional on the $\lambda_i$ is \eqref{eq:proof_sketch:avg2}, and will have variance $O(1/ |\mcA_n(l, l')|)$ as \eqref{eq:proof_sketch:avg1} is an average of $\mcA_n(l, l')$ independently distributed bounded random variables. As the variance bound is independent of $\lambda_i$ outside of the size of the set $|\mcA_n(l, l')|$, which will be $\Omega_p(n^2)$, it therefore follows that the difference between \eqref{eq:proof_sketch:avg1} and \eqref{eq:proof_sketch:avg2} will therefore also be small asymptotically unconditionally on the $\lambda_i$ too. We can therefore consider minimizing
\begin{equation}
    \sum_{l, l' \in [\kappa] }  \frac{ |\mathcal{A}_n(l, l')| }{n^2} \cdot \frac{1}{|\mathcal{A}_n(l, l')| } \sum_{(i, j) \in \mathcal{A}_n(l, l') } \sum_{x \in \{0, 1\} } \tilde{f}_n(l, l', x) \ell( B(\omega_i, \omega_j), x). \label{loss_converge:eq:emp_risk_94}
\end{equation}
We now use Jensen's inequality (which is permissible as the loss is strictly convex) and the bilinearity of $B(\cdot, \cdot)$, which gives us that
\begin{align*}
        \sum_{l, l' \in [\kappa] } & \frac{ |\mathcal{A}_n(l, l')| }{n^2}  \cdot \frac{1}{|\mathcal{A}_n(l, l')| }  \sum_{(i, j) \in \mathcal{A}_n(l, l') } \sum_{x \in \{0, 1\} } \tilde{f}_n(l, l', x) \ell( B(\omega_i, \omega_j), x) \nonumber \\
        & \geq \sum_{l, l' \in [\kappa] }  \frac{ |\mathcal{A}_n(l, l')| }{n^2} \sum_{x \in \{0, 1\} } \tilde{f}_n(l, l', x) \ell\Big( B\Big( \frac{1}{|\mathcal{A}_n(l)| } \sum_{i \in \mathcal{A}_n(l) } \omega_i, \frac{1}{|\mathcal{A}_n(l')| } \sum_{j \in \mathcal{A}_n(l') } \omega_j \Big), x \Big) \nonumber \\
        & = \sum_{l, l' \in [\kappa] } \frac{1}{n^2} \sum_{(i, j) \in \mcA_n(l, l') } \sum_{x \in \{0, 1\} } \tilde{f}_n(l, l', x) \ell( B( \widetilde{\omega}_i, \widetilde{\omega}_j), x) 
\end{align*}
where we have defined $\widetilde{\omega}_i := \tfrac{1}{|\mathcal{A}_n(l)| } \sum_{j \in \mathcal{A}_n(l) } \omega_j$ if $i \in \mcA_n(l)$, and the inequality is strict unless the $B(\omega_i, \omega_j)$ are constant across $(i, j) \in \mcA_n(l) \times \mcA_n(l')$. This means that for the purposes of minimizing \eqref{loss_converge:eq:emp_risk_94}, we know that we can restrict ourselves to only taking an embedding vector $\widetilde{\omega}_l$ per latent feature. Making use of the fact that $ n^{-2} |\mathcal{A}_n(l, l') | \cvp p_l p_{l'}$, we are left with 
\begin{equation*}
    \sum_{l, l' \in [\kappa] } p_l p_{l'} \big\{ f_n(l, l', 1) W(l, l') \ell( B(\widetilde{\omega}_l, \widetilde{\omega}_{l'}), 1)  + f_n(l, l', 0) [1 - W(l, l')] \ell( B(\widetilde{\omega}_l, \widetilde{\omega}_{l'}), 0) \big\}.
\end{equation*}
Making the identification $\eta(\lambda) = \widetilde{\omega}_l$ for $\lambda \in A_l$, we then end up exactly with $\mathcal{I}_n[K]$ where $K(l, l') = B(\eta(l), \eta(l'))$ as desired. The details in the appendix discuss how to apply the argument when $W$ is a general (sufficiently smooth) graphon and not just a stochastic block model, along with arguing that the above functions converge uniformly over the embedding vectors, and not just pointwise. 

Once we have the population risk $\mcI_n[K]$, the proof technique for the convergence of the minimizers to \eqref{framework:eq:empirical_loss} in Theorems~\ref{thm:embed_learn:converge_1},~\ref{thm:embed_learn:converge_2} and \ref{thm:embed_learn:converge_3} follow the usual strategy for obtaining consistency results - given uniform convergence of an empirical risk to a population risk, we want to show that the latter has a unique minima which is well-separated, in that points which are outside of a neighbourhood of the minima will have function values which are bounded away from the minimal value also. There are a several technical aspects which are handled in the appendix, relating to the infinite dimensional nature of our optimization problem, the non-convexity of the constraint sets $\mcZ(S_d)$ and the change in domain from embedding vectors $(\omega_1, \ldots, \omega_n)$ to kernels $K\llp$. 
%!TEX root = ..\ms.tex

\section{Proof of Theorem~\ref{thm:loss_converge}} \label{sec:app:loss_converge_proof}

For notational convenience, we will write $\bm{\omega}_n = (\omega_1, \ldots, \omega_n)$ for the collection of embedding vectors for vertices $\{1, \ldots, n \}$, and write 
\begin{equation*}
    \sum_{i, j} f(i, j) := \sum_{i, j = 1}^n f(i, j), \qquad \sum_{i \neq j} f(i, j) := \sum_{i, j \in [n], i \neq j } f(i, j).
\end{equation*}
We will also write $\bm{\lambda}_n := (\lambda_1, \ldots, \lambda_n)$ and $\bm{A}_n := (a_{ij}^{(n)})_{i, j \in [n]}$ for the collection of latent features and adjacency assignments for $\mathcal{G}_n$. We aim to prove the following result:

\begin{theorem} \label{app:loss_converge_proof:main_theorem}
    Suppose that Assumptions~\ref{assume:graphon_ass},~\ref{assume:loss},~\ref{assume:bilinear},~\ref{assume:slc}~and~\ref{assume:samp_weight_reg} hold. Let $S_d = [-A, A]^d$, and write 
    \begin{equation*}
        Z(\compactset) := \{ K: [0, 1]^2 \to \mathbb{R} \,:\, K(l, l') = B(\eta(l), \eta(l')) \text{ a.e, where } \eta: [0, 1] \to \compactset \}.
    \end{equation*}
    Then we have that
    \begin{align*}
        \big| \min_{\bm{\omega}_n \in (\compactset)^n } \mathcal{R}_n(\bm{\omega}_n) - \min_{K \in Z(\compactset) } \mathcal{I}_n[K] \big| = O_p\Big( s_n  + \frac{d^{p+1/2} \mathbb{E}[f_n^2]^{1/2} }{n^{1/2} } +  \frac{ (\log n)^{1/2} + d^{p/\gamma_s}  }{ n^{\beta/(1 + 2\beta) }    } \Big)
    \end{align*}
    where we write $\mathbb{E}[f_n^2] = \mathbb{E}[f_n(\lambda_1, \lambda_2, a_{12} )^2 ]$. If moreover we have that $\fnone$ and $\fnzero$ are piecewise constant functions on a partition $\mcQ^{\otimes 2}$ where $\mcQ$ is
    of size $\kappa$, then
    \begin{align*}
        \big| \min_{\bm{\omega}_n \in (\compactset)^n } \mathcal{R}_n(\bm{\omega}_n) - \min_{K \in Z(\compactset) } \mathcal{I}_n[K] \big|  = O_p\Big( s_n  + \frac{ d^{p +1/2} \mathbb{E}[f_n^2]^{1/2}}{n^{1/2}} + \frac{(\log \kappa)^{1/2} }{n^{1/2} } \Big).
    \end{align*}
\end{theorem}

\begin{remark}[Issues of measurability] \label{app:loss_converge_proof:remark:measurability}
    We make one technical point at the beginning of the proof to prevent repetition - throughout we will be taking infima and suprema of uncountably many random variables over sets which depend on the $\bm{\lambda}_n$ and $\bm{A}_n$. Moreover, we will want to reason about either these minimal/maximal values, or the corresponding argmin sets. We need to ensure the measurability of these types of quantities.

    We note two important facts which will allow us to do so: the fact that the $f_n(\lambda_i, \lambda_j, a_{ij})$ are measurable functions, and that the loss functions $\ell(\cdot, x)$ are continuous for $x \in \{0, 1\}$. Consequently, all of the functions we take suprema or minima over are Carath\'{e}dory; that is of the form $g : X \times S \to \mathbb{R}$, where $x \mapsto g(x, s)$ is continuous for all $s \in S$, and $s \mapsto g(x, s)$ is measurable for all $x \in X$. Here $X$ plays the role of some Euclidean space, and $S$ a probability space supporting the $\bm{\lambda}_n$ and $\bm{A}_n$. Moreover, all of our suprema and minima will be taken either over a) a non-random compact subset $K$ of $\mathbb{R}^m$ for some $m$, or b) a set of the form
    \begin{align*}
        \phi(s) & := \{ x \in K(s) \,:\, g(x, s) \leq C g(0, s) \}
    \end{align*}
    where i) $K(s) := \{ \bm{x} \in \mathbb{R}^m \,:\, \| x \| \leq f(s) \}$ for some measurable function $f(s)$ and norm $\| x \|$ on $\mathbb{R}^m$, ii) $g(x, s)$ is Carath\'{e}dory, and iii) the constant $C$ satisfies $C > 1$ (so $\phi(s)$ is non-empty). With this, we can guarantee the measurability of any quantities we will consider; an application of \citet[Theorem~8.2.9]{aubin_set-valued_2009} implies that $K(s)$, and therefore also $\phi(s)$, are measurable correspondences with non-empty compact values, and therefore the measurable maximum theorem \citep[e.g][Theorem~18.19]{aliprantis_infinite_2006} will guarantee the measurability of all the quantities we want to consider.
\end{remark}

\subsection{Replacing sampling probabilities with 
\texorpdfstring{$f_n(\lambda_i, \lambda_j, a_{ij})/n^2$}{fn(lambda(i), lambda(j), a(i,j))/n2}  } \label{app:loss_converge_proof:sec:replace_prob_with_fn}

To begin, we justify why minimizing 
\begin{equation*}
    \widehat{\mathcal{R}}_n(\bm{\omega}_n) := \frac{1}{n^2} \sum_{i \neq j} f_n(\lambda_i, \lambda_j, a_{ij} ) \ell(B(\omega_i, \omega_j), a_{ij})
\end{equation*}
is asymptotically equivalent to that of minimizing $\mathcal{R}_n(\bm{\omega}_n)$. 

\begin{lemma} \label{app:loss_converge_proof:replace_prob_with_fn}
    Assume that Assumptions~\ref{assume:loss}~and~\ref{assume:slc} hold. Then there exists a non-empty random measurable set $\Psi_n$ such that 
    \begin{equation*}
        \mathbb{P}\Big( \argmin_{\bm{\omega}_n \in (\compactset)^n } \mathcal{R}_n(\bm{\omega}_n) \cup \argmin_{\bm{\omega}_n \in (\compactset)^n } \widehat{\mathcal{R}}_n(\bm{\omega}_n) \subseteq \Psi_n \Big) \to 1, \quad
        \sup_{\bm{\omega}_n \in \Psi_n} \Big| \mathcal{R}_n(\bm{\omega}_n) - \widehat{\mathcal{R}}_n(\bm{\omega}_n) \Big| = O_p( s_n ).
    \end{equation*}   
\end{lemma}

%!TEX root = ..\..\ms.tex

\begin{proof}[Proof of Lemma~\ref{app:loss_converge_proof:replace_prob_with_fn}]
    \phantomsection\label{app:loss_converge_proof:replace_prob_with_fn:proof}
    We will argue that the loss functions will converge uniformly over sets of the form $\mathcal{R}_n(\bm{\omega}_n) \leq C\emprisk(\bm{0})$, where $C$ can be any constant strictly greater than one. Such sets contain the minima of e.g $\mcR_n(\bmomega)$, and as we are working on (stochastically) bounded level sets of $\mcR_n(\bmomega)$, this will be enough to allow us to use Assumption~\ref{assume:slc} in order to obtain the desired conclusion. With this in mind, we denote $C_{\ell, 0} = \max_{x \in \{0, 1\}} \ell(0, x)$ and then define the sets
    \begin{align*}
        \Psi_n &:= \Bigg\{ \bm{\omega}_n \in (\compactset)^n \,:\, \mathcal{R}_n(\bm{\omega}_n) \leq 2 C_{\ell, 0}  \sum_{i \neq j}  \mathbb{P}((i,j)\in S(\mcG_n)|\mathcal{G}_n)   \Bigg\}, \\
        \widehat{\Psi}_n & := \Bigg\{ \bm{\omega}_n \in (\compactset)^n \,:\, \widehat{\mathcal{R}}_n(\bm{\omega}_n) \leq C_{\ell, 0} \sum_{ i \neq j} \frac{f_n(\lambda_i,\lambda_j,a_{ij})}{n^2} \Bigg\}.
    \end{align*}
    Our aim is to show that $\widehat{\Psi}_n \subseteq \Psi_n$ with asymptotic probability $1$. Note that
    \begin{equation*}
        \mathcal{R}_n(\bm{0}) \leq C_{\ell, 0} \sum_{i \neq j}  \mathbb{P}((i,j)\in S(\mcG_n)|\mathcal{G}_n), \qquad \widehat{\mcR}_n(\bm{0}) \leq C_{\ell, 0} \sum_{i \neq j } \frac{f_n(\lambda_i, \lambda_j, a_{ij}) }{n^2} 
    \end{equation*}
    so $\bm{0} \in \Psi_n$ and $\bm{0} \in \widehat{\Psi}_n$ (meaning the sets are non-empty). Moreover, these sets will always contain the argmin sets of $\emprisk(\bmomega)$ and $\widehat{\mcR}_n(\bmomega)$ respectively (as any minimizer $\bmomega$ will satisfy e.g $\emprisk(\bmomega) \leq \emprisk(\bm{0})$). In particular, once we show that $\mathbb{P}( \widehat{\Psi}_n \subseteq \Psi_n) \to 1$ as $n \to \infty$, we will have shown the first part of the lemma, and we can then reduce to showing uniform convergence of $\mcR_n(\bmomega) - \widehat{\mcR}_n(\bmomega)$ over $\Psi_n$. Pick an arbitrary $\bm{\omega}_n \in \widehat{\Psi}_n$. Then by Assumption~\ref{assume:slc}, we get that
    \begin{align*}
        \mathcal{R}_n(\bm{\omega}_n) & = \sum_{i \neq j} \frac{ n^2 \mathbb{P}((i, j) \in S(\mcG_n) | \mathcal{G}_n) }{  f_n(\lambda_i, \lambda_j, a_{ij} ) } \frac{f_n(\lambda_i, \lambda_j, a_{ij}) }{n^2} \ell(B(\omega_i, \omega_j), a_{ij}) \\
        & \leq \sup_{i \neq j} \frac{ n^2 \mathbb{P}((i, j) \in S(\mcG_n) | \mathcal{G}_n) }{  f_n(\lambda_i, \lambda_j, a_{ij} ) } \cdot \widehat{\mathcal{R}}_n(\bm{\omega}_n) \leq C_{\ell, 0} (1 + o_p(1)) \sum_{i \neq j} \frac{ f_n(\lambda_i, \lambda_j, a_{ij} ) }{n^2}.
    \end{align*}
    By Lemma~\ref{app:loss_converge_proof:cti} - noting that with asymptotic probability $1$ all the quantities involved are positive - we have that
    \begin{equation}
        \label{eq:replace_prob_with_fn:eq1}
    \frac{ \sum_{i \neq j} n^{-2} f_n(\lambda_i, \lambda_j, a_{ij})  }{  \sum_{i \neq j} \mathbb{P}((i,j)\in S(\mcG_n)|\mathcal{G}_n)      } \leq \sup_{i \neq j} \frac{ f_n(\lambda_i, \lambda_j, a_{ij}) }{  n^2 \mathbb{P}((i,j)\in S(\mcG_n)|\mathcal{G}_n)  } = 1 + o_p(1)
    \end{equation}
    and so 
    \begin{equation*}
    \mathcal{R}_n(\bm{\omega}_n) \leq C_{\ell, 0} (1 + o_p(1))^2 \sum_{ i \neq j} \mathbb{P}((i, j) \in S(\mcG_n) | \mathcal{G}_n) \stackrel{\text{w.h.p}}{\leq} 2 C_{\ell, 0} \sum_{i \neq j} \mathbb{P}((i, j) \in S(\mcG_n) | \mathcal{G}_n)
    \end{equation*}
    for $n$ sufficiently large. This holds freely of the choice of $\bm{\omega}_n \in \widehat{\Psi}_n$, and so $\widehat{\Psi}_n \subseteq \Psi_n$ with asymptotic probability $1$. To conclude, we then note that over the set $\Psi_n$, we have 
    \begin{align*}
        \sup_{\bm{\omega}_n \in \Psi_n} & \Big|\sum_{i \neq j}  \Big[ \mathbb{P}((i,j)\in S(\mcG_n)|\mathcal{G}_n) - \frac{f_n(\lambda_i,\lambda_j,a_{ij})}{n^2} \Big] \ell(B(\omega_i, \omega_j), a_{ij})  \Big| \\
        & \leq \sup_{i \neq j} \Big| \frac{n^2 \mathbb{P}((i,j)\in S(\mcG_n)|\mathcal{G}_n) }{ f_n(\lambda_i, \lambda_j, a_{ij}) }  - 1\Big| \cdot \sup_{\bm{\omega}_n \in \Psi_n} \mathcal{R}_n(\bm{\omega}_n) \leq O_p(s_n) \cdot \mcR_n(\bm{0}) = O_p(s_n)
        %& \leq \sup_{i \neq j} \Big| \frac{ f_n(\lambda_i, \lambda_j, a_{ij})  }{ n^2 \mathbb{P}((i,j)\in S(\mcG_n)|\mathcal{G}_n)  }  - 1\Big| \cdot \mathcal{R}_n(\bm{0}) = O_p( s_n )
    \end{align*}
    as desired. Here we use the fact that $\mcR_n(\bm{0})$ is $O_p(1)$, which follows as a result of the fact that $\sum_{i \neq j} f_n(\lambda_i, \lambda_j, a_{ij}) n^{-2}$ is $O_p(1)$ by Lemma~\ref{app:loss_converge_proof:fn_op1} and $\sup_{n \geq 1} \mathbb{E}[f_n(\lambda_i, \lambda_j, a_{ij} )] < \infty$ (by Assumption~\ref{assume:slc}), and then noting that
    \begin{equation*}
        \sum_{i \neq j} \mathbb{P}\big( (i, j) \in S(\mcG_n) \,|\, \mcG_n \big) = (1 + o_p(1) ) \frac{1}{n^2} \sum_{i \neq j}  f_n(\lambda_i, \lambda_j, a_{ij})
    \end{equation*}
    analogously to \eqref{eq:replace_prob_with_fn:eq1}.
\end{proof}

\subsection{Averaging the empirical loss over possible edge assignments} \label{app:loss_converge_proof:sec:average_over_adjacency}

Now that we can work with $\empriskhat(\bmomega)$, we want to examine what occurs as we take $n \to \infty$. Intuitively, what we will attain should correspond to what occurs when we average this risk over the sampling distribution of the graph; to do so, we begin by averaging over the $a_{ij}$ (while working conditionally on the $\lambda_i$). As a result, we want to argue that $\empriskhat(\bmomega)$ is asymptotically close to
\begin{equation} \label{eq:app:loss_converge_proof:inter_fn}
    \mathbb{E}[ \empriskhat(\bmomega) | \bm{\lambda}_n ] := \frac{1}{n^2} \sum_{i \neq j} \sum_{x \in \{0, 1\} } \tilde{f}_n(\lambda_i, \lambda_j, x) \ell(B(\omega_i, \omega_j), x),
\end{equation}
where we recall
\begin{equation*}
    \tilde{f}_n(\lambda_i, \lambda_j, 1) = f_n(\lambda_i, \lambda_j, 1) W_n(\lambda_i, \lambda_j), \qquad \tilde{f}_n(\lambda_i, \lambda_j, 0) = f_n(\lambda_i, \lambda_j, 0) [1 - W_n(\lambda_i, \lambda_j)].
\end{equation*}
As the above functions depend only on the values of the $B(\omega_i, \omega_j) =: \Omega_{ij}$, we will freely interchange between the functions having argument $\Omega$ or $\bm{\omega}_n$ (whichever is most convenient, mostly for the sake of saving space), with the dependence of $\Omega$ on $\bm{\omega}_n$ implicit. We write
\begin{equation} \label{eq:app:loss_converge_proof:znsd}
    Z_{n}(\compactset) := \{ \Omega \in \mathbb{R}^{n \times n} \,:\, \Omega_{ij} = B(\omega_i, \omega_j),\, \omega_i \in \compactset \text{ for } i \in [n] \}
\end{equation}
for the corresponding set of $\Omega$ which are induced via $\bm{\omega}_n \in (\compactset)^n$, and define the metric
\begin{equation} \label{eq:app:loss_converge_proof:sell_metric}
    s_{\ell, \infty}\big(\Omega, \widetilde{\Omega} \big) := \max_{i, j \in [n]} \max \big\{ | \ell(\Omega_{ij}, 1) - \ell(\widetilde{\Omega}_{ij}, 1) | , | \ell(\Omega_{ij}, 0) - \ell(\widetilde{\Omega}_{ij}, 0) | \big\},
\end{equation}
which is induced by the choice of loss function $\ell(y, x)$ in Assumption~\ref{assume:loss}. (The injectivity constraints on the loss function specified in Assumption~\ref{assume:loss} ensure that $s_{\ell, \infty}(\Omega, \widetilde{\Omega}) = 0 \iff \Omega = \widetilde{\Omega}$; the remaining metric properties follow immediately.) We now work towards proving the following result:

\begin{theorem} \label{app:loss_converge_proof:average_over_adjacency}
    Suppose that Assumptions~\ref{assume:loss} and \ref{assume:slc} hold. Then we have that
    \begin{equation*}
        \sup_{\Omega \in Z_{n}( \compactset ) } \big| \mathbb{E}[ \widehat{\mathcal{R}}_n(\Omega) \,|\, \bm{\lambda}_n ] - \widehat{\mathcal{R}}_n(\Omega) \big| = O_p\Big( \frac{ \gamma_2(Z_n(\compactset), s_{\ell, \infty}) \mathbb{E}[f_n(\lambda_1, \lambda_2, a_{12} )^2 ]^{1/2} }{n} \Big)
    \end{equation*}
    where $\gamma_2(T, s)$ denotes the Talagrand $\gamma_2$-functional of a metric space $(T, s)$. 
\end{theorem}

Here the Talagrand $\gamma_2$-functional is defined as 
\begin{equation*}
    \gamma_2(T, s) := \inf \sup_{t \in T} \sum_{n \geq 0} 2^{n/2} \Delta( A_n(t), s) 
\end{equation*}
where the infimum is taken over all refining sequences $(\mcA_n)_{n \geq 1}$ of partitions of $T$, where $| \mcA_n | \leq 2^{2^n}$ for $n \geq 1$ and $ | \mcA_0| = 1$, $A_n(t)$ denotes the unique partition of $\mcA_n$ for which $t$ lies within the partition, and $\Delta(T, s) := \sup_{t, v \in T} s(t, v)$ denotes the diameter of $(T, s)$.  See \citet[Chapter 2]{talagrand_upper_2014} for various definitions which are equivalent up to universal constants.

\begin{remark} \label{app:loss_converge_proof:rmk:gamma_2}
    We briefly note that rather than calculating the above quantity explicitly, all we require\footnote{We note that when $T \subseteq \mathbb{R}^m$, $\gamma_2(T, s)$ can only be smaller than the metric entropy by a factor of $\log(m)$ \citep[Exercise 2.3.4]{talagrand_upper_2014}, and so this bound will be tight enough for our purposes.} are the bounds
    \begin{equation*}
        \Delta(T, s) \leq \gamma_2(T, s) \leq C \int_0^{\infty} \sqrt{ \log N(T, s, \epsilon) } \, d\epsilon,
    \end{equation*}
    where $C$ is some universal constant, and $N(T, s, \epsilon)$ is the minimal size of an $\epsilon$-covering of $T$ with respect to the metric $s$ (so the RHS is simply the metric entropy of $(T, s)$). We state the bound in terms of $\gamma_2(T, s)$ simply as it allows for the easier use of the chaining bound (Theorem~\ref{app:loss_converge_proof:generic_chaining}) stated and used later.
\end{remark}

The proof technique consists of a combination of a truncation argument, a chaining argument, and the method of exchangeable pairs. To recap from \citet{chatterjee_concentration_2005} the method of exchangeable pairs, suppose that $X$ is a random variable on a Banach space and $f$ is a measurable function such that $\mathbb{E}[f(X)] = 0$. Given an exchangeable pair $(X, X')$ (so that $(X, X') = (X', X)$ in distribution) and an anti-symmetric function $F(X, X')$ such that
\begin{equation*}
    \mathbb{E}[ F(X, X') \,|\, X] = f(X),
\end{equation*}
then provided one has $\mathbb{E}[ e^{\theta f(X)} | F(X, X')| ] < \infty$ and the "variance bound"
\begin{equation}
    \label{eq:loss_converge_proof:exch_var}
    v(X) := \frac{1}{2} \mathbb{E}\big[ | \{ f(X) - f(X') \} F(X, X') | \,\big|\, X \big] \leq C
\end{equation}
almost surely for some constant $C > 0$, then we have a concentration inequality for the tails of $f(X)$ of the form
\begin{equation*}
    \mathbb{P}\big( | f(X) | > \eta \big) \leq 2 e^{-\eta^2/2C} \text{ for all } \eta > 0.
\end{equation*}
In particular, we can interpret this as saying that $f(X)$ is sub-Gaussian. If we now had a mean zero stochastic process $\{f_t(X)\}_{t \in T}$ where we equip $T$ with a metric $s(\cdot, \cdot)$, and we could also construct an exchangeable pair $(X, X')$ and functions $F_{t, v}(X, X')$ for $t, v \in T$ such that i) $\mathbb{E}[F_{t, v}(X, X')| X] = f_t(X) - f_v(X)$ and ii) the corresponding variance term \eqref{eq:loss_converge_proof:exch_var} is bounded by $C s(t, v)^2$, we have the tail bound
\begin{equation*}
    \mathbb{P}\Big( | f_t(X) -  f_v(X) | > \eta s(t, v) \Big) \leq 2 e^{-\eta^2/2C} \text{ for all } \eta > 0.
\end{equation*}
We could then apply standard chaining results for the supremum of sub-Gaussian processes, such as those in \citet{talagrand_upper_2014}:

\begin{proposition}[{\citealp[Theorem~2.2.27]{talagrand_upper_2014}}]
    \label{app:loss_converge_proof:generic_chaining}
   Let $(T, s)$ be a metric space and suppose $(X_t)_{t \in T}$ is a mean-zero stochastic process on $(T, s)$. Suppose that there exists a constant $\sigma > 0$ such that for all $t, v \in T$, 
   \begin{equation*}
       \mathbb{P}\big( |X_t - X_v| > \eta s(t, v) \big) \leq 2 e^{- \eta^2/2\sigma^2} \text{ for all } \eta > 0.
   \end{equation*}
   Then there exist universal constants $L > 0$ and $L' > 0$ such that 
   \begin{equation*}
       \mathbb{P}\Big( \sup_{t, v \in T} |X_t - X_v| > \sigma L \big( \gamma_2(T, s)  + \eta \Delta(T, s) \big) \Big) \leq L' e^{-\eta^2}
   \end{equation*}
   for all $\eta> 0$, where $\gamma_2(T, s)$ is the Talagrand $\gamma_2$-functional of $(T, s)$ and $\Delta(T, s)$ denotes the diameter of the set $T$ with respect to $s$.
%    Here the Talagrand $\gamma_2$-functional $\gamma_2(T, s)$ is defined in the usual manner, and $\Delta(T, s) := \sup_{t, v \in T} s(t, v)$
%    denotes the diameter of $T$ with respect to $s$.
\end{proposition}

In particular, this result allows one to easily control the supremum of a stochastic process with an uncountable index, by exploiting the continuity of the underlying process. With the above result, we can rephrase Theorem~\ref{app:loss_converge_proof:average_over_adjacency} in terms of controlling the supremum of the absolute value of the stochastic process
\begin{align} \label{eq:app:loss_converge_proof:sto_pro}
    E_n(\Omega)[\bm{A}_n] &:= \widehat{\mathcal{R}}_n(\Omega) - \mathbb{E}[ \widehat{\mathcal{R}}_n(\Omega) \,|\, \bm{\lambda}_n ] \\
    & = \frac{1}{n^2} \sum_{i \neq j} f_n(\lambda_i, \lambda_j, a_{ij} ) \ell( \Omega_{ij}, a_{ij} ) - \frac{1}{n^2} \sum_{i \neq j} \sum_{x \in \{0, 1\} } \tilde{f}_n(\lambda_i, \lambda_j, x) \ell( \Omega_{ij}, x) \nonumber
\end{align}
over $\Omega \in Z_{n}(\compactset)$, where we keep track of $\bm{A}_n$ where necessary (and will suppress the dependence on this when not). To control the above stochastic process, we will use the method of exchangeable pairs, while working conditional on the $\bm{\lambda}_n$, to give us control of \eqref{eq:app:loss_converge_proof:sto_pro} for fixed $\Omega$; we can then use Proposition~\ref{app:loss_converge_proof:generic_chaining} to give us control over all the $\Omega \in Z_n(\compactset)$. We note that as our argument will partly employ a truncation argument, we require the following minor modification of the method of exchangeable pairs:

\begin{lemma} \label{app:loss_converge_proof:exchangeable_pairs_truncation}
    Suppose that $X$ is an exchangeable pair with functions $f(X)$ and $F(X, X')$ satisfying the conditions stated above, and moreover that $B \in \sigma(X)$ is an event such that $B \subseteq \{ v(X) \leq C \}$ and $\mathbb{E}[ e^{\theta f(X) } | F(X, X') | 1_B ] < \infty$ for all $\theta$. Then 
    \begin{equation*}
        \mathbb{P}\big( | f(X) | > t, B \big) \leq 2 e^{-t^2/2C} \text{ for all } t > 0.
    \end{equation*}
\end{lemma}

\begin{proof}[Proof of Lemma~\ref{app:loss_converge_proof:exchangeable_pairs_truncation}]
    The method of proof is identical to that of \citep{chatterjee_concentration_2005}, except one replaces the moment generating function of $f(X)$ with $m(\theta) := \mathbb{E}[ e^{\theta f(X)} 1_B ]$. Following the proof through gives $|m'(\theta)| \leq C |\theta| m(\theta)$, and so $m(\theta) \leq e^{C \theta^2/2}$, and so the result follows from optimizing the Chernoff bound
    \begin{align*}
        \mathbb{P}\big( f(X) > t, B\big) & \leq \mathbb{P}\big( e^{\theta f(X)} > e^{\theta t}, B \big) = \mathbb{E}\big[ 1[ e^{\theta f(X)} > e^{\theta t} ] 1_B \big] \\
        & \leq e^{-\theta t} \mathbb{E}[ e^{\theta f(X)} 1_B ]  \leq e^{-\theta t + C \theta^2/2}
    \end{align*}
    with $\theta = t/C$ as usual (and similarly so for the reverse tail).
\end{proof}

%!TEX root = ..\..\ms.tex

\begin{proof}[Proof of Theorem~\ref{app:loss_converge_proof:average_over_adjacency}]
    \phantomsection\label{app:loss_converge_proof:average_over_adjacency:proof} \emph{(Step 1: Breaking up the tail bound into controllable terms.)} To begin, we define 
    \begin{align}
        C_{n, 1}(\bm{\lambda}_n, \bm{A}_n) &:= \frac{1}{n^2} \sum_{i \neq j} f_n(\lambda_i, \lambda_j, a_{ij})^2,  
        \label{eq:app:loss_converge_proof:cn1} \\
        C_{n, 2}(\bm{\lambda}_n) & := \frac{1}{n^2} \sum_{i \neq j} \mathbb{E}\big[ f_n(\lambda_i, \lambda_j, a_{ij} )^2 \,|\, \bm{\lambda}_n \big] \nonumber \\
        & = \frac{1}{n^2} \sum_{i \neq j} \big\{ f_n(\lambda_i, \lambda_j, 1)^2 W_n(\lambda_i, \lambda_j) + f_n(\lambda_i, \lambda_j, 0)^2 (1 - W_n(\lambda_i, \lambda_j))  \big\} 
        \label{eq:app:loss_converge_proof:cn2}
    \end{align}
    and note that $\mathbb{E}[ C_{n,1}(\bm{A}_n, \bm{\lambda}_n) \,|\, \bm{\lambda}_n ] = C_{n, 2}( \bm{\lambda}_n)$. We now fix $\epsilon > 0$. By Lemma~\ref{app:loss_converge_proof:fn_op1} we know that $C_{n, 2} = O_p(\mathbb{E}[f_n^2])$ (where we understand that $\mathbb{E}[f_n^2] = \mathbb{E}[f_n(\lambda_1, \lambda_2, a_{12} )^2 ]$),
    and therefore there exists $N(\epsilon), M_2(\epsilon)$ for which, once $n \geq N(\epsilon)$, we have that 
    \begin{equation*}
        \mathbb{P}( C_{n, 2}(\bm{\lambda}_n) \geq \mathbb{E}[f_n^2] M_2) \leq \frac{\epsilon}{4}.
    \end{equation*}
    As by Markov's inequality we have that 
    \begin{equation*}
        \mathbb{P}\big( C_{n, 1}(\bm{A}_n, \bm{\lambda}_n) > t \,|\, \bm{\lambda}_n \big) \leq \frac{C_{n, 2}(\bm{\lambda}_n)}{t} \qquad \text{ a.s }
    \end{equation*}
    for any $t > 0$, if we define $B_{n, 2} := \{ C_{n, 2}(\bm{\lambda}_n) \leq \mathbb{E}[f_n^2] M_2 \}$ we therefore have that for $n \geq N(\epsilon)$ that
    \begin{equation*}
        \mathbb{P}\big( C_{n, 1}(\bm{A}_n, \bm{\lambda}_n) > t \mathbb{E}[f_n^2] M_2 \,|\, \bm{\lambda}_n \big) 1_{B_{n, 2}} \leq \frac{1}{t} \frac{ C_{n, 2}(\bm{\lambda}_n ) }{ \mathbb{E}[f_n^2] M_2} 1_{B_{n, 2}} \leq \frac{1}{t} \qquad \text{ a.s }
    \end{equation*}
    and therefore there exists $M_1(\epsilon)$ such that, once $n \geq N(\epsilon)$, we have that
    \begin{equation*}
        \mathbb{E}\Big[ \mathbb{P}\big( C_{n, 1}(\bm{A}_n, \bm{\lambda}_n) > M_1 M_2 \mathbb{E}[f_n^2] \,|\, \bm{\lambda}_n \big) 1_{B_{n, 2}} \Big] \leq \frac{\epsilon}{4}.
    \end{equation*}
    Writing $B_{n, 1} := \{ C_{n, 1}(\bm{A}_n, \bm{\lambda}_n) \leq \mathbb{E}[f_n^2] M_1 M_2 \}$, we now write 
    \begingroup
    \allowdisplaybreaks
    \begin{align*}
        \mathbb{P}\Big( \sup_{\Omega \in Z_n( \compactset) } & | E_n[\Omega] | > \eta \Big) \leq \mathbb{P}\Big( \sup_{\Omega \in Z_n( \compactset) } | E_n[\Omega] | > \eta , B_{n, 2} \Big) + \mathbb{P}(B_{n, 2}^c) \\ 
        & \leq \mathbb{E}\Bigg[ \mathbb{P}\Big( \sup_{\Omega \in Z_n(\compactset)} | E_n[\Omega] | > \eta , B_{n, 1} \,|\, \bm{\lambda}_n \Big) 1_{B_{n, 2}} \Bigg] + \mathbb{E}\big[ \mathbb{P}( B_{n, 1}^c \,|\, \bm{\lambda}_n ) 1_{B_{n, 2} } \big] + \mathbb{P}(B_{n, 2}^c) \\
        & \leq \mathbb{E}\Bigg[ \mathbb{P}\Big( \sup_{\Omega \in Z_n(\compactset) } | E_n[\Omega] - E_n[0] | > \eta/2 , B_{n, 1} \,|\, \bm{\lambda}_n \Big) 1_{B_{n, 2}} \Bigg] \\
        & \qquad + \mathbb{E}\Bigg[ \mathbb{P}\Big( | E_n[0] | > \eta/2 , B_{n, 1} \,|\, \bm{\lambda}_n \Big) 1_{B_{n, 2}} \Bigg] + \mathbb{E}\big[ \mathbb{P}( B_{n, 1}^c \,|\, \bm{\lambda}_n ) 1_{B_{n, 2} } \big] + \mathbb{P}(B_{n, 2}^c) \nonumber \\
        & := (\mathrm{I}) + (\mathrm{II}) + (\mathrm{III}) + (\mathrm{IV}) \nonumber
    \end{align*}
    \endgroup
    and control each of the four terms. For the latter two terms $(\mathrm{III})$ and $(\mathrm{IV})$ , we know that once $n \geq N(\epsilon)$, their sum is less than or equal to $\epsilon/2$, and so we focus on the details for the first two terms. For the first term, we will show that for any $\Omega, \widetilde{\Omega} \in Z_n(S_d)$ that
    \begin{equation}
        \label{eq:avg_over_adj:exch}
        \mathbb{P}\Big( \big| E_n[\Omega] - E_n[\widetilde{\Omega}] \big| > \eta , B_{n, 1} \,|\, \bm{\lambda}_n \Big) 1_{B_{n, 2} } \leq 2 \exp\Big( - \frac{\eta^2}{2 \mathbb{E}[f_n^2] M_2(1 + M_1) n^{-2} s_{\ell, \infty}\big(\Omega, \widetilde{\Omega}\big)^{2} } \Big) 
    \end{equation}
    which allows us to apply Proposition~\ref{app:loss_converge_proof:generic_chaining}, and for the second term we will get that 
    \begin{equation}
        \label{eq:avg_over_adj:exch2}
        \mathbb{P}\big( |E_n[0]| > \eta, B_{n, 1} \,|\, \bm{\lambda}_n \big) 1_{B_{n, 2} } \leq 2 \exp\Big(  - \frac{\eta^2}{2 \mathbb{E}[f_n^2] M_2(1+M_1) C_{\ell, 0}^2 n^{-2}  } \Big)
    \end{equation}
    where $C_{\ell, 0} = \max_{x \in \{0, 1\} } \ell(0, x)$. As the details are essentially identical for both, we will through the proof of \eqref{eq:avg_over_adj:exch} only. Before doing so though, we show how these results will allow us to obtain the theorem statement. Note that as a consequence of Proposition~\ref{app:loss_converge_proof:generic_chaining} (recall that $L, L'$ are universal constants introduced in the chaining bound) we have, writing $M_3 := C_M L \sqrt{2 M_2( 1 + M_1) }$ (where $C_M \geq 1$ is a constant we choose later) and $\widetilde{\eta} \geq ( \log(4L' / \epsilon) )^{1/2}$, that
    \begin{align}
        \mathbb{P}&\Big( \sup_{\Omega \in Z_n(\compactset) } | E_n[\Omega] - E_n[0] | > \tfrac{ M_3 \mathbb{E}[f_n^2]^{1/2} }{n} \big[ \gamma_2(Z_n(\compactset)) + \widetilde{\eta} \Delta(Z_n(\compactset)) \big] , B_{n, 1} \,|\, \bm{\lambda}_n \Big) 1_{B_{n, 2}} 
        \label{eq:avg_over_adj:bound_1}\\
        & \leq \mathbb{P}\Big( \sup_{\Omega, \widetilde{\Omega} \in Z_n(\compactset) } | E_n[\Omega] - E_n[\widetilde{\Omega}] | > \tfrac{ M_3 \mathbb{E}[f_n^2]^{1/2} }{n} \big[ \gamma_2(Z_n(\compactset)) + \widetilde{\eta} \Delta(Z_n(\compactset)) \big] , B_{n, 1} \,|\, \bm{\lambda}_n \Big) 1_{B_{n, 2}} \nonumber \\ 
        & \leq L' e^{-\widetilde{\eta}^2} \leq \epsilon/4. \nonumber
    \end{align}
    Here we have temporarily suppressed the dependence of the metric on $\gamma_2(T, s)$ and $\Delta(T, s)$ for reasons of space, and note that the above inequality holds provided $C_M \geq 1$. Taking expectations then allows us to show that $(\mathrm{I}) \leq \epsilon/4$ by taking any
    \begin{equation*}
    \eta \geq M_3 \Big(  \frac{ \gamma_2(Z_n(\compactset), s_{\ell, \infty}) \mathbb{E}[f_n^2]^{1/2} }{n} + \sqrt{ \log\Big( \frac{4L'}{\epsilon} \Big) } \frac{ \Delta (Z_n(\compactset), s_{\ell, \infty}) \mathbb{E}[f_n^2]^{1/2} }{n} \Big)
    \end{equation*}
    (where we have inverted the bound in \eqref{eq:avg_over_adj:bound_1} and substituted in the value of $\tilde{\eta}$). By using such a choice of $\eta$, we then note that in \eqref{eq:avg_over_adj:exch2} we get that 
    \begin{align*}
        \mathbb{P}\big( |E_n[0]| & > \eta, B_{n, 1} \,|\, \bm{\lambda}_n \big) 1_{B_{n, 2} } \\ & \leq 2 \exp\Big( - C_M^2 L^2 C_{\ell, 0}^{-2} \{ \gamma_2(Z_n(\compactset), s_{\ell, \infty}) + \tilde{\eta} \Delta (Z_n(\compactset), s_{\ell, \infty})    \} / 4 \Big). 
    \end{align*}
    Noting that $A^2 d \leq \Delta(Z_n(\compactset), s_{\ell, \infty}) \leq \gamma_2(Z_n(\compactset), s_{\ell, \infty})$ (recall Remark~\ref{app:loss_converge_proof:rmk:gamma_2}), it therefore follows that by choosing
    \begin{equation*}
        C_M = \max\{1, C_{\ell, 0} A^{-1} L^{-1} d^{-1/2} \sqrt{ \log( 8 / \epsilon ) } \}
    \end{equation*}
    in the expression for $M_3$, we get that $(\mathrm{II}) \leq \epsilon/4$ also. 
    
    Putting this altogether, as we have that $\gamma_2(Z_n(\compactset), s_{\ell, \infty}) \geq \Delta(Z_n(\compactset), s_{\ell, \infty})$, it therefore follows from the above discussion that given any $\epsilon > 0$, we will be able to find constants $N(\epsilon)$ and $M(\epsilon)$ (the value of $N$ given at the beginning of the proof; for $M$, the value of $M_3 ( 1 + \tilde{\eta} )$ from the discussion above), such that once $n \geq N(\epsilon)$, we have that 
    \begin{equation*}
        \mathbb{P}\Big( \sup_{\Omega \in Z_n(\compactset) } | E_n(\Omega) | > M \frac{ \gamma_2(Z_n(\compactset), s_{\ell, \infty}) \mathbb{E}[f_n^2]^{1/2} }{n} \Big) < \epsilon 
    \end{equation*}
    and so we get the claimed result.

    \emph{(Step 2: Deriving concentration using the method of exchangeable pairs.)} We now focus on deriving the inequality \eqref{eq:avg_over_adj:exch}. For the current discussion, we now make explicit the dependence of e.g $E_n(\Omega)[\bm{A}_n]$ on the draws of the adjacency matrix $\bm{A}_n$. Note that throughout we will be working conditionally on $\bm{\lambda}_n$, with the intention of then later restricting ourselves to only handling the $\bm{\lambda}_n$ which lie within the event $B_{n, 2}$. (Note this set has no dependence on the adjacency matrix $\bm{A}_n$, and so we are only restricting the possible values of $\bm{\lambda}_n$ which we are conditioning on when using the method of exchangeable pairs.) We now define an exchangeable pair $(\bm{A}_n, \bm{A}_n')$ as follows:
    \begin{enumerate}[label=\alph*)]
        \item Out of the set $\{ i < j \,:\, i, j \in [n] \}$, pick a pair $(I, J)$ uniformly at random.
        \item With this, we then make an independent draw $a_{I, J}' \sim \mathrm{Bernoulli}(W_n(\lambda_I, \lambda_J))$, set $a_{ij}' = a_{ij}$ for the remaining $i < j$, and set $a_{ji}' = a_{ij}'$ for $j > i$. 
    \end{enumerate}
    We then define the random variables
    \begin{equation*}
        g(\bm{A}_n) = E_n(\Omega)[\bm{A}_n] - E_n(\widetilde{\Omega})[\bm{A}_n], \qquad G(\bm{A}_n, \bm{A}_n') = \frac{n(n-1)}{2} \big( g(\bm{A}_n) - g(\bm{A}_n') \big).
    \end{equation*}
    Note that as $\mathbb{E}[ E_n(\Omega)[ \bm{A}_n] \,|\, \bm{\lambda}_n] = 0$ we have that $\mathbb{E}[ g(\bm{A}_n) \,|\, \bm{\lambda}_n ] = 0$, and similarly we have that 
    \begingroup 
    \allowdisplaybreaks
    \begin{align*}
        \mathbb{E}[ G(\bm{A}_n, \bm{A}_n') \,|\, \bm{\lambda}_n, \bm{A}_n ] & = \frac{1}{n^2} \sum_{i \neq j } \mathbb{E}\Big[ 
            f_n(\lambda_i, \lambda_j, a_{ij}) \{ \ell(\Omega_{ij}, a_{ij}) - \ell(\widetilde{\Omega}_{ij}, a_{ij}) \} \\ & \qquad \qquad \qquad - f_n(\lambda_i, \lambda_j, a_{ij}') \{ \ell(\Omega_{ij}, a_{ij}') - \ell(\widetilde{\Omega}_{ij}, a_{ij}') \}  \,|\, \bm{\lambda}_n, \bm{A}_n \Big] \nonumber \\ 
        & = \widehat{\mcR}_n(\Omega) - \widehat{\mcR}_n(\widetilde{\Omega} ) - \big\{ \mathbb{E}\big[  \widehat{\mcR}_n(\Omega) - \widehat{\mcR}_n(\widetilde{\Omega} ) \,|\, \bm{\lambda}_n \big] \big\} = g(\bm{A}_n).
    \end{align*}    
    \endgroup
    In order to obtain a concentration inequality via the method of exchangeable pairs, we first need to verify that $\mathbb{E}[ e^{\theta g(\bm{A}_n) } | G(\bm{A}_n, \bm{A}_n')|  1_{B_{n, 1} } \,|\, \bm{\lambda}_n] < \infty$ on $B_{n, 2}$ for all $\theta > 0$. To do so, we note that $g(\bm{A}_n) 1_{B_{n, 1} }$ and $g(\bm{A}_n') 1_{B_{n, 1} }$ are in fact bounded on the event $B_{n, 2}$. We argue for the former (as the arguments for both are similar). Letting $\ell_{\max}$ denote the maximum of the $\ell(\Omega_{ij}, x)$ and $\ell(\widetilde{\Omega}_{ij}, x)$ across $x \in \{0, 1\}$, we can write that  
    \begingroup 
    \allowdisplaybreaks
    \begin{align*}
        | g(\bm{A}_n) | & \leq \ell_{max} \Big( \frac{1}{n^2} \sum_{i \neq j} f_n(\lambda_i, \lambda_j, a_{ij} ) + \frac{1}{n^2} \sum_{i \neq j} \mathbb{E}[ f_n(\lambda_i, \lambda_j, a_{ij} ) \,|\, \bm{\lambda}_n ] \Big) \\
        & \leq \ell_{\max} \big( C_{n, 1}^{1/2} + C_{n, 2}^{1/2} \big) \\
        \implies | g(\bm{A}_n) | 1_{B_{n, 1} } & \leq \ell_{max} \mathbb{E}[f_n^2]^{1/2} ( M_1^{1/2} + M_1^{1/2} M_2^{1/2} ) \text{ on the event } B_{n, 2}
    \end{align*}
    \endgroup
    (where the used Jensen's inequality to obtain the bounds in terms of $C_{n, 1}$ and $C_{n, 2}$). We now work on bounding the variance term. We have that
    \begingroup
    \allowdisplaybreaks
    \begin{align*}
        v(\bm{A}_n \,|\, \bm{\lambda}_n ) & = \frac{1}{2} \mathbb{E}\big[ | \{ g(\bm{A}_n) - g(\bm{A}_n') \} G(\bm{A}_n, \bm{A}_n') |  \,|\, \bm{\lambda_n}, \bm{A}_n \big] \\
        & = \frac{n(n-1)}{4} \mathbb{E}\big[ ( g(\bm{A}_n) - g(\bm{A}_n') )^2 \,|\, \bm{\lambda}_n, \bm{A}_n \big] \\
        & \stackrel{(1)}{=} \frac{1}{2n^4} \sum_{i \neq j} \mathbb{E}\Big[ 
            \big( f_n(\lambda_i, \lambda_j, a_{ij}) \{ \ell(\Omega_{ij}, a_{ij}) - \ell(\widetilde{\Omega}_{ij}, a_{ij}) \} \\ & \qquad \qquad \qquad - f_n(\lambda_i, \lambda_j, a_{ij}') \{ \ell(\Omega_{ij}, a_{ij}') - \ell(\widetilde{\Omega}_{ij}, a_{ij}') \} \big)^2 \,|\, \bm{\lambda}_n, \bm{A}_n, (I, J) = (i, j) \Big]  \nonumber \\
        & \stackrel{(2)}{\leq} \frac{1}{n^2} \Bigg\{ \frac{1}{n^2} \sum_{i \neq j} f_n(\lambda_i, \lambda_j, a_{ij})^2 \big( \ell(\Omega_{ij}, a_{ij}) - \ell(\widetilde{\Omega}_{ij}, a_{ij})  \big)^2  \\ & \qquad \qquad \qquad +  \frac{1}{n^2} \sum_{i \neq j} \mathbb{E}\Big[ f_n(\lambda_i, \lambda_j, a_{ij} )^2 \big( \ell(\Omega_{ij}, a_{ij}) - \ell(\widetilde{\Omega}_{ij}, a_{ij})  \big)^2 \,|\, \bm{\lambda}_n \Big] \Bigg\} \nonumber \\ 
        & \stackrel{(3)}{\leq} \frac{  s_{\ell, \infty}\big(\Omega, \widetilde{\Omega}\big)^2 }{n^2} \Bigg\{ \frac{1}{n^2} \sum_{i \neq j} f_n(\lambda_i, \lambda_j, a_{ij})^2 +  \frac{1}{n^2} \sum_{i \neq j} \mathbb{E}\Big[ f_n(\lambda_i, \lambda_j, a_{ij} )^2 \,|\, \bm{\lambda}_n \Big] \Bigg\} \\
        & = \frac{  s_{\ell, \infty}\big(\Omega, \widetilde{\Omega}\big)^2 }{n^2}  \Big\{ C_{n, 1}(\bm{A}_n, \bm{\lambda}_n) + C_{n, 2}(\bm{\lambda}_n) \Big\}
    \end{align*}
    \endgroup
    (recall the definitions of $C_{n, 1}$ and $C_{n, 2}$ in \eqref{eq:app:loss_converge_proof:cn1} and \eqref{eq:app:loss_converge_proof:cn2} respectively). Here $(1)$ follows via noting that when conditioning on $(I, J)$, only the $(I, J)$ and $(J, I)$ contributions to the summation are non-zero, $(2)$ follows by using the inequality $(a-b)^2 \leq 2(a^2 + b^2)$, and $(3)$ follows via taking the maximum of the loss function differences out of the summation and using the definition of $s_{\ell, \infty}(\cdot, \cdot)$. Now, note that on the event $B_{n, 2}$, we have that 
    \begin{equation*} 
        B_{n, 1} \subseteq \Big\{ v(\bm{A}_n \,|\, \bm{\lambda}_n )\leq  \mathbb{E}[f_n^2] M_1(1  + M_2) n^{-2} s_{\ell, \infty}\big(\Omega, \widetilde{\Omega}\big)^2 \Big\},
    \end{equation*}
    and so by Lemma~\ref{app:loss_converge_proof:exchangeable_pairs_truncation} we get the desired bound. 
\end{proof}

\subsection{Approximation via a SBM} \label{app:loss_converge_proof:sec:sbm_approx}

Now that we know it suffices to examine $\mathbb{E}[ \empriskhat(\bmomega) \,|\, \bm{\lambda}_n ]$, we recall the proof sketch in Section~\ref{sec:proof_sketch}. If the $\tilde{f}_{n}(l, l', x)$ are piecewise constant functions, then this argument shows that we can reason about the distribution of the embedding vectors which lie in some particular regions (namely the sets on which the $\tilde{f}_n(l, l', x)$ are constant). In general, we need to first approximate the $\tilde{f}_{n}(l, l', x)$ by a piecewise constant function, which is possible due to the smoothness assumptions placed on them in Assumption~\ref{assume:samp_weight_reg}. Note that if the $\tilde{f}_n(l, l', x)$ are already piecewise constant, then this section can be skipped.

To formalize this further, we introduce some more notation. Let $\mathcal{P}_n = (A_{n1}, \ldots, A_{n\kappa(n)})$ be a partition of the unit interval $[0, 1]$ into $\kappa(n)$ disjoint intervals, which is a refinement of the partition $\mathcal{Q}$ of $[0, 1]$ specified in Assumption~\ref{assume:samp_weight_reg}. For now we keep $\mathcal{P}_n$ arbitrary; we will later specify the choice of the partition at the end of the proof to optimize the bound we eventually derive. We denote for $n \in \mathbb{N}$, $l \in [\kappa(n)]$
\begin{align*}
    \label{eq:app:loss_converge_proof:pnl_and_friends}
    p_n(l) := |A_{nl}|, \qquad
    \mathcal{A}_n(l) := \{ i \in [n] \,:\, \lambda_i \in A_{nl} \}, \qquad \widehat{p}_n(l) := \frac{1}{n} | \mathcal{A}_n(l) |.
\end{align*}
We now consider the intermediate loss functions 
\begin{align*}
    \mathbb{E}[ \empriskhat^{\mcP_n}(\bmomega) \,|\, \bm{\lambda}_n ] & := \frac{1}{n^2} \sum_{i \neq j} \sum_{x \in \{0, 1\} } \mathcal{P}_n^{\otimes 2}[ \tilde{f}_n(\cdot, \cdot, x) ](\lambda_i, \lambda_j) \ell( B(\omega_i, \omega_j), x), \\
    \mcI_n^{\mcP_n}[K] & := \intsq \sum_{x \in \{0, 1\} } \mcP_n^{\otimes 2}[ \tilde{f}_n(\cdot, \cdot, x) ]( l, l') \ell( K(l, l'), x) \dldl,
\end{align*}
where for any symmetric integrable function $h : [0, 1]^2 \to \mathbb{R}$ we denote
\begin{equation*}
    \mcP_n^{\otimes 2}[ h ](x, y) := \frac{1}{ | A_{nl} | |A_{nl'} | } \int_{A_{nl} \times A_{nl'} } h(u, v) \; du \, dv \qquad \text{ if } (x, y) \in A_{nl} \times A_{nl'}.
\end{equation*}
To bound the approximation error, we use the following result: 

\begin{lemma}[{\citealp[Lemma C.6]{wolfe_nonparametric_2013}}, restated] 
    \label{app:loss_converge_proof:approx_holder_by_stepping}
    Suppose that $h$ is a \linebreak symmetric piecewise H\"{o}lder$([0, 1]^2, \beta, M, \mathcal{Q}^{\otimes 2})$ function, and that $\mcP$ is a partition of $[0, 1]$ which is also a refinement of $\mcQ$. Then we have, for any $q \in [1, \infty]$, 
    %Writing $\mcP = (A_1, \ldots, A_{\kappa} )$, we have that 
    % \begin{equation*}
    %     \| h - \mathcal{P}[h] \|_q \leq M \big( \sqrt{2} \max_{i \in [\kappa]} |A_i| \big)^{\beta}.
    % \end{equation*}
    % for any $q \in [1, \infty]$. If instead $h$ is a symmetric piecewise H\"{o}lder$([0, 1]^2, \beta, M, \mathcal{Q}^{\otimes 2})$ function, then we have that
    \begin{equation*}
        \| h - \mathcal{P}^{\otimes 2}[h] \|_q \leq M \big( \sqrt{2} \max_{i \in [\kappa]} |A_i| \big)^{\beta}
    \end{equation*}
\end{lemma}

\begin{lemma} \label{app:loss_converge_proof:replace_fn_with_pnfn}
    Suppose that Assumptions~\ref{assume:graphon_ass},~\ref{assume:loss},~\ref{assume:bilinear}~and~\ref{assume:samp_weight_reg} hold. Then there exists a non-empty measurable random set $\Psi_n$ such that 
    \begin{equation*}
        \argmin_{\bm{\omega}_n \in (\compactset)^n } \mathbb{E}[\widehat{\mathcal{R}}_n^{\mathcal{P}_n}(\bm{\omega}_n) \,|\, \bm{\lambda}_n ] \cup \argmin_{\bm{\omega}_n \in (\compactset)^n } \mathbb{E}[\widehat{\mathcal{R}}_n(\bm{\omega}_n) \,|\, \bm{\lambda}_n ]  \subseteq \Psi_n
    \end{equation*}
    and 
    \begin{equation*}
        \sup_{\omega_n \in \Psi_n} \Big| \mathbb{E}[\widehat{\mathcal{R}}_n^{\mathcal{P}_n}(\bm{\omega}_n) \,|\, \bm{\lambda}_n ]  - \mathbb{E}[\widehat{\mathcal{R}}_n(\bm{\omega}_n) \,|\, \bm{\lambda}_n ] \Big| = O_p\Big(  \max_{i \in [\kappa(n)]} p_n(i)^{\beta} \cdot \max_{\omega \in S_d }\| \omega \|_2^{2p/\gamma_s}  \Big).
    \end{equation*}
    Similarly, there exists $\Phi_n$ such that 
    \begin{equation*}
        \argmin_{K \in Z(\compactset) } \mathcal{I}_n[K] \cup \argmin_{K \in Z(\compactset)} \mathcal{I}_n^{\mathcal{P}_n}[K] \subseteq \Phi_n
    \end{equation*}
    and 
    \begin{equation*}
        \sup_{K \in \Phi_n} \Big| \mathcal{I}_n[K] - \mathcal{I}_n^{\mathcal{P}_n}[K] \Big| = O\Big( \max_{l \in [\kappa(n)] } p_n(l)^{\beta} \cdot \max_{\omega \in S_d }\| \omega \|_2^{2p/\gamma_s} \Big).
    \end{equation*}
\end{lemma}

\begin{remark}[Minimizers of infinite dimensional functions]
    Note that we have referred to the argmin of $\mathcal{I}_n[K]$ and $\mcI_n^{\mcP_n}[K]$. For $\mcI_n^{\mcP_n}[K]$, the arguments in the next section will reduce this down to a finite dimensional problem, for which showing the existence of a minimizer is straightforward. For $\mathcal{I}_n[K]$, the issue is more technically involved; we show later in Corollary~\ref{app:embed_converge_proof:minimizers} that a minimizer does exist.
\end{remark}

%!TEX root = ..\..\ms.tex

\begin{proof}[Proof of Lemma~\ref{app:loss_converge_proof:replace_fn_with_pnfn}]
    \phantomsection\label{app:loss_converge_proof:replace_fn_with_pnfn:proof}
    For convenience, write $\tilde{f}_{n, x}(l, l') := \tilde{f}_n(l, l', x)$ and $\gamma = \gamma_s$. We detail the proof for the bound on $\mathbb{E}[\widehat{\mathcal{R}}_n^{\mathcal{P}_n}(\bm{\omega}_n) \,|\, \bm{\lambda}_n ]  - \mathbb{E}[\widehat{\mathcal{R}}_n(\bm{\omega}_n) \,|\, \bm{\lambda}_n ]$, as the argument for $\mathcal{I}_n[K] - \mathcal{I}_n^{\mathcal{P}_n}[K]$ works the same way. We now begin by bounding
    \begingroup
    \allowdisplaybreaks
    \begin{align*} 
        \Big| \mathbb{E}[\widehat{\mathcal{R}}_n^{\mathcal{P}_n}(\bm{\omega}_n) \,|\, \bm{\lambda}_n ]  & - \mathbb{E}[\widehat{\mathcal{R}}_n(\bm{\omega}_n) \,|\, \bm{\lambda}_n ] \Big| \\
        & \leq \frac{1}{n^2} \sum_{i \neq j} \sum_{x \in \{0, 1\}} \big| \tilde{f}_{n, x}(\lambda_i, \lambda_j) - \mathcal{P}_n^{\otimes 2}[\tilde{f}_{n, x}](\lambda_i, \lambda_j) \big| \ell( B(\omega_i, \omega_j), x)  \\
        & \leq \frac{1}{n^2} \sum_{i \neq j} \sum_{x \in \{0, 1\} } \| \tilde{f}_{n, x}- \mathcal{P}_n^{\otimes 2}[\tilde{f}_{n, x}]\|_{\infty} \cdot \ell( B(\omega_i, \omega_j), x) \\
        & \leq M \big( \sqrt{2} \max_{i \in [\kappa(n)]} p_n(i) \big)^{\beta} \cdot \frac{1}{n^2} \sum_{i \neq j} \sum_{x \in \{0, 1\} } \ell( B(\omega_i, \omega_j), x)
    \end{align*}
    \endgroup
    where in the last inequality we have used Lemma~\ref{app:loss_converge_proof:approx_holder_by_stepping}. We can then write
    \begingroup 
    \allowdisplaybreaks
    \begin{align}
        \frac{1}{n^2} & \sum_{i \neq j} \sum_{x \in \{0, 1\}} \ell( B(\omega_i, \omega_j, x) ) = \frac{1}{n^2} \sum_{i \neq j} \sum_{x \in \{0, 1\} } \tilde{f}_{n, x}^{-1}(\lambda_i, \lambda_j) \cdot \tilde{f}_{n, x}(\lambda_i, \lambda_j) \ell( B(\omega_i, \omega_j), x) 
        \label{eq:app:loss_converge_proof:div_fn} \\ 
        & \leq \Bigg( \frac{1}{n^2} \sum_{i \neq j} \sum_x \tilde{f}_{n, x}^{-\gamma}(\lambda_i, \lambda_j) \Bigg)^{1/\gamma} \cdot \Bigg[ \frac{1}{n^2} \sum_{i \neq j} \sum_x \big\{ \tilde{f}_{n, x}(\lambda_i, \lambda_j) \ell(B(\omega_i, \omega_j), x) \big\}^{\gamma/(\gamma -1) } \Bigg]^{1 - 1/\gamma}   \nonumber
    \end{align}
    \endgroup
    where we used H\"{o}lder's inequality. We now control the terms in the product. For the first, we note that as we assume that $\sup_{n \geq 1, x \in \{0, 1\} } \mathbb{E}[ \tilde{f}_{n, x}^{-\gamma} ] < \infty $, by Markov's inequality we get that 
    \begin{equation*}
        \label{eq:app:loss_converge_proof:fn_with_pnfn_bound1}
        \Bigg( \frac{1}{n^2} \sum_{i \neq j} \sum_{x \in \{0, 1\}} \tilde{f}_{n, x}^{-\gamma}(\lambda_i, \lambda_j) \Bigg)^{1/\gamma} = O_p(1).
    \end{equation*}
    For the second term, we will use a special case of Littlewood's inequality, which tells us that for $f \in L^1 \cap L^{\infty}$ we have that $\| f \|_p \leq \| f\|_1^{1/p} \|f \|_{\infty}^{1 - 1/p}$ for any $p \in [1, \infty]$; we will apply this to the sequences $f_{i, j, x} = \tilde{f}_{n, x}(\lambda_i, \lambda_j) \ell( B(\omega_i, \omega_j), x)$ and use the $\ell_1$ and $\ell_{\infty}$ norms on this sequence. If we assume the $\bmomega$ are such that we have the $\ell_1$ bound
    \begin{equation}
        \label{eq:app_loss_converge_proof:fn_with_pnfn_lossbound}
        \frac{1}{n^2} \sum_{i \neq j} \sum_{x \in \{0, 1\}} \tilde{f}_{n, x}(\lambda_i, \lambda_j) \ell(B(\omega_i, \omega_j), x) \leq C \mathbb{E}[\widehat{\mathcal{R}}_n(\bm{0}) \,|\, \bm{\lambda}_n ] 
    \end{equation}
    for some constant $C > 1$, then as we also have the $\ell_{\infty}$ bound (where we write $\tilde{f}_n = \tilde{f}_{n, 1} + \tilde{f}_{n, 0}$)
    \begingroup 
    \begin{align*}
        \max_{i \neq j} \max_{x \in \{0, 1\} } \tilde{f}_{n, x}(\lambda_i, \lambda_j) \ell(B(\omega_i, \omega_j), x)  & \leq \| \tilde{f}_n \|_{\infty} \max_{\omega, \omega' \in S_d} \max_{x \in \{0, 1\} } \ell(B(\omega_i, \omega_j), x) \nonumber \\
        & \leq \| \tilde{f}_n \|_{\infty} \lossboundc( \lossbounda + \max_{\omega \in S_d }\| \omega \|_2^{2p} )
    \end{align*}
    \endgroup
    it follows by Littlewood's inequality with $p = \gamma/(\gamma - 1)$ that
    \begin{align*}
        \Bigg[ \frac{1}{n^2} \sum_{i \neq j } \sum_x \big\{ \tilde{f}_{n, x}(\lambda_i, \lambda_j) & \ell(B(\omega_i, \omega_j), x) \big\}^{\gamma/(\gamma -1) } \Bigg]^{1 - 1/\gamma}  \\
        &  \leq C'
       \Big( \mathbb{E}[\widehat{\mathcal{R}}_n(\bm{0}) \,|\, \bm{\lambda}_n ]  \Big)^{ 1 - 1/\gamma } \cdot \max_{\omega \in S_d }\| \omega \|_2^{2p/\gamma}
    \end{align*}
    where $C'$ is some constant free of $n$. As $\| \tilde{f}_{n, x} \|_1 = O(1)$, by Markov's inequality we have that $\mathbb{E}[\widehat{\mathcal{R}}_n(\bm{0}) \,|\, \bm{\lambda}_n ] = O_p(1)$; it therefore follows that for any $\bmomega$ for which \eqref{eq:app_loss_converge_proof:fn_with_pnfn_lossbound} is satisfied, we have that 
    \begin{equation}
        \label{eq:loss_converge_proof:replace_fn_with_pnfn:bound}
        \Big| \mathbb{E}[\widehat{\mathcal{R}}_n^{\mathcal{P}_n}(\bm{\omega}_n) \,|\, \bm{\lambda}_n ]  - \mathbb{E}[\widehat{\mathcal{R}}_n(\bm{\omega}_n) \,|\, \bm{\lambda}_n ] \Big| = O_p\Big( \max_{l \in [\kappa(n)] } p_n(l)^{\beta} \cdot \max_{\omega \in S_d }\| \omega \|_2^{2p/\gamma} \Big),
    \end{equation}
    with the bound holding uniformly over such $\bmomega$. To conclude, note that when dividing and multiplying by $\tilde{f}_{n, x}$ in the argument in \eqref{eq:app:loss_converge_proof:div_fn}, we could have also done so with $\mcP_n^{\otimes 2}[\tilde{f}_{n, x} ]$ and have the same argument apply, due to the fact that
    \begin{equation*}
        \| \mcP_n^{\otimes 2}[ \tilde{f}_{n, x} ]^{-1} \|_{\gamma} \leq \| \tilde{f}_{n, x}^{-1} \|_{\gamma} \qquad \text{ and } \qquad \mathbb{E}\Big[ \mathbb{E}[\widehat{\mathcal{R}}_n^{\mathcal{P}_n}(\bm{0}) \,|\, \bm{\lambda}_n ] \Big] = \mathbb{E}[\widehat{\mathcal{R}}_n(\bm{0}) \,|\, \bm{\lambda}_n ].
    \end{equation*}
    (The first inequality is by Lemma~\ref{app:loss_converge_proof:step_fn_p_norm}.) Consequently, it therefore follows that if we define 
    \begin{equation*}
        \Psi_n = \big\{ \bm{\omega}_n \,:\, \mathbb{E}[\widehat{\mathcal{R}}_n^{\mathcal{P}_n}(\bm{\omega}_n) \,|\, \bm{\lambda}_n ] \leq C\mathbb{E}[\widehat{\mathcal{R}}_n^{\mathcal{P}_n}(\bm{0}) \,|\, \bm{\lambda}_n ] \text{ or } \mathbb{E}[\widehat{\mathcal{R}}_n(\bm{\omega}_n) \,|\, \bm{\lambda}_n ] \leq C \mathbb{E}[\widehat{\mathcal{R}}_n(\bm{0}) \,|\, \bm{\lambda}_n ] \big\}
    \end{equation*}
    for any fixed constant $C > 1$, we get that the bound derived in \eqref{eq:loss_converge_proof:replace_fn_with_pnfn:bound} holds uniformly across all such $\bmomega \in \Psi_n$, and so the stated result holds.
\end{proof}

 % \todo{Notes to Morgane: There's no mathematical reason that I can see which would prevent assuming that the $f_n W$ are Holder per piece - it'd just be a matter of being a bit more careful about how things are stated. I can add it in at some point for sake of generality. As for your comment about using the stepping operator on the $f_n W$ only, yes this is intentional. Although the bound could be improved by bounding the $\ell_1$ norm of $\mathcal{L}_n(l, l', x) = \mathbb{E}^{\zeta(l) \otimes \zeta(l')}[ \ell_n(\omega_1, \omega_2, x)]$ minus $\mathcal{P}_n[\mathcal{L}_n(l, l', x) ]$, the only thing I can say a-priori about $\mathcal{L}_n(l, l', x)$ is that it is measurable. While in that case it is possible to get bounds on $\| \mathcal{L}_n - \mathcal{P}_n[\mathcal{L}_n] \|_1$, they are restrictive in terms of the permissible partitions without much gain in the rate. In this scenario, all I can say is that there exists a partition $\mathcal{P}_n$ such that the $\ell_1$ norm would be bounded by $B_n (\log n)^{-1/2}$. As far as I can tell, it isn't straightforward to be able to control the minimum and maximum partition size from the proof - it's somewhere in Lovasz's graph limits book if you want to have a look - which is important as these are quantities which we need for later bounds.}

\subsection{Adding in the diagonal term} 

Here we show that the effect of changing the sum in $\mathbb{E}[ \empriskhat^{\mcP_n}(\bmomega) \,|\, \bm{\lambda}_n ]$ from one over all $i \neq j$ with $i, j \in [n]$, to one over all pairs $(i, j) \in [n]^2$, is asymptotically negligible. 

\begin{lemma} \label{app:loss_converge_proof:add_diag_term}
    Define the function
    \begin{equation*}
        \mathbb{E}[ \empriskhat^{\mcP_n, (1)}(\bmomega) \,|\, \bm{\lambda}_n ] := \frac{1}{n^2} \sum_{i, j} \sum_{x \in \{0, 1\}} \mcP_n^{\otimes 2}[\tilde{f}_{n, x}](\lambda_i, \lambda_j) \ell( B(\omega_i, \omega_j), x)
    \end{equation*}
    and suppose that Assumptions~\ref{assume:loss},~\ref{assume:bilinear}~and~\ref{assume:samp_weight_reg} hold. Recalling that $p \geq 1$ is the growth rate of the loss function $\ell(y, x)$, we then have that
    \begin{equation*}
        \sup_{\bm{\omega}_n \in (\compactset)^n} \big| \mathbb{E}[\empriskhat^{\mcP_n, (1)}(\bmomega) \,|\, \bm{\lambda}_n ] - \mathbb{E}[\empriskhat^{\mcP_n}(\bmomega) \,|\, \bm{\lambda}_n ] \big| = O\Big(  \frac{1}{n} \sup_{\omega \in \compactset} \| \omega_i \|_2^{2p} \Big).
    \end{equation*}
\end{lemma}

%!TEX root = ..\..\ms.tex

\begin{proof}[Proof of Lemma~\ref{app:loss_converge_proof:add_diag_term}]
    \phantomsection\label{app:loss_converge_proof:add_diag_term:proof}
    Note that $\mathbb{E}[\empriskhat^{\mcP_n, (1)}(\bmomega) \,|\, \bm{\lambda}_n ] - \mathbb{E}[\empriskhat^{\mcP_n}(\bmomega) \,|\, \bm{\lambda}_n ] \geq 0$ for all $\bm{\omega}_n$, so we work on showing an upper bound on this quantity. Writing $\tilde{f}_n(l, l') = \fnone + \fnzero$, note that as $\sup_{n \geq 1} \| \tilde{f}_{n}(\cdot, \cdot) \|_{\infty} < \infty$, we also have that $\sup_{n \geq 1} \| \mcP_n^{\otimes2}[ \tilde{f}_n(\cdot, \cdot)] \|_{\infty} < \infty$, and therefore 
    \begin{align*} 
         \mathbb{E}[ & \empriskhat^{\mcP_n, (1)}(\bmomega) \,|\, \bm{\lambda}_n ]  - \mathbb{E}[ \empriskhat^{\mcP_n}(\bmomega) \,|\, \bm{\lambda}_n ]  = \frac{1}{n^2} \sum_{i \in [n]} \sum_{x \in \{0, 1\} } \mcP_n^{\otimes2}[\tilde{f}_n](\lambda_i, \lambda_i, x) \ell(B(\omega_i, \omega_i), x)  \\
        & \leq \frac{ \| \mcP_n^{\otimes2}[\tilde{f}_n(\cdot, \cdot)] \|_{\infty} }{n^2} \sum_{i \in [n] } \sum_{x \in \{0, 1\} } \ell(B(\omega_i, \omega_i), x)  \\
        & \leq \frac{\| \mcP_n^{\otimes2}[\tilde{f}_n(\cdot, \cdot)] \|_{\infty} }{n^2} \sum_{i \in [n] } \lossboundc( \lossbounda + \| \omega_i \|_2^{2p} ) \leq O\Big(  \frac{1}{n} \sup_{\omega \in \compactset} \| \omega_i \|_2^{2p} \Big).
    \end{align*}
    Here we have used that $|B(\omega_i, \omega_i) | \leq \| \omega_i \|_2^2$, which holds regardless of whether $B(\cdot, \cdot)$ in Assumption~\ref{assume:bilinear} is a regular inner product, or a Krein inner product. As the RHS above is free of $\bmomega$, we get the claimed result.
    %therefore have uniform control of $\mathbb{E}[\empriskhat^{(1)}(\bmomega) \,|\, \bm{\lambda}_n ] - \mathbb{E}[\empriskhat(\bmomega) \,|\, \bm{\lambda}_n ]$ over $\bmomega \in (\compactset)^n$, giving the desired result.
\end{proof}

As this is a minor change to the loss function, from now on we will just rewrite
\begin{equation} \label{eq:app:loss_converge_proof:loss_with_diag}
    \mathbb{E}[ \empriskhat^{\mcP_n}(\bmomega) \,|\, \bm{\lambda}_n ] := \frac{1}{n^2} \sum_{i, j} \sum_{x \in \{0, 1\}} \mcP_n^{\otimes2}[\tilde{f}_n](\lambda_i, \lambda_j, x ) \ell( B(\omega_i, \omega_j), x).
\end{equation}
rather than explicitly writing a superscript $(1)$ each time.

\subsection{Linking minimizing embedding vectors to minimizing kernels}

With this, we now note that we can write 
\begin{equation}
    \label{eq:loss_coverge_proof:stepped_emprisk}
    \mathbb{E}[\widehat{\mathcal{R}}_n^{\mathcal{P}_n}(\bm{\omega}_n) \,|\, \bm{\lambda}_n ] = \sum_{l, l' \in [\kappa(n)]} \widehat{p}_n(l) \widehat{p}_n(l') \sum_{x \in \{0, 1\} } \Big\{ \frac{ c_n(l, l', x) }{ | \mathcal{A}_n(l) | | \mathcal{A}_n(l') | } \sum_{ \substack{i \in \mathcal{A}_n(l) \\ j \in \mathcal{A}_n(l') }  } \ell( B(\omega_i, \omega_j), x) \Big\}
\end{equation}
where 
\begin{equation*}
    c_n(l, l', x) := \frac{1}{p_n(l) p_n(l') } \int_{ A_{nl} \times A_{nl'} } \tilde{f}_n(\lambda, \lambda', x) \, d\lambda d\lambda'
\end{equation*}
and we recall that $\widehat{p}_n(l) = n^{-1} | \mcA_n(l) |$. In order to minimize $\mathbb{E}[\widehat{\mathcal{R}}_n^{\mathcal{P}_n}(\bm{\omega}_n) \,|\, \bm{\lambda}_n ]$, we can exploit the strict convexity of the $\ell(\cdot, x)$ and the bilinearity of the $B(\omega_i, \omega_j)$ in order to simplify the optimization problem.

\begin{lemma} \label{app:loss_converge_proof:embed_vector_avg_1}
    Suppose that Assumption~\ref{assume:loss},~\ref{assume:bilinear}~and~\ref{assume:samp_weight_reg} hold. Moreover suppose that the partition $\mcP_n$ used to define the above loss functions satisfies $\min_{l \in [\kappa(n)] } p_n(l) = \omega( \log(n) /n )$. Then minimizing $\mathbb{E}[\widehat{\mathcal{R}}_n^{\mathcal{P}_n}(\bm{\omega}_n) \,|\, \bm{\lambda}_n ]$ over $\bm{\omega}_n \in (\compactset)^n$ for a closed, convex and non-empty subset $\compactset \subseteq \mathbb{R}^d$ is equivalent to minimizing
    \begin{equation}
        \label{eq:app:loss_converge_proof:embed_vector_avg_1}
        \widehat{I}_n^{\mathcal{P}_n}[ \Omega] := \sum_{l, l' \in [\kappa(n)] } \widehat{p}_n(l) \widehat{p}_n(l') \sum_{x \in \{0, 1\} } c_n(l, l', x) \ell( \Omega_{l, l'}, x)
    \end{equation}
    where $\Omega_{l, l'} = B(\tilde{\omega}_l, \tilde{\omega}_{l'} )$ with the $\tilde{\omega}_{l} \in \compactset$ for $l \in [\kappa(n)]$, i.e $\Omega \in Z_{\kappa(n)}( \compactset )$, whose notation we recall from \eqref{eq:app:loss_converge_proof:znsd}). Moreover, if $\bm{\omega}_n$ is a minimizer of $\mathbb{E}[\widehat{\mathcal{R}}_n^{\mathcal{P}_n}(\bm{\omega}_n) \,|\, \bm{\lambda}_n ]$, then there must exist vectors $\tilde{\omega}_l \in S_d$ for $l \in [\kappa(n)]$ such that 
    \begin{equation*}
        B(\omega_i, \omega_j) = B(\tilde{\omega}_l, \tilde{\omega}_{l'} ) \text{ for all } (i, j) \in \mathcal{A}_n(l) \times \mathcal{A}_n(l').
    \end{equation*}
\end{lemma}

\begin{proof}[Proof of Lemma~\ref{app:loss_converge_proof:embed_vector_avg_1}]
    To ease on notation, write $\ell_x(\cdot) = \ell(\cdot, x)$ for $x \in \{0, 1\}$. Note that by Jensen's inequality and the bilinearity of $B(\cdot, \cdot)$, we have that for all $l, l' \in [\kappa(n)]$, $x \in \{0, 1\}$, that
    \begin{align*}
        \frac{1}{|\mathcal{A}_n(l)| |\mathcal{A}_n(l')| } \sum_{i \in \mathcal{A}_n(l)} \sum_{j \in \mathcal{A}_n(l')} \ell_x( B(\omega_i, \omega_j) ) & \geq \ell_x\Big( \frac{1}{|\mathcal{A}_n(l)| |\mathcal{A}_n(l')| } \sum_{i \in \mathcal{A}_n(l)} \sum_{j \in \mathcal{A}_n(l')} B(\omega_i, \omega_j) \Big) \\
        & = \ell_x\Big( B\Big( \frac{1}{|\mathcal{A}_n(l)| } \sum_{i \in \mathcal{A}_n(l)} \omega_i, \frac{1}{ |\mathcal{A}_n(l')| } \sum_{j \in \mathcal{A}_n(l')} \omega_j \Big) \Big).
    \end{align*}
    Moreover, as $\ell_x(\cdot)$ is strictly convex, note that the above inequality is an equality (for a fixed $l, l' \in [\kappa(n)]$), if and only if $B(\omega_i, \omega_j)$ is constant for all $(i, j) \in \mathcal{A}_n(l) \times \mathcal{A}_n(l')$. As by Assumption~\ref{assume:samp_weight_reg} we may deduce that $c_n(l,l',x) > 0$ for all $l, l' \in [\kappa(n)]$ (as $\fnone$ and $\fnzero$ are positive a.e) and $x \in \{0, 1\}$, it follows that if we define 
    \begin{equation*}
        \bm{\omega}_n^{\mathcal{A}_n} = \Big( \omega_j^{\mathcal{A}_n} := \frac{1}{|\mathcal{A}_n(l)|} \sum_{i \in \mathcal{A}_n(l) } \omega_i \text{ if } j \in \mathcal{A}_n(l) \Big)_{j \in [n]}
    \end{equation*}
    (note that as $\compactset$ is convex, the averages also lie within $\compactset$), then we have that $$\mathbb{E}[\widehat{\mathcal{R}}_n^{\mathcal{P}_n}(\bm{\omega}_n) \,|\, \bm{\lambda}_n ] \geq \mathbb{E}[\widehat{\mathcal{R}}_n^{\mathcal{P}_n}(\bm{\omega}_n^{\mathcal{A}_n}) \,|\, \bm{\lambda}_n ]$$ with equality iff $B(\omega_i, \omega_j)$ is equal across $(i, j) \in \mcA_n(l) \times \mcA_n(l')$, for all pairs of $l, l' \in [\kappa(n)]$. (Note that the above average is well defined as $\min_{l \in [\kappa(n)] } |\mcA_{n}(l) | \to \infty$ as $n \to \infty$ by Lemma~\ref{app:loss_converge_proof:mnom_min}, due to the condition on the sizes of the partitioning sets of $\mcP_n$.)
    
    We can then observe that $\mathbb{E}[\widehat{\mathcal{R}}_n^{\mathcal{P}_n}(\bm{\omega}_n^{\mathcal{A}_n}) \,|\, \bm{\lambda}_n ]$ is equivalent to $\widehat{I}_n^{\mathcal{P}_n}[\Omega]$ (where $\Omega_{l, l'} = B(\tilde{\omega}_l, \tilde{\omega}_{l'})$) via the correspondence
    \begin{align*}
        (\omega_1, \ldots, \omega_n) & \longrightarrow \tilde{\omega}_l := \frac{1}{|\mathcal{A}_n(l)|} \sum_{i \in \mathcal{A}_n(l) } \omega_i, \\ 
        (\tilde{\omega}_l \,:\, l \in [\kappa(n)] ) & \longrightarrow \text{ any } (\omega_1, \ldots, \omega_n) \text{ with } \tilde{\omega}_l = \frac{1}{|\mathcal{A}_n(l)|} \sum_{i \in \mathcal{A}_n(l) } \omega_i.
    \end{align*}
    Moreover, we know that $\mathbb{E}[\widehat{\mathcal{R}}_n^{\mathcal{P}_n}(\bm{\omega}_n) \,|\, \bm{\lambda}_n ] = \mathbb{E}[\widehat{\mathcal{R}}_n^{\mathcal{P}_n}(\bm{\omega}_n^{\mathcal{A}_n}) \,|\, \bm{\lambda}_n ]$ if and only if $B(\omega_i, \omega_j)$ is constant on each block $(i, j) \in \mathcal{A}_n(l) \times \mathcal{A}_n(l')$. It therefore follows that if $\bm{\omega}_n$ is a minimizer of $\mathbb{E}[\widehat{\mathcal{R}}_n^{\mathcal{P}_n}(\bm{\omega}_n) \,|\, \bm{\lambda}_n ]$, then this must be the case. As $B(\cdot, \cdot)$ is bilinear, this implies that
    \begin{equation*}
        B(\omega_i, \omega_j) := B\Big( \frac{1}{|\mathcal{A}_n(l)|} \sum_{i_1 \in \mathcal{A}_n(l) } \omega_{i_1}, \frac{1}{|\mathcal{A}_n(l')|} \sum_{j_1 \in \mathcal{A}_n(l') } \omega_{j_1} \Big) \text{ for } (i, j) \in \mathcal{A}_n(l) \times \mathcal{A}_n(l'),
    \end{equation*}
    so if we write $\tilde{\omega}_l$ as according to the above correspondence, we get the last part of the lemma statement. 
\end{proof}

As we can similarly write
\begin{equation}
    \label{eq:loss_converge_proof:stepped_in}
    I_n^{\mathcal{P}_n}[K] = \sum_{l, l' \in [\kappa(n)] } p_n(l) p_n(l') \sum_{x \in \{0, 1\} } \frac{ c_n(l, l', x)  }{ p_n(l) p_n(l') } \int_{A_{nl} \times A_{nl'} } \ell( K(\lambda, \lambda'), x) \, d \lambda d \lambda',
\end{equation}
via essentially the same argument, we get the following:

\begin{lemma} \label{app:loss_converge_proof:embed_vector_avg_2}
    Suppose that Assumption~\ref{assume:loss},~\ref{assume:bilinear}~and~\ref{assume:samp_weight_reg} hold. Then minimizing 
    \begin{equation*}
        \mathcal{I}_n^{\mathcal{P}_n}[K] = \sum_{l, l' \in [\kappa(n)] } p_n(l) p_n(l') \sum_{x \in \{0, 1\} } \frac{ c_n(l, l', x)  }{ p_n(l) p_n(l') } \int_{A_{nl} \times A_{nl'} } \ell( K(\lambda, \lambda'), x) \, d \lambda d \lambda',
    \end{equation*}
    over $K \in Z(\compactset)$ - where $\compactset \subseteq \mathbb{R}^d$ is closed, convex and non-empty, and we recall the definition of $Z(\compactset)$ from Equation~\eqref{eq:loss_converge:K_minima_set} - is equivalent to minimizing
    \begin{equation}
        \label{eq:app:loss_converge_proof:embed_vector_avg_2}
        I_n^{\mathcal{P}_n}[\Omega] = \sum_{l, l' \in [\kappa(n)] } p_n(l) p_n(l') \sum_{x \in \{0, 1\} } c_n(l, l', x) \ell( \Omega_{l, l'}, x) 
    \end{equation}
    over $\Omega \in Z_{\kappa(n)}(\compactset)$. Moreover, if $K \in Z(\compactset)$ is a minimizer of $\mathcal{I}_n^{\mathcal{P}_n}[K]$, then $K$ must be of the form (up to a.e equivalence) $K(\lambda, \lambda') = B(\eta(\lambda), \eta(\lambda'))$ for $\eta : [0, 1] \to \compactset$ which is piecewise constant on the $A_{nl}$. 
\end{lemma}

\begin{proof}[Proof of Lemma~\ref{app:loss_converge_proof:embed_vector_avg_2}]
    Note that similar to before, as we can write $K(\lambda, \lambda') = B( \eta(\lambda), \eta(\lambda') )$ for some functions $\eta(l) : [0, 1] \to \compactset$, we have that 
    \begin{align*}
        \frac{ 1  }{ p_n(l) p_n(l') } \int_{A_{nl} \times A_{nl'} } &\ell( K(\lambda, \lambda'), x) \, d \lambda d \lambda' \\
        & \geq \ell\Big( B\Big( \frac{1}{p_n(l)} \int_{A_{nl}} \eta(\lambda) \,d \lambda  , \frac{1}{p_n(l')} \int_{A_{nl'}} \eta(\lambda') \,d \lambda'  \Big), x \Big),
    \end{align*}
    where there is equality if and only $K(\lambda, \lambda')$ is constant on $A_{nl} \times A_{nl'}$ for every $l, l' \in [\kappa(n)]$. With this, the proof follows essentially identically to that of Lemma~\ref{app:loss_converge_proof:embed_vector_avg_1}.
\end{proof}

Note that by having done this, we have managed to place the problems of minimizing the functions $\mathbb{E}[\widehat{\mathcal{R}}_n^{\mathcal{P}_n}(\bm{\omega}_n) \,|\, \bm{\lambda}_n ]$ (Equation~\ref{eq:loss_coverge_proof:stepped_emprisk}) and $\mathcal{I}_n^{\mathcal{P}_n}[K]$ (Equation~\ref{eq:loss_converge_proof:stepped_in}) - the latter an infinite dimensional problem, the former $nd$ dimensional - into a common domain of optimization, from which we can compare the two. Looking at $\widehat{I}_n^{\mathcal{P}_n}[\Omega]$ and $I_n^{\mathcal{P}_n}[\Omega]$ for $\Omega \in Z_{\kappa(n)}(\compactset)$, it follows that the only remaining step is to replace the instances of $\widehat{p}_n(l)$ with $p_n(l)$ in order for us to be done:  

\begin{lemma} \label{app:loss_converge_proof:replace_pnhat_with_pn}
    Recall the definitions of $\widehat{I}_n^{\mcP_n}[\Omega]$ and $I_n^{\mcP_n}[\Omega]$ in \eqref{eq:app:loss_converge_proof:embed_vector_avg_1} and \eqref{eq:app:loss_converge_proof:embed_vector_avg_2} respectively. Then there exists a non-empty measurable random set $\Phi_n$ such that 
    \begin{equation*}
        \mathbb{P}\Big(   \argmin_{ \Omega \in Z_{\kappa(n)}(\compactset) } I_n^{\mathcal{P}_n}[ \Omega ] \cup  \argmin_{ \Omega \in Z_{\kappa(n)}(\compactset) } \widehat{I}_n^{\mathcal{P}_n}[ \Omega ] \subseteq \Phi_n \Big) \to 1
    \end{equation*}
    and 
    \begin{equation*}
        \sup_{\Omega \in \Phi_n} \big|  I_n^{\mathcal{P}_n}[ \Omega ] - \widehat{I}_n^{\mathcal{P}_n}[ \Omega ] \big| = O_p\Big(  \Big( \frac{ \log \kappa(n) }{ n  \min_{i \in [\kappa(n)]} p_n(i) } \Big)^{1/2}     \Big).
    \end{equation*}
\end{lemma}

\begin{proof}[Proof of Lemma~\ref{app:loss_converge_proof:replace_pnhat_with_pn}]
    For this, begin by observing that we have
    \begin{equation*}
        \big| I_n^{\mathcal{P}_n}[ \Omega ] - \widehat{I}_n^{\mathcal{P}_n}[ \Omega ] \big| \leq \max_{l, l' \in [\kappa(n)] } \frac{ | \widehat{p}_n(l) \widehat{p}_n(l') - p_n(l) p_n(l') |} { p_n(l) p_n(l') } \cdot I_n^{\mathcal{P}_n}[ \Omega ],
    \end{equation*}
    where as a consequence of Proposition~\ref{app:loss_converge_proof:mnomconc} we have that 
    \begin{equation*}
        \max_{l, l' \in [\kappa(n)] } \frac{ | \widehat{p}_n(l) \widehat{p}_n(l') - p_n(l) p_n(l') |} { p_n(l) p_n(l') } = O_p\Big(  \Big( \frac{ \log \kappa(n) }{ n  \min_{i \in [\kappa(n)]} p_n(i) } \Big)^{1/2}     \Big).
    \end{equation*}
    With this, the proof is similar to Lemma~\ref{app:loss_converge_proof:replace_prob_with_fn}, and so we skip repeating the details. 
\end{proof}

\subsection{Obtaining rates of convergence}

To get the bounds stated in Theorem~\ref{app:loss_converge_proof:main_theorem}, we collect and chain up the previously obtained bounds from the earlier parts. Noting that the bounds are stated in terms of suprema over sets $\Psi$ containing all the minimizers (or do so with asymptotic probability $1$), we can bound the difference in the minimal values by the supremum of the difference of the functions over $\Psi$. Indeed, suppose we have two functions $f$ and $g$ such that all the minima of $f$ and $g$ lie within a set $X$ with asymptotic probability $1$; letting $x_f$ and $x_g$ be some minima of these sets, we therefore get that on an event of asymptotic probability $1$ that
\begin{equation*}
    \min_{x} f(x) - \min_{x} g(x) = f(x_f) - g(x_g) \leq f(x_g) -  g(x_g) \leq \sup_{x \in X} | f(x) - g(x) |,
\end{equation*}
and via a similar argument for $\min_x g(x) - \min_x f(x)$ we get that 
\begin{equation*}
    \big| \min_x f(x)  - \min_x g(x) \big| \leq \sup_{x \in X} \big| f(x) - g(x) \big|.
\end{equation*}
With this in mind, we now seek to apply the results developed earlier. To do so, we need to make a choice of a sequence of partitions $\mcP_n$. To do so, we make a choice so that the $p_n(l) = \Theta( n^{-\alpha })$ uniformly over $l \in [\kappa(n)]$, and that they each are a refining partition of the partition $\mcQ$ from Assumption~\ref{assume:graphon_ass}. (This is possible simply by dividing each $Q \in \mcQ$ into intervals of the same size, each of order $n^{-\alpha}$.) Recall the notation $S_d = [-A, A]^d$; $Z(\compactset)$ from Equation~\ref{eq:loss_converge:K_minima_set}; and $Z_n(S_d)$ from Equation~\ref{eq:app:loss_converge_proof:znsd}. 
It therefore follows by collating the terms from, respectively, Lemma~\ref{app:loss_converge_proof:replace_prob_with_fn}; Theorem~\ref{app:loss_converge_proof:average_over_adjacency} + Lemma~\ref{app:loss_converge_proof:metric_entropy_znd}; Lemma~\ref{app:loss_converge_proof:replace_fn_with_pnfn}; Lemma~\ref{app:loss_converge_proof:add_diag_term}; Lemma~\ref{app:loss_converge_proof:embed_vector_avg_1}; Lemma~\ref{app:loss_converge_proof:replace_pnhat_with_pn}; Lemma~\ref{app:loss_converge_proof:embed_vector_avg_2}; and Lemma~\ref{app:loss_converge_proof:replace_fn_with_pnfn} (again), we end up with a bound of the form
\begingroup 
\allowdisplaybreaks
\begin{align}
    \Big| \min_{\bm{\omega}_n \in (\compactset)^n } & \mathcal{R}_n(\bm{\omega}_n) - \min_{K \in Z(\compactset) } \mathcal{I}_n[K] \Big| \nonumber \\
    & \leq \Big|  \min_{\bm{\omega}_n \in (\compactset)^n } \mathcal{R}_n(\bm{\omega}_n) -  \min_{\bm{\omega}_n \in (\compactset)^n } \widehat{\mathcal{R}}_n(\bm{\omega}_n) \Big| \\
    & + \Big|  \min_{\bm{\omega}_n \in (\compactset)^n } \widehat{\mathcal{R}}_n(\bmomega) -  \min_{\bm{\omega}_n \in (\compactset)^n } \mathbb{E}[ \widehat{\mathcal{R}}_n(\bmomega) \,|\, \bm{\lambda}_n ] \Big| \nonumber \\
    & + \Big|  \min_{\bm{\omega}_n \in (\compactset)^n } \mathbb{E}[\widehat{\mathcal{R}}_n(\bm{\omega}_n) \,|\, \bm{\lambda}_n ] -  \min_{\bm{\omega}_n \in (\compactset)^n } \mathbb{E}[\widehat{\mathcal{R}}_n^{\mathcal{P}_n}(\bm{\omega}_n) \,|\, \bm{\lambda}_n ]  \Big| \nonumber \\
    & + \Big|   \min_{\bm{\omega}_n \in (\compactset)^n } \mathbb{E}[\empriskhat^{\mcP_n}(\bmomega) \,|\, \bm{\lambda}_n ] -  \min_{\bm{\omega}_n \in (\compactset)^n } \mathbb{E}[\empriskhat^{\mcP_n, (1)}(\bmomega) \,|\, \bm{\lambda}_n ] \Big| \nonumber \\
    & + \Big| \min_{\bm{\omega}_n \in (\compactset)^n } \mathbb{E}[\empriskhat^{\mcP_n, (1)}(\bmomega) \,|\, \bm{\lambda}_n ]  - \min_{ \Omega \in \mcZ_{\kappa(n) }(\compactset) } \widehat{I}_n^{\mathcal{P}_n}[ \Omega ] \Big| \nonumber \\
    & + \Big| \min_{ \Omega \in \mcZ_{\kappa(n) }(\compactset) } \widehat{I}_n^{\mathcal{P}_n}[ \Omega ] - \min_{ \Omega \in \mcZ_{\kappa(n) }(\compactset) } I_n^{\mathcal{P}_n}[ \Omega ] \Big| \nonumber \\ 
    & + \Big| \min_{ \Omega \in \mcZ_{\kappa(n) }(\compactset) } I_n^{\mathcal{P}_n}[ \Omega ] -  \min_{K \in Z(\compactset) } \mathcal{I}_n^{\mathcal{P}_n}[K] \Big| + \Big| \min_{K \in Z(\compactset) } \mathcal{I}_n^{\mathcal{P}_n}[K] -  \min_{K \in Z(\compactset) } \mathcal{I}_n[K] \Big|   \\
    & = O_p\Big( s_n + \frac{d^{p+1/2} \mathbb{E}[f_n^2]^{1/2} }{n^{1/2} } +   \frac{d^p}{n} +  n^{-\alpha \beta} d^{p/\gamma_s} + \frac{ (\log n )^{1/2} }{ n^{1/2 - \alpha/2 } } \Big).
\end{align}
\endgroup

The remaining task is to balance the embedding dimension $d$ and the size of $\alpha$ in order to optimize the bound; to begin, the $d^p/n$ term is always negligible (as it is dominated by the $d^{p+1/2} \mathbb{E}[f_n^2]^{1/2} n^{-1/2}$ term). We note that when $\gamma_s = \infty$ (so the $d^{p/\gamma_s}$ term disappears), we want to balance the $n^{-\alpha \beta}$ and $n^{-1/2 + \alpha/2}$ bounds to be equal, leading to a choice of $\alpha = 1/(1 + 2\beta)$ to give an optimal bound. When $\gamma_s \in (1, \infty)$, we choose the same value of $\alpha$; we note that we can still have a bound which is $o_p(1)$ for $d = n^c$ for some sufficiently small $c = c(p, \beta, \gamma_s, \mathbb{E}[f_n^2])$. In the case where the $\tilde{f}_{n, x}$ are piecewise constant on a partition $\mcQ^{\otimes 2}$ where $\mcQ$ is of size $\kappa$, the $n^{-\alpha \beta}$ term disappears (as we no longer need to perform the piecewise approximation step given by Lemma~\ref{app:loss_converge_proof:add_diag_term} and can just have that $\mcP_n = \mcQ$ for all $n$). Consequently, the bound from Lemma~\ref{app:loss_converge_proof:replace_fn_with_pnfn} becomes $(\log \kappa / n)^{1/2}$, from which the claimed result follows.

\subsection{Proof for higher dimensional graphons}
\label{sec:app:loss_converge_proof:higher_dim}

\begin{proof}[Proof of Theorem~\ref{thm:loss_converge:high_dim}]
    Note that in following the proof argument above, the details depend only on that the $\lambda_i$ are drawn i.i.d, and does not require a particular form of the distribution, and so the result follows immediately.
\end{proof}

\subsection{Additional lemmata}

%!TEX root = ..\..\ms.tex

%We begin by discussing how we can obtain bounds on $\gamma_2(T, s)$ for a metric space $(T, s)$. Note that we have the simply lower bound $\Delta(T, s) \leq \gamma_2(T, s)$, and moreover we can obtain upper and lower bounds on $\gamma_2(T, s)$ via the metric entropy 
%\begin{equation*}
%    \frac{1}{C} \sup_{\epsilon > 0} \epsilon \sqrt{ \log N(T, s, \epsilon) } \leq \gamma_2(T, s) \leq C \int_0^{\infty} \sqrt{ \log N(T, s, \epsilon) } \, d \epsilon
%\end{equation*}
%where $C > 0$ is a universal constant and $N(T, s, \epsilon)$ is the minimal size of an $\epsilon$-covering of $T$ with respect to the metric $s$ (see e.g \citep[Chapter~2]{talagrand_upper_2014}). When $T \subseteq \mathbb{R}^m$, $\gamma_2(T, s)$ can only improve on the above entropy bound by a factor of $\log(m)$ (e.g of \citep[ Exercise 2.3.4]{talagrand_upper_2014}), and so for our applications this is excusable, as we can only lose at most a factor of $\log(nd)$ as $Z_{n}(\compactset) \subseteq \mathbb{R}^{nd}$ whenever $\compactset \subseteq \mathbb{R}^d$, and thus will only effect our overall bound up to logarithmic factors in $n$.

% We now consider $\gamma_2(Z_n(\compactset), s_{\ell, \infty})$ for $\compactset = [-A, A]^d$; we note that due to the growth conditions on the loss function, we have $\Delta(Z_n(\compactset), s_{\ell, \infty}) \geq (A^2 d)^p$, which therefore also gives a lower bound on $\gamma_2(Z_n(\compactset), s_{\ell, \infty})$. For an upper bound, we have the following:

\begin{lemma} \label{app:loss_converge_proof:metric_entropy_znd}
    Suppose that Assumptions~\ref{assume:loss}~and~\ref{assume:bilinear} hold, where $p \geq 1$ is the growth rate of the loss function, and let $\compactset = [-A, A]^d$ for some $A > 0$. Then there exists some universal constant $C > 0$ such that 
    \begin{equation*}
        \gamma_2( Z_n(\compactset), s_{\ell, \infty} ) \leq C A^{2p + 1} d^{p + 1/2} n^{1/2}.
    \end{equation*}
\end{lemma}

\begin{proof}[Proof of Lemma~\ref{app:loss_converge_proof:metric_entropy_znd}]
    We begin by upper bounding $s_{\ell, \infty}$ by a metric which is easier to work with. Using the fact that $\ell(y, x)$ is locally Lipschitz, we have that 
    \begin{align*}
        s_{\ell, \infty}(K, \widetilde{K}) & = \max_{i, j \in [n]}  \max_{x \in \{0, 1\}} \{ | \ell(K_{ij}, x) - \ell(\widetilde{K}_{ij}, x) | \}  \\
        & \leq \losslipconst \max_{i, j \in [n]} \max\{ |K_{ij}|^{p-1}, |\widetilde{K}_{ij}|^{p-1} \} \cdot | K_{ij} - \widetilde{K}_{ij} | \\
        & \leq \losslipconst \max\{ \| \widetilde{K} \|_{\infty}^{p-1} , \| K \|_{\infty}^{p-1} \} \| K - \widetilde{K} \|_{\infty} \leq \losslipconst (A^2 d)^{p-1}  \| K - \widetilde{K} \|_{\infty}.
    \end{align*}
    To handle the $\| K - \widetilde{K} \|_{\infty}$ term, recall that as $K_{ij} = B(\omega_i, \omega_j)$ and $\widetilde{K}_{ij} = B(\widetilde{\omega}_i, \widetilde{\omega}_j )$ for $\omega_i, \widetilde{\omega}_i \in \compactset$, we have that when $B(\omega, \omega') = \langle \omega, \omega' \rangle$ we can bound
    \begin{align*}
        \max_{i, j \in [n] } | \langle \omega_i, \omega_j \rangle - \langle \widetilde{\omega}_i, \widetilde{\omega}_j \rangle | & \leq \max_{i, j \in [n] } | \langle \omega_i - \widetilde{\omega}_i, \omega_j \rangle | + | \langle \widetilde{\omega}_i, \omega_j - \widetilde{\omega}_j \rangle | \\
        & \leq \Big( \max_{i \in [n] } \| \omega_i \|_1 + \max_{i \in [n]} \| \widetilde{\omega}_i \|_1 \Big) \cdot \max_{i \in [n] } \| \omega_i -\widetilde{\omega}_i \|_{\infty} \\ 
        & \leq 2 A^2 d  \max_{i \in [n] } \| \omega_i - \widetilde{\omega}_i \|_{\infty}.
    \end{align*}
    where we used the triangle inequality followed by H\"{o}lder's inequality. We can achieve the same bound when $B(\omega, \omega') = \langle \omega, \mathrm{diag}(I_{d_1}, - I_{d - d_1} ) \omega' \rangle$, by using the triangle inequality to bound
    \begin{equation*}
        | B(\omega, \omega') | \leq | \langle \omega_{[1:d_1]}, \omega_{[1:d_1]}' \rangle | + | \langle \omega_{[(d_1+1):d]}, \omega_{[(d_1 + 1):d]}' \rangle |
    \end{equation*}
    and then by applying the above argument twice. It therefore follows that in either case, letting $B(\ell_{nd}^{\infty}, A)$ denote the set $x \in \mathbb{R}^{nd}$ such that $\| x \|_{\infty} \leq A$, we have the bound
    \begin{equation*}
        \gamma_2(Z_n(\compactset), s_{\ell, \infty} ) \leq 2 \losslipconst (A^2 d)^{p} \gamma_2( B(\ell_{nd}^{\infty}, A), \| \cdot \|_{\infty} ).
    \end{equation*}
    This is because when we have two metrics $s$ and $s'$ such that $s \leq C s'$, the corresponding $\gamma_2$-functionals satisfy $\gamma_2(s) \leq C \gamma_2(s')$ \citep[Exercise~2.2.20]{talagrand_upper_2014}. The RHS is then straightforward to bound by Remark~\ref{app:loss_converge_proof:rmk:gamma_2}; note that 
    \begin{equation*}
        N( B(\ell_{nd}^{\infty}, A), \| \cdot \|_{\infty}, \epsilon ) = \Big( \frac{ 2A }{ \epsilon } \Big)^{nd}
    \end{equation*}
    and therefore 
    \begin{equation*}
        \int_0^{\infty} \sqrt{ \log N( B(\ell_{\infty}^{nd}, A), \| \cdot \|_{\infty}, \epsilon ) } \, d \epsilon \leq n^{1/2} d^{1/2} \int_0^{2A} \sqrt{  \log( 2A / \epsilon ) } \, d \epsilon = 2A \pi^{1/2} n^{1/2} d^{1/2}.
    \end{equation*}
    Combining everything gives the desired result.
\end{proof}

\begin{lemma} \label{app:loss_converge_proof:mnomconc_2}
    Let $X_n = (X_{n1}, \ldots, X_{nm}) \sim \tfrac{1}{n} \mathrm{Multinomial}(n ; p_{n})$ where the $p_{ni} > 0$, $\sum_{i=1}^m p_{ni} = 1$, $m = m(n) \to \infty$ and $n p_{n(1)}/\log(m) \to \infty$, where $p_{n(1)}$ is the minimum of the $p_{ni}$ over $i \in [m]$. Then we have that 
    \begin{equation*}
        \max_{i \in [m] } \Big| \frac{ X_{ni} - p_{ni}}{ p_{ni}} \Big| = O_p\Big(  \sqrt{ \frac{ \log m  }{ n p_{n(1)} } } \Big)
    \end{equation*}
\end{lemma}

\begin{proof}[Proof of Lemma~\ref{app:loss_converge_proof:mnomconc_2}]
    We suppress the subscript $n$ in the $X_{ni}$ and $p_{ni}$ for the proof.   Recall that $X_i \sim \tfrac{1}{n} B(n, p_i)$. By e.g \citet[Exercise 2.3.5]{vershynin_high-dimensional_2018}, for all $\epsilon \in (0, 1)$ we have that
    \begin{equation*}
        \mathbb{P}\Big( | X_i - p_i| > \epsilon p_i ) = \mathbb{P}\Big( | nX_i - np_i| > \epsilon n p_i \Big) \leq 2 \exp( - c n p_i \epsilon^2),
    \end{equation*}
    for some absolute constant $c > 0$. Therefore, by taking a union bound we get that
    \begin{align*}
        \mathbb{P}\Big( \max_{i \in [m] } \Big| \frac{ X_i - p_i}{ p_i} \Big| > \epsilon) & \leq \sum_{i=1}^m \mathbb{P}\Big( | X_i - p_i | > \epsilon p_i \Big) \\
        & \leq \sum_{i=1}^m 2 \exp( - c n \epsilon^2 p_i) \leq 2m \exp( - c n p_{(1)} \epsilon^2 ).
    \end{align*}
    In particular, given any $\delta > 0$, if we take $\epsilon = ( A \log(m) / n p_{(1)} )^{1/2}$ (which will lie in $(0, 1)$ for any fixed $A$ once $n$ is large enough), then 
    \begin{equation*}
        \mathbb{P}\Big( \max_{i \in [m] } \Big| \frac{ X_i - p_i}{ p_i} \Big| > \Big( \frac{ A \log(m) }{n p_{(1)} } \Big)^{1/2} \Big) \leq 2 e^{(1- cA)\log(m) } < \delta
    \end{equation*}
    if e.g $A = 2/c$ and $m(n) \geq 2/\delta$. The stated conclusion therefore follows.
\end{proof}

\begin{lemma} \label{app:loss_converge_proof:mnom_min}
    Let $X_n = (X_{n1}, \ldots, X_{nm}) \sim \mathrm{Multinomial}(n; p)$ with the same conditions on the $p_{ni}$ as in Lemma~\ref{app:loss_converge_proof:mnomconc_2}, and write $p_{n(m)}$ for the maximum of the $p_{ni}$ over $i \in [m]$. Then we have that
    \begin{equation*}
        \min_{i \in [m]} X_i \geq n p_{(1)} -  O_p\Big( \sqrt{n p_{(m)} \log(2m)} \Big).
    \end{equation*}
    In particular, if the $p_{ni} = \Theta(n^{-\alpha} )$ for some $\alpha \in (0, 1)$ so $m = \Theta(n^{\alpha})$, then $\min_{i \in [m] } X_i = \Omega_p( n^{1 - \alpha} )$, so $\min_{i \in [m] } X_i \cvp \infty$ as $n \to \infty$.
    %\begin{equation*}
    %    \min_{i \in [m]} X_i \gtrsim n^{1- \alpha} - O_p\Big( \sqrt{n^{1 - \alpha} \log(n) }   \Big) \cvp \infty \text{ as } n \to \infty.
    %\end{equation*}
\end{lemma}

\begin{proof}[Proof of Lemma~\ref{app:loss_converge_proof:mnom_min}]
    Again, we suppress the subscript $n$ in the $X_{ni}$ and $p_{ni}$ for the proof. Begin by noting that if $(a_i)_{i \in [m]}$ is a sequence of real numbers, then for all $j \in [m]$ we have that
    \begin{equation*}
        a_j + \max_{i} |a_i| \geq a_j + |a_j| \geq 0 \implies \min_{i \in [m]} a_i \geq - \max_{i \in [m]} |a_j|.
    \end{equation*}
    As a consequence we therefore have that (writing $X_i = \mathbb{E}[X_i] + X_i - \mathbb{E}[X_i]$)
    \begin{equation*}
        \min_{i \in [m]} X_i \geq \min_{i \in [m]} \mathbb{E}[X_i] + \min_{i \in [m]} (X_i - \mathbb{E}[X_i]) \geq n p_{(1)} - \max_{i \in [m]} \Big| X_i - n p_i \Big|
    \end{equation*}
    and so we can just apply the bound derived in Lemma~\ref{app:loss_converge_proof:mnomconc_2}.
\end{proof}

\begin{proposition} \label{app:loss_converge_proof:mnomconc}
    Let $X_n = (X_{n1}, \ldots, X_{nm}) \sim \frac{1}{n} \mathrm{Multinomial}(n, p)$, where $m = m(n) \to \infty$, $p_{n(1)}$ is the minimum of the $p_{ni}$ and $(n p_{(1)})/\log(m) \to \infty$. Then we have that
    \begin{equation*}
        \max_{i, j \in [m]} \frac{| X_{ni} X_{nj} - p_{ni} p_{nj} |}{p_{ni} p_{nj} } = O_p\left( \sqrt{ \frac{\log m}{n p_{n(1)} }   } \right).
    \end{equation*}
    In particular, if $p_{ni} = \Theta(n^{-\alpha})$ then 
    \begin{equation*}
        \max_{i, j \in [m]} \frac{| X_{ni} X_{nj} - p_{ni} p_{nj}|}{p_{ni} p_{nj}}  = O_p\Big( \frac{ \sqrt{\log n} }{  n^{1/2 - \alpha/2}}    \Big).
    \end{equation*}
    In the regime where $m$ and $p$ are fixed, we recover the standard $O_p(\tfrac{1}{\sqrt{n}})$ rate.
\end{proposition}

\begin{proof}[Proof of Proposition~\ref{app:loss_converge_proof:mnomconc_2}]
    Again, we suppress the subscript $n$ in the $X_{ni}$ and $p_{ni}$ for the proof. By the triangle inequality we have that 
    % \begin{equation*}
    %     \frac{| X_i X_j - p_i p_j|}{p_i p_j} \leq \frac{|X_i|}{p_i} \frac{|X_j - p_j|}{p_j} + \frac{|X_i - p_i|}{p_i}
    % \end{equation*}
    % and therefore 
    \begin{equation*}
        \max_{i, j \in [m] } \frac{| X_i X_j - p_i p_j|}{p_i p_j} \leq \max_{i \in [m] } \frac{|X_i|}{p_i} \max_{j \in [m] } \frac{|X_j - p_j|}{p_j} + \max_{i \in [m] } \frac{|X_i - p_i|}{p_i}.
    \end{equation*}
    As we can bound 
    \begin{equation*}
        \max_{i \in [m] } \frac{|X_i|}{p_i} \leq \max_{i \in [m] } \frac{|X_i - p_i| + p_i}{p_i} = 1 + \max_{i \in [m] } \frac{|X_i - p_i|}{p_i} = O_p(1)
    \end{equation*}
    by Lemma~\ref{app:loss_converge_proof:mnomconc_2}, using this again and the above inequality gives the desired result. 
\end{proof}

\begin{lemma}[Cauchy's third inequality] \label{app:loss_converge_proof:cti}
    Let $(a_k)_{k \geq 1}$, $(b_k)_{k \geq 1}$ and $(c_k)_{k \geq 1}$ be sequences of positive numbers. Then
    \begin{equation*}
        \min_{k \leq n} \frac{a_k}{b_k} \leq \frac{a_1c_1 + \cdots + a_n c_n}{b_1 c_1 + \cdots + b_n c_n} \leq \max_{k \leq n} \frac{a_k}{b_k}.
    \end{equation*}
\end{lemma}

\begin{proof}[Proof of Lemma~\ref{app:loss_converge_proof:cti}]
    This follows by writing
    \begin{equation*}
        \frac{a_1c_1 + \cdots + a_nc_n}{b_1c_1 + \cdots + b_n c_n} = \frac{ b_1 c_1 \big(\tfrac{a_1}{b_1} \big) + \cdots + b_n c_n \big(\tfrac{a_n}{b_n} \big)}{b_1 c_1 + \cdots + b_n c_n} 
    \end{equation*}
    and then applying the inequalities
    \begin{equation*}
        \min_{k \leq n} \frac{a_k}{b_k} \sum_{i=1}^n b_i c_i \leq \sum_{i=1}^n \frac{a_i}{b_i} b_i c_i \leq \max_{k \leq n} \frac{a_k}{b_k} \sum_{i=1}^n b_i c_i
    \end{equation*}
    and rearranging.
\end{proof}
%!TEX root = ..\..\ms.tex

\begin{lemma} \label{app:loss_converge_proof:fn_op1}
    Suppose $( g_n(\lambda_1, \lambda_2, a_{12}) )_{n \geq 1}$ is a sequence of integrable non-negative functions, where $\lambda_i \iid \mathrm{Unif}[0, 1]$ and $a_{ij} \,|\, \lambda_i, \lambda_j  \sim \mathrm{Bernoulli}(W_n(\lambda_i, \lambda_j))$. Then
    \begin{align*}
        X_n & := \frac{1}{n^2} \sum_{i \neq j} g_n(\lambda_i, \lambda_j, a_{ij}) = O_p( \mathbb{E}[g_n ] ), \\
        \mathbb{E}[X_n | \bm{\lambda}_n ] & := \frac{1}{n^2} \sum_{i \neq j} g_n(\lambda_i, \lambda_j, 1) W_n(\lambda_i, \lambda_j) + g_n(\lambda_i, \lambda_j, 0) (1 - W_n(\lambda_i, \lambda_j)) = O_p( \mathbb{E}[g_n] ). 
    \end{align*}
    % and 
    % \begin{equation*}
    %     \mathbb{E}[X_n | \bm{\lambda} ] := \frac{1}{n^2} \sum_{(i, j) \in \Delta_n} g_n(\lambda_i, \lambda_j, 1) W(\lambda_i, \lambda_j, 1) + g_n(\lambda_i, \lambda_j, 0) (1 - W(\lambda_i, \lambda_j, 1)) = O_p( \mathbb{E}[g_n(\lambda_1, \lambda_2, a_{12} ) ] ).
    % \end{equation*}
\end{lemma}

\begin{proof}[Proof of Lemma~\ref{app:loss_converge_proof:fn_op1}]
    Note that as the quantities are identically distributed sums over $n(n-1) \leq n^2$ quantities, we have 
    \begin{equation*}
        \mathbb{E}[ \mathbb{E}[X_n | \lambda_1, \ldots, \lambda_n] ] = \mathbb{E}[X_n] \leq \mathbb{E}[ g_n(\lambda_1, \lambda_2, a_{12} )] < \infty,
    \end{equation*}
    so the desired conclusions follow via an application of Markov's inequality (as the $g_n$ are non-negative, so are $X_n$ and $\mathbb{E}[X_n | \bm{\lambda}]$). 
\end{proof}

\begin{lemma} \label{app:loss_converge_proof:step_fn_p_norm}
    Suppose that $\mcP = (A_1, \ldots, A_{\kappa} )$ is a partition of $[0, 1]$, and $f : [0, 1]^2 \to \mathbb{R}$ is a function such that $f > 0$ a.e and $f^{-1} \in L^p([0, 1]^2)$. Then $\mcP^{\otimes 2}[f]^{-1} \in L^p([0, 1]^2)$, and in fact $\| \mcP^{\otimes 2}[f]^{-1} \|_p \leq \| f \|_p$. 
 \end{lemma}

 \begin{proof}[Proof of Lemma~\ref{app:loss_converge_proof:step_fn_p_norm}]
    We write
    \begin{align*}
        \| \mcP^{\otimes 2}[f]^{-1} \|_p^p & = \sum_{l, l' \in [\kappa] } |A_l| |A_{l'}| \cdot \Big( \frac{1}{ | A_l| |A_{l'} | } \int_{A_l \times A_{l'} } f \, d\mu \Big)^{-p} \\
         & \leq \sum_{l, l' \in [\kappa] } |A_l| |A_{l'} | \cdot \frac{1}{ |A_l| |A_{l'}| } \int_{A_l \times A_{l'} } f^{-p} \, d \mu = \|f^{-1} \|_p^p,
    \end{align*}
    where the second line follows by using Jensen's inequality applied to the function $x \mapsto x^{-p}$.
 \end{proof}
%!TEX root = ..\ms.tex

\section{Proof of Theorems~\ref{thm:embed_learn:converge_1}~-~\ref{thm:embed_learn:converge_3}} \label{sec:app:embed_converge_proof}

We break this section up into four parts. The first discusses properties of the $\mcI_n[K]$ we will need (such as convexity and continuity), the second considers minimizers of $\mcI_n[K]$ over particular subsets of functions, and the third examines lower and upper bounds to the difference in values of $\mcI_n[K]$ when minimized over different sets. These are then combined together to talk about the embedding vectors learned by $\emprisk(\bmomega)$, and comparing this to a suitable minimizer of $\mcI_n[K]$. 

\subsection{Properties of \texorpdfstring{$\mathcal{I}_n[K]$}{In[K]}}
\label{sec:app:embed_converge_proof:properties_of_mci}

We begin with proving various properties of $\mcI_n[K]$ which will be necessary in order to talk about constrained optimization of this function.

%!TEX root = ..\..\ms.tex

\begin{lemma} \label{app:embed_converge_proof:convexity}
    Suppose that Assumptions~\ref{assume:loss}~and~\ref{assume:samp_weight_reg} hold. Then $\mathcal{I}_n[K]$ is strictly convex on the set of $K$ for which $\mcI_n[K] < \infty$.
\end{lemma}

\begin{proof}[Proof of Lemma~\ref{app:embed_converge_proof:convexity}]
    Without loss of generality we may just consider the case where $K_1$, $K_2$ are not equal almost everywhere, so the set 
    \begin{equation*}
        A := \big\{ \llp \in [0, 1]^2 \,:\, K_1\llp \neq K_2 \llp \big\}
    \end{equation*}
    has positive Lebesgue measure. Now, letting $t \in (0, 1)$ be fixed, via strictly convexity of the loss function, we have that 
    \begin{equation*}
        E_{t, x}[K_1, K_2]\llp := t \ell( K_1\llp, x) + (1-t) \ell(K_2\llp, x) - \ell( tK_1\llp + (1-t)K_2\llp, x) > 0
    \end{equation*}
    on the set $A$ for $x \in \{0, 1\}$, and that it equals zero on the set $A^c$. As the $\tilde{f}_n(l, l', x)$ are positive a.e, it therefore follows that $E_{t, x}[K_1, K_2]\llp \tilde{f}_n(l, l', x)$ is strictly positive on $A$ and zero on $A^c$, and consequently 
    \begin{align*} 
        t\mcI_n[K_1] & + (1-t)\mcI_n[K_2] - \mcI_n[tK_1 + (1-t)K_2] \\
        & = \Big( \int_A + \int_{A^c} \Big) \sum_{x \in \{0, 1\} } E_{t, x}[K_1, K_2]\llp \tilde{f}_n(l, l', x) \, dl dl' > 0
    \end{align*}
    giving the desired conclusion.
\end{proof}
%!TEX root = ..\..\ms.tex

\begin{lemma} \label{app:embed_converge_proof:level_set}
    Suppose that Assumptions~\ref{assume:loss}~and~\ref{assume:samp_weight_reg} hold with $p \geq 1$ as the growth rate of the loss function and $\gamma_s = \infty$. For convenience denote $\tilde{f}_{n, x} = \tilde{f}_n(l, l', x)$. Then $\mathcal{I}_n[K] < \infty$ if and only if $K \in L^p([0, 1]^2)$. Moreover, we have that 
    \begin{equation*}
        \mathcal{I}_n[K] \leq C_1 \mathcal{I}_n[0] \implies \| K \|_{p}^p \leq \lossbounda + \lossboundc C_1 \big( \max_{x \in \{0, 1\} } \| \tilde{f}_{n, x}^{-1} \|_{\infty} \big)^{-1} \cdot \mcI_n[0].
    \end{equation*}
\end{lemma}

\begin{proof}[Proof of Lemma~\ref{app:embed_converge_proof:level_set}]
    Note that the $\tilde{f}_{n, x}$ are assumed to be bounded away from zero as $\gamma_s = \infty$, uniformly so by $\delta_f = (\sup_{n, x} \| \tilde{f}_{n, x}^{-1} \|_{\infty} )^{-1}$, and also are assumed to be bounded above, say by $\fnboundabove = \sup_{n, x} \| \tilde{f}_{n, x} \|_{\infty}$. To obtain the upper bound, we use the growth assumptions on the loss function to give 
    \begin{align*}
        \mcI_n[K] \leq \fnboundabove \intsq \{ \fnone + \fnzero \} \, dl dl' \leq \lossboundc \fnboundabove \intsq \big( | K\llp |^p + \lossbounda \big) \, dl dl', 
    \end{align*}
    and similarly for the lower bound we find that 
    \begin{align*}
       \mcI_n[K] \geq \delta_f \intsq \{ \ell(K(l, l'), 1) + \ell(K(l, l'), 0) \} \, dl dl'  \geq \frac{ \delta_f }{ \lossboundc} \intsq  \big( | K\llp |^p - \lossbounda \big) \, dl dl',
    \end{align*}
    giving the first part of the theorem statement. The second part then follows by using the second inequality and rearranging. 
\end{proof}
%!TEX root = ..\..\ms.tex

\begin{lemma} \label{app:embed_converge_proof:local_lipschitz}
    Suppose that Assumption~\ref{assume:loss} holds, where $p \geq 1$ denotes the growth rate of the loss function. Then $\mathcal{I}_n[K]$ is locally Lipschitz on $L^{rp}([0, 1]^2)$ for any $r \geq 1$ in the following sense: if $K_1$, $K_2 \in L^{rp}([0, 1]^2)$, then 
    \begin{align*}
        \big| \mcI_n[K_1] & - \mcI_n[K_2] \big| \leq \losslipconst \| \tilde{f}_n \|_{r/(r-1)}  \big( \| K_1 \|_{rp} + \|K_2 \|_{rp}  \big)^{p-1} \|K_1 - K_2 \|_{rp},
    \end{align*}
    where $\tilde{f}_n\llp = \fnone + \fnzero$. In particular, $\mathcal{I}_n[K]$ is uniformly continuous on bounded sets in $L^p([0, 1]^2)$.
\end{lemma}

\begin{proof}[Proof of Lemma~\ref{app:embed_converge_proof:local_lipschitz}]
    Note that by the (local) Lipschitz property of the loss function $\ell(y, \cdot)$, we have that 
    \begin{align*}
        \big| \ell( K_1\llp, x ) - \ell( K_2\llp, x) \big| \leq \losslipconst \max\{ |K_1\llp|, |K_2\llp | \}^{p-1} | K_1\llp - K_2\llp |
    \end{align*}
    for $x \in \{0, 1\}$, and therefore via the triangle inequality we obtain the bound
    \begin{align*}
        \big| \mcI_n[K_1] & - \mcI_n[K_2] \big| \\
        & \leq \losslipconst \intsq \fnsum \big( |K_1\llp| + |K_2\llp| \big)^{p-1} \big| K_1\llp - K_2\llp \big| \dldl.
    \end{align*}
    Applying the generalized H\"{o}lder's inequality with exponents $r/(r-1)$, $rp/(p-1)$ and $rp$ to each of the three products in the above integral respectively then gives that 
    \begin{align*}
        \big| \mcI_n[K_1] & - \mcI_n[K_2] \big| \leq \losslipconst \| \tilde{f}_n \|_{r/(r-1)}  \big( \| K_1 \|_{rp} + \|K_2 \|_{rp}  \big)^{p-1} \|K_1 - K_2 \|_{rp}
    \end{align*}
    as claimed. 
\end{proof}
%!TEX root = ..\..\ms.tex

\begin{proposition} \label{app:embed_converge_proof:mcI_diff}
    Suppose that Assumption~\ref{assume:loss} holds, where $p \geq 1$ denotes the growth rate of the loss function. Then $\mcI_n[K]$ is Gateaux differentiable on $L^p([0, 1]^2)$ with derivative 
    \begin{align*}
        d \mcI_n[K; H] & = \lim_{s \to 0} \frac{1}{s} \big( \mcI_n[K + sH] - \mcI_n[K] \big) \\
        & = \intsq \big\{ \fnone \ell'( K\llp, 1) + \fnzero \ell'(K\llp, 0) \big\} H\llp \dldl
    \end{align*}
    where $\ell'(y, x) := \tfrac{d}{dy} \ell(y, x)$. In particular, $\mathcal{I}_n[K]$ is subdifferentiable with sub-derivative
    \begin{equation*}
            \partial \mcI_n[K] = \fnone \ell'( K\llp, 1) + \fnzero \ell'(K\llp, 0).
    \end{equation*}
\end{proposition}

\begin{proof}[Proof of Proposition~\ref{app:embed_converge_proof:mcI_diff}]
    For the Gateaux differentiability, we begin by noting that if $K \in L^p([0, 1]^2)$, then $|K|^{p-1} \in L^{p/(p-1)}([0, 1]^2)$, and therefore by the assumed growth condition on the first derivatives of $\ell(y, x)$, it follows that $d \mcI_n[K ; H]$ is well-defined by H\"{o}lder's inequality. Writing 
    \begin{align*}
        \Big| \frac{1}{s} \big( \mcI_n[K &+ sH]  - \mcI_n[K] \big) - \intsq \sum_{x \in \{0, 1\} } \tilde{f}_{n}(l, l, x) \ell'( K\llp, x)  H\llp \dldl \Big|  \\
         & \leq \intsq \sum_{x \in \{0, 1\} } \tilde{f}_n(l, l', x) \Big| \frac{1}{s} \big\{ \ell(K\llp + sH\llp, x) - \ell(K\llp, x)  \big\} \\
         & \qquad \qquad \qquad \qquad \qquad \qquad \qquad  - H\llp \ell'(K\llp, x) \Big| \dldl,
    \end{align*}
    we note that the integrand converges to zero pointwise when $s \to 0$ as $\ell(y, x)$ is differentiable. Moreover, as
    \begin{equation*}
        | \ell( K\llp + s H\llp, x) - \ell( K\llp, x) | \leq s |H\llp| |\ell'(K\llp, x) |,
    \end{equation*}
    by the mean value inequality the integrand is dominated by 
    \begin{equation*}
        C \fnsum |H\llp| \big( a + |K\llp|^{p-1} \big) 
    \end{equation*}
    which is integrable. The dominated convergence theorem therefore gives the first part of the proposition statement. The second part therefore follows by using the fact that $\mathcal{I}_n[K]$ is convex and Gateaux differentiable, hence the sub-gradient is simply the Gateaux derivative \citep[e.g][Proposition~2.40]{barbu_convexity_2012}.
\end{proof}

\subsection{Minimizers of \texorpdfstring{$\mathcal{I}_n[K]$}{In[K]} over \texorpdfstring{$Z(S_d)$}{Z(Sd)} and related sets} 
\label{sec:app:embed_converge_proof:props of sets}

Recall that we earlier denoted 
\begin{equation*}
    Z(\compactset) = \big\{ K(l, l') = B(\eta(l), \eta(l') ) \text{ where } \eta: [0, 1] \to \compactset \big\}
\end{equation*}
with an implicit choice of the similarity measure $B(\omega, \omega')$, and $\compactset = [-A, A]^d$ for some $A > 0$ and $d \in \mathbb{N}$. To distinguish between using the regular and indefinite/Krein inner product, we define the following sets, for $d, d_1, d_2 \in \mathbb{N}$ and $A > 0$:
\begin{align*} 
    \mathcal{Z}_d^{\geq 0}(A) & := \big\{ \text{functions } K\llp = \langle \eta(l), \eta(l) \rangle \,\mid \, \eta: [0, 1] \to [-A, A]^d \big\} \\
    \mathcal{Z}_{fr}^{\geq 0} & = \mathcal{Z}_{fr}^{\geq 0}(A) := \bigcup_{d = 1}^{\infty} \mcZ^{\geq 0}_d(A), \qquad \mathcal{Z}^{\geq 0} = \mcZ^{\geq 0}(A) := \cl\big( \mathcal{Z}_{fr}^{\geq 0}(A) \big), \\
    \mathcal{Z}_{d_1, d_2}(A) & := \mathcal{Z}_{d_1}^{\geq 0} - \mathcal{Z}_{d_2}^{\geq 0} \\
    & = \big\{ \text{functions } K\llp = \langle \eta_1(l), \eta_1(l) \rangle - \langle \eta_2(l), \eta_2(l') \rangle \, \mid \, \eta_i : [0, 1] \to [-A, A]^{d_i} \big\} \\
    \mathcal{Z}_{fr} & = \mathcal{Z}_{fr}(A) := \bigcup_{d_1, d_2 = 1}^{\infty} \mathcal{Z}_{d_1, d_2}(A), \qquad \mathcal{Z} = \mcZ(A) := \cl\big( \mathcal{Z}_{fr}(A) \big).
\end{align*}
Here the closures are taken with respect to the weak topology on $L^p([0, 1]^2)$ (see Appendix~\ref{sec:app:convex_opt}), for the value of $p$ corresponding to that of the loss function in Assumption~\ref{assume:loss}. We note that the sets $\mathcal{Z}_{fr}^{\geq 0}(A)$, $\mcZ^{\geq 0}(A)$, $\mathcal{Z}_{fr}(A)$ and $\mcZ(A)$ are all independent of $A > 0$ as a result of the lemma below, whence why e.g the equalities $\mcZ^{\geq 0} = \mcZ^{\geq 0}(A)$ and $\mcZ = \mcZ(A)$ are written above.

\begin{lemma} \label{app:embed_converge_proof:mcZ_free_of_A}
    For all $d \in \mathbb{N}$ and $A > 0$ we have that $\mcZ^{\geq 0}_d(A) \subset \mcZ^{\geq 0}_d(2A) \subset \mcZ^{\geq 0}_{4d}(A)$. Consequently, the sets $\mcZ_{fr}^{\geq 0}(A)$ and $\mcZ^{\geq 0}(A)$ are independent of the choice of $A > 0$. Similarly, the sets $\mcZ_{fr}(A)$ and $\mcZ(A)$ are independent of the choice of $A > 0$.
\end{lemma}

\begin{proof}[Proof of Lemma~\ref{app:embed_converge_proof:mcZ_free_of_A}]
    We give the argument for the non-negative definite case as the other case follows with the same style of argument. The first inclusion is immediate. For the second, suppose $K \in \mcZ_d^{\geq 0}(2A)$, so we have a representation 
    \begin{equation*}
        K(l, l') = \sum_{i=1}^d \eta_i(l) \eta_i(l') \text{ where } \eta_i : [0, 1] \to [-2A, 2A].
    \end{equation*}
    Then as we can equivalently write this as 
    \begin{equation*}
        K(l, l') = \sum_{i=1}^d \Big( \underbrace{\frac{1}{2} \eta_i(l) \cdot \frac{1}{2} \eta_i(l') + \cdots + \frac{1}{2} \eta_i(l) \cdot \frac{1}{2} \eta_i(l')}_{\text{repeated four times}} \Big) 
    \end{equation*}
    with $\tfrac{1}{2} \eta_i : [0, 1] \to [-A, A]$, we have that $K \in \mcZ^{\geq 0}_{4d}(A)$, and so get the second inclusion. We therefore have that $\mcZ_{fr}^{\geq 0}(A) = \mcZ^{\geq 0}_{fr}(2A)$; as one naturally has the inclusion that $\mcZ_{fr}^{\geq 0}(A) \subset \mcZ_{fr}^{\geq 0}(A')$ for all $A < A'$, it follows that the sets $\mcZ_{fr}^{\geq 0}(A)$ are equal for all $A$, and so the same holds for the closures of these sets.
\end{proof}

From now onwards, we will always drop the dependence of $A$ from the sets $\mcZ_{fr}^{\geq 0}(A)$, $\mcZ^{\geq 0}(A)$, $\mcZ_{fr}(A)$ and $\mcZ(A)$, and only refer to $\mcZ_{fr}^{\geq 0}$, $\mcZ^{\geq 0}$, $\mcZ_{fr}$ and $\mcZ$ onwards respectively. 

%!TEX root = ..\..\ms.tex

\begin{lemma} \label{app:embed_converge_proof:minima_set_convex}
    The sets $\mathcal{Z}^{\geq 0}_{fr}$ and $\mathcal{Z}_{fr}$ are convex, and therefore their weak and norm closures in $L^p([0, 1]^2)$ coincide. Moreover, the sets $\mathcal{Z}^{\geq 0}$ and $\mathcal{Z}$ are convex.
\end{lemma}

\begin{proof}[Proof of Lemma~\ref{app:embed_converge_proof:minima_set_convex}]
    The style of argument is essentially the same for both cases, so we focus on $\mathcal{Z}_{fr}^{\geq 0}$ and $\mathcal{Z}^{\geq 0}$. Note that for any $t \in (0, 1)$ we have that $$t \mathcal{Z}_d^{\geq 0}(A) \subseteq \mathcal{Z}_d^{\geq 0}(A) \qquad \text{ and } \qquad \mathcal{Z}_{d_1}^{\geq 0}(A) + \mathcal{Z}_{d_2}^{\geq 0}(A) = \mcZ_{d_1 + d_2}^{\geq 0}(A).$$ It therefore follows that $\mcZ^{\geq 0}_{fr}$ is a convex set. A standard fact from functional analysis (see Appendix~\ref{sec:app:convex_opt}) then says that convex sets are norm closed iff they are weakly closed. Moreover, as the norm closure of a convex set is convex, we also get that $\mathcal{Z}^{\geq 0}$ is a convex set too. 
\end{proof}

\begin{remark}We note that while $\mcZ^{\geq 0}_{fr}(A)$ is a convex set, the sets $\mcZ_d^{\geq 0}(A)$ for $d > 0$ are not convex. This is analogous to how the set of $n \times n$ matrices of rank $r < n$ is not convex.\end{remark}

%!TEX root = ..\..\ms.tex

\begin{proposition} \label{app:embed_converge_proof:minima_set_compact}
    The sets $\mathcal{Z}_d^{\geq 0}(A)$ and $\mathcal{Z}_{d_1, d_2}(A)$ are weakly compact in $L^p([0, 1]^2)$ for $p \geq 1$ and any $A > 0$, $d, d_1, d_2 \in \mathbb{N}$. 
\end{proposition}

\begin{proof}[Proof of Proposition~\ref{app:embed_converge_proof:minima_set_compact}]
    We work with $\mcZ_d^{\geq 0}(A)$, knowing that the other case follows similarly. We want to argue that the set is weakly closed, and then that it is relatively weakly compact. 

    We begin by noting that the set of functions $\eta: [0, 1] \to [-A, A]^d$ is weakly compact. As this set is convex and norm closed (if $f_n \to f$ in $L^p$, we can extract a subsequence which converges a.e to $f$ and whose image will therefore lie within $[-A, A]^d$ a.e), and therefore will also be weakly closed. The compactness then follows by noting that as $[-A, A]^d$ is bounded, the set of functions $\eta: [0, 1] \to [-A, A]^d$ is also relatively weakly compact (by Banach-Alogolu in the $p > 1$ case, and Dunford-Pettis in the $p=1$ case - see Appendix~\ref{sec:app:convex_opt}).
    
    Now suppose we have a sequence $K_n \in \mathcal{Z}_d^{\geq 0}(A)$, say $K_n\llp = \sum_{i=1}^d \eta_{n, i}(l) \eta_{n, i}(l')$ for some functions $\eta_n : [0, 1] \to [-A, A]^d$ (so $\eta_{n, i}$ are the coordinate functions of $\eta_n$), such that $K_n$ converges weakly to some $K \in L^p([0, 1]^2)$. By weak compactness, we can extract a subsequence of the $\eta_n$, say $\eta_{n_k}$, which converges weakly in $L^p([0, 1])$ to some function $\eta$. Writing $q$ for the H\"{o}lder conjugate to $p$, we then know that for any functions $f, g \in L^{q}([0, 1])$ we have that 
    \begin{align*} 
        & \intsq K\llp f(l) g(l') \dldl = \lim_{n_k \to \infty} \intsq K_n\llp f(l) g(l')  \dldl \\
        & \; = \lim_{n_k \to \infty} \sum_{i=1}^d \intsq \eta_{n_k, i}(l) f(l) \eta_{n_k, i}(l') g(l') \dldl = \intsq \Big( \sum_{i=1}^d \eta_i(l) \eta_i(l') \Big) f(l) g(l') \dldl
    \end{align*}
    by using the weak convergence of the $\eta_{n_k}$. By taking $f = 1_{E}$ and $g = 1_{F}$ for arbitrary closed sets $E$ and $F$, it follows that $K$ and $\sum_{i=1}^d \eta_{i}(l)\eta_i(l')$ agree on products of closed sets, and therefore must be equal almost everywhere (as the latter is a $\pi$-system generating the Borel sets on $[0, 1]^2$). In particular, this implies that $K \in \mcZ_d^{\geq 0}(A)$. The weak compactness follows by noting that as $[-A, A]^d$ is bounded, and therefore the functions belonging to $\mcZ_d^{\geq 0}(A)$ are bounded in $L^{\infty}$, whence  $\mcZ_d^{\geq 0}(A)$ is relatively weakly compact. As we also know that $\mcZ_d^{\geq 0}(A)$ is also weakly closed, we can conclude.
\end{proof}

We now discuss minimizing $\mcI_n[K]$ over the sets introduced at the beginning of this section. It will be convenient to begin with the case where the $\fnone$ and $\fnzero$ are stepfunctions.

%!TEX root = ..\..\ms.tex

\begin{proposition} \label{app:embed_converge_proof:sbm_exist}
    Suppose that Assumption~\ref{assume:loss} holds, and further suppose that $\fnone$ and $\fnzero$ as introduced in Assumption~\ref{assume:samp_weight_reg} are piecewise constant on $\mcQ^{\otimes 2}$ (thus also bounded below), where $\mcQ$ is a partition of $[0, 1]$ into finitely many intervals, say $\kappa$ in total. Then there exists unique minimizers to the optimization problem 
    \begin{equation*}
        \min_{ K \in \mcZ^{\geq 0} } \mcI_n[K] \quad \text{ and } \quad \min_{ K \in \mcZ } \mcI_n[K].
    \end{equation*}
    Moreover, there exists $A'$ and $q \leq \kappa$ such that the minimum of $\mcI_n[K]$ over $\mcZ^{\geq 0}_d(A)$ are identical across all $A \geq A'$ and $d \geq q$, and therefore also equal to the minimizer over $\mcZ^{\geq 0}$. The same statement holds when replacing $\mcZ_d^{\geq 0}(A) \to \mcZ_{d_1, d_2}(A)$, $d  \geq q \to \min\{ d_1, d_2 \} \geq q$ and $\mcZ^{\geq 0} \to \mcZ$. 
\end{proposition}

\begin{proof}[Proof of Proposition~\ref{app:embed_converge_proof:sbm_exist}]
    We give the argument for when the constraint sets are non-negative definite, as the argument for the other case is very similar. Suppose that $\mcQ$ is of size $\kappa$ and is composed of intervals $(Q_i)_{i \in [\kappa]}$. Note that when $\fnone$ and $\fnzero$ are piecewise constant as assumed, we can argue analogously to Lemma~\ref{app:loss_converge_proof:embed_vector_avg_2} (via the strict convexity of the loss function) that any minimal value of $\mcI_n[K]$ over $\mcZ^{\geq 0}$ must be piecewise constant on $\mcQ = (Q_i)_{i \in [\kappa]}$, i.e we can write $K(l, l') = \langle \eta_i, \eta_j \rangle \text{ if } (l, l') \in Q_i \times Q_j$ for some vectors $\eta_i \in [-A, A]^d$, $i \in [\kappa]$. Moreover, by Lemma~\ref{app:embed_converge_proof:level_set} we know any minima must satisfy $\| K \|_p \leq C$ for some $C > 0$. We want to argue that the set of functions belonging to 
    \begin{equation*}
        \mcC := \{ K \,:\, \| K \|_p \leq C \} \cap \{ K \text{ piecewise constant on } \mcQ^{\otimes 2} \}
    \end{equation*}
    is weakly compact, so by Corollary~\ref{app:convex_opt:easy_min} we know that there is a unique minima to $\mcI_n[K]$ over $\mcZ^{\geq 0}$. To do so, we first note that the set is weakly closed, as $\mcC$ is convex and norm closed. In the case where $p > 1$, the set $\mcC$ is therefore weakly compact by Banach-Alagolu (see Appendix~\ref{sec:app:convex_opt}) as $\mcC$ is a weakly closed subset of the weakly compact set $\{ K \,:\, \| K \|_p \leq C \}$. In the case where $p = 1$, to apply the Dunford-Pettis criterion we need to argue that the set of functions $K \in \mcC$ is uniformly integrable. Indeed, if we let $K_{i, j}$ denote the value of $K$ on $Q_i \times Q_j$, then we can write that 
    \begin{align*}
        (\min_{i, j} |Q_i| |Q_j| ) & \cdot \max_{i, j} | K_{i, j} | \leq \sum_{i, j} |Q_i| |Q_j| | K_{i, j} | = \| K \|_1 \leq C \\
        & \implies \max_{i, j} | K_{i, j} | \leq \frac{C}{ \min_{i, j} |Q_i| |Q_j| },
    \end{align*}
    so $\sup_{K \in \mcC } \| K \|_{\infty} < \infty$, whence $\mcC$ is uniformly integrable. In both cases ($p > 1$ and $p = 1$), we therefore have that there exists a (unique) minima to $\mcI_n[K]$ over $\mcZ^{\geq 0}$. 
    
    We note that in the discussion above, we have reduced the minimization problem to one over the cone of $\kappa \times \kappa$ non-negative definite symmetric matrices. If we consider optimizing the function
    \begin{equation*}
        \tilde{I}_n[ \tilde{K} ] := \sum_{i, j \in [\kappa] } \sum_{x \in \{0, 1\}} p(i) p(j) \tilde{c}_n(i, j, x) \ell( \tilde{K}_{i, j}, x), \text{ where } \tilde{c}_n(i, j, x) = \int_{Q_i \times Q_j} \tilde{f}_n(l, l', x) \, dl dl' 
    \end{equation*}
    and $p(i) = |Q_i|$, over all non-negative definite symmetric matrices $\tilde{K}$, then we know that it has a unique minimizer $\tilde{K}^*$ with eigendecomposition $\tilde{K}^* = \sum_{i=1}^{\kappa} ( \sqrt{ \mu_i } \phi_i ) ( \sqrt{ \mu_i } \phi_i )^T$. Let $q$ equal the rank of $\tilde{K}^*$, i.e the number of $i$ for which $\mu_i \neq 0$. If we then define $K^*(l, l') = \langle \sqrt{ \mu_i } \phi_i, \sqrt{ \mu_j } \phi_j \rangle \text{ if } (l, l') \in Q_i \times Q_j$, it therefore follows that $K^*$ is the unique minima to $\mcI_n[K]$ over $\mcZ^{\geq 0}$. Moreover, the above representation tells us that $K^* \in \mcZ_d^{\geq 0}(A)$ as soon as $d \geq q$ and $A \geq A' = \max_{i \in [\kappa] } \| \sqrt{ \mu_i } \phi_i \|_{\infty}$, and therefore $K^*$ is the unique minima of $\mcI_n[K]$ over all such $\mcZ_d^{\geq 0}(A)$ too.
\end{proof}
%!TEX root = ..\..\ms.tex

\begin{corollary} \label{app:embed_converge_proof:minimizers}
    Suppose that Assumptions~\ref{assume:loss} holds with $p \geq 1$ as the growth rate of the loss, and Assumption~\ref{assume:samp_weight_reg} holds with $\gamma_s = \infty$, so $\mcI_n[K] < \infty$ iff $K \in L^p([0, 1]^2)$ by Lemma~\ref{app:embed_converge_proof:level_set}. Then there exists solutions to 
    \begin{equation*}
        \min_{K \in \mcZ_d^{\geq 0}(A) } \mcI_n[K] \quad \text{and} \quad \min_{K \in \mcZ_{d_1, d_2}(A) } \mcI_n[K]
    \end{equation*}
    for any $n$, $d$, $d_1$, $d_2$ and $A$. Moreover, there exists unique solutions to 
    \begin{equation*}
        \min_{K \in \mcZ^{\geq 0} } \mcI_n[K] \quad \text{and} \quad \min_{K \in \mcZ } \mcI_n[K].
    \end{equation*}
    Additionally, the minimizers of $\mcI_n[K]$ over $\mcZ^{\geq 0}$ and $\mcZ$ are continuous in the functions \linebreak $\{\fnone, \fnzero\}$ in the following sense: if we have functions $(\fnone, \fnzero)$, $(\tilde{f}_{\infty}(l, l', 1), \tilde{f}_{\infty}(l, l', 0))$ with minimizers $$K_n^* = \argmin I[K ; (\fnone, \fnzero)], \quad K_{\infty}^* = \argmin I[K ; (\tilde{f}_{\infty}(l, l', 1), \tilde{f}_{\infty}(l, l', 0)]$$ over $\mcZ^{\geq 0}$ or $\mcZ$, then if $\max_{x \in \{0, 1\} } \| \tilde{f}_n(\cdot, \cdot, x) - \tilde{f}_{\infty}(\cdot, \cdot, x) \|_{\infty} \to 0$ as $n \to \infty$, we have that $K_n^*$ converges weakly in $L^p([0, 1]^2)$ to $K_{\infty}^*$.
\end{corollary}

\begin{proof}[Proof of Corollary~\ref{app:embed_converge_proof:minimizers}]
    The first statement follows by combining Lemmas~\ref{app:embed_converge_proof:convexity},~\ref{app:embed_converge_proof:local_lipschitz} and Proposition~\ref{app:embed_converge_proof:minima_set_compact} and applying Corollary~\ref{app:convex_opt:easy_min}. For the second, we note that the optimization domains are convex by Lemma~\ref{app:embed_converge_proof:minima_set_convex}. In the case where $p > 1$, Lemma~\ref{app:embed_converge_proof:level_set} and Banach-Alagolu allows us to argue that the minima over $\mcZ^{\geq 0}$ and $\mcZ$ lies within a weakly compact set, and so such a minima exists and is unique. 

    In the $p = 1$ case, we already know that a minima to $\mcI_n[K]$ exists when the $\fnone$ and $\fnzero$ are piecewise constant on some partition $\mcQ^{\otimes 2}$, where $\mcQ$ is a partition of $[0, 1]$. Consider the function
    \begin{equation*}
        I[K ; g] = \intsq \sum_{x \in \{0, 1\} } g(l, l', x) \ell( K(l, l'), x) \, dl dl'
    \end{equation*}
    defined on $L^p([0, 1]^2) \times V_{\delta}$, where $V_{\delta} = \{ \text{symmetric } f \in L^{\infty}([0, 1]^2 \times \{0, 1\} ) \,:\, \delta \leq f \leq \delta^{-1} \text{ a.e} \}$ for some $\delta > 0$, so $\mcI_n[K] = I[K ; (\tilde{f}_n(\cdot, \cdot, 1), \tilde{f}_n(\cdot, \cdot, 0) )]$. We then know by Proposition~\ref{app:embed_converge_proof:sbm_exist} that a unique minimizer to $I[K; g]$ exists on a set of $g$ which is dense in $V_{\delta}$ (namely, symmetric stepfunctions). We now verify that $I[K ; g]$ satisfies the conditions in Theorem~\ref{app:convex_opt:l1_minima}. The strict convexity condition in a) follows by Lemma~\ref{app:embed_converge_proof:convexity}. We now note that via the same type of argument as in Lemma~\ref{app:embed_converge_proof:local_lipschitz}, we have that
    \begin{equation}
        \label{eq:app:embed_converge_proof:general_lip}
        \big| I[K; g] - I[ \tilde{K}; \tilde{g} ] \big| \leq \losslipconst \delta^{-1} \| K - \tilde{K} \|_{  L^1([0, 1]^2)} + \lossboundc ( \lossbounda + \| \tilde{K} \|_{  L^1([0, 1]^2)}) \| g - \tilde{g} \|_{ L^{\infty}([0, 1]^2 \times \{0, 1\} ) }
    \end{equation}
    from which the continuity condition b) holds. Moreover, by the same type of argument in Lemma~\ref{app:embed_converge_proof:level_set}, if we have that $I[K ; g] \leq \lambda$ then $\| K \|_1 \leq \lossbounda + \lossboundc \delta^{-1} \lambda$, and so this plus \eqref{eq:app:embed_converge_proof:general_lip} verifies condition c). With this, we can apply Theorem~\ref{app:convex_opt:l1_minima}, from which we get the claimed existence result when $p = 1$, along with continuity of the minimizers for $p \geq 1$.
\end{proof}

\subsection{Upper and lower bounds} 
\label{sec:app:embed_converge_proof:upper_lower_bounds}

In order to get a convergence result for the learned embeddings, we need some upper and lower bounds on quantities of the form $\mcI_n[K] - \mcI_n[K^*]$, where $K^*$ is the unique minima of $\mcI_n[K]$ over either $\mcZ^{\geq 0}$ or $\mcZ$. We begin with lower bounds in terms of quantities involving $K - K^*$.

\begin{lemma} \label{app:embed_converge_proof:kkt_conditions}
    Suppose that Assumptions~\ref{assume:loss}~and~\ref{assume:samp_weight_reg} hold, where $p \geq 1$ is the growth rate of the loss function. Let $\mathcal{C}$ be a weakly closed convex set in $L^p([0, 1]^2)$, and let $q$ be the H\"{o}lder conjugate to $p$. Then $K^*$ is the unique minima of $\mcI_n[K]$ over $\mathcal{C}$ if and only if 
    \begin{equation*}
        -\partial \mcI_n[K^*] \in \mathcal{N}_{\mathcal{C}}(K^*) = \big\{ L \in L^q([0, 1]^2) \,:\, \langle L, K^* - C \rangle \geq 0 \text{ for all } C \in \mathcal{C} \big\}.
    \end{equation*}
\end{lemma}

\begin{proof}
    By the strict convexity of $\mcI_n[K]$ and the KKT conditions.
\end{proof}

%!TEX root = ..\..\ms.tex

\begin{proposition} \label{app:embed_converge_proof:curvature_at_minima}
    Suppose that Assumptions~\ref{assume:loss} and~\ref{assume:samp_weight_reg} hold with $p \geq 1$ as the growth rate of the loss function and $\gamma_s = \infty$. Suppose $\mathcal{C}$ is a weakly closed convex set of $L^p([0, 1]^2)$, and that there exists a minima (whence unique) $K^*$ to $\mcI_n[K]$ over $\mcC$. Write $\tilde{f}_{n, x}(l, l') = \tilde{f}_n(l, l', x)$. Then for any $K \in \mathcal{C}$, we have the following:
    \begin{enumerate}[label=\roman*)]
        \item If $\ell''(y, x) \geq c > 0$ for some constant $c > 0$ for all $y \in \mathbb{R}$ and $x \in \{0, 1\}$ (for example the probit loss - see Lemma~\ref{app:embed_converge_proof:curvature_at_minima_prob_losses}), then 
        \begin{equation*}
            \mcI_n[K] - \mcI_n[K^*] \geq \frac{c}{2} \big( \max_{x \in \{0, 1\}}  \| \tilde{f}_{n, x}^{-1} \|_{\infty} \big)^{-1} \intsq (K\llp - K^*\llp)^2 \dldl.
        \end{equation*}
        \item Suppose that $\ell(y, x)$ is the cross entropy loss. Then
        \begin{equation*}
            \mcI_n[K] - \mcI_n[K^*] \geq \frac{1}{4}  \big( \max_{x \in \{0, 1\}} \| \tilde{f}_{n, x}^{-1} \|_{\infty} \big)^{-1} \intsq e^{-|K^*\llp|} \psi( |K\llp - K^*\llp|  ) \dldl,
        \end{equation*}
        where $\psi(x) = \min\{ x^2, 2x \}$.
    \end{enumerate}
\end{proposition}

\begin{proof}[Proof of Proposition~\ref{app:embed_converge_proof:curvature_at_minima}]
    Let $K_t = t K + (1-t)K^*$; therefore $K_0 = K^*$ and $K_1 = K$. Now, as $\ell(y, x)$ is twice differentiable in $y$ for $x \in \{0, 1\}$, by the integral version of Taylor's theorem we have that 
    \begin{align*}
        \ell( K, x) & = \ell(K^*, x) + \ell'(K^*, x) (K - K^*) + \int_0^1 (1-t) \ell''(K_t, x) (K - K^*)^2 \, dt
    \end{align*}
    for $x \in \{0, 1\}$. Therefore, if we multiply by $\tilde{f}_n(l, l', x)$, sum over $x \in \{0, 1\}$ and integrate over the unit square, it follows that 
    \begin{align*}
        \mcI_n[K] & = \mcI_n[K^*]   +  \intsq \partial\mcI_n[K^*]\llp (K\llp - K^*\llp) \dldl \\ & \qquad + \intsq \int_0^1 (1-t) \sum_{x \in \{0, 1\} } \tilde{f}_n(l, l', x) \ell''( K_t\llp , x) (K\llp - K^*\llp)^2 \dldl \, dt,
    \end{align*}
    where we have used the expression for $\partial \mcI_n[K]$ as derived in Proposition~\ref{app:embed_converge_proof:mcI_diff}. By the KKT conditions stated in Corollary~\ref{app:embed_converge_proof:kkt_conditions}, as $K^*$ is the unique minima to the constrained optimization problem, we get that
    \begin{equation*}
        \mcI_n[K] - \mcI_n[K^*] \geq \intsq \int_0^1 (1-t) \sum_{x \in \{0, 1\} } \tilde{f}_n(l, l', x) \ell''( K_t\llp , x) (K\llp - K^*\llp)^2 \dldl \, dt.
    \end{equation*}
    In order to lower bound the RHS further, we then work with the two specified cases in order. In the case where $\ell''(y, x) \geq c > 0$ for some constant $c > 0$ for all $y \in \mathbb{R}$ and $x \in \{0, 1\}$, then we get the bound
    \begin{equation*}
        \mcI_n[K] - \mcI_n[K^*] \geq \frac{c}{2} \intsq \fnsum (K\llp - K^*\llp)^2 \dldl
    \end{equation*}
    after integrating over $t \in [0, 1]$, from which we get the stated bound by using the fact that $\fnone$ and $\fnzero$ are bounded away from zero. In the cross entropy case, this follows by using the expression given in Lemma~\ref{app:embed_converge_proof:curvature_at_minima_prob_losses} and then using Fubini. 
\end{proof}

We now want to work on obtaining upper bounds for $\mcI_n[K] - \mcI_n[K^*]$, in the case where $K$ is a minimizer to $\mcI_n[K]$ over one of the sets $\mcZ_d^{\geq 0}(A)$ or $\mcZ_{d_1, d_2}(A)$. 
%We provide a qualitative version which shows that we can take this difference to $0$ as $n$ and $d(n) \to \infty$ for some sequence $d(n) \to \infty$, and then a quantative version which gives a speed of convergence depending on how well $K^*$ can be approximated in the $L^2([0, 1]^2)$ norm by an element within $\mcZ_d^{\geq 0}(A)$ or $\mcZ_{d_1, d_2}(A)$. 

%\input{appendix/embed_converge_proof/finite_rank_min_converge_1.tex}
%!TEX root = ..\..\ms.tex

\begin{lemma} \label{app:embed_converge_proof:kernel_approx}
    Suppose that Assumption~\ref{assume:loss} holds with $1 \leq p \leq 2$ and Assumption~\ref{assume:samp_weight_reg} holds with $\gamma_s = \infty$, and let $K_n^*$ be the unique minima of $\mcI_n[K]$ over $\mcZ^{\geq 0}$. Moreover suppose that $K_n^* \in L^2([0, 1]^2)$ for all $n \geq 1$, so we can therefore write
    \begin{equation}
        \label{eq:embed_converge_proof:l2_lim}
        K_n^*(l, l') = \sum_{k=1}^{\infty} \mu_{n, k} \phi_{n, k}(l) \phi_{n, k}(l'),
    \end{equation}
    where we understand the equality sign above to be understood as a limit in $L^2([0, 1]^2)$. Here the $\mu_{n, k} \geq 0$ for each $n$ are sorted in monotone decreasing order in $k$, and $\langle \phi_{n, i} , \phi_{n, j} \rangle = \delta_{ij}$ for each $n$. Additionally assume that $\| \sqrt{ \mu_{n, i} } \phi_{n, i} \|_{\infty} \leq A'$ for all $n, i$. Then for any $A \geq A'$, we get that
    \begin{align*}
        \Big| \min_{K \in \mcZ^{\geq 0} } \mcI_n[K] - \min_{K \in \mcZ_d^{\geq 0}(A) } \mcI_n[K] \Big| \leq 2^{p-1} \losslipconst \max_{x \in \{0, 1\}} \| \tilde{f}_{n, x} \|_{\infty}  \| K_n^* \|_2^{p-1} \Big( \sum_{k = d + 1}^{\infty} |\mu_{n, k} |^2 \Big)^{1/2}.
    \end{align*} 
    In the case when $K_n^*$ is the unique minima to $\mcI_n[K]$ over $\mcZ$, we again assume that $K_n^* \in L^2([0, 1]^2)$ for all $n$, so the expansion \eqref{eq:embed_converge_proof:l2_lim} still holds. Here the $\mu_{n, k}$ may not be non-negative, and are sorted so that $|\mu_{n, k} | \geq | \mu_{n, k+1} |$ for all $n, k$. Additionally assume that $\| \sqrt{ | \mu_{n, i} | } \phi_{n, i} \|_{\infty} \leq A'$ for all $n, i$. For each $n$, define $J^{(\pm)}_n := \{ i \,:\, \pm \mu_{n, i} > 0 \}$, and given a sequence $d = d(n)$, define 
    \begin{equation*}
        d_1 = d_1(n) := |  J^{(+)}_n \cap [d] |, \quad d_2 = d_2(n) := |  J^{(-)}_n \cap [d] |.
    \end{equation*}
    We then have for any $A \geq A'$ that
    \begin{equation*}
        \Big| \min_{K \in \mcZ } \mcI_n[K] - \min_{K \in \mcZ_{d_1, d_2}(A) } \mcI_n[K] \Big| \leq 2^{p-1} \losslipconst \max_{x \in \{0, 1\}} \| \tilde{f}_{n, x} \|_{\infty}  \| K_n^* \|_2^{p-1} \Big( \sum_{k = d + 1}^{\infty} |\mu_{n, k} |^2 \Big)^{1/2}.
    \end{equation*}
\end{lemma}

\begin{proof}[Proof of Lemma~\ref{app:embed_converge_proof:kernel_approx}]
    Note that
    \begin{equation*}
        K_{n, d}^* := \sum_{ k =1}^d \mu_{n, k} \phi_{n, k}(l) \phi_{n, k}(l') 
    \end{equation*}
    is a best rank-$d$ approximation to $K_n^*$, with the assumption that $\| \sqrt{ \mu_{n, i} } \phi_{n, i} \|_{\infty} \leq A'$ implying $K_{n, d}^* \in \mcZ_d^{\geq 0}(A)$ for each $d$. Consequently we have that $\min_{K \in \mcZ_d^{\geq 0}(A) } \mcI_n[K] \leq \mcI_n[K_{n, d}^* ]$ and therefore 
    \begin{equation*}
        \Big| \min_{K \in \mcZ^{\geq 0} } \mcI_n[K] - \min_{K \in \mcZ_d^{\geq 0}(A) } \mcI_n[K] \Big| \leq \mcI_n[K_{n, d}^*] - \mcI_n[K_n^*]. 
    \end{equation*}
    We then apply Proposition~\ref{app:embed_converge_proof:local_lipschitz} with $r = 2/p$, noting that 
    \begin{equation*}
        \| K_{n, d}^* \|_2 \leq \| K_n^* \|_2, \qquad \| K_{n, d}^* - K_n^* \|_2 = \Big( \sum_{k = d + 1}^{\infty} |\mu_{n, k} |^2 \Big)^{1/2},
    \end{equation*}
    to get the first stated result. The argument in the case where $\mcZ^{\geq 0}$ is replaced with $\mcZ$ is the same, after noting that our choice of $d_1$ and $d_2$ forces the best rank-$d$ approximation to be within $\mcZ_{d_1, d_2}(A)$. 
\end{proof}

\begin{remark} \label{rmk:embed_converge_proof:smoothness}
    Note that the eigenvalue bound obtained via the Parseval identity $      \sum_{k=1}^{\infty} \mu_k^2 = \| K^* \|_2^2 $ is that $|\mu_k| \leq \| K^* \|_2 k^{-1/2}$, which is unable to give rates of convergence of the best rank-$d$ approximation of $K^*$ to $K$, as the series $\sum_{k=1}^{\infty} k^{-1}$ is not summable. Under some additional smoothness conditions on $K^*$, we can obtain summable eigenvalue bounds (see Section~\ref{sec:app:holder_props}).
 \end{remark}

\begin{corollary} \label{app:embed_converge_proof:kernel_approx_2}
    Suppose that Assumption~\ref{assume:loss} holds with $1 \leq p \leq 2$ and Assumption~\ref{assume:samp_weight_reg} holds with $\gamma_s = \infty$, and let $K_n^*$ be the unique minima of $\mcI_n[K]$ over $\mcZ^{\geq 0}$. Suppose that one of the following sets of regularity conditions hold: 
    \begin{enumerate}[label=(\Alph*)]
        \item The $K_n^*$ satisfy $\sup_{n \geq 0} \| K_n^* \|_{\infty} < \infty$ and are $\mcQ^{\otimes 2}$-piecewise equicontinuous (that is, for all $\epsilon > 0$ there exists $\delta > 0$ such that whenever $x, y$ lie within the same partition of $\mcQ^{\otimes 2}$ and $\| x - y \| < \delta$, we have that $| K_n^*(x) - K_n^*(y) | < \epsilon$ for all $n$).
        \item The $K_n^*$ are each piecewise H\"{o}lder($[0, 1]^2$, $\beta$, $M$, $\mcQ^{\otimes 2}$) and $\sup_{n \geq 0} \| K_n^* \|_{\infty} < \infty$.
    \end{enumerate}
    Then there exists $A'$ such that whenever $A \geq A'$, we have that 
    \begin{equation*}
        \sup_{n} \Big| \min_{K \in \mcZ^{\geq 0} } \mcI_n[K] - \min_{K \in \mcZ_d^{\geq 0}(A) } \mcI_n[K] \Big| = \begin{cases} 
            o(1) \text{ as } d \to \infty & \text{ if (A) holds,} \\
            O\big( d^{- (1/2 + \beta)} \big) & \text{ if (B) holds. }\end{cases}
    \end{equation*}
    In the case where $K_n^*$ is the unique minima of $\mcI_n[K]$ over $\mcZ$ and either (A) or (B) as above hold, define $d_1, d_2$ as according to Lemma~\ref{app:embed_converge_proof:kernel_approx}. Then there exists $A'$ such that whenever $A \geq A'$, the above bound becomes 
    \begin{equation*}
        \sup_{n} \Big| \min_{K \in \mcZ } \mcI_n[K] - \min_{K \in \mcZ_{d_1, d_2}(A) } \mcI_n[K] \Big| = \begin{cases} 
            o(1) \text{ as } d \to \infty & \text{ if (A) holds,} \\
            O\big( d^{-\beta} \big) & \text{ if (B) holds. }\end{cases}
    \end{equation*}
\end{corollary}

\begin{proof}[Proof of Corollary~\ref{app:embed_converge_proof:kernel_approx_2}]
    Under the given assumptions, this is a consequence of Lemma~\ref{app:embed_converge_proof:kernel_approx}, Theorem~\ref{app:holder_props:eigenvalue_decay_of_holder} and Proposition~\ref{app:holder_props:optimalK_in_opt_domain}.
\end{proof}

\subsection{Convergence of the learned embeddings}

%!TEX root = ..\..\ms.tex

\begin{theorem} \label{app:embed_converge_proof:embed_convergence}
    Suppose that Assumptions~\ref{assume:loss} holds with either the cross-entropy loss (so $p = 1$) or a loss function satisfying $\ell''(y, x) \geq c > 0$ for all $y \in \mathbb{R}$, $x \in \{0, 1\}$ with $p = 2$; Assumptions~\ref{assume:graphon_ass}~\ref{assume:bilinear}~and~\ref{assume:slc} hold; and that Assumption~\ref{assume:samp_weight_reg} holds with $\gamma_s = \infty$. Suppose that $\whbmomega$ is any minimizer of $\mcR_n(\bmomega)$ over the set $\bmomega \in ([-A, A]^d)^n$, where we require that $A \geq A'$ for a constant $A'$ specified as part of one of the three regularity conditions listed below. Write $r_n$ for the relevant rate from Theorem~\ref{app:loss_converge_proof:main_theorem}, and define the function $\gamma(\beta) = \beta +1/2$ if $B(\omega, \omega')$ the regular inner product, or $\gamma(\beta) = \beta$ if $B(\omega, \omega')$ is a Krein or indefinite inner product in Assumption~\ref{assume:bilinear}. Let $K_n^*$ be the unique minima of $\mcI_n[K]$ over $\mcZ^{\geq 0}$ or $\mcZ$, depending on whether $B(\omega, \omega') = \langle \omega, \omega' \rangle$ or $\langle \omega, I_{d_1, d_2} \omega' \rangle$ respectively. We now assume one of the following sets of regularity conditions: 
    \begin{enumerate}[label=(\Alph*)]
        \item The $K_n^*$ are $\mcQ^{\otimes 2}$-piecewise equicontinuous (see Corollary~\ref{app:embed_converge_proof:kernel_approx_2}) and $\sup_{n \geq 1} \| K_n^* \|_{\infty} < \infty$. Moreover, the embedding dimension $d = d(n)$ is chosen so that $r_n \to 0$ (for example, one can take $d = \log(n)$ or $d = n^c$ for $c$ sufficiently small), and $d_1$, $d_2$ are chosen as described in Corollary~\ref{app:embed_converge_proof:kernel_approx_2}. Finally, we let $A'$ be the constant specified in Corollary~\ref{app:embed_converge_proof:kernel_approx_2}.
        \item In addition to (A), we assume that the $K_n^*$ are piecewise H\"{o}lder($[0, 1]^2$, $\beta$, $M$, $\mcQ^{\otimes 2}$) continuous for some constants $\beta$, $M > 0$ free of $n$.
        \item The functions $\fnone$ and $\fnzero$ are piecewise constant on $\mcQ^{\otimes 2}$. Moreover, the values of $A'$, $d$, $d_1$ and $d_2$ are chosen to satisfy the conditions in the last two sentences of Theorem~\ref{app:embed_converge_proof:sbm_exist}.
    \end{enumerate}
    We then have that
    \begin{equation*}
        \frac{1}{ n^2 } \sum_{i, j} \big| K_n^*(\lambda_i, \lambda_j) - B( \widehat{\omega}_i, \widehat{\omega}_j ) \big| =  \begin{cases} o_p(1) & \text{ if (A) holds, } \\ O_p( \tilde{r}_n^{1/2} )  & \text{ if (B) holds, } \\ O_p( r_n^{1/2} ) & \text{ if (C) holds.} \end{cases}
    \end{equation*}
    where $\tilde{r}_n = r_n + (\log(n)/n)^{\beta/2} + d^{-\gamma(\beta) }$.
\end{theorem}

\begin{remark}
    We note that when $K_n^* = \optimalK$ as defined in \eqref{eq:embed_learn:optimalK}, condition (B) will be satisfied by Corollary~\ref{app:holder_props:optimalK_bounded_and_holder}.
\end{remark}

\begin{proof}[Proof of Theorem~\ref{app:embed_converge_proof:embed_convergence}]
    Let $\whbmomega$ be a minimizer of $\mathcal{R}_n(\bm{\omega}_n)$ over $\bmomega \in (\compactset)^n = ([-A, A]^d)^n$. We begin with associating a kernel $K$ to a collection of embedding vectors $\bmomega$. To do so, given $\bm{\lambda}_n$, let $\lambda_{n, (i)}$ be the associated order statistics for $i \in [n]$, and $\pi_n$ be the mapping which sends $i$ to the rank of $\lambda_i$. We then define the sets
    \begin{equation*}
        A_{n, i} = \Big[ \frac{i - 1/2}{n+1}, \frac{i + 1/2}{n+1} \Big] \text{ for } i \in [n]
    \end{equation*}
    and the function
    \begin{equation*}
        \widehat{K}_n(l, l') = \begin{cases} B( \whomega_i, \whomega_j ) & \text{ if } (l, l') \in A_{n, \pi_n(i) } \times A_{n, \pi_n(j) }, \\ 
            0 & \text{ if } l \text{ or } l' \in [0, 1] \setminus \cup_{j=1}^n A_{n, j}.
        \end{cases}
    \end{equation*}
    The purpose of defining $\widehat{K}_n$ to have a "border" around the edges of $[0, 1]^2$ is so that we can allow the sets $A_{n, i}$ to be the same size, to simplify the bookkeeping below.

    We will now work on upper bounding $\mcI_n[\widehat{K}_n] - \mcI_n[K_n^*]$ to give us a rate at which this quantity converges. We will then lower bound this by some norm of $\widehat{K}_n - K_n^*$, which will be comparable to the quantity for which we give a rate of convergence for.

    \emph{Step 1: Bounding from above.} By the triangle inequality, we have that
    \begin{align*}
        \mcI_n[\widehat{K}_n] - \mcI_n[K_n^*] & \leq  \Big| \mcI_n[ K_n^* ] - \min_{K \in \mcZ_d^{\geq 0}(A) } \mcI_n[K] \Big| + \Big|\min_{K \in \mcZ^{\geq 0}_d(A) } \mcI_n[K] - \mcR_n(\whbmomega) \Big| \\
        & \qquad + \Big| \mcR_n(\whbmomega) - \mcI_n[ \widehat{K}_n ] \Big| = (\mathrm{I}) + (\mathrm{II}) + (\mathrm{III}).
    \end{align*}
    We note that $(\mathrm{II})$ is $O_p(r_n)$ by Theorem~\ref{app:loss_converge_proof:main_theorem}. The other two parts require more discussion depending on which of (A), (B) or (C) hold; we begin by bounding $(\mathrm{I})$ first.
    
    \emph{Step 1A: Bounding (I).} Here we apply Corollary~\ref{app:embed_converge_proof:kernel_approx_2} for when either (A) or (B) hold, and Theorem~\ref{app:embed_converge_proof:sbm_exist} for when (C) holds. In the latter case, we note that the conditions on $A'$ and $d$ (respectively $A'$, $d_1$ and $d_2$) imply that the minimizer to $\mcI_n[K]$ over $\mcZ^{\geq 0}$ (respectively $\mcZ$) is equal to the minimizer over $\mcZ_d^{\geq 0}(A)$ (respectively $\mcZ_{d_1, d_2}(A)$) whenever $A \geq A'$. It therefore follows that in either of the three cases, when $B(\omega, \omega') = \langle \omega, \omega' \rangle$ we know that whenever $A \geq A'$ we have that
    \begin{align*}
        \Big| \min_{K \in \mcZ^{\geq 0} } \mcI_n[K] - \min_{K \in \mcZ_d^{\geq 0}(A) } \mcI_n[K] \Big| = \begin{cases} o(1) & \text{ if (A) holds,} \\ O(d^{-(\beta + 1/2)} ) & \text{ if (B) holds, } \\
            0 & \text{ if (C) holds.} \end{cases} 
    \end{align*}
    In the case where $B(\omega, \omega') = \langle \omega, I_{d_1,d_2} \omega' \rangle$, we similarly have that
    \begin{align*}
        \Big| \min_{K \in \mcZ } \mcI_n[K] - \min_{K \in \mcZ_{d_1, d_2}(A) } \mcI_n[K] \Big| = \begin{cases} o(1) & \text{ if (A) holds,} \\ O( d^{-\beta} ) & \text{ if (B) holds,} \\
            0 & \text{ if (C) holds.} \end{cases} 
    \end{align*}
    %Note: when (A) holds, we are careful to make sure that $d$ (or $\min\{d_1, d_2\}$) is taken sufficiently slowly so that the rate of $r_n$ still goes to zero as $n \to \infty$, and in the case of (C) we take $d$ sufficiently large (but not diverging). 

    \emph{Step 1B: Bounding (III).} We will detail the argument and bounds under condition (B) first, and then describe what changes under conditions (A) and (C) afterwards. We begin by defining the quantity
    \begin{equation*}
        \tilde{c}_n(i, j, x) := \frac{1}{ | A_{n, \pi_n(i) } | | A_{n, \pi_n(i) }| } \int_{ A_{n, \pi_n(i)} \times A_{n, \pi_n(j) } } \tilde{f}_n(l, l', x) \, dl dl'
    \end{equation*}
    so we can therefore write (as $\widehat{K}_n$ is piecewise constant)
    \begin{align*}
        \mcI_n[ \widehat{K}_n ] & = \frac{1}{(n+1)^2} \sum_{i, j \in [n] } \sum_{x \in \{0, 1\} } \ell( B( \whomega_i, \whomega_j ), x) \tilde{c}_n(i, j , x) + \frac{(n-1)}{(n+1)^2} \big( \ell(0, 1) + \ell(0, 0) \big) \\
        & =  \widetilde{\mcI}_n[\widehat{K}_n] + O(n^{-1})  \text{ where } \tilde{\mcI}_n[\widehat{K}_n] := \frac{1}{(n+1)^2} \sum_{i, j \in [n] } \sum_{x \in \{0, 1\} } \ell( B( \whomega_i, \whomega_j ), x) \tilde{c}_n(i, j , x).
    \end{align*}
    Note that the $O(n^{-1})$ term holds uniformly across any choice of embedding vectors $\bmomega$. Recalling the function 
    \begin{equation*}
        \mathbb{E}[ \widehat{\mcR_n}(\bmomega) | \bm{\lambda}_n ] := \frac{1}{n^2} \sum_{i \neq j} \sum_{x \in \{0, 1\} } \tilde{f}_n(\lambda_i, \lambda_j, x) \ell( B(\omega_i, \omega_j), x)
    \end{equation*}
    from \eqref{eq:app:loss_converge_proof:inter_fn}, we introduce the function 
    \begin{equation*}
        \mathbb{E}[ \widehat{\mcR}_{n, (1)}(\bmomega) | \bm{\lambda}_n ] := \frac{1}{n^2} \sum_{i, j \in [n] } \sum_{x \in \{0, 1\} } \tilde{f}_n(\lambda_i, \lambda_j, x) \ell( B(\omega_i, \omega_j), x),
    \end{equation*}
    where we have added the diagonal term $i = j, i \in [n]$, and note that analogously to Lemma~\ref{app:loss_converge_proof:add_diag_term} (and with the exact same proof) we have that 
    \begin{equation}
        \label{eq:converge_proof:add_diag_bound}
        \sup_{\bmomega \in (S_d)^n } \Big| \mathbb{E}[ \widehat{\mcR}_{n, (1)}(\bmomega) | \bm{\lambda}_n ] - \mathbb{E}[ \widehat{\mcR_n}(\bmomega) | \bm{\lambda}_n ] \Big| = O\Big( \frac{d^p}{n} \Big).
    \end{equation}
    We can therefore write 
    \begingroup 
    \allowdisplaybreaks
    \begin{align}
        \big| \mcI_n[\widehat{K}_n] & - \mcR_n(\whbmomega) \big| \leq \big| \widetilde{\mcI}_n[ \widehat{K}_n ] - \mcR_n(\whbmomega) \big| + O(n^{-1} )  \nonumber \\
        & \leq \Big| \frac{1}{(n+1)^2 } \sum_{ i, j \in [n] } \sum_{x \in \{0, 1\} } \ell( B( \whomega_i, \whomega_j ), x) \big\{  \tilde{c}_n(i, j , x) - \tilde{f}_n(\lambda_i, \lambda_j, x) \big\} \nonumber \\
        & \qquad + \frac{1}{(n+1)^2 } \Big( \sum_{i, j \in [n] } \sum_{x \in \{0, 1\} } \ell( B(\whomega_i, \whomega_j) , x) \tilde{f}_n(\lambda_i, \lambda_j, x) \Big) - \mcR_n(\whbmomega) \Big| + O( n^{-1} )\nonumber  \\
        & \leq \frac{1}{ (n+1)^2 } \sum_{i, j \in [n] } \sum_{x \in \{0, 1\} } \ell( B(\whomega_i, \whomega_j), x) \big| \tilde{c}_n(\lambda_i, \lambda_j, x) - \tilde{f}_n(\lambda_i, \lambda_j, x) \big| \nonumber \\ 
        & \qquad + \Big| \big( 1 - \frac{1}{n+1} ) \big)^2 \mathbb{E}[ \widehat{\mcR}_{n, (1)}(\whbmomega) | \bm{\lambda}_n ] - \mcR_n(\whbmomega) \big| + O(n^{-1} ) \nonumber \\
        & \leq \frac{1}{ (n+1)^2 } \sum_{i, j \in [n] } \sum_{x \in \{0, 1\} } \ell( B(\whomega_i, \whomega_j), x) \big| \tilde{c}_n(\lambda_i, \lambda_j, x) - \tilde{f}_n(\lambda_i, \lambda_j, x) \big| \nonumber \\
        & \qquad + O(n^{-1} ) \mathbb{E}[ \widehat{\mcR}_{n, (1)}(\whbmomega) | \bm{\lambda}_n ] + \big| \mathbb{E}[ \widehat{\mcR}_{n, (1)}(\whbmomega) | \bm{\lambda}_n ]  -\mcR_n(\whbmomega) \big| + O(n^{-1} ). \label{eq:converge_proof:bound_on_III}
    \end{align} 
    \endgroup
    We begin by bounding the second and third terms above. We note that the third term can be bounded above by $O_p(r_n)$ by combining Lemma~\ref{app:loss_converge_proof:replace_prob_with_fn}, Theorem~\ref{app:loss_converge_proof:average_over_adjacency} and the bound \eqref{eq:converge_proof:add_diag_bound}. This also tells us that $\mathbb{E}[ \widehat{\mcR_n}(\whbmomega) | \bm{\lambda}_n ] = O_p(1)$, so the second term will be $O_p(n^{-1})$.

    For the first term, we exploit the smoothness of the $\tilde{f}_n(l, l', x)$, noting that we need to take some care in handling that it is only piecewise smooth. To handle the piecewise aspect, write $\mcQ = (Q_1, \ldots, Q_{\kappa})$, where the $Q_i$ are ordered so that if $x \in Q_i$ and $y \in Q_j$, then $x < y$ iff $i < j$. We then define the sets $N_{\lambda, n, k} = \{ j \,:\, \lambda_j \in Q_j \}$, $N_{A, n, k} = \{ j \,:\, A_{n, \pi_n(j) } \subseteq Q_k \}$, 
    \begin{equation*}
        M_{n, k} = \{ j \,:\, \lambda_j \in Q_k, A_{n, \pi_n(j) } \subseteq Q_k \} = N_{\lambda, n, k} \cap N_{A, n, k}, \qquad M_n = \bigcup_{k = 1}^{\kappa} M_{n, k}.
    \end{equation*}
    We want to determine the size of the set $M_n$. To do so, we note that as $\mcQ$ is a partition of $[0, 1]$, we have that the $N_{\lambda, n, k}$ are pairwise disjoint (and similarly so for the $N_{A, n, k}$), and therefore so are the $M_{n, k}$. To determine the size of the $M_{n, k}$, we note that as $\pi_n(\cdot) : [n] \to [n]$ is a bijection (sending the index $i$ to the order rank of $\lambda_i$ out of the $(\lambda_1, \ldots, \lambda_n)$), the size of $M_{n, k}$ is equal to the size of $\pi_n^{-1}( N_{\lambda, n, k} ) \cap \pi_n^{-1}( N_{A, n, k})$. We then note that the sets $\pi_n^{-1}( N_{\lambda, n, k} )$ are sets of contiguous integers, which begin and end at points 
    \begin{equation*}
        1 + \sum_{l=1}^{k-1} |N_{\lambda, n, k} |, \qquad \sum_{l=1}^k |N_{\lambda, n, k}|
    \end{equation*}
    respectively. Note that as $|N_{\lambda,n,k}|$ is $B(n, |Q_k|)$ distributed, we have that $|N_{\lambda, n, k}| = n|Q_k| + O_p(\sqrt{n})$ (for example by Proposition~\ref{app:loss_converge_proof:mnomconc_2}) and therefore the beginning and endpoints are equal to
    \begin{equation*}
        n \sum_{l=1}^{k-1} |Q_l| + O_p(\sqrt{n}), \qquad n \sum_{l=1}^k |Q_l| + O_p(\sqrt{n}).
    \end{equation*}
    Similarly, the sets $\pi_n^{-1}( N_{A, n, k} )$ are sets of contiguous integers beginning and ending at the points
    \begin{equation*}
        n \sum_{l=1}^{k-1} |Q_l| + O(1), \qquad n \sum_{l=1}^k |Q_l| + O(1)
    \end{equation*}
    respectively. It therefore follows that the size of the intersection, and therefore $|M_{n, k}|$, must be at least $n|Q_k| - E_{n, k}$ where $E_{n, k} \geq 0$, $E_{n, k} = O_p(\sqrt{n} )$. Consequently, as the $M_{n, k}$ are disjoint we have that $|M_n| \geq n - O_p(\sqrt{n})$, and so $|M_n^c| \leq O_p(\sqrt{n})$. 
    
    With this, we now begin bounding
    \begin{equation*}
       \big| \tilde{c}_n(\lambda_i, \lambda_j, x) - \tilde{f}_n(\lambda_i, \lambda_j, x) \big|
    \end{equation*}
    considering separately the cases where $i,j \in M_n$, and when either $i \not\in M_n$ or $j \not\in M_n$. In the case where $i, j \in M_n$, we get that 
    \begingroup 
    \allowdisplaybreaks 
    \begin{align}
        \big| \tilde{c}_n(i, j, x) & - \tilde{f}_n(\lambda_i, \lambda_j, x) \big| \leq \frac{1}{ |A_{n, i} | |A_{n, j} | } \int_{ A_{n, i} \times A_{n, j} } \big| \tilde{f}_n(l, l', x) - \tilde{f}_n(\lambda_{n, (i)}, \lambda_{n, (j) } , x) \big| \, dl dl' \nonumber \\ 
        & \leq \fnholderconst \sup_{(l, l') \in A_{n, i} \times A_{n, j} } \| (l, l') - (\lambda_{n, (i) }, \lambda_{n, (j) } ) \|_2^{\beta} \nonumber \\
        & \leq \fnholderconst 2^{\beta/2} \Big( \frac{1}{2n} + \max_{i \in [n] } \big| \lambda_{n, (i)} - \frac{i}{n+1} \big| \Big)^{\beta} = O_p\Big( \Big( \frac{ \log(n) }{n } \Big)^{\beta/2} \Big),
        \label{eq:converge_proof:diff_fc}
    \end{align}
    \endgroup
    where the last equality follows by Lemma~\ref{app:embed_converge_proof:order_stat_conc}, and we note that the stated bound holds uniformly over all $n$ and pairs of indices $i, j \in M_n$. In the case where either $i \not\in M_n$ or $j \not\in M_n$, then all we can say is that the difference of the two quantities is uniformly bounded above by $\sup_{n, x} \| \tilde{f}_{n, x} \|_{\infty}$. To summarize, we have that 
    \begin{equation}
        \big| \tilde{c}_n(i, j, x) - \tilde{f}_n(\lambda_i, \lambda_j, x) \big| \leq \begin{cases} 
            O_p\big( ( \log n)/n )^{\beta/2} \big) & \text{ if } i, j \in M_n, \\
             \sup_{x, n} \| \tilde{f}_{n, x} \|_{\infty} & \text{ otherwise,}
        \end{cases}
    \end{equation}
    holding uniformly across the vertices. We therefore have that 
    \begingroup 
    \allowdisplaybreaks 
    \begin{align}
        \frac{1}{n^2} & \sum_{i, j \in [n] } \sum_{x \in \{0, 1\} } \ell( B(\widehat{\omega}_i, \widehat{\omega}_j ), x ) \big| \tilde{c}_n(i, j, x) - \tilde{f}_n(\lambda_i, \lambda_j, x) \big| \nonumber \\
        & \leq \frac{1}{n^2} \Big( \sum_{i, j \in M_n } + \sum_{i \text{ or } j \in M_n^c } \Big) \sum_{x \in \{0, 1\} } \ell( B(\widehat{\omega}_i, \widehat{\omega}_j ), x ) \big| \tilde{c}_n(i, j, x) - \tilde{f}_n(\lambda_i, \lambda_j, x) \big| \nonumber \\
        & \leq \frac{1}{n^2} \sum_{i, j \in M_n} \sum_{x \in \{0, 1\}} \ell( B(\widehat{\omega}_i, \widehat{\omega}_j ), x ) \cdot O_p\big( ( \log n)/n )^{\beta/2} \big) \nonumber \\
        & \qquad \qquad \qquad \qquad + \frac{ |M_n^c|^2 + 2 |M_n| |M_n^c| }{ (n+1)^2 }  \cdot  \sup_{x, n} \| \tilde{f}_{n, x} \|_{\infty} A^2 d^p \nonumber \\
        & \leq O_p\big( ( \log n)/n )^{\beta/2} \big) \cdot \frac{1}{n^2} \sum_{i, j \in [n] } \sum_{x \in \{0, 1\} } \ell( B(\widehat{\omega}_i, \widehat{\omega}_j ), x ) + O_p(d^p / n^{1/2} ).
        \label{eq:converge_proof:diff_sum_bound}
    \end{align}
    \endgroup
    To finalize the above bound, we want to argue that 
    \begin{equation}
        \label{eq:converge_proof:average_loss_bound}
        \frac{1}{n^2} \sum_{i, j \in [n] } \sum_{x \in \{0, 1\} } \ell( B(\widehat{\omega}_i, \widehat{\omega}_j ), x ) = O_p(1).
    \end{equation}
    To do so, we note that as $\emprisk(\whbmomega) \leq \emprisk(\bm{0})$, by combining Lemma~\ref{app:loss_converge_proof:replace_prob_with_fn},~Theorem~\ref{app:loss_converge_proof:average_over_adjacency} and the bound \eqref{eq:converge_proof:add_diag_bound}, we know that $$\mathbb{E}[ \widehat{\mcR}_{(1), n}(\whbmomega) \,|\, \bm{\lambda}_n ] \leq 2  \mathbb{E}[ \widehat{\mcR}_{(1), n}(\bm{0}) \,|\, \bm{\lambda}_n ]$$ with asymptotic probability one. One of the intermediate steps in the proof of Lemma~\ref{app:loss_converge_proof:replace_fn_with_pnfn} then shows that this implies \eqref{eq:converge_proof:average_loss_bound} as desired.
    
    Consequently, it therefore follows by combining \eqref{eq:converge_proof:diff_sum_bound} and \eqref{eq:converge_proof:average_loss_bound} with \eqref{eq:converge_proof:bound_on_III} that we get $$(\mathrm{III}) = O_p( (\log(n)/n)^{\beta/2} + d^p n^{-1/2} + r_n ).$$ Here the $d^p n^{-1/2}$ term is negligible compared to $r_n$. We now discuss how this bound changes when (A) and (C) hold. In the case of (A), the equicontinuity condition implies that we can guarantee that the bound \eqref{eq:converge_proof:diff_fc} is $o_p(1)$, and so we obtain the bound $(\mathrm{III}) = o_p(1)$ after piecing together the other parts. In the case of (C), we note that the bound \eqref{eq:converge_proof:diff_fc} is equal to zero, and consequently the bound in \eqref{eq:converge_proof:diff_sum_bound} is $O_p(d^p n^{-1/2})$, so we have the bound $(\mathrm{III}) = O_p(r_n)$.

    \emph{Step 2: Lower bounding and concluding.} To summarize what we have shown so far in Step 1, we have obtained the bounds
    \begin{equation*}
        \mcI_n[\widehat{K}_n] - \mcI_n[K_n^*] =  \begin{cases} o_p(1) & \text{ if (A) holds, } \\ O_p( \tilde{r}_n) \text{ where } \tilde{r}_n = r_n + (\log(n)/n)^{\beta/2} + d^{-\gamma(\beta) } & \text{ if (B) holds, } \\ O_p(r_n)  & \text{ if (C) holds;} \end{cases}
    \end{equation*}
    where $\gamma(\beta) = \beta$ or $1/2 + \beta$, depending on whether $B(\omega, \omega')$ is an indefinite or the regular inner product on $\mathbb{R}^d$ respectively. To proceed, we work first in the case when (B) holds, and the loss function $\ell(y, x)$ is the cross-entropy loss. We then discuss afterwards what occurs when either (A) or (C) hold, along with when the loss function instead satisfies $\ell''(y, x) \geq c > 0$. 
    
    We now note that as $K_n^*$ is the unique minima of $\mcI_n[K]$ under either the constraint set $\mcZ^{\geq 0}$ or $\mcZ$, Proposition~\ref{app:embed_converge_proof:curvature_at_minima} tells us that we can obtain a lower bound on $\mcI_n[\widehat{K}_n] - \mcI_n[K_n^*]$ of the form 
    \begin{equation}
        \label{eq:converge_proof:lower_bound}
        \intsq \psi\big( | \widehat{K}_n\llp - K_n^*\llp | \big) e^{- |K_n^*(l, l')| } \, dl dl' \leq 4 \max_{x \in \{0, 1\} } \| \tilde{f}_{n, x}^{-1} \|_{\infty} \big( \mcI_n[\widehat{K}_n] - \mcI_n[K_n^*] \big)
    \end{equation}
    where $\psi(x) = \min\{ x^2, 2x \}$. As $K_n^*$ is assumed to be uniformly bounded in $L^{\infty}([0, 1]^2)$, and $\| \tilde{f}_{n, x}^{-1} \|_{\infty}$ is assumed to be uniformly bounded too, this implies that
    \begin{equation*}
        \intsq \psi\big( | \widehat{K}_n\llp - K_n^*\llp | \big) \, dl dl' = O_p( \tilde{r}_n),
    \end{equation*}
    and therefore by Lemma~\ref{app:embed_converge_proof:min_to_L1_lemma_crossent} we get that
    \begin{equation}
        \label{eq:converge_proof:knstar_widehatK}
        \intsq \big|  \widehat{K}_n(l, l') - K_n^*(l, l') \big| = O_p( \tilde{r}_n^{1/2} ).
    \end{equation}
    We now introduce the function 
    \begin{equation*}
        \bar{K}_n^*\llp = \begin{cases} K_n^*(\lambda_i, \lambda_j) & \text{ if } \llp \in A_{n, \pi_n(i) } \times A_{n, \pi_n(j) }  \\ 
            0 & \text{ if } l \text{ or } l' \in [0, 1] \setminus \cup_{i=1}^n A_{n, i} \end{cases}
    \end{equation*}
    and note that by the same arguments as in \eqref{eq:converge_proof:diff_sum_bound} above, it follows that 
    \begin{equation}
        \label{eq:converge_proof:barK_knstar}
        \intsq \big| \bar{K}_n^*\llp - K_n^*\llp \big| \, dl dl' = O_p\Big( \frac{ \| K_n^* \|_{\infty}}{n^{1/2} } + \Big(  \frac{\log(n) }{n} \Big)^{\beta/2} \Big). 
    \end{equation}
    Note that the term above decays faster than $\tilde{r}_n$, and as we are interested in the regime where $\tilde{r}_n \to 0$, it will be dominated by an $O_p(\tilde{r}_n^{1/2})$ term also. It therefore follows by the triangle inequality that 
    \begin{equation}
        \label{eq:converge_proof:done}
        \begin{split}
        \frac{1}{(n+1)^2} \sum_{i, j \in [n] } &\big| K_n^*(\lambda_i, \lambda_j) - B(\whomega_i, \whomega_j) \big| = \intsq \big| \bar{K}_n^*\llp - \widehat{K}_n\llp \big| \, dl dl' \\
        & \leq \intsq \big| \bar{K}_n^*\llp - K_n^*\llp \big| + \big| K_n^*\llp - \widehat{K}_n\llp \big| \, dl dl' = O_p( \tilde{r}_n^{1/2} )
        \end{split}
    \end{equation}
    as desired. In the case where (A) holds, we know that the bound \eqref{eq:converge_proof:knstar_widehatK} is now $o_p(1)$, and \eqref{eq:converge_proof:barK_knstar} will also be $o_p(1)$ by the asymptotic equicontinuity condition, and so \eqref{eq:converge_proof:done} will be $o_p(1)$ too. In the case where (C) holds, we firstly note that Theorem~\ref{app:embed_converge_proof:sbm_exist} implies that $\sup_{n \geq 1} \| K_n^* \|_{\infty} < \infty$, and so the parts of the argument relying on this assumption still go through. We then have that \eqref{eq:converge_proof:knstar_widehatK} will be $O_p( r_n^{1/2} )$, and \eqref{eq:converge_proof:barK_knstar} will be $O_p( \| K_n^* \|_{\infty} n^{-1/2} )$, and so \eqref{eq:converge_proof:done} will be $O_p( r_n^{1/2} )$. In the case where the loss function $\ell(y, x)$ is such that $\ell''(y, x) \geq c > 0$ for all $y$ and $x$ - we state the bounds for when (B) holds, as the argument does not change between the cases - we note that in \eqref{eq:converge_proof:lower_bound}, Proposition~\ref{app:embed_converge_proof:curvature_at_minima} instead tells us that 
    \begin{equation*}
        \Big( \intsq \big( \widehat{K}_n\llp - K_n^*\llp  \big)^2 \, dl dl' \Big)^{1/2} \leq \Big( 2 c^{-1} \max_{x \in \{0, 1\} } \| \tilde{f}_{n, x}^{-1} \|_{\infty} \cdot \big( \mcI_n[\widehat{K}_n] - \mcI_n[K_n^*]    \big)  \Big)^{1/2}. 
    \end{equation*}
    Consequently, \eqref{eq:converge_proof:knstar_widehatK} becomes 
    \begin{equation*}
        \Big( \intsq \big( \widehat{K}_n\llp - K_n^*\llp  \big)^2 \, dl dl' \Big)^{1/2} = O_p( \tilde{r}_n^{1/2} ),
    \end{equation*} 
    from which we can obtain the $L^1([0, 1]^2)$ bound in \eqref{eq:converge_proof:knstar_widehatK} by Jensen's inequality to therefore obtain the same bound as in \eqref{eq:converge_proof:done}.
\end{proof}

\subsection{Graphon with high dimensional latent features}
\label{app:embed_converge:high_dim}

\begin{proof}[Proof of Theorem~\ref{thm:embed_converge:high_dim}]
    Recall that for Algorithm~\ref{alg:random_walk}, we have that 
    \begin{align*}
        \tilde{f}_n(\bm{\lambda}, \bm{\lambda}', 1) & = \frac{ 2 k W(\bm{\lambda}, \bm{\lambda}') }{ \mcE_W }, \\
        \tilde{f}_n(\bm{\lambda}, \bm{\lambda}', 0) & = \frac{ l(k+1) (1 - \rho_n W(\bm{\lambda}, \bm{\lambda}'))} { \mcE_W(\alpha) \mcE_W(\alpha) } \big\{ W(\bm{\lambda}, \cdot) W(\bm{\lambda}', \cdot)^{\alpha} + W(\bm{\lambda}, \cdot)^{\alpha} W(\bm{\lambda}', \cdot) \big\}. 
    \end{align*}
    In particular, as the graphon $W(\bm{\lambda}, \bm{\lambda}')$ on $[0, 1]^q$ is equivalent to a graphon $W'$ on $[0, 1]$ which is H\"{o}lder with exponent $\beta_W q^{-1}$ by Theorem~\ref{thm:graphon_equiv}, it follows that
    \begin{align*}
        \tilde{f}_n'(\lambda, \lambda', 1) & := \frac{ 2 k W'(\lambda, \lambda') }{ \mcE_{W'} }, \\
        \tilde{f}_n'(\lambda, \lambda', 0) & := \frac{ l(k+1) (1 - \rho_n W'(\lambda, \lambda'))} { \mcE_{W'} \mcE_{W'}(\alpha) } \big\{ W'(\lambda, \cdot) W'(\lambda, \cdot)^{\alpha} + W'(\lambda, \cdot)^{\alpha} W'(\lambda', \cdot) \big\}
    \end{align*} 
    will be H\"{o}lder with exponent $\alpha \beta_W q^{-1}$ by Lemma~\ref{app:sampling:prod_deg_holder}. Similarly by Theorem~\ref{thm:graphon_equiv} and Lemma~\ref{app:sampling:prod_deg_tails}, we also know that $\tilde{f}_n'(\lambda, \lambda', 1)$ and $\tilde{f}_n'(\lambda, \lambda', 0)$ are bounded above uniformly in $n$, and are bounded below and away from zero uniformly in $n$. Consequently, we can then apply Theorem~\ref{thm:embed_learn:converge_2} to get the stated result.
\end{proof}

\subsection{Additional lemmata}
\label{sec:app:embed_converge_proof:add_lemma}

%!TEX root = ..\..\ms.tex

\begin{lemma} \label{app:embed_converge_proof:curvature_at_minima_prob_losses}
    Suppose that Assumption~\ref{assume:loss_prob} holds, so 
    \begin{equation*}
        \ell(y, x) = - x \log\big( F(y) \big) - (1-x) \log\big(1 - F(y) \big)
    \end{equation*}
    for some c.d.f function $F$. If $F(y) = \Phi(y)$ is the c.d.f of a standard Normal distribution, then $\ell''(y, x) \geq (4/\pi - 1) > 0$ for all $y \in \mathbb{R}$, $x \in \{0, 1\}$. If $F(y) = e^y/(1+e^y)$ is the c.d.f of the logistic distribution (so $\ell(y, x)$ is the cross entropy loss), then we have that 
    \begin{equation*}
        \int_0^1 (1 - t) \ell''(ty + (1-t) y^*) (y - y^*)^2 \, dt \geq \frac{1}{4} e^{-|y^*| } \min\{ |y - y^*|^2, 2|y - y^*| \}.
    \end{equation*}
\end{lemma}

\begin{proof}[Proof of Lemma~\ref{app:embed_converge_proof:curvature_at_minima_prob_losses}]
    \phantomsection\label{app:embed_converge_proof:curvature_at_minima_prob_losses:proof}
    Note that if the loss function is of the stated form with a symmetric, twice differentiable c.d.f $F$, we get that
    \begin{equation*}
        \frac{d^2}{dy^2} \ell(y, x) = \frac{ F'(y)^2 + (1 - F(y)) F''(y) }{ (1 - F(y))^2 }
    \end{equation*}
    for $x \in \{0, 1\}$. Due to the relation $F(y) + F(-y) = 1$, it follows that $F'$ is even and $F''$ is odd, meaning that the two derivatives for $x \in \{0, 1\}$ will be equal, and the second derivative is an even function in $y$. Consequently, we only need to work with $y > 0$. 

    With this, we begin with working with the probit loss. Note that by \citet[Formula 7.1.13]{abramowitz_handbook_1964} we have the tail bound
    \begin{equation*}
        \frac{2 \phi(y) }{y + \sqrt{y^2 + 4} } \phi(y) \leq 1 - \Phi(y) = \mathbb{P}(Z \geq y) \leq \frac{ 2 \phi(y) }{  y + \sqrt{y^2 + 8/ \pi} } \text{ for } y > 0
    \end{equation*}
    where $\phi(\cdot)$ is the corresponding p.d.f. It follows that the second derivative of $\ell(y, x)$ is therefore bounded below by (for $y > 0$)
    \begin{align*}
        \frac{1}{4} \big( y & + \sqrt{y^2 + 8/\pi} \big)^2 - \frac{1}{2} y \big( y + \sqrt{y^2 + 4} \big) = \frac{2}{\pi} + \frac{1}{2} x^2 \big( \sqrt{ 1 + \tfrac{8}{\pi x^2} } - \sqrt{1 + \tfrac{4}{x^2} } \big).
    \end{align*}
    This function is monotonically decreasing, and by the use of L'Hopitals rule we have that
    \begin{align*}
        \lim_{x \to \infty} x^2 \big( \sqrt{ 1 + \tfrac{8}{\pi x^2} } - \sqrt{1 + \tfrac{4}{x^2} } \big) & = \lim_{x \to \infty} \frac{ \sqrt{ 1 + \tfrac{8}{\pi x^2} } - \sqrt{1 + \tfrac{4}{x^2} } }{ x^{-2} } \\
        & = \lim_{x \to \infty} \frac{ - x^{-3} \big( \tfrac{8}{\pi} ( 1 + \tfrac{8}{\pi x^2} )^{-1/2} - 4 ( 1 + \tfrac{4}{x^2} )^{-1/2}    \big)   }{   - 2x^{-3} } = \frac{4}{\pi} - 2;
    \end{align*}
    it follows that $\ell''(y, x)$ will be bounded below by $4/\pi - 1 > 0$.

    If $F(y) = e^y/(1+e^y)$, then we claim that
    \begin{equation*}
        \frac{d^2}{dy^2} \ell(y, x) = \frac{e^y}{(1+e^y)^2} \geq \frac{1}{4} e^{-|y|}
    \end{equation*}
    for $x \in \{0, 1\}$. To see that this inequality is true, note that we can rearrange it to say that 
    \begin{equation*}
        e^{y+|y|} \geq \frac{1}{4} (1 + e^y )^2 = \frac{1}{4} \big( 1 + e^y + e^{2y} \big).
    \end{equation*}
    In the case when $y \geq 0$, the inequality follows by noting that the polynomial $1 + 2x - 3x^2$ is non-negative for $x \geq 1$ and substituting in $x = e^y$, and in the case when $y < 0$ follows by noting that the two functions which we are comparing are even. With this inequality we therefore have that 
    \begin{align*}
        \int_0^1 (1 - t) \ell''(ty + (1-t) y^*) (y - y^*)^2 \, dt & \geq \int_0^1 (1-t) e^{-|ty + (1- t)y^* | } (y - y^*)^2 \, dt \\
        & \geq \int_0^1  (1- t) e^{-|y^*|} e^{-t|y - y^*| } ( y - y^*)^2 \, dt \\
        & = e^{-|y^*|} \big\{  |y - y^*| + e^{-|y-y^*| } - 1 \big\} \\
        & \geq \frac{1}{4} e^{-|y^*| } \min\{ |y - y^*|^2, 2|y - y^*| \}.
    \end{align*}
    where in the second line we used the triangle inequality, and in the last line we used the inequality $x + e^{-x} - 1 \geq 0.25 \min\{ x^2, 2x \}$. (This last inequality can be derived by noting that the inequality holds at $x = 0$, and that the derivatives of the functions also satisfy the inequality.)
\end{proof}

\begin{lemma} \label{app:embed_converge_proof:order_stat_conc}
    Let $\mu_{n, i} \iid \mathrm{Unif}[0, 1]$ for $i \in [n]$, and let $\lambda_{n, (i)}$ be the associated order statistics. Then 
    \begin{equation*}
        \max_{i \in [n] } \Big| \lambda_{n, (i)} - \frac{i}{n+1} \Big| = O_p\Big( \sqrt{ \frac{ \log(2n) }{n} } \Big)
    \end{equation*}
\end{lemma}

\begin{proof}[Proof of Lemma~\ref{app:embed_converge_proof:order_stat_conc}]
    As the $\lambda_{n, (i)} \sim \mathrm{Beta}(i, n+1-i)$, we have by \citet[Theorem~2.1]{marchal_sub-gaussianity_2017} that
    \begin{equation*}
        \mathbb{E}\Big[ \exp\Big( \mu \Big\{  \lambda_{n, (i) } - \frac{i}{n+1} \Big\} \Big) \Big] \leq \exp\Big( \frac{ \mu^2 }{ 8 (n+2) } \Big) \text{ for all } \mu \in \mathbb{R},
    \end{equation*}
    i.e the $\lambda_{n, (i)} - \tfrac{i}{n+1}$ are sub-Gaussian random variables. The desired result therefore follows by using standard maximal inequalities for sub-Gaussian random variables. 
\end{proof}

\begin{lemma} \label{app:embed_converge_proof:min_to_L1_lemma_crossent}
    Suppose that $(g_n)$ is a sequence of measurable functions on $[0, 1]^2$ such that
    \begin{equation*}
        \int \min\{ |g_n|^2, c|g_n| \} \, d\mu = o(r_n)
    \end{equation*}
    where $(r_n)$ is a sequence converging to zero. Then $\int |g_n| d\mu = o(r_n^{1/2})$.
\end{lemma}

\begin{proof}[Proof of Lemma~\ref{app:embed_converge_proof:min_to_L1_lemma_crossent}]
    Recall that for $x > 0$, $x^2 \leq cx $ if and only if $x \leq c$, and therefore by Jensen's inequality we have that 
    \begin{align*}
        \int |g_n| 1[ |g_n| & \geq c] \,d\mu + \Big( \int |g_n| 1[ |g_n| \leq c ] \,d\mu \Big)^2  \\
        & \leq \int \{ |g_n| 1[ |g_n| \geq c] + |g_n|^2 1[ |g_n| \leq c] \} \, d\mu = \int \min\{ |g_n|^2, c|g_n| \} \, d\mu.
    \end{align*}
    Therefore by decomposing $\int |g_n| \, d\mu$ into parts where $|g_n| \geq c$ and $|g_n| \leq c$, we get contributions $o(r_n)$ and $o(r_n^{1/2})$ respectively, and so the desired result follows.
\end{proof}
%!TEX root = ..\ms.tex

\section{Additional results from Section~\ref{sec:embed_learn}} \label{sec:app:additional_results}

%!TEX root = ..\..\ms.tex

\phantomsection \label{app:other_results:link_prediction_consistency}

\begin{proof}[Proof of Proposition~\ref{thm:embed_learn:link_prediction_consistency}]
    Throughout, we denote $s_{ij} = B(\whomega_i, \whomega_j)$ and $\tilde{s}_{ij} = K_n^*(\lambda_i, \lambda_j)$. In the case where $d(s, b)$ is Lipschitz for $b \in \{0, 1\}$, if we let $M$ be the maximum of the Lipschitz constants for $d(s, 1)$ and $d(s, 0)$, and write $d(s, b) = b d(s, 1) + (1- b) d(s, 0)$, we get that for any $B \in \mathbb{A}_n$ that
    \begin{equation*}
        \Big| \mcL(S, B) - \mcL(\widetilde{S}, B) \Big| \leq \frac{M}{n^2} \sum_{i \neq j} \big| s_{ij}  - \tilde{s}_{ij} \big|,
    \end{equation*}
    and therefore we can apply Theorem~\ref{app:embed_converge_proof:embed_convergence} (which encapsulates Theorems~\ref{thm:embed_learn:converge_1},~\ref{thm:embed_learn:converge_2}~and~\ref{thm:embed_learn:converge_3}) to give the first claimed result.
    When $d(s, b)$ is the zero-one loss, we can write 
    \begin{equation*}
        \big| D_{\tau}(S, B) - D_{\tau'}(\widetilde{S}, B) \big| \leq \frac{1}{n^2} \sum_{i \neq j} \big| \mathbbm{1}[ s_{ij} < \tau] - \mathbbm{1}[ \tilde{s}_{ij} < \tau' ] \big|,
    \end{equation*}
    where we note that the RHS is free of $B$. We now note that the $\big| \mathbbm{1}[ s_{ij} < \tau] - \mathbbm{1}[ \tilde{s}_{ij} < \tau' ] \big|$ term equals $1$ iff either a) $s_{ij} < \tau$ and $\tilde{s}_{ij} \geq \tau'$, or b) $s_{ij} \geq \tau$ and $\tilde{s}_{ij} < \tau'$; otherwise it equals $0$. If $\tau' = \tau + \epsilon$ for $\epsilon > 0$, then a) implies that $| s_{ij} - \tilde{s}_{ij} | > \epsilon$. If b) holds, then either
    \begin{enumerate}[label=\roman*)]
        \item $s_{ij} \in [\tau, \tau + 2\epsilon]$, $\tilde{s}_{ij} \in [\tau - \epsilon, \tau + \epsilon ]$, and therefore $|s_{ij} - \tilde{s}_{ij} | \leq 3 \epsilon$; or
        \item one of the above conditions does not hold, in which case $| s_{ij} - \tilde{s}_{ij} | > \epsilon$.
    \end{enumerate} 
    If we instead take $\epsilon < 0$, then the above statements still hold provided we write $\epsilon \to |\epsilon|$; without loss of generality, we work with $\epsilon > 0$ onwards. Consequently, we get
    \begin{align*}
        \sup_{B \in \mathbb{A}_n } \big| D_{\tau}(S, B) & - D_{\tau + \epsilon}(\widetilde{S}, B) \big| \nonumber \\
        & \leq \frac{1}{n^2} \sum_{i \neq j} \mathbbm{1}\big[ \big| s_{ij} - \tilde{s}_{ij} | > \epsilon \big]  + \frac{1}{n^2} \sum_{i \neq j} \mathbbm{1}\big[ \tilde{s}_{ij} \in [\tau - \epsilon, \tau + \epsilon], |s_{ij} - \tilde{s}_{ij} | < 3 \epsilon \big] \nonumber \\
        & \leq \frac{1}{ \epsilon n^2} \sum_{i \neq j} \big| s_{ij} - \tilde{s}_{ij} | + \frac{1}{n^2} \sum_{i \neq j} \mathbbm{1}\big[ \tilde{s}_{ij} \in [\tau - \epsilon, \tau + \epsilon] \big].
    \end{align*}
    The first term will converge to zero in probability by Theorem~\ref{app:embed_converge_proof:embed_convergence} provided $\epsilon \to 0$ as $n \to \infty$ with $\epsilon = \omega( \tilde{r}_n )$, where $\tilde{r}_n$ is the convergence rate from Theorem~\ref{app:embed_converge_proof:embed_convergence}. For the second term, we want to control this term uniformly over all $\tau \in \mathbb{R} \setminus E$, where we recall that $E$ is the finite set of exceptions for the regularity condition stated in Equation~\eqref{eq:embed_learn:stab_condition}. Begin by noting that as the $K_n^*$ are uniformly bounded (as a result of the assumptions within Theorem~\ref{app:embed_converge_proof:embed_convergence}), we can reduce the above supremum to being over $\tau \in [-A, A] \setminus E$ for some $A > 0$ free of $n$. With this, if we write 
    \begin{equation*}
        X_{n, \tau, \epsilon} := \frac{1}{n^2} \sum_{i \neq j} \mathbbm{1}\big[ K_n^*(\lambda_i, \lambda_j) \in [\tau - \epsilon, \tau + \epsilon] \big],
    \end{equation*}
    then if we let $N(\epsilon)$ be a minimal $\epsilon$-covering of $[-A, A]$ (which has cardinality $\leq 4A \epsilon^{-1}$), we know that 
    \begin{align*}
        & \sup_{\tau \in [-A, A] \setminus E } X_{n, \tau, \epsilon} \leq 2 \sup_{\tau \in N(\epsilon) \setminus E } X_{n, \tau, \epsilon } \nonumber \\
        & \quad \leq 2 \sup_{\tau \in N(\epsilon) } \big| X_{n, \tau, \epsilon} - \mathbb{E}\big[ X_{n, \tau, \epsilon} \big] \big| + 2 \sup_{\tau \in N(\epsilon) \setminus E} \big| \big\{ \llp \in [0, 1]^2 \,:\, K_n^*\llp \in [\tau - \epsilon, \tau + \epsilon] \big\} \big|.
    \end{align*}
    Here, the first inequality follows by noting that for any $\tau \in [-A, A] \setminus E$, there exist two points $\tau_1, \tau_2 \in N(\epsilon)$ (pick the closest points to the left and right of $\tau$ within $N(\epsilon)$) such that 
    \begin{align*}
        \mathbbm{1}\big[ K_n^*(\lambda_i, \lambda_j) & \in [\tau - \epsilon, \tau + \epsilon] \big] \nonumber \\
        & \leq \mathbbm{1}\big[ K_n^*(\lambda_i, \lambda_j) \in [\tau_1 - \epsilon, \tau_1 + \epsilon] \big] + \mathbbm{1}\big[ K_n^*(\lambda_i, \lambda_j) \in [\tau_2 - \epsilon, \tau_2 + \epsilon] \big],
    \end{align*}
    and the second inequality follows by adding and subtracting
    \begin{equation*}
        \mathbb{E}[X_{n, \tau, \epsilon} ] = \big| \big\{ \llp \in [0, 1]^2 \,:\, K_n^*\llp \in [\tau - \epsilon, \tau + \epsilon] \big\} \big|.
    \end{equation*}
    With the regularity assumption, we know that 
    \begin{equation*}
        \sup_{\tau \in N(\epsilon) \setminus E} \big| \big\{ \llp \in [0, 1]^2 \,:\, K_n^*\llp \in [\tau - \epsilon, \tau + \epsilon] \big\} \big| \to 0
    \end{equation*}
    as $\epsilon \to 0$ uniformly in $n$. As for the $\sup_{\tau \in N(\epsilon) } \big| X_{n, \tau, \epsilon} - \mathbb{E}\big[ X_{n, \tau, \epsilon} \big] \big|$ term, by a union bound and the bounded differences concentration inequality \citep[Theorem~6.2]{boucheron_concentration_2016}, we have that 
    \begin{equation*}
        \mathbb{P}\Big( \sup_{\tau \in N(\epsilon) } \big| X_{n, \tau, \epsilon} - \mathbb{E}\big[ X_{n, \tau, \epsilon} \big] \big|  \geq \delta \Big) \leq \frac{ 4 A}{\epsilon} e^{-n \delta^2/8} 
    \end{equation*}
    which converges to zero for any fixed $\delta > 0$ provided $\epsilon^{-1} = O(n^c)$ for any constant $c > 0$. In particular, this tells us that $\sup_{\tau \in [-A, A] \setminus E} X_{n, \tau, \epsilon} \cvp 0$ provided $\epsilon \to 0$ with $\epsilon = \omega( \tilde{r}_n )$ as $n \to \infty$, and so the desired conclusion follows.  
\end{proof}

%!TEX root = ..\..\ms.tex

\begin{proof}[Proof of Proposition~\ref{thm:embed_learn:sbm_example_1}]
    \phantomsection\label{thm:embed_learn:sbm_example_1:proof}
    By the argument in the proof of Proposition~\ref{app:embed_converge_proof:sbm_exist}, we know that we can reduce the problem of optimizing $\mcI_n[K]$ over $K \in \mcZ^{\geq 0}$ to minimizing the function 
    \begin{equation*}
        \mcI_n[K] = \frac{1}{4}\Big( - p K_{11} + \log(1 + e^{K_{11}} ) - p K_{22} + \log(1 + e^{K_{22}} ) - 2q K_{12} + 2\log(1 + e^{K_{12}} ) \Big)
    \end{equation*} 
    over all positive definite matrices 
    \begin{equation*}
        K = \begin{pmatrix} K_{11} & K_{12} \\ K_{21} & K_{22} \end{pmatrix} \text{ where } K_{12} = K_{21},
    \end{equation*}
    and that a unique solution to this optimization problem exists. Note that the positive definite constraint forces that $K_{11}, K_{22} \geq 0$ and $K_{11} K_{22} \geq K_{12}^2$. Now, as the above function is symmetric in $K_{11}$ and $K_{22}$ and the function $-px + \log(1+e^x)$ is strictly convex for all $p \in (0, 1)$, it follows by convexity that a minima of $\mcI_n[K]$ must have $K_{11} = K_{22}$. This therefore simplifies the above problem to solving the convex optimization problem 
    \begin{align*}
        \text{minimize: } & - p K_{11} + \log(1 + e^{K_{11}} ) - q K_{12} + \log(1 + e^{K_{12}} ) \\ 
        \text{subject to: } & K_{11} \geq 0, K_{11} - K_{12} \geq 0, K_{11} + K_{12} \geq 0.
    \end{align*}
    Letting $\mu_i \geq 0$ be dual variables for $i \in \{1, 2, 3\}$, the KKT conditions for this problem state that any minima must satisfy 
    \begin{align*}
        -p + \sigma(K_{11}) - \mu_1 - \mu_2 - \mu_3 & = 0, \\
        -q + \sigma(K_{22}) + \mu_2 - \mu_3 & = 0, \\
        \mu_1 K_{11} = 0, \qquad \mu_2( K_{11} - K_{12} ), \qquad \mu_3( K_{11} + K_{12} ) & = 0.
    \end{align*}
    We now work case by case, considering what occurs on the interior of the constraint region; then the edges $K_{11} = \pm K_{12}$ with $K_{11} > 0$; and then we finish with $K_{11} = K_{12} = 0$:
    \begin{itemize}
        \item In the case where $K_{11} > 0$ and $K_{11} > |K_{12}|$, the solution is given by $K_{11} = \sigma^{-1}(p)$ and $K_{12} = \sigma^{-1}(q)$, which is feasible provided $p > 1/2$, $p > q$ (if $q \geq 1/2$) and $p > 1 - q$ (if $q < 1/2$). 
        \item In the case where $K_{11} > 0$ and $K_{11} = -K_{12}$, then $\mu_1 = \mu_2 = 0$, and so the optimal solution has $K_{11} = \sigma^{-1}( (1 + p -q)/2)$ with $\mu_3 = (1 - p -q)/2$, which is feasible provided $p > q$ but $p + q < 1$. 
        \item In the case where $K_{11} > 0$ and $K_{11} = K_{12}$, then $\mu_1 = \mu_3 = 0$, so $K_{11} = \sigma^{-1}( (p + q)/2)$, and so is feasible if $q > p$ and $p + q > 1$. 
        \item The only remaining case is when $K_{11} = K_{12} = 0$, and occurs in the complement of the union of the above cases, i.e when $q > p$ and $p + q \leq 1$. 
    \end{itemize}
    As the optimization problem is feasible (in that we can guarantee that a minima exists) for all values of $p, q \in (0, 1)$, and each of the above cases correspond to a partition of the $(p, q)$ space with a unique minima in each case, these do indeed correspond to the minima of $\mcI_n[K]$ in each of the designated regimes, as stated.
\end{proof}

%!TEX root = ..\..\ms.tex

\begin{proposition} \label{thm:extra_results:learn_nothing}
    Suppose that the loss function in Assumption~\ref{assume:loss_prob} is the cross-entropy loss. Then the minima of $\mcI_n[K]$ over $\mcZ^{\geq 0}$ is equal to a constant $c \geq 0$ if and only if 
    \begin{equation*}
        \tilde{f}_n(l, l', 1) \preccurlyeq \tilde{f}_n(l, l', 0) \max\bigg\{1, \frac{\intsq \tilde{f}_n(x, y, 1) \, dx dy}{ \intsq \tilde{f}_n(x, y, 0) \, dx dy} \Bigg\}
    \end{equation*}
    where $\preccurlyeq$ denotes the positive definite ordering (see Section~\ref{sec:app:holder_props}) on symmetric kernels $[0, 1]^2 \to \mathbb{R}$. In the case where we have that $\fnone = k W(l, l')$ and $\fnzero = k (1 - W(l, l'))$ for some $k$ (such as when the sampling scheme is uniform vertex sampling as in Algorithm~\ref{alg:psamp}), this condition is equivalent to $W \preccurlyeq \max\{1/2, \intsq W(l, l') \, dl dl' \}$. 
    % This occurs if and only if a) all but the largest eigenvalue of $W$ are negative, and b) the eigenvector $\phi$ coresponding to the largest eigenvalue of $W$ satisfies
    % \begin{equation*}
    %     \int_0^1 W(x, y) \phi(x) \phi(y) \, dx dy \leq \max\Big\{ \frac{1}{2}, \intsq W(l, l') \, dl dl' \Big\} \Big( \int_{[0, 1]} \phi(x) \, dx \Big)^2.
    % \end{equation*}
\end{proposition} 

\begin{proof}[Proof of Proposition~\ref{thm:extra_results:learn_nothing}]
    \phantomsection\label{thm:extra_resukts:learn_nothing:proof}
    We begin by noting that if $K^*(l, l') = c \geq 0$ is the minima of $\mcI_n[K]$ over $\mcZ^{\geq 0}$, then the KKT conditions guarantee that
    \begin{equation}
        \label{eq:learn_nothing:KKT}
        \intsq \Big\{ \tilde{f}_n(l, l', 1) \frac{1}{1+e^c} - \tilde{f}_n(l, l', 0) \frac{e^c}{1+e^c} \Big\} \cdot \big( c - K(l, l') \big) \, dl dl' \geq 0
    \end{equation}
    for all $K \in \mcZ^{\geq 0}$. In the case where $c > 0$, by choosing $K(l, l') = b$ and varying $b$ either side of $c$, it follows that we in fact must have that
    \begin{equation*}
        c \cdot \Big( \frac{A_1}{1+e^c} - \frac{A_0 e^c}{1+e^c} \Big) = 0 \text{ where } A_x = \intsq \tilde{f}_n(l, l', x) \, dl dl' \text{ for } x \in \{0, 1\}.
    \end{equation*}
    It therefore follows that if $K = c$ is the minima, then we necessarily have that $c = \log(A_1/A_0)$, which is greater than $0$ if and only if $A_1 > A_0$. Substituting this value of $c$ back into \eqref{eq:learn_nothing:KKT} and rearranging then tells us that for all $K \in \mcZ^{\geq 0}$ we have that
    \begin{align}
        \label{eq:learn_nothing:kkt_1}
        \intsq \Big\{ \tilde{f}_n(l, l', 1) \frac{A_0}{A_0 + A_1} &- \tilde{f}_n(l, l', 0) \frac{A_1}{A_0 + A_1} \Big\} K(l, l') \, dl dl' \nonumber \\
        &  \leq \log(A_1/A_0 ) \frac{A_1 A_0 - A_0 A_1}{A_0 + A_1} = 0.
    \end{align}
    In the case where $c = 0$, we instead immediately obtain 
    \begin{equation}
        \label{eq:learn_nothing:kkt_2}
        \intsq \Big\{ \tilde{f}_n(l, l', 1) - \tilde{f}_n(l, l', 0)  \Big\} \cdot K(l, l') \, dl dl' \leq 0
    \end{equation}
    from \eqref{eq:learn_nothing:KKT}. As the $\tilde{f}_n \in L^{\infty}$ and are non-negative, by a density argument we can extend \eqref{eq:learn_nothing:kkt_1} and \eqref{eq:learn_nothing:kkt_2} to hold for all non-negative definite kernels $K \in L^2$. Consequently, if we write $\preccurlyeq$ for the positive definite ordering of symmetric kernels, this is equivalent to saying that 
    \begin{equation*}
        \tilde{f}_n(l, l', 1) \preccurlyeq \tilde{f}_n(l, l', 0) \max\Big\{1, \frac{A_1}{A_0} \Big\}.
    \end{equation*}
    Specializing further to the case where $\fnone = kW(l, l')$ and $\fnzero = k(1 - W(l, l'))$, this simplifies to saying that (recalling the notation $\mcE_W = \intsq W(l, l') \, dl dl'$)
    \begin{equation*}
        W \preccurlyeq (1 - W) \max\Big\{ 1, \frac{ \mcE_W }{1 - \mcE_W } \Big\} \qquad \iff \qquad W \preccurlyeq \max\Big\{ \frac{1}{2}, \intsq W(l, l') \, dl dl' \Big\},
    \end{equation*}
    and so we are done.
    % By Lemma~\ref{thm:extra_results:loewner_lemma_1}, this occurs if and only if all but the largest eigenvalue of $W$ is non-negative, and the largest eigenvalue $\lambda_1(W)$ with eigenvector $\phi$ satisfies
    % \begin{equation}
    %     \intsq \phi(x) W(x, y) \phi(y) \, dx dy \leq \max\Big\{ \frac{1}{2}, \mcE_W \Big\} \Big( \int_{[0, 1] } \phi(x) \, dx \Big)^2
    % \end{equation}
    % as claimed.
\end{proof}

\section{Proof of results in Section~\ref{sec:sampling_formula}}
\label{sec:app:sampling_formula_proof}

We begin with several results which give concentration and quantative results for various summary statistics of the network (e.g the number of edges and the degree), before giving the
sampling formula (and rates of convergence) for each of the algorithms we discuss in Section~\ref{sec:sampling_formula}.

\subsection{Large sample behavior of graph summary statistics}

%!TEX root = ..\..\ms.tex

\phantomsection\label{app:sampling:summary_stats_converge:proof}

\begin{proposition} \label{app:sampling:summary_stats_converge}
    Let $\mcG_n = (\mcV_n, \mcE_n)$ be a graph drawn from a graphon process with generating graphon $W_n(x, y) = \rho_n W(x, y)$ for some sequence $(\rho_n)$ with $\rho_n \to 0$. Recall that part of Assumption~\ref{assume:graphon_ass} requires that $W(\lambda, \cdot) \in L^{\gamma_d}([0, 1]^2)$ for some $\gamma_d \in (1, \infty]$. Then we have the following:
    \begin{enumerate}[label=\alph*)]
        \item Letting $\degree_n(i)$ denote the degree of a vertex $i \in \mcV_n$ with latent feature $\lambda_i$, we have for all $t > 0$ that
        \begin{equation*}
            \mathbb{P}\Big( \big| \frac{ \degree_n(i) }{ (n-1) \rho_n W(\lambda_i, \cdot) } - 1 \big| \geq t \,|\, \lambda_i \Big) \leq 2\exp\Big( \frac{ -n \rho_n t^2  W(\lambda_i, \cdot) }{4 ( 1 + 2t) } \Big).
        \end{equation*}

        \item Under the additional requirement that Assumption~\ref{assume:graphon_ass} holds with $\gamma_d \in (1, \infty]$, we have that 
        \begin{equation*}
            \max_{i \in [n] } \Big| \frac{ \degree_n(i) }{ (n-1) \rho_n W(\lambda_i, \cdot) } - 1 \Big| = \begin{cases} O_p\Big( ( \log n)^{1/2} (n \rho_n)^{-1/2} \Big) & \text{ if } \gamma_d = \infty, \\
                O_p\Big( \big( n^{(\gamma_d - 1)/\gamma_d } \rho_n \big)^{-1/2} \Big) & \text{ if } \gamma_d \in (1, \infty). \end{cases}
        \end{equation*}

        \item Under the additional requirement that Assumption~\ref{assume:graphon_ass} holds, we have that 
        \begin{equation*}
            \max_{ i \in [n] } \frac{1}{ \degree_n(i) } = \begin{cases} O_p\Big( (n \rho_n)^{-1} \Big) & \text{ if } \gamma_d = \infty, \\
                O_p\Big( \big( n^{(\gamma_d - 1)/\gamma_d } \rho_n \big)^{-1} \Big) & \text{ if } \gamma_d \in (1, \infty); \end{cases}
        \end{equation*}
        and 
        \begin{equation*}
            \min_{ i \in [n] } \degree_n(i)  = \begin{cases} \Omega_p\Big( n \rho_n \Big) & \text{ if } \gamma_d = \infty, \\
                \Omega_p\Big(  n^{(\gamma_d - 1)/\gamma_d } \rho_n \Big) & \text{ if } \gamma_d \in (1, \infty). \end{cases}
        \end{equation*}

        \item We have that 
        \begin{equation*}
            \mathbb{P}\Big( \big| \frac{ \sum_{i=1}^n W_n(\lambda_i, \cdot)^{\alpha} }{ n \rho_n^{\alpha} \mathcal{E}_W(\alpha) } - 1 \big| \geq t \Big) \leq 2\exp\Big( \frac{ - n \mcE_W(\alpha) t^2 }{ 2(1+t) } \Big),
        \end{equation*}
        where we write $\mathcal{E}_W(\alpha) := \int_0^1 W(\lambda, \cdot)^{\alpha} \, d\lambda$, and consequently
        \begin{equation*}
            \sum_{i = 1}^n W_n(\lambda_i, \cdot)^{\alpha} = n \rho_n^{\alpha} \mcE_W(\alpha) \cdot \big( 1 + O_p( n^{-1/2} ) \big).
        \end{equation*}

        \item Writing $E_n := E[\mcG_n]$ for the number of edges of $\mcG_n$, we have for all $t > 0$ that
        \begin{equation*}
            \mathbb{P}\Big( \big| \frac{ 2 E_n }{ n(n-1) \rho_n \mathcal{E}_W } - 1 \big| \geq t \Big) \leq \exp\Big( \frac{-n \rho_n \mcE_W t^2 }{4(1+2t) } \Big)
        \end{equation*}
        and consequently $E_n = n^2 \rho_n \mcE_W \cdot \big(1 + O_p( (n \rho_n)^{-1/2} ) \big)$.

        \item Under the additional requirement that Assumption~\ref{assume:graphon_ass} holds with $\gamma_d \in (1, \infty]$, we have that 
        \begin{equation*}
            \max_{i \in [n] } \Big| \frac{ \degree_n(i) /2 E_n }{ W(\lambda_i, \cdot)/ n \mcE_W } - 1 \Big| = \begin{cases} O_p\Big( ( \log n)^{1/2} (n \rho_n)^{-1/2} \Big) & \text{ if } \gamma_d = \infty, \\
                O_p\Big( \big( n^{(\gamma_d - 1)/\gamma_d } \rho_n \big)^{-1/2} \Big) & \text{ if } \gamma_d \in (1, \infty). \end{cases}
        \end{equation*}
        % \item Let $g : [0, 1]^2 \to [0, \infty]$ be a bounded measurable function not identically equal to zero, and define 
        % \begin{equation*}
        %     T_{n, i} = \frac{1}{n-1} \sum_{j \in [n] \setminus \{i \} } a_{ij} g(\lambda_j), \text{ so } \mathbb{E}[ T_{n, i} \,|\, \lambda_i ] = \rho_n \int_0^1 W(\lambda_i, y) g(y) \, dy
        % \end{equation*}
        % is uniformly bounded away from $0$ for all $i$, $n$. Then for all $t \geq 0$ we have that 
        % \begin{equation*}
        %     \mathbb{P}\Big(  \Big| \frac{ T_{n, i} }{ \mathbb{E}[T_{n, i} \,|\, \lambda_i ] } - 1 \Big| \geq t \Big) \leq 2\exp\Big( \frac{ - n \mathbb{E}[T_{n, i} \,|\, \lambda_i ]  t^2 }{ 8 \| g \|_{\infty} (1+t) } \Big)
        % \end{equation*}
        % and whence that $T_{n, i} = \mathbb{E}[T_{n, i} \,|, \lambda_i] (1 + O_p((\log n/n \rho_n )^{1/2} ) )$. Similarly, if we write 
        % \begin{equation*}
        %     \widetilde{T}_{n, i} = \frac{1}{n-1} \sum_{j \in [n] \setminus \{i \} } (1 - a_{ij} ) g(\lambda_j), \text{ so } \mathbb{E}[ \widetilde{T}_{n, i} \,|\, \lambda_i ] = \int_0^1 (1 - \rho_n W(\lambda_i, y) ) g(y) \, dy, 
        % \end{equation*}
        % then $\widetilde{T}_{n, i} = \mathbb{E}[ \widetilde{T}_{n, i} \,|, \lambda_i] (1 + O_p((\log n/n)^{1/2} ) )$.
    \end{enumerate}
\end{proposition}

\begin{proof}[Proof of Proposition~\ref{app:sampling:summary_stats_converge}]  
    For a), begin by noting that for the degree we can write 
    \begin{align*}
        \degree_n(i) \disteq \sum_{j \in [n] \setminus i} \mathbbm{1}\Big[ U_j \leq W_n(\lambda_i, \lambda_j) \Big] 
    \end{align*}
    where $U_j \iid U[0, 1]$. We then form an exchangeable pair $( \bm{\lambda}_{n, -i}, \tilde{\bm{\lambda}}_{n, -i} )$ (where we work conditional on $\lambda_i$ and write $\bm{\lambda}_{n, -i} = (\lambda_j)_{j \leq n, j \neq i}$) by selecting a vertex $J \sim \mathrm{Unif}([n] \setminus \{i \} )$ and then redrawing $\tilde{\lambda}_{J} \sim U[0, 1]$ and otherwise setting $\tilde{\lambda}_{j} = \lambda_{j}$ for $j \neq J$. Writing $\bm{\lambda}_{n, -i}'$ and $U_j'$ for independent copies of $\bm{\lambda}_{n, -i}$ and $U_j$, and also writing $\mathrm{deg}_n(i)[ \bm{\lambda}_{n, -i} ]$ to make the dependence on $\bm{\lambda}_{n, -i}$ explicit, we have that
    \begin{align*}
        \mathbb{E}\Big[ & \frac{ \mathrm{deg}_n(i)[ \bm{\lambda}_{n, -i} ] }{ W_n(\lambda_i, \cdot) } -  \frac{ \degree_n(i)[ \tilde{\bm{\lambda}}_{n, -i}  ] }{ W_n(\lambda_i, \cdot) } \,\Big|\, \lambda_i, \bm{\lambda}_{n, -i} \Big]  \nonumber \\
        & = \frac{1}{ (n-1) W_n(\lambda_i, \cdot) } \sum_{j \neq i} \Big\{ \mathbbm{1}\Big[ U_j \leq W_n(\lambda_i, \lambda_j) \Big] - \mathbb{E}\Big[ \mathbbm{1}\Big[ U_j' \leq W_n(\lambda_i, \lambda_j') \Big] \,\big|\, \lambda_i \Big] \Big\} \nonumber \\
        & = \frac{\degree_n(i)[ \bm{\lambda}_{n, -i} ] }{(n-1) W_n(\lambda_i, \cdot) } - 1.
    \end{align*}
    We then have that 
    \begin{align*}
        v(\bm{\lambda}_{n, -i} ) &= \frac{1}{2(n-1)} \mathbb{E}\Big[ \Big( \frac{ \degree_n(i)[ \bm{\lambda}_{n, -i} ] }{ W_n(\lambda_i, \cdot) }  -  \frac{ \degree_n(i)[ \tilde{\bm{\lambda}}_{n, -i}  ] }{ W_n(\lambda_i, \cdot) } \Big)^2 \,\Big|\, \lambda_i, \bm{\lambda}_{n, -i} \Big] \\
        & = \frac{1}{2(n-1)^2 W_n(\lambda_i, \cdot)^2 } \sum_{j \neq i} \Big\{ \mathbb{E}\Big[ \Big( \mathbbm{1}\Big[ U_j \leq W_n(\lambda_i, \lambda_j) \Big] - \mathbbm{1}\Big[ U_j' \leq W_n(\lambda_i, \lambda'_j) \Big] \Big)^2 \, \Big| \, \lambda_i \Big] \nonumber \\
        & \leq \frac{1}{ (n-1)^2 W_n(\lambda_i, \cdot)^2 } \big( \degree_n(i)[ \bm{\lambda}_{n, -i} ] + (n-1) W_n(\lambda_i, \cdot) \big) \\
        & \leq \frac{2}{n W_n(\lambda_i, \cdot) } \Big(  \frac{ \degree_n(i)[ \bm{\lambda}_{n, -i} ] }{  (n-1) W_n(\lambda_i, \cdot) } + 2 \Big),
    \end{align*}
    where we used the inequality $(a-b)^2 \leq 2(a^2 + b^2)$ to obtain the penultimate line, and the inequality $1/(n-1) \leq 2/n$ in the last. With this, we apply a self-bounding exchangeable pair concentration inequality \citep[Theorem~3.9]{chatterjee_concentration_2005} which states that for an exchangeable pair $(X, X')$ and mean-zero function $f(X)$, if we have that the associated variance function $v(X)$ (see Equation~\ref{eq:loss_converge_proof:exch_var} in Section~\ref{app:loss_converge_proof:sec:average_over_adjacency} for a recap) satisfies $v(X) \leq B f(X) + C$, then we have that 
    \begin{equation}
        \label{eq:summary_stats_converge:self_bound_conc}
        \mathbb{P}\Big( | f(X) | \geq t \Big) \leq 2 \exp\Big( \frac{ -t^2}{2C + 2Bt} \Big).
    \end{equation}

    For b), by part a) and taking a union bound, we get that 
    \begin{equation*}
        \mathbb{P}\Big( \max_{i \in [n] } \Big| \frac{ \degree_n(i) }{ (n-1) \rho_n W(\lambda_i, \cdot) } - 1 \Big| \geq t \Big) \leq 2n \mathbb{E}\Big[ \exp\Big( \frac{ -n \rho_n t^2  W(\lambda, \cdot) }{4 ( 1 + 2t) } \Big) \Big]
    \end{equation*}
    where the expectation is over $\lambda \sim U(0, 1)$. If there exists a constant $c_W > 0$ such that $W(\lambda, \cdot) \geq c_W$ a.e, then we can upper bound this expectation by $2n\exp(- c_W n \rho_n t^2/4(1+2t) )$. Consequently, if one takes $t = C (\log n/ n \rho_n)^{1/2}$ for some $C$ sufficiently large, this quantity will decay towards zero as $n \to \infty$, giving us the first part of the result. For the second part of b), note that for a positive random variable $X$ we have
    \begin{align*}
        \mathbb{E}[ e^{-\lambda X} ] = \mathbb{E}\Big[ \int_X^{\infty} \lambda e^{-\lambda t} \, dt \Big] = \mathbb{E}\Big[ \int_0^{\infty} 1[X \leq t] \lambda e^{-\lambda t} \, dt \Big] = \int_0^{\infty} \lambda e^{-\lambda t} \mathbb{P}\big( X \leq t \big) \, dt
    \end{align*}
    by Fubini's theorem, and therefore we get that 
    \begin{equation*}
        2n \mathbb{E}\Big[ \exp\Big( \frac{ -n \rho_n t^2  W(\lambda, \cdot) }{4 ( 1 + 2t) } \Big) \Big] = 2n \lambda(n, t) \int_0^{\infty} e^{- s \lambda(n, t) } \mathbb{P}\big( W(\lambda, \cdot) \leq s ) \, ds.
    \end{equation*}
    where we write $\lambda(n, t) = n \rho_n t^2/4(1+2t)$. When $W(\lambda, \cdot)^{-1} \in L^{\gamma_d}([0, 1]^2)$ for some $\gamma_d > 1$, as a consequence of Markov's inequality we get that $\mathbb{P}( W(\lambda, \cdot) \leq s ) \leq C s^{\gamma_d}$ for some constant $C > 0$, and consequently that 
    \begin{equation*}
        2n \lambda(n, t) \int_0^{\infty} e^{- s \lambda(n, t) } \mathbb{P}\big( W(\lambda, \cdot) \leq s ) \, ds \leq 2C n \lambda(n, t) \int_0^{\infty} s^{\gamma_d} e^{- s \lambda(n, t) }  \, ds = \frac{ 2 C n \Gamma(\gamma_d + 1) }{ \lambda(n, t)^{\gamma_d} }.
    \end{equation*}
    In particular, if one takes $t = C ( n^{(\gamma_d - 1)/\gamma_d } \rho_n )^{-1/2}$, then for any $\epsilon > 0$ one can choose $C$ sufficiently large such that the RHS is less than $\epsilon$ for $n$ sufficiently large, and so we get the stated result.

    For c), we note that by the prior result that 
    \begin{equation*}
        \degree_n(i) = (n - 1) \rho_n W(\lambda_i, \cdot) \cdot \Big( 1 + O_p( r_n ) \Big)
    \end{equation*}
    holds uniformly across all the vertices, and $r_n = (\log n/ n \rho_n)^{1/2}$ if $\gamma_d = \infty$ or $r_n = (n^{(\gamma_d - 1)/\gamma_d } \rho_n )^{-1/2}$ if $\gamma_d \in (1, \infty)$. As a result of the delta method (by considering the function $f(x) = x^{-1}$ about $x = 1$), it therefore follows that 
    \begin{equation*}
        \frac{1}{ \degree_n(i) } = \frac{ 1}{ (n-1) \rho_n W(\lambda_i, \cdot) } \big( 1 + O_p(r_n) \big)
    \end{equation*}
    holds uniformly across all vertices too. With these two results, it follows that to study the minimum degree (or maximum reciprocal degree) we can instead focus on the i.i.d sequence $W(\lambda_i, \cdot)$. In the case where $W(\lambda, \cdot)$ is bounded away from zero (i.e when $\gamma_d = \infty$), $W(\lambda_i, \cdot)^{-1}$ is bounded above and consequently
    \begin{equation*}
        \frac{1}{\degree_n(i) } \leq \frac{O_p(1)}{ n \rho_n W(\lambda_i, \cdot) } \leq O_p( (n \rho_n)^{-1} ).
    \end{equation*}
    In the case where $\gamma_d < \infty$, the fact that $\mathbb{P}( W(\lambda, \cdot)^{-1} \geq s ) \leq C s^{-\gamma_d}$ implies that $W(\lambda_i, \cdot)^{-1}$ has tails dominated by a Pareto distribution with shape parameter $\gamma_d$ and scale parameter $1$. It is known from extreme value theory that the maximum of $n$ i.i.d such random variables, say $Z_n$, is such that $ n^{-1/\gamma} Z_n = O_p(1)$ \citep[Example~21.15]{vaart_asymptotic_1998}, and consequently we have that $\max_{i \in [n] } W(\lambda_i, \cdot)^{-1}$ is $O_p( n^{1/\gamma_d} )$. Combining this all together gives that $\max_{i \in [n] } \degree_n(i)^{-1} = O_p\big( (n^{(\gamma_d - 1)/\gamma_d } \rho_n)^{-1} \big)$. As the minimum degree is the reciprocal of the maximum of the $\degree_n(i)^{-1}$, the other part follows immediately. 

    For d), we choose a similar exchangeable pair as above, except we now no longer work conditional on some $\lambda_i$ (and choose $J \sim \mathrm{Unif}[n]$), in which case we see that 
    \begin{align*}
        \mathbb{E}\Big[  \frac{ \sum_{i=1}^n W_n(\lambda_i, \cdot)^{\alpha} }{ \rho_n^{\alpha} \mathcal{E}_W(\alpha) } & - \frac{ \sum_{i=1}^n W_n(\tilde{\lambda}_i, \cdot)^{\alpha} }{ \rho_n^{\alpha} \mathcal{E}_W(\alpha) } \,\Big|\, \bm{\lambda}_n \Big]  = \frac{ \sum_{i=1}^n W_n(\lambda_i, \cdot)^{\alpha} }{n \rho_n^{\alpha} \mathcal{E}_W(\alpha) } - 1 
    \end{align*}
    and we get an associated stochastic variance term 
    \begin{align*}
        v(\bm{\lambda}_n) &:= \frac{1}{2n} \mathbb{E}\Big[ \Big( \frac{ \sum_{i=1}^n W_n(\lambda_i, \cdot)^{\alpha} }{ \rho_n^{\alpha} \mathcal{E}_W(\alpha) }  - \frac{ \sum_{i=1}^n W_n(\tilde{\lambda}_i, \cdot)^{\alpha} }{ \rho_n^{\alpha} \mathcal{E}_W(\alpha) } \Big)^2 \,\Big|\, \bm{\lambda}_n \Big] \\
        & = \frac{1}{2n^2 \mathcal{E}_W(\alpha)^2 } \sum_{i=1}^n \mathbb{E}\big[ \big( W(\lambda_i, \cdot)^{\alpha} - W(\lambda'_i, \cdot)^{\alpha} \big)^2  \,\big|\, \lambda_i \big] \nonumber \\
        & \leq \frac{1}{n^2 \mathcal{E}_W(\alpha)^2 } \sum_{i=1}^n \big\{ W(\lambda_i, \cdot)^{2 \alpha} + \mathcal{E}(2\alpha) \big\} \leq \frac{1}{n \mcE_W(\alpha) } \Big[ \frac{ \sum_{i=1}^n W_n(\lambda_i, \cdot)^{\alpha} }{n \rho_n^{\alpha} \mathcal{E}_W(\alpha) } + 1 \Big] 
    \end{align*}
    where in the last line we used the inequalities $(a-b)^2 \leq 2(a^2 + b^2)$, $W(\lambda, \cdot)^{2\alpha} \leq W(\lambda, \cdot)^{\alpha}$ and $\mcE(2\alpha) \leq \mcE(\alpha)$ (the last two hold as $W(\lambda, \cdot) \in [0, 1]$). We get the stated concentration inequality by applying \eqref{eq:summary_stats_converge:self_bound_conc}.

    For the concentration of the edge set in e), we will form an exchangeable pair $(\bm{A}_n, \tilde{\bm{A}}_n)$ by drawing a vertex $I$ uniformly at random from $[n]$, then letting (for $j < k$) $\tilde{a}_{jk} = a_{jk}$ if $j, k \neq I$ and otherwise redrawing $\tilde{a}_{jk} | \lambda_j, \lambda_k \sim \mathrm{Bern}(W(\lambda_j, \lambda_k) )$ if either $j = I$ or $k = I$. We then set $\tilde{a}_{jk} = \tilde{a}_{kj}$ for $k > j$. If we define 
    \begin{equation*}
        F(\bm{A}_n, \tilde{\bm{A}}_n ) = \frac{1}{(n-1) \rho_n \mcE_W} \Big( \sum_{i < j} a_{ij} - \sum_{i < j} \tilde{a}_{ij} \Big)
    \end{equation*}
    then we can calculate that 
    \begin{align*}
        \mathbb{E}\big[ F(\bm{A}_n, \tilde{\bm{A}}_n) \,|\, \bm{A}_n \big] &= \frac{1}{(n-1) \rho_n \mcE_W} \cdot \frac{1}{n} \sum_{k = 1}^n \Big\{ \sum_{ \substack{i < j \\ i \text{ or } j = k } } a_{ij} - \sum_{\substack{i < j \\ i \text{ or } j = k } } \rho_n \mcE_W \Big\} \\ 
        & = \frac{ 2 \sum_{i < j} a_{ij} }{ n(n-1) \rho_n \mcE_W } - 1.
    \end{align*}
    The associated stochastic variance term is then of the form, letting $(a'_{ij})$ be an independent copy of $(a_{ij})$,
    \begingroup 
    \allowdisplaybreaks
    \begin{align*}
        v(\bm{A}_n) & = \frac{1}{ n (n-1)^2 \rho_n^2 \mcE_W^2 } \mathbb{E}\Big[ \Big( \sum_{i < j} a_{ij} - \sum_{i < j} \tilde{a}_{ij} \Big)^2 \,|\, \bm{A}_n \Big] \\
        & = \frac{1}{ n (n-1)^2 \rho_n^2 \mcE_W^2 } \cdot \frac{1}{n} \sum_{k=1}^n \mathbb{E}\Big[ \Big( \sum_{\substack{i < j \\ i \text{ or } j = k }} a_{ij} - a_{ij}' \Big)^2 \,|\, \bm{A}_n \Big] \\
        & \leq \frac{1}{ n (n-1)^2 \rho_n^2 \mcE_W^2 } \sum_{k=1}^n \sum_{\substack{i < j \\ i \text{ or } j = k }} \mathbb{E}\big[ ( a_{ij} - a'_{ij} )^2 \,|\, \bm{A}_n \big] \\ 
        & \leq \frac{2 \sum_{i < j} a_{ij} + 2n(n-1) \rho_n \mcE_W }{ n(n-1)^2 \rho_n^2  \mcE_W  } \leq \frac{2}{n \rho_n \mcE_W} \Big[ \frac{2 \sum_{i < j} a_{ij} }{ n (n-1) \rho_n \mcE_W} + 2 \Big],
    \end{align*}
    \endgroup
    where the first inequality follows by Cauchy-Schwarz, the second by using the inequality $(a-b)^2 \leq 2(a^2 + b^2) = 2(a+b)$ when $a, b \in \{0, 1\}$, and the third by using the inequality $1/(n-1) \leq 2/n$. The stated concentration inequality then holds by applying \eqref{eq:summary_stats_converge:self_bound_conc}.
    % For the concentration of the edge set in e), we will form an exchangeable pair $(\bm{A}_n, \tilde{\bm{A}}_n )$ by drawing a pair of vertices $(I, J)$ uniformly from the set $\{ i < j \,:\, i, j \in [n] \}$, redrawing $\tilde{a}_{I, J} \sim \mathrm{Bern}( W_n(\lambda_I, \lambda_J) )$ but otherwise keeping $\tilde{a}_{i, j} = a_{i, j}$ for $(i, j) \neq (I, J)$. Consequently we see that 
    % \begin{equation}
    %     \mathbb{E}\Big[ \frac{ \sum_{i < j} a_{i, j} }{ \rho_n \mathcal{E}_W } -  \frac{ \sum_{i < j} \tilde{a}_{i, j} }{ \rho_n \mathcal{E}_W } \,\Big|\, \bm{A}_n \Big] = \frac{2 \sum_{i < j} a_{i, j} }{ n(n-1) \rho_n \mathcal{E}_W } - 1
    % \end{equation}
    % and get an associated stochastic variance term 
    % \begin{align}
    %     v(\bm{A}_n) & = \frac{2}{n(n-1) } \mathbb{E}\Big[ \Big( \frac{ \sum_{i < j} a_{i, j} }{ \rho_n \mathcal{E}_W } -  \frac{ \sum_{i < j} \tilde{a}_{i, j} }{ \rho_n \mathcal{E}_W } \Big)^2 \,\Big| \, \bm{A}_n \Big] \\
    %     & = \frac{1}{n^2 (n-1)^2 \rho_n^2 \mathcal{E}_W^2 } \sum_{i < j} \mathbb{E}[ (a_{i, j} - a_{i, j}' )^2 \,|\, a_{i, j} ] \nonumber \\
    %     & \leq \frac{2 \sum_{i < j} a_{i,j} + 2n(n-1) \rho_n \mathcal{E}_W }{n^2(n-1)^2 \rho_n^2 \mathcal{E}_W^2} \leq \frac{2}{n^2 \rho_n \mcE_W} \Big[  \frac{2 \sum_{i < j} a_{i,j} }{ n(n-1) \rho_n \mcE_W} + 2 \Big]
    % \end{align}
    % which we can use to give the stated concentration inequality. 

    For part f), we simply combine some of the earlier parts, and write
    \begin{align*}
        \Big| \frac{ \degree_n(v) }{2 E_n} & \cdot \frac{ n \mcE_W}{ W(\lambda_v, \cdot) } - 1 \Big| \leq \frac{ n^2 \rho_n \mcE_W}{ 2 E_n} \cdot \Big| \frac{ \degree_n(v) }{n \rho_n W(\lambda_v, \cdot) } - 1 \Big| + \Big| \frac{ n^2 \rho_n \mcE_W }{ 2 E_n} - 1 \Big| = O_p( \tilde{s}_n ),
    \end{align*}
    where $\tilde{s}_n$ is the rate obtained from part b).
\end{proof}
%!TEX root = ..\..\ms.tex

\phantomsection\label{app:sampling:random_walk_and_unigram:proof}
\begin{proposition} \label{app:sampling:random_walk_and_unigram}
    Write $E_n := E[\mcG_n]$, and let $\pi(\cdot \,|\, \mathcal{G}_n)$ be the stationary distribution of a simple random walk on $\mathcal{G}_n$, so $\pi(v \,|\, \mathcal{G}_n) = \degree_n(v)/2E_n$ for all $v \in \mcV_n$, and let $(\tilde{v}_i)_{i \geq 1}$ be a simple random walk on $\mathcal{G}_n$ where $\tilde{v}_1 \sim \pi(\cdot \,|\, \mcG_n)$. Write
    \begin{equation*}
        Q_{k}(v \,|\, \mcG_n) = \mathbb{P}\big( \tilde{v}_i = v \text{ for some } i \leq k \,|\, \mcG_n \big) \text{ and }
        \mathrm{Ug}_{\alpha}(v \,|\, \mcG_n) = \frac{ Q_k(v \,|\, \mcG_n)^{\alpha} }{ \sum_{u \in \mcV_n} Q_k(u \,|\, \mcG_n)^{\alpha} }
    \end{equation*}
    be the corresponding unigram distribution for any $\alpha > 0$. Suppose that Assumption~\ref{assume:graphon_ass} also holds with $\gamma_d \in (1, \infty]$. Then for $k \geq 3$, we have that 
    \begin{equation*}
        \max_{v \in \mcV_n} \Big| \frac{ Q_k(v \,|\, \mcG_n) }{ k W(\lambda_v, \cdot) / n \mcE_W  } - 1 \Big| = O_p\big( \tilde{s}_n(\gamma_d) \big) \text{ and }
        \max_{v \in \mcV_n} \Big| \frac{ \mathrm{Ug}_{\alpha}(v \,|\, \mcG_n)  }{  W(\lambda_v, \cdot)^{\alpha} /  n \mcE_W(\alpha) } - 1 \Big| = O_p\big( \tilde{s}_n(\gamma_d) \big)
    \end{equation*}
    where $\tilde{s}_n(\gamma_d) = ( n^{(\gamma_d - 1) / \gamma_d} \rho_n )^{-1/2}$ if $\gamma_d \in (1, \infty)$ and $\tilde{s}_n(\infty) = ( \log(n) /n \rho_n )^{1/2}$.
\end{proposition}

\begin{proof}[Proof of Proposition~\ref{app:sampling:random_walk_and_unigram}]
    We begin by handling the probability that a vertex is sampled in a simple random walk of length $k$; the idea is to show that the self-intersection probability of the walk is negligible. Note that by stationarity of the simple random walk we have for all $i$ that
    \begin{equation*}
        \mathbb{P}\big( \tilde{v}_i = v \,|\, \mcG_n \big) = \frac{ \mathrm{deg}_n(v) }{2 E_n}.
    \end{equation*}
    Also note that for any sequence of events $A_i$, we have that 
    \begin{equation*}
        \Big( \sum_{i=1}^k \mathbbm{1}[A_i] \Big) - \mathbbm{1}[\cup_{j=1}^k A_j ] = \sum_{i=1}^{k-1} \mathbbm{1}[A_i \cap \cup_{j > i} A_j]
    \end{equation*}
    (simply consider the LHS and RHS when $x \in A_i$ exactly when $i \in S \subseteq [k]$). Therefore if we let $A_i = \{ \tilde{v}_i = v\}$ and take expectations, we get the inequality 
    \begingroup 
    \allowdisplaybreaks
    \begin{align*}
        \big| Q_k(v \,|\, \mcG_n) & - \frac{k \mathrm{deg}_n(v) }{2 E_n} \big| = \big| Q_k(v \,|\, \mcG_n) - \sum_{i=1}^k \mathbb{P}\big( \tilde{v}_i = v \,|\, \mcG_n \big) \big| \\
        & \leq \sum_{i=1}^{k-1} \mathbb{P}\big( \tilde{v}_i = v, \tilde{v}_j = v \text{ for some } j \in [i+1, k] \,|\, \mcG_n \big) \\
        & = \sum_{i=1}^{k-1} \mathbb{P}( \tilde{v}_i = v \,|\, \mcG_n) \mathbb{P}\big( \tilde{v}_j = v \text{ for some } j \in [i+1, k] \,|\, \mcG_n, \tilde{v}_i = v \big) \\
        & = \frac{\degree_n(v) }{2 E_n} \sum_{i=1}^{k-1} \mathbb{P}\big( \tilde{v}_j = v \text{ for some } j \in [2, k-i+1] \,|\, \mcG_n, \tilde{v}_1 = v \big) \\
        & \leq \frac{ k \degree_n(v) }{2 E_n} \mathbb{P}\big( \tilde{v}_j = v \text{ for some } j \in [2, k] \,|\, \mcG_n, \tilde{v}_1 = v \big)
    \end{align*}%
    \endgroup
    To proceed with bounding the self-intersection probability, write $N(v \,|\, \mcG_n)$ for the set of neighbours of a vertex $v$ in $\mcG_n$, so by the Markov property we can write 
    \begingroup
    \allowdisplaybreaks
    \begin{align*} 
        \mathbb{P}\big( \tilde{v}_j & = v \text{ for some } j \in [2, k] \,|\, \mcG_n, \tilde{v}_1 = v\big)  \\
        & = \sum_{u \in N(v \,|\, \mcG_n) } \mathbb{P}\big( \tilde{v}_j = v \text{ for some } j \in [3, k] \,|\, \mcG_n, \tilde{v}_2 = u \big) \mathbb{P}\big( \tilde{v}_2 = u \,|\, \tilde{v}_1 = v \big) \\
        & = \sum_{u \in N(v \,|\, \mcG_n) } \frac{ 2 E_n }{ \degree_n(u) \degree_n(v) } \mathbb{P}( \tilde{v}_j = v \text{ for some } j \in [3, k] \,|\, \mcG_n, \tilde{v}_2 = u \big) \mathbb{P}\big( \tilde{v}_2 = u \,|\, \mcG_n \big) \\
        & \leq \sum_{u \in \mcV_n}  \frac{ 2 E_n }{ \degree_n(u) \degree_n(v) } \mathbb{P}( \tilde{v}_j = v \text{ for some } j \in [3, k] \,|\, \mcG_n, \tilde{v}_2 = u \big) \mathbb{P}\big( \tilde{v}_2 = u \,|\, \mcG_n \big) \\
        & \leq Q_{k-2}(v \,|\, \mcG_n) \max_{u \in \mcV_n} \frac{ 2 E_n }{ \degree_n(u) \degree_n(v) } \leq (k-2) \max_{u \in \mcV_n} \frac{1}{\degree_n(u) },
    \end{align*}
    \endgroup
    where in the last line we pulled the max term out of the summation, used stationarity of the simple random walk, and that $Q_k(v \,|\, \mcG_n) \leq k \degree_n(v) / 2 E_n$ for all $k$. By part c) of Proposition~\ref{app:sampling:summary_stats_converge}, it therefore follows that
    \begin{equation*}
        \max_{v \in \mcV_n} \Big| \frac{ Q_k(v \,|\, \mcG_n) }{ k \degree_n(v) / 2E_n } - 1 \Big| = \begin{cases} O_p\Big( (n \rho_n)^{-1} \Big) & \text{ if } \gamma_d = \infty, \\
            O_p\Big( \big( n^{(\gamma_d - 1) / \gamma_d} \rho_n \big)^{-1} \Big) & \text{ if } \gamma_d \in (1, \infty). \end{cases}
    \end{equation*}
    By part f) of Proposition~\ref{app:sampling:summary_stats_converge}, we can then control the denominator to find that 
    \begin{equation*}
        \max_{v \in \mcV_n} \Big| \frac{ Q_k(v \,|\, \mcG_n) }{ k W(\lambda_v, \cdot) / n \mcE_W  } - 1 \Big| = O_p\big( \tilde{s}_n(\gamma_d) \big).
    \end{equation*}
    For the large sample behaviour of the unigram distribution, we may then deduce that
    \begin{align*}
        & \Big| \frac{ \sum_{u \in \mcV_n} Q_k(u \,|\, \mcG_n)^{\alpha} - \sum_{u \in \mcV_n} (k W(\lambda_u, \cdot) / n\mcE_W )^{\alpha} }{   \sum_{u \in \mcV_n} (k W(\lambda_u, \cdot) / n\mcE_W )^{\alpha}   } \Big| \\
        & \qquad \qquad \qquad \qquad \qquad \qquad \qquad  \leq \max_{u \in \mcV_n} \Big| \frac{ Q_k(u \,|\, \mcG_n)^{\alpha} }{  (k W(\lambda_u, \cdot) / n\mcE_W )^{\alpha}   } - 1 \Big| =  O_p\big( \tilde{s}_n(\gamma_d) \big)
    \end{align*}
    for any $\alpha > 0$ (where we used Lemma~\ref{app:loss_converge_proof:cti} followed by the delta method applied to $f(x) = x^{\alpha}$). Combining this with part d) of Proposition~\ref{app:sampling:summary_stats_converge} then allows us to get the desired conclusion.
\end{proof}

\subsection{Sampling formula for different sampling schemes}

Here it will be convenient to define the rate function
\begin{equation*}
    \tilde{s}_n(\gamma) = \begin{cases} 
        ( n^{(\gamma-1)/\gamma } \rho_n )^{-1/2} & \text{ if } \gamma \in (1, \infty), \\
        ( \log(n) )^{1/2} ( n \rho_n)^{-1/2} & \text{ if } \gamma = \infty
    \end{cases}
\end{equation*}
which depends on the choice of the sparsifying sequence $\rho_n$ used to generate the model; we note that $\tilde{s}_n(\gamma_d) = o(1)$ under
our assumptions. Propositions~\ref{ass:sampling:psamp_formula}~to~\ref{ass:sampling:rw_uni_stat_formula} correspond to Propositions~\ref{sec:sampling:psamp_formula}~to~\ref{sec:sampling:rw_uni_stat_formula} in
Section~\ref{sec:sampling_formula}.

\begin{proposition} \label{ass:sampling:psamp_formula}
    Suppose that Assumption~\ref{assume:graphon_ass} holds. Then for Algorithm~\ref{alg:psamp}, Assumptions~\ref{assume:slc} and \ref{assume:samp_weight_reg} hold with
    \begin{equation*}
        f_n(\lambda_i, \lambda_j, a_{ij} ) = k(k-1), 
        %\qquad g_n(\lambda_i) = k
    \end{equation*}
    $s_n = 0$, $\mathbb{E}[f_n^2] = \rho_n k^2(k-1)^2$ and $\beta = \beta_W$ and $\gamma_s = \gamma_W$.
\end{proposition}

\begin{proof}[Proof of Proposition~\ref{ass:sampling:psamp_formula}]
    Here a vertex is sampled with probability $k/n$, and any two distinct vertices are sampled with probability $k(k-1)/n(n-1)$; the stated formulae therefore follow immediately. We then calculate that $\mathbb{E}[f_n(\lambda_i, \lambda_j, a_{ij} )^2] = k^2(k-1)^2$ and $\| \fnone \|_{\infty}, \| \fnzero \|_{\infty} \leq k(k-1)$. Under the stated assumptions, the integrability conditions on $\fnone$ and $\fnzero$ then follow directly.
\end{proof}

\begin{proposition} \label{ass:sampling:unif_edge_uni_formula}
    Suppose that Assumption~\ref{assume:graphon_ass} holds. Then for Algorithm~\ref{alg:unifedge+ns}, Assumptions~\ref{assume:slc} and \ref{assume:samp_weight_reg} hold with 
    \begin{align*}
        f_n(\lambda_i, \lambda_j, a_{ij} ) & = \begin{dcases*}
          \frac{2k}{\mcE_W \rho_n} & if $a_{ij} = 1$, \\
          \frac{ 2k l }{ \mcE_W \mcE_W(\alpha) } \big\{ W(\lambda_i, \cdot) W(\lambda_j, \cdot)^{\alpha} + W(\lambda_j, \cdot) W(\lambda_i, \cdot)^{\alpha} \big\} & if $a_{ij} = 0$;
    \end{dcases*}
    %g_n(\lambda_i) & = \frac{2k}{\mcE_W }\Big( W(\lambda_i, \cdot) +  \frac{ l W(\lambda_i, \cdot)^{\alpha} }{ \mcE_W(\alpha)} \cdot \int_0^1 (1 - \rho_n W(\lambda_i, y) ) W(y, \cdot) \, dy \Big)
    \end{align*}
    with $s_n = \tilde{s}_n(\gamma_d)$, $\mathbb{E}[f_n^2] = O(\rho_n^{-1})$, and $\beta = \beta_W \min\{ \alpha, 1 \}$ and $\gamma_s = \min\{ \gamma_W, \gamma_d, \gamma_d/\alpha \}$.
\end{proposition}

%!TEX root = ..\..\ms.tex

\phantomsection\label{sampling:unif_edge_uni_formula:proof}
\begin{proof}[Proof of Proposition~\ref{ass:sampling:unif_edge_uni_formula}]
    Let $S_0(\mcG_n)$ denote the $k$ edges which are sampled without replacement from the edge set of $\mcG_n$, and recall that $E_n = E[\mcG_n]$ denotes the number of edges of $\mcG_n$.
    We then have that
    \begin{equation*}
        \mathbb{P}\big( (u, v) \in S_0(\mcG_n) \,|\, \mcG_n \big) = a_{uv} { E_n - 1 \choose k - 1} { E_n \choose k}^{-1} = \frac{k  a_{uv}}{  E_n } =  \frac{2k a_{uv}}{\mcE_W \rho_n n^2} \big( 1 + O_p( (n\rho_n)^{-1/2} ) \big)
    \end{equation*}
    where we note that the $O_p(\cdot)$ term has no dependence on $u$ or $v$. Note by Lemma~\ref{app:sampling:binom_coeff_lemma} we have that 
    \begin{equation*}
        1 - { E_n - \degree_n(u) \choose k} { E_n \choose k}^{-1} 
        = \frac{k \degree_n(u) }{E_n } \Big( 1 + O\Big( \frac{\degree_n(u) }{E_n} \Big) \Big) =  \frac{k \degree_n(u) }{E_n } \Big( 1 + O_p( n^{-1} ))
    \end{equation*}
    uniformly across all vertices $u$, and consequently 
    \begingroup 
    \allowdisplaybreaks
    \begin{align*}
        \mathbb{P}\big( u &\in \mcV(S_0(\mcG_n)) \,|\, \mcG_n \big) = 1 - \mathbb{P}\big( \text{no edge containing a vertex $u$ is sampled from $\mcE_n$} \,|\, \mcG_n \big) \\
        & = 1 - { E_n - \degree_n(u) \choose k} { E_n \choose k}^{-1} 
        = \frac{k \degree_n(u) }{E_n } \big( 1 + O_p\big( n^{-1} \big) \big) \\
        & = \frac{2 k W(\lambda_u, \cdot) }{ \mcE_W n} \big( 1 + O_p\big( \tilde{s}_n(\gamma_d) \big) \big) \nonumber 
    \end{align*}
    \endgroup
    where the last equality follows by Proposition~\ref{app:sampling:summary_stats_converge}. The same arguments as in Proposition~\ref{app:sampling:random_walk_and_unigram} tell us that
    \begin{equation}
        \label{eq:unif_edge_uni_formula:unigram}
        \mathrm{Ug}_{\alpha}\big( v \,|\, \mcG_n \big) =  \frac{ W(\lambda_v, \cdot)^{\alpha} }{ n \mcE_W(\alpha) } \big( 1 + O_p\big( \tilde{s}_n(\gamma_d) \big) \big).
    \end{equation}
    With this, we are now in a position to derive the sampling formula for the specified sampling scheme. As $(u, v)$ can only be part of $S_0(\mcG_n)$ or $S_{ns}(\mcG_n)$ (not both), we can write that 
    \begingroup
    \allowdisplaybreaks
    \begin{align*}
        \mathbb{P}\big( (u, v) \in S(\mcG_n)  & \,|\, \mcG_n \big)  = \mathbb{P}\big( (u, v) \in S_0(\mcG_n) \,|\, \mcG_n \big) + \mathbb{P}\big( (u, v) \in S_{ns}(\mcG_n) \,|\, \mcG_n \big) &  \\
        & =  \frac{2k a_{uv}}{\mcE_W \rho_n n^2} \big( 1 + O_p( (n\rho_n)^{-1/2} ) \big) & \nonumber \\
        & \qquad  + \mathbb{P}\big( u \in \mcV(S_0(\mcG_n)), v \not\in \mcV(S_0(\mcG_n)), (u, v) \in S_{ns}(\mcG_n) \,|\, \mcG_n \big) & \text{(I)} \nonumber   \\
        & \qquad  + \mathbb{P}\big( u \not\in \mcV(S_0(\mcG_n)), v \in \mcV(S_0(\mcG_n)), (u, v) \in S_{ns}(\mcG_n) \,|\, \mcG_n \big) & \text{(II)}  \nonumber \\
        & \qquad  + \mathbb{P}\big( u, v \in \mcV(S_0(\mcG_n)), (u, v) \not\in S_0(\mcG_n), (u, v) \in S_{ns}(\mcG_n) \,|\, \mcG_n \big). & \text{(III)} \nonumber 
    \end{align*}
    \endgroup
    We begin with (I) and (II); as they are symmetric in $(u, v)$ we can just consider (I). Writing on occasion $\mcV_0 = \mcV(S_0(\mcG_n))$ for reasons of space, we have
    \begingroup
    \allowdisplaybreaks
    \begin{align*}
        \mathbb{P}&\big( u \in \mcV_0, v \not\in \mcV_0, (u, v) \in S_{ns}(\mcG_n) \,|\, \mcG_n \big) \\
        & = \mathbb{P}\big( (u, v) \in S_{ns}(\mcG_n) \,|\, u \in \mcV_0, v \notin \mcV_0, \mcG_n \big) \mathbb{P}\big( u \in \mcV_0, v \notin \mcV_0 \,|\, \mcG_n \big) \nonumber \\
        & = (1 - a_{uv} ) \mathbb{P}\big( B(l, \mathrm{Ug}_{\alpha}(v \,|\, \mcG_n) ) \geq 1 \big) \cdot \Big[ \mathbb{P}\big( v \not\in \mcV_0 \,|\, \mcG_n \big) - \mathbb{P}\big( u, v \not\in \mcV_0 \,|\, \mcG_n \big) \Big]. \nonumber
    \end{align*}
    \endgroup
    By Lemma~\ref{app:sampling:binom_coeff_lemma} and \eqref{eq:unif_edge_uni_formula:unigram}, we know that 
    \begin{equation*}
        \mathbb{P}\big( B(l, \mathrm{Ug}_{\alpha}(v \,|\, \mcG_n) ) \geq 1 \big) = \frac{ l W(\lambda_v, \cdot)^{\alpha} }{ n \mcE_W(\alpha) } \big( 1 + O_p\big( \tilde{s}_n(\gamma_d) \big).
    \end{equation*}
    As for the $\mathbb{P}\big( v \not\in \mcV(S_0(\mcG_n)) \,|\, \mcG_n \big) - \mathbb{P}\big( u, v \not\in \mcV(S_0(\mcG_n)) \,|\, \mcG_n \big)$ term, we note that it equals (as without loss of generality we can assume $a_{uv} = 0$)
    \begingroup 
    \allowdisplaybreaks
    \begin{align*}
        -\mathbb{P}\big( v &\in \mcV(S_0(\mcG_n)) \,|\, \mcG_n \big) + 1 - \mathbb{P}\big( u, v \not\in \mcV(S_0(\mcG_n)) \,|\, \mcG_n \big) \\
        &=  -1 + { E_n - \degree_n(v) \choose k} {E_n \choose k }^{-1}   + 1 - { E_n - \degree_n(u) - \degree_n(v) \choose k }{ E_n \choose k }^{-1} \\
        & = \frac{ 2 k W(\lambda_u, \cdot) }{n \mcE_W }  \big( 1 + O_p\big( \tilde{s}_n(\gamma_d) \big) \big) 
    \end{align*}
    \endgroup
    by Lemma~\ref{app:sampling:binom_coeff_lemma}, and whence
    \begin{equation*}
        \text{(I)} = (1 - a_{uv} ) \frac{ 2 kl W(\lambda_v, \cdot)^\alpha W(\lambda_u, \cdot) }{n^2 \mcE_W \mcE_W(\alpha) } \big( 1 + O_p\big( \tilde{s}_n(\gamma_d) \big) \big).
    \end{equation*}
    For (III), we begin by noting that as 
    \begin{equation*}
        \mathbb{P}( A \cap B ) = \mathbb{P}(A) + \mathbb{P}(B) - (1 - \mathbb{P}(A^c \cap B^c) )
    \end{equation*}
    for any events $A$ and $B$, we have by Lemma~\ref{app:sampling:binom_coeff_lemma2} that
    \begingroup 
    \allowdisplaybreaks
    \begin{align*}
        \mathbb{P}\big( u, v \in \mcV(S_0(\mcG_n)) \big) & = 1 - { E_n - \degree_n(u) \choose k}{ E_n \choose k }^{-1} + 1 - { E_n - \degree_n(v) \choose k }{ E_n \choose k}^{-1} \\
        & - \Bigg( 1 - { E_n - \degree_n(u) - \degree_n(v) + a_{uv} \choose k}{ E_n \choose k }^{-1} \Bigg) \\
        & = \Big(  \frac{2 k a_{uv} }{ n^2 \rho_n \mcE_W } + \frac{4k(k-1) W(\lambda_u, \cdot) W(\lambda_v, \cdot) }{ \mcE_W^2 n^2} \Big) \cdot \big( 1 + O_p\big( \tilde{s}_n(\gamma_d) \big) \big). 
    \end{align*}
    \endgroup
    As by a similar argument to above we know that 
    \begin{equation*}
        \mathbb{P}\big( (u, v) \in S_{ns}(\mcG_n) \,|\, u, v \in \mcV(S_0(\mcG_n)) ) = (1 - a_{uv} ) \frac{ l( W(\lambda_u, \cdot)^{\alpha} + W(\lambda_v, \cdot)^{\alpha} )) }{   n \mcE_W(\alpha) } \big( 1 + O_p\big( \tilde{s}_n(\gamma_d) \big) \big),
    \end{equation*}
    it therefore follows that the (III) term will be asymptotically negligible, leaving us with the sampling formula
    \begin{align*}
        \mathbb{P}\big( (u, v) & \in S(\mcG_n) \,|\, \mcG_n \big) = a_{uv} \cdot  \frac{ 2k }{ n^2 \mcE_W \rho_n } \big( 1 + O_p\big( (n\rho_n)^{-1/2} \big) \big) \\
        & + (1 - a_{uv} ) \cdot \frac{2kl \{ W(\lambda_u, \cdot) W(\lambda_v, \cdot)^{\alpha} + W(\lambda_v, \cdot) W(\lambda_u, \cdot)^{\alpha} \}  }{n^2 \mcE_W \mcE_W(\alpha) } \big( 1 + O_p\big( \tilde{s}_n(\gamma_d) \big) \big) 
    \end{align*}
    from which we get the stated result for the sampling formula and convergence rate. The remaining properties to check can then be done so via routine calculation and the use of Lemmas~\ref{app:sampling:prod_deg_tails}~and~\ref{app:sampling:prod_deg_holder}.
\end{proof}

\begin{proposition} \label{ass:sampling:unif_edge_induced_formula}
    Suppose that Assumption~\ref{assume:graphon_ass} holds. Then for Algorithm~\ref{alg:unifedge+inducedsg}, Assumptions~\ref{assume:slc} and \ref{assume:samp_weight_reg} hold with 
    \begin{align*}
        f_n(\lambda_i, \lambda_j, a_{ij} ) & = \begin{dcases*}
         \frac{4k}{\mcE_W \rho_n} + \frac{ 4k(k-1) W(\lambda_i, \cdot) W(\lambda_j, \cdot) }{ \mcE_W^2  } & if $a_{ij} = 1$, \\
         \frac{ 4k(k-1) W(\lambda_i, \cdot) W(\lambda_j, \cdot) }{ \mcE_W^2 }  & if $a_{ij} = 0$;
    \end{dcases*} 
    %g_n(\lambda_i) & = \frac{2k W(\lambda_i, \cdot)}{\mcE_W } 
    \end{align*}
    with $s_n = \tilde{s}_n(\gamma_d)$, $\beta = \beta_W$, and $\mathbb{E}[f_n^2] = O(\rho_n^{-1})$ and $\gamma_s = \min\{ \gamma_d, \gamma_W \}$.  
\end{proposition}

%!TEX root = ..\..\ms.tex

\phantomsection \label{sampling:unif_edge_induced_formula:proof}
\begin{proof}[Proof of Propsition~\ref{ass:sampling:unif_edge_induced_formula}]
    We note that most of the calculations can be taken from Proposition~\ref{sec:sampling:unif_edge_uni_formula}. Begin by noting that $(u, v)$ is selected either as part of $S_0(\mcG_n)$, or $u, v \in \mcV(S_0(\mcG_n))$ but $(u, v)$ is not selected as part of $S_0(\mcG_n)$ (and that these occurrences are mutually exclusive). The probability of the first we know from earlier, and the probability of the second is given by
    \begin{equation*}
        \mathbb{P}\big( u, v \in \mcV(S_0(\mcG_n)) \,|\, (u, v) \not\in S_0(\mcG_n), \mcG_n \big) \cdot \mathbb{P}\big( (u, v) \not\in S_0(\mcG_n) \,|\, \mcG_n \big).
    \end{equation*}
    The second term in the product equals $1 - 2k a_{uv} \mcE_W^{-1} \rho_n^{-1} n^{-2} (1 + O_p( (n\rho_n)^{-1/2} ))$, and the first equals
    \begingroup 
    \allowdisplaybreaks
    \begin{align*}
        1 & - { E_n - \degree_n(u)  \choose k}{ E_n - a_{uv} \choose k }^{-1} + 1 - { E_n -  \degree_n(v)  \choose k }{ E_n - a_{uv} \choose k}^{-1} \\
        & - \Bigg( 1 - { E_n - (\degree_n(u) + \degree_n(v) - a_{uv}) \choose k}{ E_n - a_{uv} \choose k }^{-1} \Bigg) \\
        & = \Big( \frac{k a_{uv} }{ E_n - a_{uv} } + \frac{ k(k-1) \deg_n(u) \deg_n(v) }{ (E_n - a_{uv} )^2 } \Big) (1 + O_p(n^{-1} ) ) \\
        & = \Big( \frac{ 2k a_{uv} }{\mcE_W \rho_n n^2} + \frac{ 4k(k-1) W(\lambda_u, \cdot) W(\lambda_v, \cdot) }{ \mcE_W^2 n^2 } \Big) \big( 1 + O_p\big( \tilde{s}_n(\gamma_d) \big) \big),
    \end{align*}
    \endgroup
    where we have used Lemma~\ref{app:sampling:binom_coeff_lemma2} followed by Proposition~\ref{app:sampling:summary_stats_converge}. It therefore follows that 
    \begin{equation*}
        \mathbb{P}\big( (u, v) \in S(\mcG_n) \,|\, \mcG_n \big) = \Big( \frac{ 4k a_{uv} }{ \mcE_W \rho_n n^2 } + \frac{ 4k(k-1) W(\lambda_u, \cdot) W(\lambda_v, \cdot) }{ \mcE_W^2 n^2 } \Big) \big( 1 + O_p\big( \tilde{s}_n(\gamma_d) \big) \big).
    \end{equation*}
    The remaining properties to check can then be done so via routine calculation and the use of Lemmas~\ref{app:sampling:prod_deg_tails}~and~\ref{app:sampling:prod_deg_holder}.
    %As for the probability that $u \in S(\mcG_n)$, this can occur iff $u \in S_0(\mcG_n)$, which we derived the asymptotic probability for earlier in Proposition~\ref{sec:sampling:unif_edge_uni_formula}.
\end{proof}

\begin{proposition} \label{ass:sampling:rw_uni_stat_formula}
    Suppose that Assumption~\ref{assume:graphon_ass} holds. Then for Algorithm~\ref{alg:unifedge+inducedsg} with choice of initial distribution $\pi_0( v \,|\, \mcG_n) = \degree_n(v)/2E_n$, Assumptions~\ref{assume:slc} and \ref{assume:samp_weight_reg} hold with
    \begin{align*}
        f_n(\lambda_i, \lambda_j, a_{ij} ) & = \begin{dcases*}
         \frac{2k}{\mcE_W \rho_n} & if $a_{ij} = 1$, \\
         \frac{ l(k+1) }{\mcE_W \mcE_W(\alpha) } \big\{ W(\lambda_i, \cdot) W(\lambda_j, \cdot)^{\alpha} + W(\lambda_j, \cdot) W(\lambda_i, \cdot)^{\alpha} \big\} & if $a_{ij} = 0$;
    \end{dcases*} 
    %g_n(\lambda_i) & = \frac{ k W(\lambda_i, \cdot) }{\mcE_W} +  \frac{ (k+1)l W(\lambda_i, \cdot)^{\alpha} }{ \mcE_W(\alpha) \mcE_W } \int_0^1 (1 - \rho_n W(\lambda_i, y) ) W(y, \cdot) \, dy 
    \end{align*}
    with $s_n = \tilde{s}_n(\gamma_d)$, $\mathbb{E}[f_n^2] = O(\rho_n^{-1})$, and $\beta = \beta_W \min\{ \alpha, 1 \}$ and $\gamma_s = \min\{ \gamma_W, \gamma_d, \gamma_d/\alpha \}$.
\end{proposition}

%!TEX root = ..\..\ms.tex

\phantomsection\label{sampling:rw_uni_stat_formula:proof}
\begin{proof}[Proof of Proposition~\ref{ass:sampling:rw_uni_stat_formula}]
    We begin by handling the probability that $(u, v)$ appears within $S_0(\mcG_n)$. Letting $(\tilde{v}_i)_{i \leq k+1}$ be a SRW on $\mcG_n$, we first note that for any $(u, v)$ and $i \geq 1$, we have that 
    \begin{align*}
        \mathbb{P}\big( \tilde{v}_i = u, \tilde{v}_{i+1} = v \,|\, \mcG_n \big) & = \mathbb{P}\big( \tilde{v}_{i+1} = v \,|\, \mcG_n, \tilde{v}_i = u) \mathbb{P}\big( \tilde{v}_{i} = u \,|\, \mcG_n \big) \\
        & = \frac{a_{uv}}{ \degree_n(u) } \cdot \frac{ \degree_n(u) }{2 E_n} = \frac{ a_{uv} }{ 2 E_n }.
    \end{align*}
    Writing $A_i(u \to v) = \{ \tilde{v}_i = u, \tilde{v}_{i+1} = v \}$ for $i \leq k$ and $u, v \in \mcV_n$, we then have
    \begin{equation*}
        \mathbb{P}\big( (u, v) \in S_0(\mcG_n) \,|\, \mcG_n \big) = \mathbb{P}\Big( \bigcup_{i=1}^{k} \big\{ A_i(u \to v) \cup A_i(v \to u) \big\} \,|\, \mcG_n \Big).
    \end{equation*}
    By bounding the probability of the walk intersecting through either $u$ or $v$ twice in a way analogous to that in Proposition~\ref{app:sampling:random_walk_and_unigram}, and then using Proposition~\ref{app:sampling:summary_stats_converge}, we get that 
    \begin{align*}
        \mathbb{P}\big( (u, v) \in S_0(\mcG_n) \,|\, \mcG_n \big) & = \frac{k a_{uv} }{ E_n} \big( 1 + O_p( \tilde{s}_n(\gamma_d)^2 ) \big) \\
        & = \frac{ 2k a_{uv} }{\mcE_W \rho_n n^2} \big( 1 + O_p( \max\{\tilde{s}_n(\gamma_d)^2, (n \rho_n)^{-1/2} \} ) \big).
    \end{align*}
    As for the negative samples, if we write $A_i(u) = \{ \tilde{v}_i = u \}$ for $i \leq k + 1$ and $u \in \mcV_n$, and $B_i(v|u) = \{ v \text{ selected via negative sampling from } u \}$, we can write 
    \begin{align*}
        \mathbb{P}\big( (u, v) \in S_{ns}(\mcG_n) \,|\, \mcG_n \big) = \mathbb{P}\Big( \bigcup_{i=1}^{k+1} \big( A_i(u) \cap B_i(v|u) \big) \cup \big( A_i(v) \cap B_i(u|v) \big) \Big).
    \end{align*}
    Note that $A_i(u) \cap A_i(v) = \emptyset$ for $u \neq v$, and moreover that
    \begin{align*}%
        \mathbb{P}\big( A_i(u) \cap B_i(v|u) \,|\, \mcG_n \big) & = \mathbb{P}\big( A_i(u) \,|\, \mcG_n \big) \mathbb{P}\big( B_i(v|u) \,|\, \mcG_n\big) \\
        & = \frac{ \degree_n(u) }{2 E_n} \cdot \mathbb{P}\big( B(l, \mathrm{Ug}_{\alpha}(v \,|\, \mcG_n) ) \geq 1 \,|\, \mcG_n \big) (1 - a_{uv} ).
    \end{align*}%
    Now, via the same arguments as in Proposition~\ref{app:sampling:random_walk_and_unigram} with regards to the self intersection probability of the random walk, we have that
    \begin{align*}
        \mathbb{P}\big( (u, v) \in S_{ns}(\mcG_n) \,|\, \mcG_n \big) = \Big( \sum_{i=1}^{k+1} \big\{ \mathbb{P}\big( & A_i(u) \cap B_i(v|u) \,|\, \mcG_n \big)  \\
        & + \mathbb{P}\big( A_i(v) \cap B_i(u|v) \,|\, \mcG_n \big) \big\} \Big) \big( 1 + O_p\big( \tilde{s}_n(\gamma_d)^2 \big) \big),
    \end{align*}
    Combining Proposition~\ref{app:sampling:random_walk_and_unigram} and Lemma~\ref{app:sampling:binomial_asymp} therefore gives 
    \begin{align*}
        \mathbb{P}\big( (u, v) & \in S_{ns}(\mcG_n) \,|\, \mcG_n \big) \\
        & = (1 - a_{uv} ) \frac{ l(k+1) \big\{ W(\lambda_u, \cdot) W(\lambda_v, \cdot)^{\alpha} + W(\lambda_v, \cdot) W(\lambda_u, \cdot)^{\alpha} \big\} }{n^2 \mcE_W \mcE_W(\alpha) } \big( 1 + O_p\big( \tilde{s}_n(\gamma_d) \big) \big).
    \end{align*}
    The remaining properties to check can then be done so via routine calculation and the use of Lemmas~\ref{app:sampling:prod_deg_tails}~and~\ref{app:sampling:prod_deg_holder}.
\end{proof}

\phantomsection\label{sampling:sgd_variance:proof}
\begin{proof}[Proof of Proposition~\ref{thm:sampling:sgd_variance}]
    We begin with the expectation; note that by the strong local convergence
    property of the sampling scheme we have that 
    \begin{align*}
        \mathbb{E}[G_i | \mcG_n] & = \sum_{j \in \mcV_n} \mathbb{P}\big( (i, j) \in S(\mcG_n) \,|\, \mcG_n \big) \omega_j \ell'(\langle \omega_i, \omega_j \rangle, a_{ij}) \\ 
        & = \frac{1}{n^2} \sum_{j \in \mcV_n \setminus \{i\} } \big\{ \frac{ 2 a_{ij} }{ \mcE_W \rho_n} + \frac{ 2 l H(\lambda_i, \lambda_j) (1 - a_{ij} )}{ \mcE_W \mcE_W(\alpha)} \big\} \omega_j \ell'(\langle \omega_i, \omega_j \rangle, a_{ij} ) \cdot (1 + o_p(s_n))
    \end{align*}
    where $H(\lambda_i, \lambda_j) := W(\lambda_i, \cdot) W(\lambda_j, \cdot)^{\alpha} + W(\lambda_j, \cdot) W(\lambda_i, \cdot)^{\alpha}$ is free of $k$, and so the first part of the theorem statement holds. 

    For the variance of the estimate, we look at $G_{ir}$, the $r$-th entry of $G_i$, and note that as for $k \neq l$ the events $\mathbbm{1}[ (i, k) \in S(\mcG_n) ]$ and $\mathbbm{1}[ (i, l) \in S(\mcG_n) ]$ are not necessarily independent,
    we have that 
    \begin{align*}
        \mathrm{Var}[ G_{ir} \,|\, \mcG_n] & = \frac{1}{k^2} \sum_{ j \in \mcV_n \setminus \{i\} } \mathrm{Var}\big( \mathbbm{1}\big[ (i, j) \in S(\mcG_n) \big] \,|\, \mcG_n\big) \omega_{jr}^2 c_{ij}^2 \\ 
        & + \frac{1}{k^2} \sum_{j, s \in \mcV_n \setminus \{i\}, k \neq l }  \mathrm{Cov}\big( \mathbbm{1}\big[ (i, j) \in S(\mcG_n) \big], \mathbbm{1}\big[ (i, s) \in S(\mcG_n) \big] \,|\, \mcG_n\big) \omega_{jr} \omega_{sr} c_{ij} c_{is}
    \end{align*}
    where we write $c_{ij} = \ell'( \langle \omega_i, \omega_j \rangle, a_{ij} )$ to reduce notation. To study these terms, we make use of the fact that
    \begin{align*}
        \mathrm{Var}( \mathbbm{1}[A] ) = \mathbb{P}(A) \cdot \big(1 - \mathbb{P}(A) \big), \quad \mathrm{Cov}( \mathbbm{1}[A], \mathbbm{1}[B] ) = \mathbb{P}(A , B) - \mathbb{P}(A) \cdot \mathbb{P}(B). 
    \end{align*}
    In particular, we have that
    \begin{align*}
        \mathrm{Var}\big( \mathbbm{1}\big[ (i, j) \in S(\mcG_n) \big] \,|\, \mcG_n\big) & = \frac{f_n(\lambda_i, \lambda_j, a_{ij})}{n^2} \cdot \Big( 1 - \frac{f_n(\lambda_i, \lambda_j, a_{ij})}{n^2} \Big) \cdot (1 + o_p(s_n) ) \\ 
        & = \frac{f_n(\lambda_i, \lambda_j, a_{ij})}{n^2} \cdot (1 + o_p(s_n) )
    \end{align*}
    by the strong local convergence assumption holding. Studying the covariance term requires more care; in particular, we note
    the covariance will depend on both of the values of $a_{ij}$ and $a_{ik}$.
    The case where $a_{ij} = 1$ and $a_{ik} = 1$ will be most involved, and
    so we focus on this case first. Recall that in this case, $(i, j)$ and
    $(i, k)$ can only be sampled as part of a random walk; letting
    $\tilde{v}_1, \ldots, \tilde{v}_{k+1}$ denote the vertices obtained
    on a random walk, we define the events 
    \begin{align*}
        A_{l}(i \to j) & := \{ \tilde{v}_{l} = i, \tilde{v}_{l+1} = j \}, & A_l(i, j) & := A_l(i \to j) \cup A_l(j \to i), \\ 
        A(i, j) & := \bigcup_{l=1}^{k} A_l(i, j), & A_{m <}(i, j) &:= \bigcup_{l = m+1}^{k} A_l(i, j)
    \end{align*}
    and so we want to study the covariance of the events $A(i, j)$ and $A(i, s)$. For now, we will also write $\mathbb{P}_{\mcG_n}$ to refer to probabilities
    computed conditional on the realization of the graph $\mcG_n$. 
    Recalling the identity
    \begin{equation*}
        \mathbbm{1}\big[ \cup_{l = 1}^k A_l \big] = \sum_{i=1}^k \mathbbm{1}[A_l] - \sum_{l=1}^{k-1} \mathbbm{1}\big[ A_l \cap \cup_{j > l} A_j \big],
    \end{equation*}
    for any sequence of events $(A_l)_{l \leq k}$, by applying this identity
    twice we can derive that
    {
    \allowdisplaybreaks
    \begin{align*}
        \mathbb{P}_{\mcG_n}(A(i, j) \cap A(i, s)) & = \sum_{l = 1}^k \sum_{m=1}^k \mathbb{P}_{\mcG_n}( A_l(i, j) \cap A_m(i, s) ) \\
        & - \sum_{l = 1}^k \sum_{m=1}^{k-1} \mathbb{P}_{\mcG_n}( A_l(i, j) \cap A_m(i, s) \cap A_{m<}(i, s)) \\ 
        & - \sum_{l = 1}^{k-1} \sum_{m=1}^k \mathbb{P}_{\mcG_n}( A_l(i, j) \cap A_m(i, s) \cap A_{l<}(i, j) )  \\
        & + \sum_{l = 1}^{k-1} \sum_{m=1}^{k-1} \mathbb{P}_{\mcG_n}( A_l(i, j) \cap A_m(i, s) \cap A_{l<}(i, j) \cap A_{m<}(i, s) ) 
    \end{align*}
    }
    For the terms in the first sum, we can expand this as
    {
    \allowdisplaybreaks
    \begin{align*} 
        \mathbb{P}_{\mcG_n}( A_l(i, j) \cap A_m(i, s)) & = \mathbb{P}_{\mcG_n}( \tilde{v}_l = i, \tilde{v}_{l+1} = j, \tilde{v}_m = i, \tilde{v}_{m+1} = s) \\
        & \quad + \mathbb{P}_{\mcG_n}( \tilde{v}_l = i, \tilde{v}_{l+1} = j, \tilde{v}_m = i, \tilde{v}_{m+1} = s) \\
        & \quad + \mathbb{P}_{\mcG_n}( \tilde{v}_l = i, \tilde{v}_{l+1} = j, \tilde{v}_m = i, \tilde{v}_{m+1} = s) \\
        & \quad + \mathbb{P}_{\mcG_n}( \tilde{v}_l = i, \tilde{v}_{l+1} = j, \tilde{v}_m = i, \tilde{v}_{m+1} = s).
    \end{align*}
    }We note that when $l = m$, all the probabilities equal $0$, and when $l = m \pm 1$ there are two contributions of the form e.g
    \begin{equation*}
        \mathbb{P}_{\mcG_n}( \tilde{v}_{m-1} = j, \tilde{v}_m = i, \tilde{v}_{m+1} = s) = \frac{1}{\deg(i) 2 E_n}
    \end{equation*}
    (where we have used the Markov property and the stationarity of the random walk), with the remaining terms equaling zero. The contributions of the terms where $l = m \pm 2$ are all of
    the order e.g 
    {\allowdisplaybreaks
    \begin{align*}
        \mathbb{P}_{\mcG_n}( \tilde{v}_m = i, \tilde{v}_{m+1} = j, \tilde{v}_{m+2} = i, \tilde{v}_{m+3} = s ) = \frac{1}{ 2 E_n \deg(i) \deg(j) } = \frac{1}{ \deg(i) 2 E_n O_p( n \rho_n )}
    \end{align*}
    }(where the bounds hold uniformly over any $(i, j, s)$). For terms
    $l = m \pm r$ where $r \geq 3$, we get terms of the order e.g
    {\allowdisplaybreaks
    \begin{align*}
        \mathbb{P}_{\mcG_n}( \tilde{v}_m = i &, \tilde{v}_{m+1} = j, \tilde{v}_{m+r} = i, \tilde{v}_{m+3} = s ) \\ 
        & = \frac{1}{ \deg(i) } \cdot \mathbb{P}_{\mcG_n}( \tilde{v}_{m+r} = i \,|\, \tilde{v}_{m+1} = j) \cdot \frac{1}{ 2 E_n} = \frac{1}{2 \deg(i) E_n} \cdot \mathbb{P}_{\mcG_n}( \tilde{v}_{r} = i \,|\, \tilde{v}_{1} = j) \\ 
        & =  \frac{1}{2 \deg(i) E_n \deg(j)} \cdot \sum_{u_2, \ldots, u_{r-1}} \frac{a_{i u_{r-1}} a_{u_{r-1} u_{r-2} } \cdots a_{u_2 j} }{ \deg( u_{r-1} ) \cdots \deg( u_2 ) } \\
        & =  \frac{1}{2 \deg(i) E_n O_p( n \rho_n )} \cdot O_p(1)
    \end{align*}
    } where the $O_p(1)$ term follows by using the fact that $\deg(i) = n \rho_n W(\lambda_i, \cdot) (1 + O_p(s_n))$ uniformly across $i$, and that the number of paths of length $r-2$ between $i$ and $j$ is $O_p((n \rho_n)^{r-2})$ uniformly across $i$ and $j$. By similar arguments, the terms in the other sums will be an order of magnitude less than that of the terms from the first sum (they will be multiplied by factors no greater in magnitude than $1/\deg(i)$), and consequently
    it follows that when $a_{ij} = a_{is} = 1$, we have that
    \begin{equation*}
        \mathrm{Cov}_{\mcG_n}( A(i, j), A(i, s)) = \frac{2(k-1)}{ W(\lambda_i, \cdot) \mcE_W n^3 \rho_n^2  }(1 + o_p( s_n) )
    \end{equation*}
    where we already have calculated the asymptotics for $\mathbb{P}_{\mcG_n}(A(i, j))$ and $\mathbb{P}_{\mcG_n}(A(i, s))$ in Proposition~\ref{app:sampling:random_walk_and_unigram}, and we applied Proposition~\ref{app:sampling:summary_stats_converge} to handle the degree term. 

    When $a_{ij} = 1$ and $a_{is} = 0$, the covariance is equal to zero, as once $i$ has been sampled as part of the random walk, the pair $(i, s)$
    can only be subsampled from the negative sampling distribution, which
    does so independently of the process from the random walk; the same
    argument applies for when $a_{ij} = 0$ and $a_{is} = 1$. 
    
    The final case
    to consider is when $a_{ij} = 0$ and $a_{is} = 0$; to handle this term,
    we note that if $i$ is not sampled as part of the random walk, then
    the events that $(i, j)$ and $(i, s)$ are sampled as part of the negative
    sampling distribution are independent. As a result, we only need to focus
    on conditioning on the events where $i$ does appear in the random walk; note that if $i$ appears multiple times, then the pairs $(i, j)$ and
    $(i, s)$ could be sampled during any of the corresponding negative sampling
    steps. if we let $X_m^{(l)} \sim \mathrm{Multinomial}(l ; (p_j)_{j \neq i})$ be drawn independently for $m \geq 1$ (which corresponds to the vertices
    negative sampled) with probability $p_j = l W(\lambda_j, \cdot)^{\alpha} / n \mcE_W(\alpha) (1 + o_p(s_n))$ according to the unigram distribution (by Proposition~\ref{app:sampling:random_walk_and_unigram}), and let $Y$ be the
    number of times the vertex $i$ appears in the random walk, then we have
    that
    {\allowdisplaybreaks
    \begin{align*}
        \mathrm{Cov}_{\mcG_n}& ( (i, j) \in S_{ns}(\mcG_n) , (i, s) \in S_{ns}(\mcG_n) )  \\ 
        & = \sum_{r = 1}^{k} \mathrm{Cov}_{\mcG_n}( (i, j) \in S_{ns}(\mcG_n) , (i, s) \in S_{ns}(\mcG_n) \,|\, Y = r ) \mathbb{P}_{\mcG_n}( Y = r ) \\ 
        & = \sum_{r=1}^{k} \mathrm{Cov}\Big( \sum_{m = 1}^r X_{mj}^{l} \geq 1, \sum_{m=1}^r X_{ms}^{(l)} \geq 1 \Big) \mathbb{P}_{\mcG_n}( Y = r ) \\ 
        & = \sum_{r = 1}^k \mathrm{Cov}( X_{1j}^{(rl)} \geq 1, X_{1s}^{(rl)} \geq 1 ) \mathbb{P}_{\mcG_n}( Y = r ) \\ 
        & = - \frac{l^2 W(\lambda_j, \cdot)^{\alpha} W(\lambda_s, \cdot)^{\alpha}}{ n^2 \mcE_W(\alpha)^2  } \cdot  (1 + O_p(n^{-1})) \cdot \sum_{r = 1}^k r \mathbb{P}_{\mcG_n}( Y = r) \\
        & = - \frac{l^2 W(\lambda_j, \cdot)^{\alpha} W(\lambda_s, \cdot)^{\alpha}}{ n^2 \mcE_W(\alpha)^2  } \cdot (1 + O_p(n^{-1})) \cdot \mathbb{E}_{\mcG_n}[Y] \\ 
        & = - \frac{ kl^2 W(\lambda_j, \cdot)^{\alpha} W(\lambda_s, \cdot)^{\alpha} W(\lambda_i, \cdot) }{ n^3 \mcE_W \mcE_W(\alpha)^2 } \cdot (1 + o_p(s_n))
    \end{align*}
    }where in the fourth line, we used the fact that the sum of independent multinomial distributions is multinomial; in the fifth line we used Lemma~\ref{app:sampling:multinomial_covariance}; and in the last line, we used the fact that as $Y = \sum_{r=1}^{k+1} 1[ \tilde{v}_r = i]$, by linearity of expectations we have
    \begin{equation*}
        \mathbb{E}_{\mcG_n}[Y] = \sum_{r = 1}^{k+1} \mathbb{P}_{\mcG_n}( \tilde{v}_r = i) = \frac{ k \mathrm{deg}(i)}{ 2 E_n} = \frac{ k W(\lambda_i, \cdot) }{ n \mcE_W } (1 + o_p(s_n))
    \end{equation*}
    where again we have used Proposition~\ref{app:sampling:summary_stats_converge}.

    Pulling this altogether, it follows that
    \begin{align*}
        \mathrm{Var}[ G_{ir} \,|\, \mcG_n] & = \frac{1}{k n^2} \sum_{ j \in \mcV_n \setminus \{i\} } \Big\{ \frac{ 2 a_{ij} }{ \mcE_W \rho_n} + \frac{ 2 l H(\lambda_i, \lambda_j) (1 - a_{ij} )}{ \mcE_W \mcE_W(\alpha)} \Big\} \omega_{jr}^2 c_{ij}^2 \cdot (1 + o_p(s_n)) \\ 
        & + \frac{1}{k} \sum_{j, s \in \mcV_n \setminus \{i\}, j \neq s }  \widetilde{H}(\lambda_i, \lambda_j, \lambda_s, a_{ij}, a_{is} ) \omega_{jr} \omega_{sr} c_{ij} c_{is} \cdot (1 + o_p( s_n) )
    \end{align*}
    where we write 
    \begin{equation*}
        \widetilde{H}(\lambda_i, \lambda_j, \lambda_s, a_{ij}, a_{is} ) := \frac{2(1-k^{-1}) a_{ij}a_{is} }{ W(\lambda_i, \cdot) \mcE_W n^3 \rho_n^2  } - (1 - a_{ij})(1 - a_{is}) \frac{ l^2 W(\lambda_j, \cdot)^{\alpha} W(\lambda_s, \cdot)^{\alpha} W(\lambda_i, \cdot) }{ n^3 \mcE_W \mcE_W(\alpha)^2 }
    \end{equation*}
    To bound the variance, we note that uniformly across all $i$ we have that
    \begin{align*}
        \sum_{j \in \mcV_n \setminus \{ i \} } a_{ij} = O_p( (n \rho_n) ), \sum_{j, s \in \mcV_n \setminus \{i \}, j \neq s } a_{ij} a_{is} = O_p( (n^2 \rho_n^2 ) ).
    \end{align*}
    To conclude, we note that under the assumption that the embedding vectors $\| \omega_j \|_{\infty} \leq A$ for all $j$, and as the gradient of the
    cross entropy is absolutely bounded by $1$ (and consequently so are the $c_{ij}$ and $c_{is}$), by applying H\"{o}lder's
    inequality we find that 
    \begin{equation*}
        \mathrm{Var}[ G_{ir} \,|\, \mcG_n ] = O_p( \frac{1}{kn })
    \end{equation*}
    uniformly across all $i$ and $r$, and so the stated conclusion follows.
\end{proof}

\subsection{Additional quantative bounds}

\begin{lemma} \label{app:sampling:binomial_asymp}
    Suppose that $X_{n, m} \sim B(k, p_{n, m})$ for $n \geq 1$, $m \leq n$ with $\max_{m \leq n} p_{n, m} \to 0$ as $n \to \infty$. Then 
    \begin{equation*}
        \max_{m \leq n} \Big| \frac{ \mathbb{P}(X_{n, m} \geq 1) }{ k p_{n, m} } - 1 \Big| = O(\max_{m \leq n} p_{n, m} ).
    \end{equation*}
\end{lemma}

\begin{proof}[Proof of Lemma~\ref{app:sampling:binomial_asymp}]
    The result follows by noting that 
    \begin{align*}
        \mathbb{P}( X_{n, m} \geq 1) = 1 - (1 - p_{n, m})^{k} = \sum_{r=1}^{k} (-1)^{r-1} { k \choose r } p_{n, m}^{r}
    \end{align*}
    and whence
    \begin{equation*}
        \Big| \frac{ \mathbb{P}(X_{n, m} \geq 1) }{ k p_{n, m} } - 1 \Big| = \sum_{r = 2}^k (-1)^{r-1} \frac{1}{k} {k \choose r} p_{n, m}^{r-1} = O(\max_{m \leq n} p_{n, m} ).
    \end{equation*}
    as desired.
\end{proof}

\begin{lemma} \label{app:sampling:binom_coeff_lemma}
    Suppose that $m, r \to \infty$ with $m \gg r$ and $k = O(1)$. Then we have that 
    \begin{equation*}
            1 - { m -r \choose k } { m \choose k }^{-1} = \frac{rk}{m}\Big( 1 + O\Big(\frac{r}{m} \Big) \Big).
    \end{equation*}
\end{lemma}

\begin{proof}[Proof of Lemma~\ref{app:sampling:binom_coeff_lemma}]
    We begin by recalling Stirling's approximation, which tells us that 
    \begin{equation*}
        \Gamma(n+1) = \sqrt{2 \pi n} \Big( \frac{n}{e} \Big)^n \Big( 1 + \frac{1}{12n} + o\Big( \frac{1}{n} \Big) \Big).
    \end{equation*}
    We can then write
    \begingroup 
    \allowdisplaybreaks
    \begin{align*} 
        1 - { m -r \choose k } & { m \choose k }^{-1} = 1 - \frac{ \Gamma(m - r +1 ) \Gamma(m - k + 1)  }{ \Gamma(m+1) \Gamma(m -r - k +1)  } \\
        & = 1 - \frac{ (m - r)^{m-r} (m - k)^{m-k}  }{  m^m (m - r - k)^{m - r - k} }\big( 1 + O(m^{-1} ) \big) \\
        & = 1 - \Big[ \Big( 1 - \frac{r}{m} \Big)^k \cdot \Big( 1 - \frac{k}{m} \Big)^r \cdot \Big(1 + \frac{ rk / m}{m - r - k} \Big)^{m - r - k} \Big] \cdot \big( 1 + O( m^{-1} ) \big).
    \end{align*}
    \endgroup
    Letting $(A)$ denote the $[\cdots]$ term, and using that $\log(1+x) = x - x^2/2 + x^3/3 + o(x^3)$ and $\exp(x) = 1 + x + x^2/2 + o(x^2)$ as $x \to 0$, we have that 
    \begingroup 
    \allowdisplaybreaks
    \begin{align*}
        \log(A) &= k \log\Big(1 - \frac{r}{m} \Big) + r \log\Big( 1 - \frac{k}{m} \Big) + (m - r -k ) \log\Big( 1 + \frac{ rk/m}{m - r -k} \Big) \\ 
        & = - \frac{rk}{m} - \frac{ k r^2 }{2 m^2} + o( r^2 m^{-2} )  \qquad \implies (A) = 1 - \frac{rk}{m}\Big( 1 + O\Big(\frac{r}{m} \Big) \Big).
    \end{align*}
    \endgroup
    Combining this all together gives the stated result.
\end{proof}

\begin{lemma} \label{app:sampling:binom_coeff_lemma2}
    Suppose that $m, r_1, r_2 \to \infty$ with $m \gg r_1, r_2$, $r_1$ and $r_2$ of the same order, and $k, c = O(1)$ with $k > 1$. Then we have that 
    \begin{align*}
        1 - { m - r_1 \choose k} {m \choose k}^{-1} &+ 1 - { m - r_2 \choose k} {m \choose k}^{-1} - \Bigg[ 1 - { m - (r_1 + r_2 - c) \choose k} {m \choose k}^{-1} \Bigg] \\
        & = \Big( \frac{ kc}{m} + \frac{ k(k-1) r_1 r_2}{m^2} \Big) \Big( 1 + O\Big( \frac{ r_1 + r_2}{m} \Big) \Big).
    \end{align*}
\end{lemma}

\begin{proof}[Proof of Lemma~\ref{app:sampling:binom_coeff_lemma2}]
    The argument is the same as in Lemma~\ref{app:sampling:binom_coeff_lemma}, except we need to use the higher ordered termed expansion 
    \begin{equation*}
        1 - { m - r \choose k} {m \choose k}^{-1} = \frac{r k}{m} \Big( 1 - \frac{ r(k-1) }{2m} + o\Big( \frac{r}{m} \Big) \Big),
    \end{equation*}
    in order to get the stated result. With this, the result follows by routine calculations which we therefore omit.
\end{proof}

\begin{lemma} \label{app:sampling:prod_deg_tails}
    Suppose that $g: [0, 1] \to [0, 1]$ is such that $g^{-1} \in L^{\gamma}([0, 1])$ for some $\gamma \in [1, \infty]$. Then the function $f(x, y) = ( g(x) g(y)^{\alpha} + g(x)^{\alpha} g(y) )^{-1}$ belongs to $L^{\tilde{\gamma}}([0, 1]^2)$ where $\tilde{\gamma} = \min\{ \gamma, \gamma/ \alpha \}$.
\end{lemma}

\begin{proof}[Proof of Lemma~\ref{app:sampling:prod_deg_tails}]
    Note that we have that $f(x, y) \leq (g(x) g(y)^{\alpha} )^{-1} + (g(y) g(x)^{\alpha} )^{-1}$. As we have that $g^{-1} \in L^{\gamma}([0, 1])$, it follows that $g^{-\alpha} \in L^{\gamma/\alpha}([0, 1])$, and consequently $g(x)^{-1} g(y)^{-\alpha} \in L^{\tilde{\gamma}}([0, 1]^2)$, so the conclusion follows. 
\end{proof}

\begin{lemma} \label{app:sampling:prod_deg_holder}
    Suppose that $W: [0, 1]^2 \to [0, 1]$ is piecewise H\"{o}lder$([0, 1]^2, \beta, L, \mathcal{Q}^{\otimes 2})$ for some partition $\mathcal{Q}$ of $[0, 1]$. Then 
    \begin{enumerate}[label=\alph*)]
        \item The degree function $W(\lambda, \cdot)$ is piecewise H\"{o}lder$([0, 1], \beta, L, \mathcal{Q})$;
        \item The function $W(x, \cdot) W(y, \cdot)^{\alpha} + W(x, \cdot)^{\alpha} W(y, \cdot)$ is piecewise H\"{o}lder($[0, 1]^2$, $\beta_{\alpha}$, $L'$, $\mcQ^{\otimes 2}$) where $\beta_{\alpha} = \beta \min\{ \alpha, 1\}$ and $L' = 4 L \max\{ 1, \alpha \}$.
    \end{enumerate}
\end{lemma}

\begin{proof}[Proof of Lemma~\ref{app:sampling:prod_deg_holder}]
    The first part follows immediately by noting that, whenever $x, y \in \mcQ$,
    \begin{equation*}
        | W(x, \cdot) - W(y, \cdot) | \leq \sum_{Q'\in \mcQ} \int_{Q'} | W(x, z) - W(y, z) | \, dz \leq L | x - y |^{\beta}
    \end{equation*}
    by using the H\"{o}lder properties of $W$. For the second part, note that the function $x \mapsto x^{\alpha}$ is H\"{o}lder$([0, 1], \min\{\alpha, 1\}, C_{\alpha})$ where $C_{\alpha} = \max\{ \alpha, 1 \}$, and so $W(\lambda, \cdot)$ is piecewise H\"{o}lder($[0, 1]$, $\min\{\alpha \beta, \beta \}$, $L C_{\alpha}$, $\mcQ$). To conclude, by the triangle inequality we then get that whenever $(x_1, y_1)$, $(x_2, y_2) \in Q \times Q'$, we have
    \begin{align*}
        | W(x_1, \cdot ) & W(y_1, \cdot)^{\alpha} - W(x_2, \cdot) W(y_2, \cdot)^{\alpha} | \\
        & \leq W(x_1, \cdot) | W(y_1, \cdot)^{\alpha} - W(y_2, \cdot)^{\alpha} | + W(y_2, \cdot)^{\alpha} | W(x_1, \cdot) - W(x_2, \cdot) | \\ 
        & \leq L C_{\alpha} |y_1 - y_2|^{\min\{\alpha \beta, \beta \}} + L | x_1 - x_2 |^{\beta} \leq 2 L C_{\alpha} \| x - y \|_2^{\min\{ \alpha \beta, \beta \} }, 
    \end{align*}
    giving the stated result.
\end{proof}

\begin{lemma} \label{app:sampling:multinomial_covariance}
    Let $X \sim \mathrm{Mutinomial}(l ; p_1, \ldots, p_{n})$ be such that we have that $p_i = \Theta(n^{-1})$ uniformly across all $i$. Then
    \begin{equation*}
        \mathrm{Cov}(X_i \geq 1 , X_j \geq 1) = - l p_i p_j \cdot (1 + O(n^{-1})).
    \end{equation*}
\end{lemma}

\begin{proof}[Proof of Lemma~\ref{app:sampling:multinomial_covariance}]
    Note that 
    \begin{equation*}
        \mathbb{P}(X_i \geq 1, X_j \geq 1) = \mathbb{P}(X_i \geq 1) + \mathbb{P}(X_j \geq 1) - (1 - \mathbb{P}(X_i = 0, X_j = 0)) 
    \end{equation*}
    and consequently we get that 
    \begin{align*}
        \mathrm{Cov}(X_i \geq 1& , X_j \geq 1) \\ 
        & = 1 - (1 - p_i)^l - (1 - p_j)^l + (1 - p_i - p_j)^l - (1 - (1 - p_i)^l) (1 - (1 - p_j)^l) \\
        & = (1 - p_i - p_j)^l - (1 - p_i - p_j + p_i p_j)^l \\ 
        & = l p_i p_j (1 - p_i - p_j)^{l-1} \cdot ( 1 + O(n^{-2} ) ) = l p_i p_j \cdot (1 - O(n^{-1}) )
    \end{align*} 
    as desired.
\end{proof}
%!TEX root = ..\ms.tex

\section{Optimization of convex functions on \texorpdfstring{$L^p$}{Lp} spaces}
\label{sec:app:convex_opt}

In this section we summarize the necessary functional analysis needed in order to study the minimizers of convex functionals on $L^p$ spaces.

\subsection{Weak topologies on \texorpdfstring{$L^p$}{Lp}} \label{sec:app:convex_opt:weak_topology}

The material stated in this section is textbook, with \citet{aliprantis_infinite_2006, barbu_convexity_2012, brezis_functional_2011} and \citet{riesz_functional_1990} all useful references. We begin with a Banach space $X$, whose continuous dual space $X^*$ consists of all continuous linear functionals $X \to \mathbb{R}$. The weak topology on $X$ is the coarsest topology on $X$ for which these functionals remain continuous. (The norm topology on $X$ is also referred to as the strong topology.) We can describe this topology via a base of neighbourhoods
\begin{equation*}
    N(L, x, \epsilon) := \big\{ y \in X \,:\, L(y - x) < \epsilon \big\}
\end{equation*}
for $L \in X^*$, $x \in X$ and $\epsilon > 0$. For sequences, we say that a sequence $(x_n)_{n \geq 1}$ converges weakly to some element $x$ provided $y(x_n) \to y(x)$ as $n \to \infty$ for all $y \in X^*$. We now state some useful facts about weak topologies on Banach spaces:
\begin{enumerate}[label=\alph*)]
    \item A non-empty convex set is closed in the weak topology iff it is closed in the strong topology. (The corresponding statement for open sets is not true.)
    \item A convex, norm-continuous function $f: X \to \mathbb{R}$ is lower semi-continuous (l.s.c) in the weak topology; that is, the level sets $L_{\lambda} := \{ x \,:\, f(x) \leq \lambda \}$ are weakly closed for all $\lambda \in \mathbb{R}$.
    %\item If $(x_n)_{n \geq 1}$ converges in the weak topology to some element $x$, then $(x_n)_{n \geq 1}$ is norm bounded.
    \item The weak topology on $X$ is Hausdorff.
    %\item (Eberlein-Smulian) A set is compact in the weak topology on $X$ if and only it is sequentially compact. (Note that weak topologies on a Banach space are metrizable iff $X$ is finite dimensional, so this is a non-trivial result). 
\end{enumerate}

\begin{corollary} \label{app:convex_opt:easy_min}
    Let $X$ be a Banach space and $f : X \to \mathbb{R}$ be a convex, norm continuous function, and let $A$ be a weakly compact set. Then there exists a minimizer of $f$ over $A$. If the set $A$ is convex and $f$ is strictly convex, then the minima is unique.
\end{corollary}

\begin{proof}[Proof of Corollary~\ref{app:convex_opt:easy_min}]
    By applying a) and b) above and using Weierstrass' theorem in the weak topology, we get the first part; the second part is standard.
\end{proof}

Specializing now to the case where $X = L^p(\mu) = L^p(X, \mcF, \mu)$ where $(X, \mcF, \mu)$ is a $\sigma$-finite measure space, the Riesz representation theorem guarantees that for $p \in [1, \infty)$, if $q$ is the H\"{o}lder conjugate of $p$ so $q^{-1} + p^{-1} = 1$, then the mapping 
\begin{equation*}
    g \in L^q(\mu) \mapsto L_g(\cdot) \in (L^p(\mu))^* \qquad \text{ where } L_g(f) := \int_{X} f g \, d \mu := \langle f, g \rangle
\end{equation*}
gives an isometric isomorphism between $(L^p(\mu))^*$ and $L^q(\mu)$. The relatively weakly compact sets (that is, the sets whose weak closures are compact) in $L^p(\mu)$ can be characterized as follows:

\begin{enumerate}[label=\alph*)]
    \item (Banach–Alaoglu) For $p > 1$, the closed unit ball $\{ x \in L^p(\mu) \,:\, \| x \|_p \leq 1 \}$ is weakly compact, and the relatively weakly compact sets are exactly those which are norm bounded.
    \item (Dunford-Pettis) A set $A \subset L^1(\mu)$ is relatively weakly compact if and only if the set $A$ is uniformly integrable. (This is a stricter condition than in the $p > 1$ case.)
\end{enumerate}

\subsection{Minimizing functionals over \texorpdfstring{$L^1(\mu)$}{L1(mu)}} \label{sec:app:conex_opt:L1}

Note that to apply Corollary~\ref{app:convex_opt:easy_min}, we require the optimization domain $A$ to be weakly compact. In the case where we are optimizing over $L^p(\mu)$ for $p = 1$, we note that the uniform integrability property is stricter than that of norm-boundedness. We are mainly motivated by wanting to optimize the functional $\mcI_n[K]$ over a weakly closed set which is only norm-bounded, which therefore will cause us trouble in the regime where $p = 1$. However, if the function we are seeking to optimize is more structured, we can still guarantee the existence of a minimizer; this is the purpose of the next result. 

\begin{theorem} \label{app:convex_opt:l1_minima}
    Let $P$ be a norm closed subset of a Banach space $U$ equipped with a norm $\| \cdot \|_U$, and let $(P, \mcP)$ denote the corresponding subspace topology on $P$. Let $X$ be a Banach space equipped with strong and weak topologies $\mcS$ and $\mcW$, and whose norm is denoted $\| \cdot \|_X$. Let $I[K; g] : X \times P \to \mathbb{R}$ be a function which is bounded below, and has the following additional properties:
    \begin{enumerate}[label=\alph*)]
        \item $K \mapsto I[K ; g]$ is strictly convex for all $g \in P$;
        \item $(K, g) \mapsto I[K ; g]$ is $\mcS \times \mcP$-continuous;
        \item For any $\lambda$ such that the level set $L_{\lambda} := \{ (K, g) \,:\, I[K ; g] \leq \lambda \}$ is non-empty, there exists a constant $C_{\lambda}$ for which
        \begin{equation} \label{eq:app:convex_opt:l1_minima_save}
            \big| I[K ; g] - I[K ; \tilde{g} ] \big| \leq C_{\lambda} \| g - \tilde{g} \|_U
        \end{equation}
        for any $(K, g) \in L_{\lambda}$ and $\tilde{g} \in P$.
    \end{enumerate}
    Let $\mcC$ be a weakly closed convex set in $X$, and let $\tilde{\mu}(g) := \argmin_{K \in \mcC} I[K;g]$. By the strict convexity, there exists a set $A$ for which $\tilde{\mu}(g) = \{ \mu(g) \}$ if $g \in A$ and $\tilde{\mu}(g) = \emptyset$ for $g \in A^c$. If there exists a dense set $D$ for which $D \subseteq A$, then $A = P$, and the function $\mu(g)$ is $\mcP$-to-$\mcW$ continuous.
\end{theorem} 

The purpose of the above theorem is that provided we can argue the existence of a minimizer on a dense set of values of $g$, then we can exploit the continuity and convexity of $I[K ; g]$ in order to upgrade our existence guarantee to hold for all functions $g$. In order to prove the above result, we require two intermediate results: one is a simple topological result, and the other a refinement of a version of Berge's maximum principle introduced in \citet{horsley_berges_1998}. Before doing so, we introduce some terminology:
\begin{enumerate}[label=\alph*)]
    \item A correspondence $B : P \twoheadrightarrow X$ is a set-valued mapping for which every $p \in P$ is assigned a subset $B(p) \subseteq X$. (A function is therefore a singleton valued correspondence.)
    \item The graph of a correspondence $B$ is the subset of $P \times X$ given by $\{ (p, B(p) ) \,:\, p \in P \}$.
    \item Let $\mcP$ be a topology on $P$, and $\tau$ a topology on $X$. Then we say that $B$ is $\mcP$-to-$\tau$ lower hemicontinuous if the set $\{ p \,:\, B(p) \cap U \neq \emptyset \}$ is open in $\mcP$ for every open set $U$ in $\tau$.
    \item We say a correspondence $B$ is $\mcP$-to-$\tau$ upper hemicontinuous if the set $\{ p \,:\, B(p) \subseteq U \}$ is open in $\mcP$ for all open sets $U \in \tau$. 
    \item When $B$ is a bond-fide function, the above notions in c) and d) are the same as lower semi-continuity (l.s.c) and upper semi-continuity (u.s.c) for functions respectively.
\end{enumerate}

\begin{lemma} \label{app:convex_opt:hemicont}
    Let $(P, \mcP)$ and $(X, \mcX)$ be topological spaces. Suppose that $B: P \twoheadrightarrow X$ is at most singleton valued, with $A$ denoting the set of $p$ for which $B(p) \neq \emptyset$, so $B(p) = \{ b(p) \}$ for $p \in A$ and $B(p) = \emptyset$ if $p \in A^c$. If $B$ is an upper hemicontinuous correspondence, then $A$ is closed in $P$, and $b : A \to X$ is a continuous function with respect to the subspace topology on $A$ induced by $X$. In particular, if $A$ is also dense, then $A = P$. 
\end{lemma}

\begin{proof}[Proof of Lemma~\ref{app:convex_opt:hemicont}]
    Note that by the upper hemicontinuity property, $(A^c) = \{ p: B(p) \subseteq \emptyset \}$ is open and whence $A$ is closed. As for the continuity, we want to show that $b^{-1}(U)$ is open in the subspace topology on $A$ given any open set $U$ in $X$. As $b^{-1}(U) = A \cap \{  p \,:\, B(p) \subseteq U \}$, this is indeed the case. For the final statement, we simply note that $A = \mathrm{cl}(A) = P$, where the first equality is because $A$ is closed, and the second as $A$ is dense.
\end{proof}

\begin{theorem}[Summary and extension of \citealp{horsley_berges_1998}] \label{app:convex_opt:maximum_theorem}
    Let $(P, \mathcal{P})$ be a \linebreak Hausdorff topological space, and let $X$ be a Banach space equipped with topologies $\mcS$ (informally, a "strong" topology) and $\mcW$ (informally, a "weak" topology). Let $B : P \twoheadrightarrow X$ be a correspondence, and suppose that $f : X \times P$ is a function. Define the sets
    \begin{align}
        R &:= \big\{ (z, p, x) \in X \times P \times X \,:\, f(z, p) \geq f(x, p) \big\}, \label{eq:convex_opt:R_set} \\
        \widehat{X}(p) & := \big\{ x \in B(p) \,:\, f(z, p) \geq f(x, p) \text{ for all } z \in B(p) \big\}. \label{eq:convex_opt:Xp_set}
    \end{align}
    Then we have the following:
    \begin{enumerate}[label=\alph*)]
        \item Suppose that $B$ is $\mcP$-to-$\mcS$ lower hemicontinuous, the graph of $B$ is $\mcP \times \mcW$-closed in $P \times X$, and that the set $R$ is $\mcS \times \mcP \times \mcW$-closed in $X \times P \times X$. Then the graph of $\mcX$ is also $\mcP \times \mcW$-closed in $P \times X$.
        \item If in addition to a) we have that $B$ is $\mcP$-to-$\mcW$ upper hemicontinuous and has $\mcW$-compact values, then $\widehat{X}$ is also $\mcP$-to-$\mcW$ upper hemicontinuous and has $\mcW$-compact values.
        \item If in addition to a) we have that $B$ is $\mcP$-to-$\mcW$ upper hemicontinuous and $\widehat{X}$ is $\mcW$-compact valued, then $\widehat{X}$ is $\mcP$-to-$\mcW$ upper hemicontinuous. 
    \end{enumerate}
\end{theorem}

\begin{proof}[Proof of Theorem~\ref{app:convex_opt:maximum_theorem}]
    The first two parts are simply Theorem~2.2 and Corollaries~2.3 and 2.4 of \citet{horsley_berges_1998} applied to the relation defined by the set $R$ above. The third is a modification of the argument in Corollary~2.4. Begin by writing $\widehat{X} = B \cap \widehat{X}$. It is known that the intersection of a closed correspondence $\phi$ and a upper hemicontinuous, compact-valued correspondence $\psi$ is upper hemicontinuous and compact-valued \citep[Theorem~17.25, p567]{aliprantis_infinite_2006}; one can show with the same proof that if $\psi$ is only upper hemicontinuous and closed-valued, and $\phi \cap \psi$ is compact valued, then $\phi \cap \psi$ is upper hemicontinuous also. From this, part c) follows.
\end{proof}

\begin{proof}[Proof of Theorem~\ref{app:convex_opt:l1_minima}]
    Our aim is to apply Theorem~\ref{app:convex_opt:maximum_theorem}, using the correspondence $B(g) = \mcC$ for all $g \in P$, and $f(K, g) = I[K ; g]$ (now writing $x \to K$ and $p \to g$). As this correspondence is constant, the graph of $B$ is closed in $\mcP \times \mcW$, as it simply equals $P \times \mcC$ and $\mcC$ is weakly closed. As $\mcC$ is convex and weakly closed, it is also strongly closed, and therefore the correspondence $B(g)$ is both $\mcP$-to-$\mcS$ lower hemicontinuous and $\mcP$-to-$\mcW$ upper hemicontinuous. Note that $\widehat{X}(g)$ as defined in \eqref{eq:convex_opt:Xp_set} is the correspondence which defines the minima set of $I[K ;g]$ for each $g \in P$ and so equals $\tilde{\mu}(g)$; via the strict convexity of $I[K ; g]$ for each $g$, we know that $\widehat{X}(g)$ is at most a singleton, and therefore is $\mcW$-compact valued (as the empty set and singletons are compact). 
    
    Consequently, in order to apply part c) of Theorem~\ref{app:convex_opt:maximum_theorem}, the remaining part is to show that the set $\mcR$ as defined in \eqref{eq:convex_opt:R_set} is $\mcS \times \mcP \times \mcW$-closed. To do so, we will argue that the complement $R^c$ is open. Fix a point $(K_0, g_0, K_0') \in X \times P \times X$. As $I[K_0; g_0] < I[K_0' ; g_0]$, there exists $\lambda \in \mathbb{R}$ such that $I[K_0; g_0] < \lambda < I[K_0'; g_0]$. Note that if we can find 
    \begin{enumerate}[label=\alph*)]
        \item a $\mcS$-nbhd (neighbourhood) $N_S$ of $K_0$ and a $\mcP$-nbhd $N_P$ of $g_0$ such that $I[K; g] < \lambda$ for all $(K, g) \in N_S \times N_P$; and 
        \item a $\mcW$-nbhd $N_W$ of $K_0'$ and a $\mcP$-nbhd $N_P'$ of $g_0$ such that $I[K; g] > \lambda$ for all $(K, g) \in N_W \times N_P'$;
    \end{enumerate}
    then $N_S \times (N_P \cap N_P') \times N_W$ would be a $\mcS \times \mcP \times \mcW$-nbhd of $(K_0, g_0, K_0')$ contained in $R^c$, whence $R^c$ would be open. To do so, we want to show that a) $I[K ; g]$ is $\mcS \times \mcP$-u.s.c and b) $I[K ; g]$ is $\mcW \times \mcP$-l.s.c.
    % For b), we want to show that the level sets $L_{\lambda} = \{ I[K; g] \leq \lambda \}$ are closed. Let $(K_n, g_n)$ be a sequence in $L_{\lambda}$ which converges in $\mcW \times \mcV$ to $(K, g)$. We will show for all $\epsilon > 0$ that $I[K^* ; g^*] \leq \lambda + \epsilon$, so taking $\epsilon \to 0$ gives the desired result. By Mazur's lemma, there exists a function $M : \mathbb{N} \to \mathbb{N}$ and sequences $\alpha(n)_k$ defined for $n \in \mathbb{N}$ and $k \in [n, M(n)]$ such that $\alpha(n)_k \geq 0$, $\sum_{k = n}^{M(n)} \alpha(n)_k = 1$ and $\tilde{K}_n := \sum_{m = n}^{M(n)} \alpha(n)_m K_m$ converges in $\mcS$ to $K^*$. By the convexity of $I[K; g]$, we get that 
    % \begin{align*}
    %     I[ \tilde{K}_n ; g_n ] & \leq \sum_{m = n}^{M(n) } \alpha(n)_m \big( I[ K_m ; g_n ] - I[ K_m ; g_m ] + I[ K_m ; g_m ] \big) \\
    %     & \leq \lambda + \sum_{m = n}^{M(n) } \alpha(n)_m \cdot \big| I[ K_m ; g_n ] - I[ K_m ; g_m ] \big|.
    % \end{align*}
    % As $K_n$ converges weakly to some $K$, we know that $K_n$ is a norm-bounded sequence, and therefore as $I[K; g]$ is uniformly continuous in $g$ for $K$ fixed on norm-bounded sets of $K$, it follows that $\mcI[ \tilde{K}_n ; g_n] \leq \lambda + \epsilon$ for all sufficiently large $n$. Using the $\mcS \times \mcV$-continuity of $\mcI[K ; g ]$, it therefore follows that $\mcI[K^* ; g^*] \leq \lambda + \epsilon$ as desired.

    Part a) follows immediately by the assumption that $I[K; g]$ is $\mcS \times \mcP$-continuous. For b), it suffices to show that the level sets $L_{\lambda} = \{ (K, g) \,:\, I[K ; g] \leq \lambda \}$ are $\mcW\times\mcV$-closed. To do so, let $(K_{\alpha}, g_{\alpha} )_{\alpha \in A}$ be a net which converges to $(K^*, g^*)$; note that as the weak and norm topologies on a Banach space are Hausdorff and the product topology on Hausdorff topologies is Hausdorff, the limit is unique. We aim to show that for any $\epsilon > 0$, we have that $I[K^*, g^*] \leq \lambda + \epsilon$, so the conclusion follows by taking $\epsilon \to 0$. 
    
    To do so, we begin by noting that as $g_{\alpha}$ is a net converging to $g^*$ in a metrizable space (the topology $\mcP$ is induced by the metric $d(f, g) = \| f - g \|_U$), we can find a cofinal subsequence (that is, a subnet which is a sequence) $(\alpha_i)_{i \geq 1}$ along which $g_{\alpha_i} \to g^*$ as $i \to \infty$. (Indeed, we simply note that for each $i$, we can find $\alpha_i$ for which $d(g_{\beta}, g) \leq 1/i$ for all $\beta \geq \alpha_i$.) With this, we now note that for each $\alpha_i$, $K^*$ must be in the weak closure of $\mathrm{conv}(K_{\beta} \,:\, \beta \geq \alpha_i)$ (i.e, the convex hull of the $K_{\beta}$ for $\beta \geq \alpha_i$, which therefore contains each $K_{\beta}$ for $\beta \geq \alpha_i$). As this is a convex set, the weak and strong closures of this set are equal, and consequently $K^*$ must be in the strong closure of each of the $\mathrm{conv}(K_{\beta} \,:\, \beta \geq \alpha_i)$ too. Consequently, we can therefore always find some element $\tilde{K}_{\alpha_i} \in \mathrm{conv}(K_{\beta} \,:\, \beta \geq \alpha_i)$ for which $\| \tilde{K}_{\alpha_i} - K^* \|_X \leq 1/i$. In particular, we therefore have that the sequence $(\tilde{K}_{\alpha_i}, g_{\alpha_i} )_{i \geq 1}$ $\mcS \times \mcV$-converges to $(K^*, g^*)$.
    % Note that K(b) for b \geq a_i is a subnet as it is monotone (as we're using the same indexing as before), and also has a cofinal image (simply note that for all a \leq a_i, a_i fits the bill, and then for all a \geq a_i we're back in A so we can use the upward directed property of A.)

    To proceed further, we note that for each $i$, there exists $(\mu(i)_{\beta} )_{\beta \geq \alpha_i}$ such that all but finitely many of the $\mu(i)_{\beta}$ are zero, with the non-zero elements positive and $\sum_{\beta \geq \alpha_i} \mu(i)_{\beta} = 1$, with $\tilde{K}_{\alpha_i} = \sum_{\beta \geq \alpha_i} \mu(i)_{\beta} K_{\beta}$. The convexity of $I[K; g]$ plus the continuity condition \eqref{eq:app:convex_opt:l1_minima_save} then implies that 
    \begingroup 
    \allowdisplaybreaks
    \begin{align*}
        I[\tilde{K}_{\alpha_i} ; g_{\alpha_i} ] & \leq \sum_{\beta \geq \alpha_i} \mu(i)_{\beta} I[K_{\beta}; g_{\alpha_i} ] \\
        & = \sum_{\beta \geq \alpha_i} \mu(i)_{\beta} \big\{ I[K_{\beta}; g_{\alpha_i} ] - I[K_{\beta}; g_{\beta} ] + I[K_{\beta}; g_{\beta} ] \big\} \\
        & \leq \lambda + \sum_{\beta \geq \alpha_i} \mu(i)_{\beta} \big| I[K_{\beta}; g_{\alpha_i} ] - I[K_{\beta}; g_{\beta} ] \big| \leq \lambda + \sum_{\beta \geq \alpha_i} \mu(i)_{\beta} C_{\lambda} \| g_{\alpha_i} - g_{\beta} \|_P   \\ 
        & \leq \lambda + C_{\lambda} \sum_{\beta \geq \alpha_i} \mu(i)_{\beta}\big\{ \| g_{\alpha_i} - g^* \|_P + \| g_{\beta} - g^* \|_P \big\}.
    \end{align*}
    \endgroup 
    In particular, given any $\epsilon > 0$, we can choose $j \in \mathbb{N}$ such that $\|g_{\beta} - g \|_U \leq \epsilon / (2 C_{\lambda} )$ for all $\beta \geq \alpha_j$, and whence for $i \geq j$ we have that 
    \begin{equation*}
        I[\tilde{K}_{\alpha_i} ; g_{\alpha_i} ] \leq \lambda + \epsilon \sum_{\beta \geq \alpha_i } \mu(i)_{\beta} = \lambda + \epsilon.
    \end{equation*}
    Consequently passing to the strong limit using the $\mcS \times \mcP$-continuity of $I[K ; g]$ gives us that $I[K^* ; g^*] \leq \lambda + \epsilon$, as desired.

    With this, we can now apply part c) of Theorem~\ref{app:convex_opt:maximum_theorem} to conclude that $\mu(g)$ is $\mcP$-to-$\mcW$ upper hemicontinuous. The desired result then follows by applying Lemma~\ref{app:convex_opt:hemicont}.
\end{proof}
%!TEX root = ..\ms.tex

\section{Properties of piecewise H\"{o}lder functions and kernels} 
\label{sec:app:holder_props}

In this section we discuss some useful properties of symmetric, piecewise H\"{o}lder continuous functions, relating to the decay of their eigenvalues when viewed as operators between $L^p$ spaces. Letting $q$ be the H\"{o}lder conjugate of $p$ (so $p^{-1} + q^{-1} = 1$), for a symmetric function $K \in L^{\infty}([0, 1]^2)$ we can consider the operator $T_K: L^p([0, 1]) \to L^q([0, 1])$ defined by 
\begin{equation} \label{eq:app:holder_props:tk}
    T_K[f](x) := \int_0^1 K(x, y) f(y) \, dy.
\end{equation}
We usually refer to $K$ as the kernel of such an operator. $T_K$ is then self-adjoint, in that for any functions $f, g \in L^p([0, 1])$ we have that $\langle T_K[f], g \rangle = \langle f, T_K[g] \rangle$, where $\langle f, g \rangle = \int f g \, d \mu$. 

%  \citep{dunford_linear_1988} for a general reference)
We introduce some terminology and theoretical results concerning such operators. We say that an operator $T$ is compact if the image of the ball $\{ f \in L^p([0, 1]) \,:\, \| f \|_p \leq 1 \}$ under $T$ is relatively compact in $L^q([0, 1])$. If $K \in L^{\infty}([0, 1]^2)$, then $T_K$ is a compact operator. An operator $T$ is of finite rank $r$ if the range of $T$ is of dimension $r$. We say that an operator $T$ is positive if $\langle T[f], f \rangle \geq 0$ for all $f \in L^p([0, 1])$. This induces a partial ordering on the operators, where $T_1 \preccurlyeq T_2$ iff $T_2 - T_1$ is positive. In the case when $p = q = 2$, if $K$ is positive, then there exists a unique positive square root of $K$ (say $J$) such that $J^2 = K$, i.e that $K[f] = J[J[f]]$ for all $f \in L^2([0, 1])$. Again in the case where $p = q =2$, as $T_K$ is a self-adjoint compact operator, by the spectral theorem \citep[e.g][Theorem~7.46]{fabian_functional_2001} there exists a sequence of eigenvalues $\mu_i(K) \to 0$ and eigenvectors $\phi_i$ (which form an orthonormal basis of $L^2([0, 1])$) such that 
\begin{equation*}
    T_K[f] = \sum_{n = 1}^{\infty} \mu_n(K) \langle f, \phi_n \rangle \phi_n \text{ for all } f \in L^2([0, 1]^2), \qquad K(x, y) = \sum_{n = 1}^{\infty} \mu_n(K) \phi_n(x) \phi_n(y)
\end{equation*}
where the latter sum is understood to converge in $L^2$, and $\| K \|_{L^2([0, 1]^2)} = \sum_{n = 1}^{\infty} \mu_n(K)^2 < \infty$. Supposing that $T_K$ is also positive, then one can prove \citep[e.g][Theorem~3.A.1]{konig_eigenvalue_1986} that $T_K$ is trace class, in that $\| K \|_{\mathrm{tr}} := \sum_{n = 1}^{\infty} \mu_n(K) < \infty$, and we refer to this as the trace, or trace norm, of $T_K$. 

We now give some useful properties of the algebraic properties of piecewise H\"{o}lder continuous functions, before proving a result concerning the eigenvalues of $T_K$ when $K$ is piecewise H\"{o}lder.

\begin{lemma} \label{app:holder_props:ratio_of_holder}
    Let $f, g : [0, 1]^2 \to \mathbb{R}$ be two piecewise H\"{o}lder$([0, 1]^2, \beta, M, \mathcal{Q})$ continuous functions, which are both bounded below by $\delta > 0$ and bounded above by $C > 0$, so $0 < \delta \leq f, g \leq C$. Then:
    \begin{enumerate}[label=\roman*)]
        \item For any scalar $A$, $Af$ is piecewise H\"{o}lder($[0, 1]^2$, $\beta$, $|A|M$, $\mathcal{Q}$), and $f + g$ is piecewise H\"{o}lder($[0, 1]^2$, $\beta$, $2M$, $\mathcal{Q})$.
        \item $f/(f+g)$ is bounded below by $\delta/(\delta + C)$ and bounded above by $C/(C + \delta)$;
        \item $f/g$ and $f/(f+g)$ are H\"{o}lder$([0, 1]^2, \beta, 2CM\delta^{-2}, \mcQ)$ continuous.
        \item If $F$ is a continuous distribution function satisfying the conditions in Assumption~\ref{assume:loss_prob}, then $ \| F^{-1}(f/(f+g)) \|_{\infty} \leq C' = C'(F, \delta, C)$, and $F^{-1}(f/(f+g))$ is \linebreak H\"{o}lder($[0, 1]^2, \beta, M', \mcQ)$ where $M' = M'(F, \delta, C, M)$. 
    \end{enumerate}
\end{lemma}

\begin{proof}[Proof of Lemma~\ref{app:holder_props:ratio_of_holder}]
    Part i) is immediate. Part ii) follows by noting that as $f$ and $g$ are bounded below by $\delta$ and above by $C$, we have that
    \begin{equation*}
        \frac{\delta}{C} \leq \frac{f}{g} \leq \frac{C}{\delta}   \implies 0 < \frac{\delta}{\delta + C} \leq \frac{ f}{ f + g} \leq \frac{ C}{ C + \delta } < 1.
    \end{equation*} 
    As $F^{-1}$ is a monotone bijection $(0, 1) \to \mathbb{R}$, we therefore get the first part of iv) also. For iii), for any $Q \in \mcQ$ and $x, y \in Q$ we have that 
    \begin{align*} 
        \Big| \frac{ f(x) }{ g(x) } - \frac{ f(y) }{ g(y) } \Big| &= \Big| \frac{ f(x) g(y) - f(y) g(x) }{ g(x) g(y) } \Big| \leq \delta^{-2} | f(x) (g(y) - g(x) ) + g(x) (f(x) - f(y) ) | \\
        & \leq \delta^{-2} \big( | f(x) ||g(y) - g(x) | + |g(x)| |f(x) - f(y) | \big) \leq 2 C M \delta^{-2} \| x - y \|^{\beta}
    \end{align*}
    giving the first part of iii). For the second, note that we can write $f/(f+g) = h(f/g)$ where $h(x) = x/(1+x)$ is $1$-Lipschitz; consequently $f/(f+g)$ has the same H\"{o}lder properties as $f/g$. As $F^{-1}$ is Lipschitz on compact sets and we know that $f/(f+g)$ is contained within a compact interval (say $J$), the same reasoning gives that $F^{-1}(f/(f+g))$ is also H\"{o}lder with the same exponent and partition, and a constant depending only on the H\"{o}lder constant of $f/(f+g)$, the upper/lower bounds on $f/(f+g)$ and the Lipschitz constant of $F^{-1}$ on $J$. This then gives the second part of iv).
\end{proof}

To have the next theorem hold in slightly more generality, we introduce the notion of $\mcP$-piecewise equicontinuity of a family of functions $\mcK$, which holds if for all $\epsilon > 0$, there exists $\delta > 0$ such that whenever $x, y$ lie within the same partition of $\mcP$ and $\| x - y \| < \delta$, we have that $|K(x) - K(y) | < \epsilon$ for all $K \in \mcK$. 

%!TEX root = ..\..\ms.tex

\begin{theorem} \label{app:holder_props:eigenvalue_decay_of_holder}
    Suppose that $K: [0, 1]^2 \to \mathbb{R}$ is H\"{o}lder($[0, 1]^2$, $\beta$, $M$, $\mcQ^{\otimes 2}$) continuous and symmetric. For such a $K$, define $T_K$ as in \eqref{eq:app:holder_props:tk}, so $T_K$ is a self-adjoint, compact operator. Writing $\mu_d(K)$ for the eigenvalues of $T_K$ sorted in decreasing order of magnitude, we have that 
    \begin{equation*}
        \sup_{K \in \text{H\"{o}lder}\big([0, 1]^2, \beta, M, \mcQ^{\otimes 2} \big)} \Big( \sum_{i = d+1}^{\infty} \mu_i(K)^2 \Big)^{1/2} = O(d^{-\beta} )
    \end{equation*}
    or that $|\mu_d(K)| = O(d^{-(1/2 + \beta) })$ (also uniformly over such $K$). If $T_K$ is also positive, then this bound can be improved to $\mu_d(K) = O(d^{-(1 + \beta) })$ uniformly, or
    \begin{equation*}
        \sup_{K \text{ positive, } K \in \text{H\"{o}lder}\big([0, 1]^2, \beta, M, \mcQ^{\otimes 2}\big)} \Big( \sum_{i = d+1}^{\infty} \mu_i(K)^2 \Big)^{1/2} = O(d^{-(1/2 + \beta)} )
    \end{equation*}
    For any given $m \in \mathbb{N}$ and $A > 0$, the second bound stated also holds uniformly across $T_K$ for which $ \| K \|_{\infty} \leq A$ and $T_K$ having at most $m$ negative eigenvalues. More generally, suppose that $\mcK$ is a family of $\mcQ^{\otimes 2}$-piecewise equicontinuous functions, in which case we have that
    \begin{equation*}
        \sup_{K \in \mcK} \Big( \sum_{i = d+1}^{\infty} \mu_i(K)^2 \Big)^{1/2} = o(1).
    \end{equation*}
\end{theorem}

\begin{proof}[Proof of Theorem~\ref{app:holder_props:eigenvalue_decay_of_holder}]
    \phantomsection\label{app:holder_props:eigenvalue_decay_of_holder:proof}
    We adapt the proofs of \citet[Lemma~1]{reade_eigen-values_1983} and the main result of \citet{reade_eigenvalues_1983} so that they apply when $K$ is \emph{piecewise} H\"{o}lder, and to track the constants from the aforementioned proofs so we can argue that the bounds we adapt hold uniformly across all $K$ which are H\"{o}lder($[0, 1]^2$, $\beta$, $M$, $\mcQ^{\otimes 2}$). The idea of these proofs is to exploit the smoothness of $K$ to build finite rank approximations whose error in particular norms is easy to calculate, giving eigenvalue bounds. We then discuss how the proofs can be modified for the equicontinuous case.
    
    Starting when a-priori $T_K$ is not known to be positive, for any kernel $R_d$ corresponding to an operator of rank $\leq d$, we know that $\sum_{k=d+1}^{\infty} \mu_k(K)^2 \leq \| K - R_d \|_2^2$. As $K$ is piecewise H\"{o}lder continuous with respect to a partition $\mcQ^{\otimes 2}$, one strategy is to choose $R_d$ to be piecewise constant on a partition $\mcP_d$ which is a refinement of $\mcQ$. 
    
    To do so, begin by writing $\mcQ = (Q_1, \ldots, Q_k)$ for some $k$. For $d \gg (\min_i |Q_i|)^{-1}$, note that we can find $\tilde{n}_i(d) \in \mathbb{N}$ for $i \in [k]$ such that $(\tilde{n}_i - 1)/d \leq |Q_i| \leq (\tilde{n}_i + 1)/d$. By summing over the $i$ index, this implies that $\sum_i \tilde{n}_i - k \leq d \leq \sum_i \tilde{n}_i + k$, and so we can choose $n_i(d) \in \{ \tilde{n}_i(d) - 1, \tilde{n}_i(d), \tilde{n}_i(d) + 1\}$ such that $\sum_i n_i(d) =  d$ by the pigeonhole principle, as there are $2k$ possible values of the sum, yet $3^k$ possible choices of $n_i(d)$. With this, we can define a partition $\mcP_d = (A_{d, 1}, \ldots, A_{d, d})$ of $[0, 1]$ where the $A_{d, j}$ are intervals of length $|A_{d, j}| = |Q_i|/n_i(d)$ stacked alongside each other in consecutive order, where $i$ such that $\sum_{r = 1}^{i-1} n_r(d) \leq j \leq \sum_{r =1}^i n_r(d)$. This is a refining partition of $\mcQ$, and moreover 
    \begin{equation*}
        \Big| \frac{ |Q_i| }{ n_i(d) \cdot d } - 1 \Big| \leq \frac{1}{d} \implies \big| A_{d, j} \big| =  d^{-1}(1 + d^{-1} E_{d, j}) \text{ where } |E_{d, j}| \leq k (\min_i |Q_i| )^{-1}.
    \end{equation*}
    
    With this, if we define $R_d$ as being a piecewise constant on $\mcP_d^{\otimes 2}$, equal to the value of $K$ on the midpoint of the $A_{di} \times A_{dj}$, then $R_d$ is the kernel of an operator of rank $\leq d$ by Lemma~\ref{app:holder_props:piecewise_constant_are_finite_rank}. We then note that by the piecewise H\"{o}lder properties of $K$, and as $R_d$ is piecewise constant on a refinement of $\mcQ$, if $(u, v) \in A_{d, i} \times A_{d, j}$ then 
    \begin{equation*}
        | K(u, v) - R_d(u, v) | \leq M 2^{-\beta} ( |A_{d, i}|^2 + |A_{d, j}|^2 )^{\beta/2} \leq M 2^{-\beta/2} d^{-\beta} k^{\beta} (\min_i |Q_i| )^{-\beta} 
    \end{equation*}
    Consequently $\| K - R_d \|_2 \leq \| K - R_d \|_{\infty} \leq O(d^{-\beta})$ (where the implied constant attached to the $O(\cdot)$ term depends only on $M$, $\beta$ and the partition $\mcQ$), and so we get the first part of the result. 
    
    Note that if we only know that the $K$ belong to a equicontinuous family $\mcK$, then we can still apply the same construction and find that $ \sup_{K \in \mcK} \| K - R_d \|_{\infty} \to 0$ as $d \to \infty$. Indeed, given $\epsilon > 0$, let $\delta > 0$ be such that once $\| (u, v) - (u', v') \|_2 < \delta$ we have that $ | K(u, v) - K(u', v') | < \epsilon$ for all $K \in \mcK$. Then provided we choose $d$ to be so that the $|A_{d, i} | < \delta$, the above construction guarantees us that $|K(u, v) - R_d(u, v) | < \epsilon$ a.e uniformly over all $K \in \mcK$. 

    For the case where $K$ is non-negative definite, we will use a version of the Courant-Fischer min-max principle \citep[Lemma~1]{reade_eigenvalues_1983}, which states that if $R_d$ is a kernel of a rank $\leq d$ symmetric operator, then $\sum_{k=d+1}^{\infty} \mu_k(K) \leq \| K - R_d \|_\mathrm{tr}$. Define 
    \begin{equation*}
        S_d(u, v) = \sum_{i=1}^d |A_{d, i}|^{-1} \phi_i(u) \phi_i(v) \text{ where } \phi_i(u) = \mathbbm{1}[u \in A_{d, i} ].
    \end{equation*}
    Note that $S_d$ is non-negative definite, of rank $\leq d$, and $0 \curlyeqprec S_d \curlyeqprec I$ as, by Jensen's inequality, 
    \begin{equation*}
        \langle S_d[f], f \rangle = \sum_{i=1}^d |A_{d, i}|^{-1} \Big( \int_{A_{d, i} } f(x) \, dx \Big)^2 \leq \sum_{i=1}^d \int_{A_{d, i}} f(x)^2 \, dx = \langle f, f \rangle
    \end{equation*}
    for any function $f \in L^2([0, 1])$. Therefore if we define $R_d = J S_d J$ (where $J$ is the square root of $K$), then by Lemma~\ref{app:holder_props:mercer} we know that $R_d$ is of rank $\leq d$ and $0 \preccurlyeq J S_d J \preccurlyeq K$. By following through the arguments in \citet*[p.155]{reade_eigenvalues_1983} (noting that in Lemma~\ref{app:holder_props:mercer} we verify that the trace of a piecewise continuous kernel is given by its integral over the diagonal), we may then argue that 
    \begin{align*}
        \| K -  J S_d J \|_{\mathrm{tr}} & = \sum_{i=1}^d |A_{d, i}|^{-1} \int_{A_{d, i} \times A_{d, i} } \frac{1}{2} ( K(u, u) + K(v, v) ) - K(u, v) \, du \, dv \\
        & \leq \sum_{i=1}^d |A_{d, i}|^{-1}  \int_{A_{d, i} \times A_{d, i} }  M |u - v|^{\beta} \, du \, dv \leq \sum_{i=1}^d M |A_{d, i}|^{1 + \beta} = O(d^{-\beta})
    \end{align*}
    and so $\mu_d(K) = O(d^{-(1+\beta)} )$ as desired, with the implied constant depending only on $M$ and $\mcQ$; this then gives the stated bound on $( \sum_{k = d+1}^{\infty} \mu_k(K)^2 )^{1/2}$. In the case where $K$ has $m$ negative eigenvalues, note that the eigenvectors are piecewise H\"{o}lder by Lemma~\ref{app:holder_props:eigenvectors_holder}, and the eigenvalues are bounded above by $\| K \|_{2} \leq \| K \|_{\infty}$. In particular, for each $m$, if we subtract the negative part of $K$ from itself then we still have a class of piecewise H\"{o}lder continuous functions with partition $\mcQ$, exponent $\beta$ and constant depending on $M$, $m$ and $\| K \|_{\infty}$. We can then apply the above result (as we are only interested in tail bounds for the eigenvalues), and get tail bounds which depend only on these quantities again.  
\end{proof}

We want to apply these results to $K$ of the form 
\begin{equation}
    \label{eq:holder_props:optimalK}
    \optimalK := F^{-1}\Big( \frac{ \fnone}{ \fnone + \fnzero } \Big)
\end{equation}
where $F$ is a c.d.f as in Assumption~\ref{assume:loss_prob}, and the $\fnone$ and $\fnzero$ come from Assumption~\ref{assume:samp_weight_reg}. By the above results, we can obtain the following:

\begin{corollary} \label{app:holder_props:optimalK_bounded_and_holder}
    Suppose that Assumptions~\ref{assume:graphon_ass}~and~\ref{assume:samp_weight_reg} hold with $\gamma_s = \infty$, and that $F$ is a c.d.f satisfying the properties stated in Assumption~\ref{assume:loss_prob}. Denote $\tilde{f}_{n, x}(l, l') = \tilde{f}_n(l, l', x)$. Then there exists $A'$, free of $n$ and depending only on $\sup_{n, x} \| \tilde{f}_{n, x} \|_{\infty}$, $\sup_{n, x } \| \tilde{f}_{n, x}^{-1} \|_{\infty}$ and $F$, such that $\sup_n \| \optimalK \|_{\infty} \leq A < \infty$ where $\optimalK$ is as in \eqref{eq:holder_props:optimalK}. Moreover, there exists $L'$ depending only on $\sup_{n, x} \| \tilde{f}_{n, x} \|_{\infty}$, $\sup_{n, x } \| \tilde{f}_{n, x}^{-1} \|_{\infty}$, $\fnholderconst$ and $F$ - so again free of $n$ - such that $\optimalK$ is piecewise H\"{o}lder($[0, 1]^2$, $\beta$, $L'$, $\mcQ^{\otimes 2}$) for all $n$.
\end{corollary}

\begin{proof}[Proof of Corollary~\ref{app:holder_props:optimalK_bounded_and_holder}]
    Apply Lemma~\ref{app:holder_props:ratio_of_holder}.
\end{proof}

\begin{proposition} \label{app:holder_props:optimalK_in_opt_domain}
    Suppose that Assumption~\ref{assume:loss} holds with $1 \leq p \leq 2$, where $p$ is the growth rate of the loss function $\ell$, that Assumption~\ref{assume:graphon_ass} holds, and Assumption~\ref{assume:samp_weight_reg} holds with $\gamma_s = \infty$. Then we have that $\optimalK \in \mcZ$; if $\optimalK$ is positive for all $n$, then we moreover have that $\optimalK \in \mcZ^{\geq 0}$. Moreover, there exists $A'$ free of $n$ such that whenever $A \geq A'$, denoting $K_{n, d_1, d_2}$ for the best rank $(d_1, d_2)$ approximation in $L^2$ to $\optimalK$ (that is, the operator $S_1 - S_2$ for which $\| \optimalK - (S_1 - S_2) \|_2$ is minimized over all positive rank $d_i$ operators $S_i$ for $i \in \{1, 2\}$), then $K_{n, d_1, d_2} \in \mcZ_{d_1, d_2}(A)$ for all $n$, $d_1$ and $d_2$. 
    
    In the case when $\optimalK$ is positive, then $K_{n, d_1, d_2}$ is also positive for all $d_1$ and $d_2$, and consequently $K_{n, d_1, d_2} \in \mcZ^{\geq 0}_{d_1}(A)$ for all $n$, $d_1$ and $d_2$. In fact, the same conclusions above hold provided $K \in \mcK$ where $\mcK$ is a family of $\mcQ^{\otimes 2}$-piecewise equicontinuous functions with $\sup_{K \in \mcK} \| K \|_{\infty} < \infty$, with the choice of $A'$ holding uniformly over all $K \in \mcK$.  
\end{proposition}

\begin{proof}[Proof of Proposition~\ref{app:holder_props:optimalK_in_opt_domain}]
    Let $\mu_i(\optimalK)$ and $\phi_{n, i}$ denote, respectively, the eigenvalues and eigenvectors of $\optimalK$. Working with the eigenvalues, note that $\sup_{n, i} | \mu_i(\optimalK) | \leq \| \optimalK \|_2 \leq \| \optimalK \|_{\infty}$, which is bounded uniformly in $n$ by Corollary~\ref{app:holder_props:optimalK_bounded_and_holder}. As for the eigenvectors, we note that by Lemma~\ref{app:holder_props:eigenvectors_holder} they are all piecewise H\"{o}lder($[0, 1]$, $\beta$, $L$, $\mcQ$) (where $L$ is as in Corollary~\ref{app:holder_props:optimalK_bounded_and_holder}); as they all have $L^2$ norm equal to one, it therefore follows by Lemma~\ref{app:holder_props:holder_plus_L1_bound} that the eigenvectors are also uniformly bounded in $L^{\infty}$. As we now can write
    \begin{align*}
        \optimalK\llp & = \sum_{i \,:\, \mu_i(\optimalK) > 0} \Big( |\lambda_i( \optimalK )|^{1/2} \phi_{n, i}(l) \Big) \Big( |\lambda_i( \optimalK )|^{1/2} \phi_{n, i}(l') \Big) \\
        & \qquad -  \sum_{i \,:\, \mu_i(\optimalK) < 0} \Big( |\lambda_i( \optimalK )|^{1/2} \phi_{n, i}(l) \Big) \Big( |\lambda_i( \optimalK )|^{1/2} \phi_{n, i}(l') \Big),
    \end{align*}
    where the sum is understood to converge in $L^2$ (and therefore also in $L^p([0, 1]^2)$ for any $p \in [1, 2]$), the desired conclusion follows with $A' = \sup_{n, i} \big| \lambda_i( \optimalK )  \big|^{1/2} \cdot \sup_{n, i} \| \phi_{n, i} \|_{\infty}$. In the case where the $K$ lie within a piecewise equicontinuous class $\mcK$ where $\sup_{K \in \mcK} \| K \|_{\infty} \leq A$, the same arguments hold and therefore the stated conclusion does too.
\end{proof}

\subsection{Additional lemmata}

\begin{lemma} \label{app:holder_props:piecewise_constant_are_finite_rank}
    Let $K: [0, 1]^2 \to \mathbb{R}$ be symmetric and piecewise constant on a partition $\mcP^{\otimes 2}$, where $\mcP$ is a partition of $[0, 1]$. Then if $\mcP$ is of size $r$, $T_K$ is of rank $\leq r$. 
\end{lemma}

\begin{proof}[Proof of Lemma~\ref{app:holder_props:piecewise_constant_are_finite_rank}]
    Suppose $\mcP = (A_1, \ldots, A_r)$ for some intervals $A_r$, and define the matrix $M_{i, j} = K(u, v)$ where we can choose any $(u, v) \in A_i \times A_j$ and have $M$ be well defined as $K$ is piecewise constant. Then as $M$ is a $r$-by-$r$ symmetric matrix, by the spectral theorem, there exists $\lambda_i \in \mathbb{R}$ (possibly allowing for zero eigenvalues) and eigenvectors $v_i \in \mathbb{R}^r$ such that $M = \sum_{i=1}^r \lambda_i v_i v_i^T$. Then if we define functions $\phi_i : [0, 1] \to \mathbb{R}$ by $\phi_i(l) = v_{i, j}$ for $l \in A_j$, $j \in [r]$, we have that $K(u, v) = \sum_{i=1}^r \lambda_i \phi_i(u) \phi_i(v)$ and therefore $T_K$ is of rank $\leq r$. 
\end{proof}

\begin{lemma} \label{app:holder_props:eigenvectors_holder}
    Suppose that $K: [0, 1]^2 \to \mathbb{R}$ is H\"{o}lder($[0, 1]^2$, $\beta$, $M$, $\mcQ^{\otimes 2}$) continuous and symmetric. Then for any $f \in L^2$ we have that $T_K[f]$ is H\"{o}lder($[0, 1]$, $\beta$, $M \|f \|_2$, $\mcQ$). In particular, $T_K$ is a self adjoint, compact operator. Moreover, the eigenvectors of $T_K$, normalized to have $L^2([0, 1])$ norm $1$, can be taken to each be piecewise H\"{o}lder($[0, 1]$, $\beta$, $M$, $\mcQ$), and are uniformly bounded in $L^{\infty}([0, 1])$. 
    
    Similarly, if $\mcK$ is a $\mcQ^{\otimes 2}$-piecewise equicontinuous family of symmetric functions $[0, 1]^2 \to \mathbb{R}$, then the collection of all the eigenvectors of $T_K$ for $K \in \mcK$ are $\mcQ$-piecewise equicontinuous and uniformly bounded in $L^{\infty}([0, 1])$. 
\end{lemma}

\begin{proof}[Proof of Lemma~\ref{app:holder_props:eigenvectors_holder}]
    Let $f: [0, 1] \to \mathbb{R}$. Beginning with the H\"{o}lder case, for any pair $x, y \in Q \in \mcQ$ we have
    \begin{align*}
        | T_K[f](x) & - T_K[f](y) | \leq \int_0^1 | K(x, z) - K(y, z) | | f(z) | \, dz \\
        & = \sum_{Q \in \mcQ} \int_Q | K(x, z) - K(y, z) | | f(z) | \, dz \nonumber \\
        & \leq \sum_{Q \in \mcQ} \int_Q M |x - y|^{\beta} |f(z)| \, dz = M |x - y|^{\beta} \cdot \int_0^1 |f(z)| \, dz \leq M \| f \|_2 |x - y|^{\beta},
    \end{align*}
    so the image of the $L^2([0, 1])$ ball is contained within the class of H\"{o}lder($[0, 1]$, $\beta$, $M \|f \|_2$, $\mcQ$) functions. This implies the claimed results, where the compactness of the operator follows by using the Arzela-Ascoli theorem with this fact, and the statement on eigenvectors of $T_K$ is immediate by the above derivation and an application of Lemma~\ref{app:holder_props:holder_plus_L1_bound}. For the case where we have some equicontinuous family $\mcK$, let $\epsilon > 0$, so there exists some $\delta > 0$ such that whenever $\| (x, u) - (y, x) \|_2 < \delta$ and $(x, y), (u, v)$ lie within the same partition of $\mcQ^{\otimes 2}$, we have that $| K(x, u) - K(y, v) | <  \epsilon$ for all $K \in \mcK$. Therefore, if $|x - y| < \delta$, $\| (x, z) - (y, z) \|_2 < \delta$ for all $z$ and so we get that
    \begin{align*}
        | T_K[f](x) - T_K[f](y) | \leq \int_Q | K(x, z) - K(y, z) | | f(z) | \, dz \leq \epsilon \| f \|_1 \leq \epsilon \|f \|_2 = \epsilon,
    \end{align*}
    giving the desired conclusion.
\end{proof}

\begin{lemma}[Mercer's theorem + more for piecewise continuous kernels]
    \label{app:holder_props:mercer}
    Let $K : [0, 1]^2 \to \mathbb{R}$ be a symmetric piecewise continuous function on $\mcQ^{\otimes 2}$, according to some partition $\mcQ$ of $[0, 1]$, for which the associated operator $T_K$ is positive. Then $\| K \|_{\mathrm{tr}} = \int_0^1 K(u, u) \, du$. Moreover, if $J$ is the unique positive square root of $K$ and $S$ is an operator of rank $\leq d$ such that $0 \preccurlyeq S \preccurlyeq I$, then $JSJ$ is of rank $\leq d$, the corresponding kernel is piecewise continuous, and $0 \preccurlyeq JSJ \preccurlyeq K$. 
\end{lemma}

\begin{proof}[Proof of Lemma~\ref{app:holder_props:mercer}]
    Note that in the case where $K$ is positive and continuous, it is well known as a consequence of Mercer's theorem that we can write the trace norm of $K$ as the integral over the diagonal of $K$. In the case where $K$ is piecewise continuous, if we write $\lambda_i$ and $\phi_i$ for the eigenvalues and (normalized) eigenfunctions of $T_K$, then we know that the eigenfunctions are piecewise continuous (by the argument in Lemma~\ref{app:holder_props:eigenvectors_holder}). By following the arguments in the proof of Mercer's theorem for the continuous case \citep[e.g][p245-246]{riesz_functional_1990}, one can argue that 
    \begin{equation} \label{eq:embed_converge:diagonal_form_of_K}
        K(u, u) = \sum_{i=1}^{\infty} \lambda_i \phi_i(x)^2 
    \end{equation}
    convergences pointwise for all $u \in [0, 1]$ except at (potentially) the discontinuity points of $u \mapsto K(u, u)$, of which there are only finitely many. Therefore by the monotone convergence theorem, we then get that 
    \begin{equation*}
        \| K \|_{\mathrm{tr} } = \lim_{N \to \infty} \sum_{i=1}^N \mu_i(K) = \lim_{N \to \infty} \int_0^1 \sum_{i=1}^N \mu_i(K) \phi_i(u)^2 \, du = \int_0^1 K(u, u) \, du.
    \end{equation*}
    Moreover, as a consequence of Dini's theorem, we know that for any $x \in \mathrm{int}(Q)$ for some $Q \in \mcQ$, there exists a compact set $C$ such that $x \in C \subseteq Q$ and the convergence in \eqref{eq:embed_converge:diagonal_form_of_K} is uniform on $C$. This last part then allows us to follow through the proof of \citet[Lemma~2]{reade_eigenvalues_1983} to note that if $J(u, v)$ is the unique non-negative definite square root of $K$, then $J[f]$ is piecewise continuous for any $f \in L^2([0, 1])$. It then follows by the same argument as in \citet[Lemma~3]{reade_eigenvalues_1983} that if $S$ is an operator of rank $\leq d$ such that $0 \preccurlyeq S \preccurlyeq I$ and $K$ is a non-negative definite operator which is piecewise continuous with square root $J$, then $JSJ$ is of rank $\leq d$, is piecewise continuous and satisfies $0 \preccurlyeq JSJ \preccurlyeq K$. 
\end{proof}

\begin{lemma}  \label{app:holder_props:holder_plus_L1_bound}
    Let $X \subseteq \mathbb{R}^d$ be compact, and let $(f_n)_{n \geq 1}$ be a sequence of piecewise \linebreak H\"{o}lder($X$, $\beta$, $M$, $\mcQ$) functions. If we also suppose that $\sup_{n \geq 1} \| f_n \|_{L^p(X)}$ for any $p \geq 1$, then \linebreak $\sup_{n \geq 1} \| f_n \|_{L^{\infty}(X) } < \infty$. The same conclusion follows if we have a sequence $f_n$ of piecewise equicontinuous functions.
\end{lemma}

\begin{proof}[Proof of Lemma~\ref{app:holder_props:holder_plus_L1_bound}]
    Without loss of generality we may suppose that $p = 1$ (as uniform boundedness in any $L^p$ norm with $p > 1$ implies uniform boundedness in $p = 1$ when $X$ is compact). If we pick $Q \in \mcQ$ and $x \in \mathrm{int}(\mcQ)$ (so that $f_n(x)$ is well defined as $f_n$ is piecewise continuous on $\mcQ$), by the triangle inequality and integrating we then have that
    \begin{align*}
        | f_n(x) | & \leq \int_{Q} | f_n(x) - f_n(y) | \, dy + \int_{Q} |f_n(y) | \, dy \\
        & \leq \int_{Q} M \| x - y \|_2^{\beta} \, dy + \int_{Q} |f_n(y)| \, dy \leq M \mu(X) \mathrm{diam}(X)^{\beta}  + \|f_n \|_{L^1(X)}
    \end{align*}
    where $\mu(X)$ denotes the Lebesgue measure of $X$. As the RHS is finite and bounded uniformly in $n$, we get the desired result. The same argument works in the piecewise equicontinuous case.
\end{proof}

\bibliography{graphbibtex}

\end{document}